\documentclass[11pt,letterpaper,reqno]{amsart}
\usepackage[foot]{amsaddr}
\usepackage[margin=1in]{geometry}
\usepackage{amsmath,amsthm,amsfonts,amssymb,mathrsfs}
\usepackage{hyperref,color,enumerate,cleveref}
\usepackage{graphicx}
\usepackage{bm}

\makeatletter
\def\paragraph{\@startsection{paragraph}{4}%
  \z@\z@{-\fontdimen2\font}%
  {\normalfont\bfseries}}
\makeatother

\allowdisplaybreaks

\newtheorem{thm}{Theorem}[section]

\newtheorem{lem}[thm]{Lemma}

\theoremstyle{definition}
\newtheorem{asp}[thm]{Assumption}

\newtheorem{defi}[thm]{Definition}
\newtheorem{remark}[thm]{Remark}

\def\eps{\varepsilon}
\def\d{\mathrm{d}}
\def\D{\mathrm{D}}
\def\P{\mathbb{P}}
\def\E{\mathbb{E}}
\def\Z{\mathbb{Z}}
\def\R{\mathbb{R}}
\def\Tr{\operatorname{Tr}}
\def\Pois{\operatorname{Poisson}}
\def\Law{\operatorname{Law}}
\def\GP{\operatorname{GP}}
\def\PP{\operatorname{PP}}
\def\Id{\mathrm{Id}}
\def\op{\mathrm{op}}

\def\diag{\operatorname{diag}}
\def\IE{\operatorname{IE}}
\def\GE{\operatorname{GE}}

\def\cR{\mathcal{R}}
\def\cS{\mathcal{S}}
\def\cD{\mathcal{D}}
\def\cE{\mathcal{E}}
\def\cF{\mathcal{F}}
\def\cN{\mathcal{N}}
\def\sF{\mathrm{F}}
\def\a{\mathbf{a}}

\def\e{\mathbf{e}}
\def\v{\mathbf{v}}
\def\x{\mathbf{x}}
\def\y{\mathbf{y}}
\def\z{\mathbf{z}}
\def\A{\mathbf{A}}
\def\B{\mathbf{B}}

\def\M{\mathbf{M}}
\def\bN{\mathbf{N}}
\def\bP{\mathbf{P}}
\def\U{\mathbf{U}}
\def\V{\mathbf{V}}
\def\W{\mathbf{W}}
\def\X{\mathbf{X}}
\def\Y{\mathbf{Y}}
\def\bZ{\mathbf{Z}}
\def\1{\mathbf{1}}
\def\beps{\boldsymbol{\varepsilon}}
\def\btheta{\boldsymbol{\theta}}
\def\bxi{\boldsymbol{\xi}}
\def\bnu{\boldsymbol{\nu}}
\def\bPi{\boldsymbol{\Pi}}

\def\set{\mathscr{S}_\kappa}
\def\Ct{C_\theta}
\def\Rt{R_\theta}
\def\Cf{C_f}
\def\Rf{R_f}
\def\pCt{\Phi_{\Ct}}
\def\pCf{\Phi_{\Cf}}
\def\pRt{\Phi_{\Rt}}
\def\pRf{\Phi_{\Rf}}
\def\pRfs{\Phi_{\Rf^*}}
\def\trsfrm{\mathcal{T}}
\def\trsfrmA{\mathcal{T}_{\mathcal{S}_\xi \to {\mathcal{S}}_\theta}}
\def\trsfrmB{\mathcal{T}_{{\mathcal{S}}_\theta \to \mathcal{S}_\xi}}
\def\trsfrmD{\mathcal{T}^\delta}
\def\trsfrmDA{\mathcal{T}^\delta_{\mathcal{D}_\xi \to {\mathcal{D}}_\theta}}
\def\trsfrmDB{\mathcal{T}^\delta_{{\mathcal{D}}_\theta \to \mathcal{D}_\xi}}
\def\nI{(\mathrm{I})}
\def\nII{(\mathrm{II})}
\def\nIII{(\mathrm{III})}
\def\nIV{(\mathrm{IV})}
\def\nV{(\mathrm{V})}
\def\nVI{(\mathrm{VI})}
\newcommand{\lbddst}[2]{\mathsf{dist}_{\lambda,T}(#1,#2)}
\newcommand{\flr}[1]{{\lfloor {#1} \rfloor}}
\newcommand{\cil}[1]{{\lceil {#1} \rceil}}

\newcommand{\norm}[1]{\|#1\|}
\newcommand{\normop}[1]{\|#1\|_{\mathrm{op}}}
\newcommand{\prn}[1]{\left({#1}\right)} 
\newcommand{\brk}[1]{\left[{#1}\right]} 

\definecolor{lw}{RGB}{11,0,249}

\title[High-dimensional learning dynamics of multi-pass SGD]
{High-dimensional learning dynamics of multi-pass Stochastic Gradient
Descent in multi-index models}

\author{Zhou Fan}
\author{Leda Wang}
\email{zhou.fan@yale.edu, leda.wang@yale.edu}
\address{Department of Statistics and Data Science, Yale University}

\begin{document}

\begin{abstract}
The learning dynamics of Stochastic Gradient Descent (SGD) are believed to
play an important role in influencing
the generalization properties of trained models. While the precise
dynamics of how online (one-pass) SGD optimizes the population risk are now
understood in a broad range of high-dimensional problems, a comparable
understanding of multi-pass SGD for empirical risk minimization is still
largely lacking.

In this work, we develop an asymptotically exact characterization of the
dynamics of multi-pass SGD for a non-convex multi-index model,
in the high-dimensional regime where sample size $n$ and data dimension $d$
increase proportionally. We consider any sub-linear batch size $\kappa \asymp
n^\alpha$ with $\alpha \in [0,1)$ and a commensurate ``critical'' scaling of
the learning rate, and show that the coordinate-wise dynamics of SGD admit a novel continuous-time scaling limit,
given by a system of dynamical mean-field
equations that are driven by a scalar Poisson jump process representing the
asymptotic limit of SGD sampling noise. This scaling limit connects several
previously
disparate theories around SGD, namely (i) the classical interpretation of SGD as
an approximation to gradient flow on the empirical risk, (ii) the
characterization of online SGD via an autonomous ODE in its overlap statistics,
and (iii) the coincidence between multi-pass SGD dynamics and diffusion
approximations in least-squares problems. We illustrate how these perspectives
are recovered by specializing our result to small learning rates, large
sample sizes, and linear models, respectively.
Outside of these specializations, the scaling limit of SGD is distinct both
from diffusion approximations and gradient flow, and may enable SGD to
converge to local optima of the non-convex empirical risk with different
statistical properties.
\end{abstract}

\maketitle


\section{Introduction}

Introduced by \cite{robbins1951stochastic}, Stochastic Gradient
Descent (SGD) and its variants have played an important role in machine
learning, and are often the optimization methods of choice in large-scale
learning systems \cite{bottou2010large,bottou2018optimization}. While the
optimization theory of SGD in convex problems is richly understood
\cite{nemirovski2009robust,moulines2011non,rakhlin2012making,bach2013non},
many aspects of SGD in high-dimensional non-convex learning problems,
including the nature of local minimizers that it reaches and the differences in
generalization performance with gradient descent and gradient flow,
remain a mystery.

In this work, we will study the dynamics of SGD in a prototypical
high-dimensional and non-convex setting of empirical risk minimization for
learning
a multi-index model, with isotropic random data $\x \in \R^d$ and labels
\[y=\sigma^*(\x^\top \btheta^*,\eps), \qquad \btheta^* \in \R^{d \times k^*}.\]
Given $n$ such data observations $(\x_1,y_1),\ldots,(\x_n,y_n)$, we will
consider the standard mini-batch, multi-pass SGD procedure for optimizing an
empirical risk with coordinate-separable regularizer,
\begin{equation}\label{eq:ERMintro}
\cR(\btheta)=\sum_{i=1}^n L(\sigma(\x_i^\top \btheta),y_i)+\sum_{j=1}^d
G(\theta_j), \qquad \btheta=[\theta_1,\ldots,\theta_d]^\top \in \R^{d \times
k}.
\end{equation}
Our main result will provide an asymptotically exact characterization of the
learning dynamics of SGD in this model, with a general sub-linear batch
size $\kappa \asymp n^\alpha$ for $\alpha \in [0,1)$,
in a high-dimensional asymptotic limit as $n,d \to \infty$ proportionally.

To motivate this result and orient it in the existing literature, let us briefly
summarize three distinct theories on the dynamics of SGD in
high-dimensional and/or non-convex landscapes:

\begin{enumerate}[1.]
\item A body of classical literature studies SGD dynamics in the
setting of vanishing learning rate $\eta \to 0$, showing that in this regime
SGD has an ODE scaling limit that coincides with (non-stochastic) gradient-flow
optimization of the empirical risk
\cite{ljung1977analysis,ljung1992stochastic,kushner2003stochastic}.

More recent work has studied various stochastic diffusion approximations
to SGD \cite{mandt2017stochastic,li2017stochastic}, showing quantitative
$O(\eta)$ or $O(\eta^2)$ approximation guarantees for small learning rates $\eta$
\cite{feng2018semigroups,hu2019diffusion,li2019stochastic}, and illustrating
how stochasticity may lead to additional diffusion and
implicit regularization effects along minimizers of the loss
landscape \cite{blanc2020implicit,li2022happens}.

These results suggest that SGD behaves similarly to gradient flow, with
additional diffusivity properties near criticality. However, they pertain to a 
``low-dimensional'' regime of $\eta \to 0$ for fixed $n,d$, and
seem to be at odds with empirical observations that SGD may reach risk
minimizers different from gradient flow in high-dimensional problems under
commonly-used learning rate scalings \cite{keskar2017large,jastrzkebski2017three}.

\item Novel phenomena are known to arise in high dimensions for an analogous
\emph{one-pass} or \emph{online learning} setting,
where a new sample $(\x_i,y_i)$ is used in each SGD iteration. A large body of
work, initiating with
\cite{saad1995dynamics,saad1995exact,riegler1995line,biehl1995learning} and the mathematical proofs of their findings by
\cite{wang2019solvable} and \cite{goldt2019dynamics},
has recognized that for online SGD in multi-index and related models,
asymptotic dynamics in a high-dimensional ($n,d \to \infty$)
and large learning-rate limit are characterized by a closed/autonomous ODE
for the asymptotic limit of the fixed-dimensional overlap parameters
$\frac{1}{d}(\btheta,\btheta^*)^\top
(\btheta,\btheta^*) \in \R^{(k+k^*) \times (k+k^*)}$, which
is not simply the projection of the gradient flow ODE on the population risk.

Many generalizations of these results are now known, including to a broader
class of models and to SDE approximations near critical points of the risk landscape by
\cite{ben2022high}; to two-layer neural networks with more general dimension
scalings by \cite{veiga2022phase} and \cite{arnaboldi2023high}; and to
non-isotropic data and adaptive dynamics by
\cite{collins2024hitting,collins2024high}. Within this literature,
the sample complexity needed for online SGD to escape from symmetric initial
conditions and saddle points of the population risk has been of particular interest
\cite{ge2015escaping,fang2019sharp,arous2021online,barak2022hidden,tan2023online,abbe2023sgd,damian2023smoothing,damian2025generative,joshi2025learning}.

Although such results have illuminated how SGD may navigate complex
high-dimensional landscapes, there are significant differences between
online and multi-pass dynamics \cite{lee2024neural,dandi2024benefits}.
In particular, studying multi-pass
dynamics with a finite training sample seems requisite for understanding
optimization of the empirical (rather than population) risk and
associated issues of overfitting versus generalization.

\item A line of work 
\cite{paquette2021dynamics,paquette2021sgd,lee2022trajectory,paquette2022implicit,marshall2024clip,paquette2025homogenization}
has studied multi-pass SGD in settings of least-squares linear regression with
ridge regularization, establishing asymptotic characterizations of quadratic
observables in a high-dimensional, large learning-rate limit that is analogous
to the above settings of online SGD.
This work illustrated that memory effects due to sample re-use in multi-pass
SGD manifest in the form of Volterra integral equations in the asymptotic
dynamics. An interesting finding, highlighted in
\cite{paquette2022implicit,paquette2025homogenization}, is that although the
asymptotic dynamics of SGD do not coincide with the ODE
approximation given by gradient flow, they do in fact coincide with a
simplified SDE approximation dubbed ``Homogenized SGD''.

Such a finding is surprising, given that the studied learning rate
$\eta$ in these works is not small in the sense of classical low-dimensional
theory \cite{li2019stochastic,hu2019diffusion}, and raises a tantalizing
question of whether this equivalence may hold true more generally in
non-convex learning problems beyond linear regression.
\end{enumerate}

\noindent These three viewpoints on SGD dynamics are each illuminating, but
the connections between them are unclear. The main contribution of our work
is to develop an asymptotic characterization of multi-pass SGD in the general
multi-index model \eqref{eq:ERMintro} which serves to connect and unify
these theories.

Our characterization is based upon a mathematically rigorous instantiation of
dynamical mean-field theory (DMFT) from statistical physics
\cite{sompolinsky1981dynamic,sompolinsky1982relaxational,crisanti1993spherical,cugliandolo1993analytical}.
In particular, we extend results of \cite{celentano2021high} that
studied high-dimensional dynamics of gradient flow for a similar
class of multi-index models; we discuss in Section \ref{sec:proof-novelty} 
the main technical innovations of this extension and further related
literature.

\subsection{Summary of results}

Our main result establishes a continuous-time scaling limit for the
learning dynamics of SGD over $O(1)$ training epochs, under
high-dimensional asymptotics as $n,d \to \infty$ proportionally. This will
take the form of a characterization for the limiting value of any
coordinate-separable observables,
\[\lim_{n,d \to \infty} \frac{1}{d}\sum_{j=1}^d
\psi(\theta_j^{t_1},\ldots,\theta_j^{t_m},\theta_j^*),
\qquad \lim_{n,d \to \infty} \frac{1}{n}\sum_{i=1}^n
\psi(\x_i^\top \btheta^{t_1},\ldots,\x_i^\top \btheta^{t_m},\x_i^\top
\btheta^*),\]
depending on the latent parameter $\btheta^*=[\theta_1^*,\ldots,\theta_d^*]^\top
\in \R^{d \times k^*}$ and SGD iterates
$\btheta^t=[\theta_1^t,\ldots,\theta_d^t]^\top \in \R^{d \times k}$
at fixed times $t_1,\ldots,t_m$ in units of training epochs. Such
observables can encompass, for example, the overlap parameters
$d^{-1}(\btheta^t,\btheta^*)^\top(\btheta^t,\btheta^*)$ and evaluations of the
training and test losses. To understand the similarities and differences
between SGD and its commonly studied diffusion approximations,
we will establish also an analogous
characterization of the Stochastic Modified Equation (SME) that constitutes
the most standard SDE approximation to SGD \cite{li2017stochastic}.

As these characterizations involve dynamical mean-field theory and
are somewhat complex, we summarize here a few qualitative aspects/implications
of our results:

\begin{enumerate}[1.]
\item For any sub-linear choice of batch size
\[\kappa \asymp n^\alpha\]
with $\alpha \in [0,1)$, and for a commensurate ``critical'' scaling of the learning rate
$\eta \asymp n^\alpha$,\footnote{Our scaling conventions for the model are $\|\x_i\|_2 \asymp
1$ and $\|\btheta\|_\sF \asymp \sqrt{d}$. Under a common alternative scaling of
$\|\x_i\|_2 \asymp \sqrt{d}$ and $\|\btheta\|_\sF \asymp 1$, the equivalent
learning rate is $\tilde \eta=\eta/d \asymp n^{-1+\alpha}$.} there is a
well-defined continuous-time scaling limit for
the SGD iterates under a time rescaling by $n^{1-\alpha}$,
i.e.\ to time units of training epochs.

This limit is the same for any choice of $\alpha \in [0,1)$ (but
different from the discrete-time limit for linear batch sizes $\kappa
\asymp n$), depending only
on $\bar\kappa=\lim_{n,d \to \infty} \kappa/n^\alpha$ and
$\bar \eta=\lim_{n,d \to \infty} \eta/n^\alpha$. This corroborates
empirical observations 
\cite{jastrzkebski2017three,he2019control} that learning
dynamics of SGD may depend moreso on the ratio of learning rate to batch
size than the absolute size of the batch.

\item The random sampling of mini-batches in SGD leads to a dynamical
mean-field limit in which a scalar process $\{\xi^t\}_{t \geq 0}$
that tracks the distributional dynamics of $\{\{\x_i^\top \btheta^t\}_{i \in
[n]}\}_{t \geq 0}$ is
driven by a univariate Poisson jump process $\{z^t\}_{t \geq 0}$.
For the SME diffusion approximation of SGD, the dynamics of $\{\xi^t\}_{t \geq 0}$ are
analogously driven by a univariate Brownian motion with drift whose
mean/covariance match those of the Poisson process $\{z^t\}_{t \geq 0}$.

The presence of this Poisson jump process constitutes the primary
difference between the asymptotic dynamics of SGD versus gradient flow or SME.

\item In a further limit of small learning rate $\bar \eta \to 0$ and
rescaling of time, the asymptotic dynamics of both
SGD and SME reduce to those of gradient flow. This agrees with the classical
picture of SGD as an approximation to the gradient-flow ODE under small learning rate asymptotics, and clarifies that in this high-dimensional context,
the notion of a ``small'' learning rate under which such an ODE approximation
is accurate may be understood as $\eta \ll 
n^\alpha$. For larger learning rates represented by any
fixed value of $\bar \eta=\lim_{n,d \to \infty} \eta/n^\alpha>0$, the
asymptotic dynamics of SGD and gradient flow are distinct.

\item In the special case of least-squares linear regression, the
dynamical mean-field equations simplify to a system of deterministic
integro-differential equations which, in particular,
depend only on the first and second moment statistics of the Poisson
process/Brownian diffusion $\{z^t\}_{t \geq 0}$. Thus, in this case, the
dynamics of quadratic observables of SME and SGD coincide even in the large
learning-rate setting of fixed $\bar\eta>0$, in agreement with the results of
\cite{paquette2025homogenization}. Outside of linear regression,
the asymptotic dynamics of SGD and its SME diffusion approximation are also distinct.

\item Setting $\gamma=\lim_{n,d \to \infty} n/d$, in a further limit
$\gamma \to \infty$ and re-indexing of time,
the dynamical mean-field equations for both SGD and SME
reduce to a much simpler distribution-dependent SDE. Further
specializing this SDE
to a ridge regularizer recovers the known autonomous ODE for 
overlap parameters in online SGD \cite{saad1995dynamics}.
Comparison of these limiting dynamics around initialization $t=0$ 
illustrates that learnability of $\btheta^*$ by multi-pass SGD with linear
sample complexity $n \asymp d$ is dictated by the generative rather than
information exponent \cite{arous2021online,damian2024computational},
in agreement with previous work of \cite{lee2024neural,dandi2024benefits} who first established analogous insights for more
specialized versions of gradient descent dynamics.

In this way, our result bridges the online and multi-pass SGD theories, and
also obtains a generalization of the online SGD theory to 
coordinate-separable regularizers beyond ridge regularization.

\item While we do not carry out a general theoretical analysis of the long-time
behavior of this scaling limit for SGD (and leave this as an important
direction for future work), we make available a software implementation for
numerically simulating this limit.

Since the driving Poisson process $\{z^t\}_{t \geq 0}$ in the SGD
scaling limit is discrete, with $O_\P(1)$ jumps over time horizons of
$O(1)$ training epochs, this enables more rapid numerical simulation of
the DMFT equations for SGD than of the analogous equations for gradient flow
or SME. Our simulations confirm that in non-convex landscapes, SGD may reach
local minimizers of the empirical risk with different generalization
properties from those reached by either SME or gradient flow, and these
differences are deterministic in nature and well-predicted by the DMFT theory.
\end{enumerate}

We will elaborate upon each of the above points further in
Section \ref{sec:results}, after formally stating our main results.

\subsection{Related literature and technical innovations}\label{sec:proof-novelty}

Dynamical mean-field theory encompasses a body of techniques initially developed
in statistical physics to study dynamics in models of spin glasses and disordered systems
\cite{sompolinsky1981dynamic,sompolinsky1982relaxational,crisanti1993spherical,cugliandolo1993analytical}.
For such models, the first mathematically rigorous formalizations of DMFT were
obtained by
\cite{arous1995large,grunwald1996sanov,arous1997symmetric,guionnet1997averaged,ben2006cugliandolo}.

In recent years, DMFT has been applied to study gradient-based
learning dynamics in a variety of statistical learning models
\cite{sarao2019afraid,mannelli2019passed,sarao2020complex,sarao2020marvels,sarao2021analytical,liang2023high}, with
substantial focus on training dynamics in various neural network
architectures
\cite{bordelon2022self,bordelon2024dynamical,bordelon2024infinite,montanari2025dynamical,han2025precise,martin2026high}.
While many such analyses were initially non-rigorous, an important result of
\cite{celentano2021high} developed a new method of mathematically
formalizing the DMFT theory by discretizing in time. Rigorous DMFT analyses
building upon this approach are now available for a wide class of
discrete-time and continuous-time dynamics 
\cite{gerbelot2024rigorous,han2025entrywise,fan2025dynamical1,fan2025dynamical2,han2025long,chen2025learning,celentano2025state,dandi2025sequential}.

We highlight within this literature a few existing analyses of variants of SGD and related
sequential dynamics: A line of work of
\cite{mignacco2021dynamical,mignacco2021stochasticity} and \cite{mignacco2022effective}
developed (non-rigorously) a discrete-time DMFT limit for SGD
with a linear batch size of $\kappa \asymp n$ samples per iteration, as well as a
continuous-time limit for ``persistent SGD'' dynamics that maintains this
linear batch
size while replacing samples individually. The former discrete-time limit was
shown mathematically as part of the results of \cite{gerbelot2024rigorous}.
A classical paper of \cite{grunwald1996sanov} established a continuous-time
DMFT limit for Glauber dynamics in the Sherrington-Kirkpatrick model, as the
minimizer of a large deviations rate functional; we believe that
a translation of this result to a more explicit stochastic process description
may have qualitative similarities with the characterization that we show here for SGD.
More recently, \cite{dandi2025sequential} obtained an alternative
discrete-time DMFT limit for a variant of Glauber dynamics 
with sequential batch updates having $\kappa \asymp n$ variables per batch.

In the context of this literature, the main differences and innovations of our
analyses are the following:

\begin{itemize}
\item \emph{Emergence of a continuous-time limit for $\alpha<1$.}
In contrast to \cite{gerbelot2024rigorous},
we study a regime of sublinear batch size
\[\kappa \asymp n^\alpha\]
for some $\alpha \in [0,1)$ strictly less than 1. Thus a growing number
$O(n^{1-\alpha})$ of SGD iterations constitutes a single training epoch in
which each coordinate of $\btheta^t$ undergoes a $O_\P(1)$ change, resulting
in a continuous-time rather than discrete-time scaling limit of the dynamics
as $n,d \to \infty$.
Within an epoch, each training sample is selected $O_\P(1)$ times, and thus
sampling noise does not vanish in this limit but is instead represented by
a stochastic Poisson process. Illustrating that standard multi-pass
SGD admits such a continuous-time scaling limit is a primary contribution of our work.

This limit is different from that of persistent SGD
\cite{mignacco2021dynamical} which preserves a batch size of $\kappa \asymp n$.
As sub-linear batch sizes $\kappa \ll n$ (including the most practically
relevant regime of $\kappa \asymp 1$) leads to greater
variability of the stochastic gradient noise, we believe that
the standard SGD dynamics may be qualitatively different from
persistent SGD and lead to different generalization
properties in non-convex models, and this would be interesting to
investigate further.

In addition, the continuous-time nature of the limit dynamics
enables their analyses via stochastic calculus techniques, by which
we are able to show how this limit may be specialized to recover
the DMFT theory of gradient flows \cite{celentano2021high}, overlap ODEs of
online SGD \cite{saad1995dynamics}, and coincidence with
stochastic diffusion (SDE) approximations in least-squares
problems \cite{paquette2025homogenization}. We believe that such connections
would be difficult to establish in a discrete-time theory.

\item \emph{Stochastic Volterra equation and discretization error.}
Our analyses develop upon the proof strategy introduced in
\cite{celentano2021high} to analyze gradient flow, by ``discretizing'' the SGD
dynamics in time and establishing contractivity of the dynamical fixed-point
equations for the correlation/response kernels to pass to continuous time.

For gradient flow, the time discretization of \cite{celentano2021high} is a
standard Euler approximation of the ODE. For SGD, we instead construct a
``$\delta$-discretized'' SGD process that performs a Poissonized time
embedding of the SGD updates, and freezes the gradient computation in
time intervals of $\delta=O(1)$ epochs. The $\delta$-discretized
dynamics admit a discrete-time asymptotic
limit as $n,d \to \infty$ via a reduction to Approximate Message Passing (AMP).

Several main technical innovations of our work lie in addressing two central
difficulties that arise when taking $\delta \to 0$ to pass this
characterization to continuous time. The first difficulty is that in the
continuous-time DMFT scaling limit, the defining equation for
$\{\xi^t\}_{t \geq 0}$ is a stochastic Volterra integral with integration
against the Poisson process/Brownian diffusion $\{z^t\}_{t \geq 0}$.
Thus the usual
approach of constructing a response process $\{\frac{\partial f(\xi^t)}{\partial
w^s}\}_{t \geq s \geq 0}$ by formal differentiation through this integral
is not well-defined (as an ordinary stochastic process, rather than a process
of Schwarz distributions). To address this issue in a unified way that
encompasses both SGD and SME, we define this response process
as a bounded operator on $L^4([0,t])$, and derive a closed stochastic
Volterra equation for its evaluation on test paths $x^{[0:t]} \in
L^4([0,t])$. The stochastic integral defining $\{\xi^t\}_{t \geq 0}$ also
leads to covariance kernels $C_f^{t,s}$ and $C_\theta^{t,s}$ 
that are not twice-continuously-differentiable along the diagonal, and hence
Gaussian processes $\{u^t\}_{t \geq 0}$ and $\{w^t\}_{t \geq 0}$ in the DMFT
limit that do not have differentiable sample paths; this necessitates also an
integrated definition of the response process
$\{\frac{\partial \theta^t}{\partial u^s}\}_{t \geq s \geq 0}$. We discuss
these issues further in Remark~\ref{rmk:response}. These differences lead to a
system of DMFT equations that are more complex than those of prior works
\cite{celentano2021high,fan2025dynamical1}, and we develop new techniques
based around stochastic calculus to establish contractivity and existence/uniqueness of their solutions.

The second difficulty lies in showing that the $\delta$-discretized SGD
dynamics and original (Poissonized) SGD dynamics in $\R^d$
coincide asymptotically as $\delta \to 0$, i.e.\
\[
\limsup_{n,d\to\infty} \sup_{t \leq T}
\frac{\|\btheta^t-\btheta_\delta^t\|_\sF^2}{d}=O(\delta),
\]
where importantly the limit $\delta \to 0$ must be taken after the
high-dimensional limit $n,d \to \infty$. We apply arguments extending those of
\cite{paquette2025homogenization} developed in the linear model, which use the
Poissonized time embedding of SGD to analyze this discretization error
via a stochastic Gr\"onwall argument and Burkholder-Davis-Gundy
martingale inequalities. This analysis is simpler with a Lipschitz loss
$L(\cdot)$, but is more intricate in settings of pseudo-Lipschitz loss
functions including the most commonly-used squared loss. In such settings,
we develop a three-step bootstrapping argument
that first uses a martingale analysis to obtain a high-moment bound
$\E\sup_{t\leq T}\|\btheta^t\|_\sF^{2p}\lesssim n^{p}$ for any $p \geq 1$,
then performs a leave-one-out analysis to obtain refined
control of $\sup_{i=1}^n \|\x_i^\top \btheta^t\|_2$, and finally
applies this back into the martingale analysis to show
$\sup_{t\leq T}\|\btheta^t\|_\sF^2 \lesssim n$
and $\sup_{t\leq T}\|\btheta^t-\btheta_\delta^t\|_\sF^2 \lesssim \delta n$
almost surely for all large $n,d$.
\end{itemize}

\subsection*{Concurrent work}
Independent work of \cite{nishiyama2026high} (concurrent with the
posting of our paper on arXiv) establishes a DMFT characterization of the
high-dimensional limit dynamics for SME, referred to therein as Stochastic
Gradient Flow (SGF). Their DMFT equations are equivalent to ours for SME
although formulated in an alternative way, with a central difference being
that they apply a Gaussian integration-by-parts technique in the SME
setting of Gaussian $\{z^t\}_{t \geq 0}$
to give a different rigorous definition for the averaged response
$\E[\frac{\partial f(\xi^t)}{\partial w^s}]$.
\cite{nishiyama2026high} obtain convergence in probability over times $t \in
[0,T^*]$ for some $(L,G,\sigma)$-dependent time horizon $T^*<\infty$, whereas
our results establish a stronger notion of a.s.\ convergence over $[0,T]$ for
any fixed $T>0$, and we characterize also the SGD dynamics in addition to SME.

\subsection*{Notational conventions} For vectors $\x \in \R^k$ and $\y \in \R^l$,
we write $\x \otimes \y \in \R^{k \times l}$ for the outer product. For a
function $g:\R^k \to \R^k$ and input $\btheta \in \R^{d \times k}$, we 
write $g(\btheta) \in \R^{d \times k}$ for the application of $g(\cdot)$
to each row. We write $\frac{1}{n}\sum_{i=1}^n \delta_{\v_i}$ for the empirical
distribution on $\R^k$ of the rows of
$\V=[\v_1,\ldots,\v_n]^\top \in \R^{n \times k}$.

The notation
$\|\cdot\|$ without subscript will refer to the Euclidean norm $\|\cdot\|_2$ for
vectors and Frobenius norm $\|\cdot\|_\sF$ for matrices. We write
$\|\cdot\|_\op$ for the matrix $\ell_2$-to-$\ell_2$ operator norm, and
$\|\cdot\|_\infty$ for the element-wise $\ell_\infty$ norm. $\Id_k$ is the $k
\times k$ identity matrix, and we omit the subscript when the dimension is
clear. We write $x \lesssim
y$ if $x \leq Cy$ for a constant $C>0$, which does not depend on the dimensions $n,d$ but will usually depend on the fixed
time horizon $T>0$. We write $x \asymp y$ if
both $x \lesssim y$ and $y \lesssim x$.

\section{Main results}\label{sec:results}

\subsection{Model and dynamics}\label{sec:model}

We consider a standard setting of supervised learning, where
\[(\x_1,y_1),\ldots,(\x_n,y_n) \in \R^d \times \R\]
are $n$ independent training observations. It will be assumed that the data
$\x_i \in \R^d$ has independent coordinates with isotropic covariance, and
that the label $y_i \in \R$ follows a multi-index model: For some parameter $\btheta^* \in \R^{d \times k^*}$ where $k^* \geq 1$,
noise variables
\[\eps_1,\ldots,\eps_n \in \R,\]
and label map $\sigma^*:\R^{k^*} \times \R \to \R$, each label $y_i$
takes the form
\begin{equation}\label{eq:multiindex}
y_i=\sigma^*(\x_i^\top \btheta^*,\eps_i).
\end{equation}

Our main result will characterize the training dynamics of a multi-pass
Stochastic Gradient Descent (SGD) algorithm and its Stochastic Modified Equation
(SME) approximant for minimizing a regularized
empirical risk over $\btheta \in \R^{d \times k}$,
\begin{equation}\label{eq:ERM}
\cR(\btheta)=\sum_{i=1}^n L\big(\sigma(\x_i^\top \btheta),y_i\big)
+\sum_{j=1}^d G(\theta_j).
\end{equation}
Here $\sigma:\R^k \to \R$ is a non-linear activation for some $k \geq 1$,
and $L:\R \times \R \to \R$ and $G:\R^k \to \R$ are a smooth loss function and regularizer. The gradient of $\cR(\btheta)$ takes a form
\begin{equation}\label{eq:riskgradient}
\nabla \cR(\btheta)=\sum_{i=1}^n
\x_i \otimes f(\x_i^\top \btheta,\x_i^\top \btheta^*,\eps_i)
+g(\btheta) \in \R^{d \times k},
\end{equation}
where the functions $f:\R^k \times \R^{k^*} \times \R \to \R^k$ and
$g:\R^k \to \R^k$ are defined by
\begin{equation}\label{eq:fgdef}
f(\xi,w^*,\eps)=L'(\sigma(\xi),\sigma^*(w^*,\eps))
\,\nabla \sigma(\xi),
\qquad g(\theta)=\nabla G(\theta),
\end{equation}
$L'(\cdot,\cdot)$ is the derivative of $L(\cdot,\cdot)$ in its first argument,
and $g(\cdot)$ is applied row-wise to $\btheta \in \R^{d \times k}$.\\

{\bf SGD.} The SGD dynamics, initialized at
\[\bar\btheta^0 \equiv \btheta^0\in \R^{d \times k},\]
are given by
\begin{equation}\label{eq:SGD}
\bar\btheta^{t+1}=\bar\btheta^t-\eta^t \underbrace{\left(\frac{1}{\kappa}
\sum_{i \in S^t} \x_i \otimes
f(\x_i^\top \bar\btheta^t,\x_i^\top \btheta^*,\eps_i)
+\frac{1}{n}\,g(\bar\btheta^t)\right)}_{:=\v^t},
\qquad S^t \sim \operatorname{Unif}(\set),
\end{equation}
where
\[\set=\{S \subset [n]:|S|=\kappa\}\]
denotes the collection of all subsets of $\{1,\ldots,n\}$ of size $\kappa$,
and $\kappa \geq 1$ is the batch size.
In each iteration, $\eta^t>0$ is a learning rate, and
$S^t \in \set$ is a batch of $\kappa$ samples chosen independently and
uniformly at random. The update $\v^t$ in
\eqref{eq:SGD} has expectation $n^{-1}\nabla_{\btheta} \cR(\bar\btheta^t)$
over randomness of $S^t$, and constitutes a stochastic approximation for 
this rescaled gradient.

We will study a regime of general sub-linear batch size
\[\kappa \asymp n^\alpha, \qquad \alpha \in [0,1),\]
in a high-dimensional limit where $n,d \to \infty$. (The setting of $\alpha=0$
encompasses a constant batch size independent of $n$, including
single-sample SGD with $\kappa=1$.) Then
$\frac{n}{\kappa} \asymp n^{1-\alpha}$ iterations of the SGD dynamics
\eqref{eq:SGD} constitute one ``epoch'' of training in which, on expectation,
each training sample is used once. We define a continuous-time embedding
$\{\btheta^t\}_{t \geq 0}$ of the SGD dynamics \eqref{eq:SGD}, with time
re-indexed on the scale of epochs, by
\begin{equation}\label{eq:SGDrescaled}
\btheta^t=\bar\btheta^{\lfloor tn^{1-\alpha} \rfloor}.
\end{equation}
Our results will characterize the training dynamics 
over a fixed number of epochs not depending on $n,d$.

Under our model scaling conventions (to be specified in Assumption
\ref{asp:model}) where
$\|\bar\btheta^t\|_\sF \asymp \sqrt{d}$ and $\|\x_i\|_2 \asymp 1$,
we consider a learning rate that scales also as
\[\eta^t \asymp n^\alpha.\]
One may check that this is the scaling under which both the
mean and variance per coordinate of $\eta^t\v^t$ accumulated over one training
epoch has $O(1)$ size, and thus non-trivial learning may occur over $O(1)$
epochs (in problems without symmetry at initialization). For single-sample SGD where $\alpha=0$ and
$\kappa=1$, this coincides also with learning rate scalings studied previously
in e.g.\ \cite{goldt2019dynamics,tan2023online,paquette2025homogenization}
(in both online and multi-pass settings).

We formalize the above conditions as the following assumption.

\begin{asp}[Batch size and learning rate]\label{asp:SGD}
There exist constants $\alpha \in [0,1)$, $\bar \kappa>0$, and a
Lipschitz-continuous learning rate schedule $\{\bar\eta^t\}_{t \geq 0}$
not depending on $n,d$, such that over each fixed time horizon $T>0$,
\[\lim_{n,d \to \infty} |\kappa/n^\alpha-\bar \kappa|=0, \qquad
\lim_{n,d \to \infty} \sup_{t \in [0,T]}
|\eta^{\lfloor tn^{1-\alpha}\rfloor}/n^\alpha-\bar \eta^t|=0.\]
\end{asp}

{\bf SME.} Under Assumption \ref{asp:SGD}, we define the SME approximation
\cite{li2017stochastic} to the above SGD algorithm as the continuous-time
diffusion process
\begin{equation}\label{eq:SME}
\d\btheta^t={-}\bar\eta^t\left(\sum_{i=1}^n \x_i \otimes
f(\x_i^\top \btheta^t,\x_i^\top \btheta^*,\eps_i)
+g(\btheta^t)\right)\d t
+\frac{\bar\eta^t}{\sqrt{\bar \kappa}} \sum_{i=1}^n \Big(\x_i \otimes
f(\x_i^\top \btheta^t,\x_i^\top \btheta^*,\eps_i)\Big) \d b_i^t
\end{equation}
where $\{b_i^t\}_{t \geq 0}$ are independent standard Brownian motions for
$i=1,\ldots,n$. More precisely, this may be understood as a diffusion
approximation to the time-rescaled SGD process $\{\btheta^t\}_{t \geq 0}$ in
\eqref{eq:SGDrescaled}, where the drift
and anisotropic diffusion terms of \eqref{eq:SME} are defined to match the mean and
covariance of the stochastic gradient updates defining \eqref{eq:SGDrescaled}.
(We note that \cite{paquette2025homogenization} has studied a further ``Homogenized
SGD'' approximation where each factor 
$f(\x_i^\top \btheta^t,\x_i^\top \btheta^*,\eps_i)$ in the diffusion term of
\eqref{eq:SME} is replaced by
$n^{-1}\sum_{i=1}^n f(\x_i^\top \btheta^t,\x_i^\top \btheta^*,\eps_i)$. We
will study instead the diffusion \eqref{eq:SME} which provides a
closer approximation to SGD.)

\subsection{Model assumptions}

We impose the following two assumptions on the data model and empirical risk.

\begin{asp}[Data model and asymptotic scaling]\label{asp:model}
\phantom{ }
\begin{enumerate}[(a)]
\item As $n,d \to \infty$, $n/d \to \gamma$ for a constant $\gamma>0$, and
$k,k^* \geq 1$ are constants not depending on $n,d$.
\item $\x_1,\ldots,\x_n \in \R^d$ are independent samples.
Each data vector $\x_i=(x_{i1},\ldots,x_{id}) \in \R^d$ has independent entries
with
\[\E[x_{ij}]=0, \quad \E[x_{ij}^2]=d^{-1}, \quad \E[|x_{ij}|^p] \leq
C_pd^{-p/2} \text{ for each $p \geq 3$ and a constant $C_p>0$}.\]
\item The initialization $\btheta^0=[\theta_1^0,\ldots,\theta_d^0]^\top
\in \R^{d \times k}$, parameter
$\btheta^*=[\theta_1^*,\ldots,\theta_d^*]^\top \in \R^{d \times k^*}$,
and noise variables $\beps=(\eps_1,\ldots,\eps_n)$ defining
$\y=(y_1,\ldots,y_n)$ via the multi-index model \eqref{eq:multiindex}
are (possibly deterministic and) independent of $\X=[\x_1,\ldots,\x_n]^\top$
and satisfy, almost surely as $n,d \to \infty$,
weakly and in Wasserstein-$p$ for each fixed order $p \geq 1$,
\[\frac{1}{d}\sum_{j=1}^d \delta_{(\theta_j^0,\theta_j^*)} \Rightarrow
\Law(\theta^0,\theta^*),
\qquad \frac{1}{n}\sum_{i=1}^n \delta_{\eps_i} \Rightarrow \Law(\eps).\]
Here $\Law(\theta^0,\theta^*)$ and $\Law(\eps)$ are any probability
distributions on $\R^k \times \R^{k^*}$ and $\R$, respectively, that have
finite moment generating functions in a neighborhood of 0.
\end{enumerate}
\end{asp}

Assumption \ref{asp:model}(c) allows a correlation between $\theta^0$ and
$\theta^*$ in the limiting law, and thus encompasses settings of
``warm'' initializations $\btheta^0$ having non-zero initial overlap with $\btheta^*$. In
the commonly studied example where columns of $\btheta^0,\btheta^*$ are
independent and uniformly distributed on the sphere $\cS^{d-1}(\sqrt{d})$,
$\Law(\theta^0,\theta^*)$ is the standard normal law on $\R^{k+k^*}$.

\begin{asp}[Loss and regularizer]\label{asp:lipschitz}
Let $f:\R^k \times \R^{k^*} \times \R \to \R^k$ and $g:\R^k \to \R^k$ be the
functions defined by \eqref{eq:fgdef}. For a constant $C>0$, these satisfy
\[\|f(\xi,w^*,\eps)\|_2 \leq C(1+\|\xi\|_2 + \|w^*\|_2+|\eps|),
\qquad \|g(\theta)\|_2 \leq C(1+\|\theta\|_2).\]
Furthermore, $f$ and $g$ are twice continuously-differentiable,
with all first and second order partial derivatives uniformly bounded by $C$.
\end{asp}

These conditions hold, for example, with sufficiently smooth and Lipschitz
activation $\sigma(\cdot)$, smooth and pseudo-Lipschitz loss $L(\cdot,\cdot)$
(e.g.\ squared loss $L(\hat y,y)=\frac{1}{2}(y-\hat y)^2$), and smooth and
pseudo-Lipschitz regularizer $G(\cdot)$ (e.g.\ ridge regularizer
$G(\theta)=\frac{\lambda}{2}\|\theta\|_2^2$).


\subsection{Definitions of the DMFT limit processes}\label{sec:DMFTlimits}

The high-dimensional limits of the SGD dynamics
\eqref{eq:SGDrescaled} and SME dynamics \eqref{eq:SME} are described via two
$\R^k$-valued stochastic process $\{\theta^t\}_{t \geq 0}$ and
$\{\xi^t\}_{t \geq 0}$, defined in the probability
space of the random variables 
\begin{equation}\label{eq:thetaeps}
(\theta^0,\theta^*) \in \R^k \times \R^{k^*}, \qquad \eps \in \R
\end{equation}
in Assumption \ref{asp:model}(c). Together with an additional Gaussian
vector $w^* \in \R^{k^*}$,
$\Law(\theta^t,\theta^*)$ and $\Law(\xi^t,w^*,\eps)$ will represent the limiting
empirical distributions of the coordinates
of $(\btheta^t,\btheta^*)$ and $(\X\btheta^t,\X\btheta^*,\beps)$,
respectively.

\subsubsection{Limit processes for SGD}\label{sec:DMFT-SGD}
Fix any time horizon $T>0$.
For SGD, these processes $\{\theta^t\}_{t \in [0,T]}$ and
$\{\xi^t\}_{t \in [0,T]}$ are defined via the equations
\begin{align}
\theta^t&=\theta^0-\int_0^t \bar\eta^r \Big({\gamma}\Gamma^r
\theta^r+g(\theta^r)+{\gamma}R_f^r(\theta^{[r]})+
{\gamma}R_f^{r,*}\theta^*\Big)\d r+\sqrt{\gamma}\,u^t,
\label{eq:def-cont-theta}\\
\xi^t&={-}\int_0^t \frac{\bar\eta^r}{\bar\kappa} R_\theta^{t,r}
f(\xi^{r-},w^*,\eps)\d z^r+w^t,\label{eq:def-cont-xi}
\end{align}
where $\bar\kappa,\{\bar\eta^t\}_{t \in [0,T]},\gamma$ are as defined in
Assumptions \ref{asp:SGD} and \ref{asp:model},
\begin{equation}\label{eq:GPs}
\{u^t\}_{t \in [0,T]} \sim \GP(0,C_f),
\quad (\{w^t\}_{t \in [0,T]},w^*) \sim \GP(0,C_\theta)
\end{equation}
are centered Gaussian processes with covariance kernels $C_f,C_\theta$, and
\begin{equation}\label{eq:poisson-z}
\{z^t\}_{t \in [0,T]} \sim \PP(\bar\kappa)
\end{equation}
is a $\{0,1,2,\ldots\}$-valued homogeneous Poisson jump process
with rate $\bar\kappa$. (Thus $z^t \sim \Pois(\bar\kappa t)$ marginally for each
$t \geq 0$.) These processes (\ref{eq:def-cont-theta}--\ref{eq:def-cont-xi})
are characteristic of DMFT theory in their correlated Gaussian fields
$\{u^t,w^t\}_{t \in [0,T]}$ and response/memory
kernels $R_f^t(\cdot),R_\theta^{t,r}$, which capture the aggregate effects
of interactions between a single coordinate $\theta_i^t$ and the other coordinates $\{\theta^t_j\}_{j \neq i}$ 
over time \cite{sompolinsky1981dynamic}.

Here $(\theta^0,\theta^*)$, $\eps$, $\{u^t\}_{t \in [0,T]}$, $(\{w^t\}_{t \in
[0,T]},w^*)$,
and $\{z^t\}_{t \in [0,T]}$ are mutually independent and
constructed on a filtered probability space $(\Omega,\cF,\{\cF_t\}_{t \in
[0,T]},\P)$ where $\{\cF_t\}_{t \in [0,T]}$ is a complete and right-continuous
filtration, $\theta^0,\theta^*,\eps,w^*$ are $\cF_0$-measurable, and
$u^t,w^t,z^t$ are $\cF_t$-adapted and c\`adl\`ag.
We write $x^{[s]}=\{x^t\}_{t \in [0,s]}$ for the restriction of a
process $\{x^t\}$ to times $t \in [0,s]$, and
$x^{t-}=\lim_{s \uparrow t} x^s$ for any c\`adl\`ag process $\{x^t\}_{t \in
[0,T]}$. In these equations,
$(C_\theta,R_\theta,C_f,R_f,R_f^*,\Gamma)$ denote the following deterministic
objects:
\begin{itemize}
\item $C_\theta$ is the joint covariance kernel of a $\R^k$-valued Gaussian
process $\{w^t\}_{t \in [0,T]}$ and a Gaussian vector $w^* \in \R^{k^*}$.
We denote $C_\theta^{t,s}=\E[w^t \otimes w^s] \in \R^{k \times k}$
for $t,s \in [0,T]$,
$C_\theta^{t,*}=\E[w^t \otimes w^*] \in \R^{k \times k^*}$,
and $C_\theta^{*,*}=\E[w^* \otimes w^*] \in \R^{k^* \times k^*}$.
\item $C_f$ is the covariance kernel of a $\R^k$-valued Gaussian process
$\{u^t\}_{t \in [0,T]}$. We denote likewise
$C_f^{t,s}=\E[u^t \otimes u^s] \in \R^{k \times k}$.
\item $R_\theta \equiv \{R_\theta^{t,s}\}_{t,s \in [0,T]}$ is a $\R^{k \times
k}$-valued process on $[0,T] \times [0,T]$,
where $R_\theta^{t,s}=0$ if $t<s$.
\item $R_f \equiv \{R_f^t\}_{t \in [0,T]}$ is a process of linear operators,
where $R_f^t:L^4([0,t],\R^k) \to \R^k$ for each $t \in [0,T]$.
\item $R_f^* \equiv \{R_f^{t,*}\}_{t \in [0,T]}$ is a $\R^{k \times k^*}$-valued
process on $[0,T]$.
\item $\Gamma \equiv \{\Gamma^t\}_{t \in [0,T]}$ is a $\R^{k \times k}$-valued
process on $[0,T]$.
\end{itemize}
These objects are defined self-consistently from the laws of
(\ref{eq:def-cont-theta}--\ref{eq:def-cont-xi}) via a system of
dynamical fixed-point relations: We define auxiliary response
processes $\{r_\theta^{t,s}\}_{t,s \geq 0}$ on $\R^{k \times k}$
and $\{r_f^{t,*}\}_{t \geq 0}$ on $\R^{k \times k^*}$ by
\begin{align}
r_\theta^{t,s}&=\Id_k-
\int_s^t \bar\eta^r\bigg[\Big({\gamma}\Gamma^r+\D
g(\theta^r)\Big)r_\theta^{r,s}+{\gamma}
R_f^r(r_\theta^{[r],s})\bigg]\d r \text{ for } t \geq s,
\label{eq:def-cont-deriv-t}\\
&\hspace{1in}r_\theta^{t,s}=0 \text{ for } t<s,\notag\\
r_f^{t,*}&={-}\D_\xi f(\xi^t,w^*,\eps) \int_0^t
\frac{\bar\eta^r}{\bar\kappa}
R_\theta^{t,r}r_f^{r-,*}\d z^r
+\D_{w^*} f(\xi^t,w^*,\eps),\label{eq:def-cont-deriv-xi-star}
\end{align}
and a process of random linear operators $\{r_f^t\}_{t \geq 0}$
with $r_f^t:L^4([0,t],\R^k) \to \R^k$ by
\begin{align}
r_f^t(x^{[t]})&={-}\D_\xi
f(\xi^t,w^*,\eps) \int_0^t \frac{\bar\eta^r}{\bar\kappa}
R_\theta^{t,r}\bigg(r_f^{r-}(x^{[r]})+\D_\xi
f(\xi^{r-},w^*,\eps)x^r\bigg)\d z^r.\label{eq:def-cont-resp-xi}
\end{align}
Here, $\D g(\theta)$ denotes the derivative (i.e.\ Jacobian matrix in $\R^{k
\times k}$) of $g(\cdot)$, and similarly
$\D_\xi f(\xi,w^*,\eps)$ and $\D_{w^*} f(\xi,w^*,\eps)$
denote the derivatives of $f(\cdot)$ in $\xi$ and $w^*$. In
(\ref{eq:def-cont-deriv-t}--\ref{eq:def-cont-resp-xi}),
$r_\theta^{[r],s}=\{r_\theta^{t,s}\}_{t \in [0,r]}$,
$R_f^r(r_\theta^{[r],s})$ is understood as its application to 
each column of $r_\theta^{[r],s}$, i.e.
\[R_f^t(x)=[R_f^t(x_{:,1}),\ldots,R_f^t(x_{:,k})] \text{ for }
x=[x_{:,1},\ldots,x_{:,k}] \in L^4([0,t],\R^k),\]
and $r_f^{t-,*}=\lim_{s \uparrow t} r_f^{s,*}$
and $r_f^{t-}(x^{[t]})=\lim_{s \uparrow t} r_f^s(x^{[s]})$.
The fixed-point conditions defining
$(C_\theta,R_\theta,C_f,R_f,R_f^*,\Gamma)$ are then given by
\begin{align}
C_\theta^{t,s}&=\E[\theta^t \otimes \theta^s] \text{ for } t,s \in [0,\infty)
\cup \{*\}
\label{eq:def-cont-C-t}\\
R_\theta^{t,s}&=\E[r_\theta^{t,s}]
\label{eq:def-cont-R-t}\\
C_f^{t,s}&=\E\Big[\int_0^t \frac{\bar\eta^r}{\bar\kappa}f(\xi^{r-},w^*,\eps)\d z^r
\otimes \int_0^s \frac{\bar\eta^r}{\bar\kappa}f(\xi^{r-},w^*,\eps)\d z^r\Big]\label{eq:def-cont-C-h}\\
R_f^t(x^{[t]})&=\E\big[r_f^t(x^{[t]})\big]\label{eq:def-cont-R-h}\\
R_f^{t,*}&=\E[r_f^{t,*}]
\label{eq:def-cont-R-h-star}\\
\Gamma^t&=\E[\D_\xi f(\xi^t,w^*,\eps)]\label{eq:def-cont-Gamma}
\end{align}
We clarify that in (\ref{eq:def-cont-R-h}), the input process $x^{[t]}=\{x^s\}_{s \in [0,t]}$ is
understood as deterministic, and the expectation is taken with respect to
$\{\xi^t\}_{t \in [0,T]}$, $\{z^t\}_{t \in [0,T]}$, $w^*$, and $\eps$ defining
(\ref{eq:def-cont-resp-xi}). This yields a deterministic linear operator
$R_f^t:L^4([0,t],\R^k)\to \R^k$, which is then applied to the stochastic
inputs in (\ref{eq:def-cont-theta}) and (\ref{eq:def-cont-deriv-t}).

\begin{remark}
The Poisson process $\{z^t\}_{t \in [0,T]}$ in this DMFT scaling limit
arises from independent random sampling of the SGD mini-batches. Our
techniques are extendable to other types of SGD sampling schemes such as
sequential scan, although this would lead to different (and usually more
complex) forms of the limiting dynamics. We refer to \cite{dandi2025sequential}
for an example of such a result in a related sequential dynamics setting.
\end{remark}

\subsubsection{Limit processes for SME}

For the SME diffusion process $\{\btheta^t\}_{t \geq 0}$
of (\ref{eq:SME}), the description of its high-dimensional limit is identical
to the above
equations (\ref{eq:thetaeps}--\ref{eq:def-cont-Gamma}) for SGD, except
with the Poisson jump process $\{z^t\}_{t \in [0,T]}$ in (\ref{eq:poisson-z})
replaced by the scalar diffusion process
\begin{equation}\label{eq:gaussian-z}
z^t=\int_0^t \bar\kappa\,\d s+\int_0^t \sqrt{\bar\kappa}\,\d b^s,
\end{equation}
where $\{b^t\}_{t \in [0,T]}$ is a $\cF_t$-adapted standard univariate Brownian
motion. I.e., the Poisson process integrals in (\ref{eq:def-cont-xi}),
(\ref{eq:def-cont-deriv-xi-star}), (\ref{eq:def-cont-resp-xi}), and
(\ref{eq:def-cont-C-h}) with respect to $\d z^r$
are replaced by It\^o integrals with respect to
$\bar\kappa\,\d r+\sqrt{\bar\kappa}\,\d b^r$. Note
that the two processes $\{z^t\}_{t \in [0,T]}$ defined by \eqref{eq:poisson-z}
and \eqref{eq:gaussian-z} have the same mean and covariance, but differ
otherwise in law.

\begin{remark}[Forms of the response processes]\label{rmk:response}
The preceding DMFT equations for SGD and SME are an extension of those for
gradient flow described in \cite{celentano2021high}. In particular, the
connections between our response processes and those of
\cite{celentano2021high} are as follows:

One may understand $r_f^{t,*}$ as a linear response
\begin{equation}\label{eq:rstarinterp}
r_f^{t,*}=\frac{\partial f(\xi^t,w^*,\eps)}{\partial w^*},
\end{equation}
where its evolution equation (\ref{eq:def-cont-deriv-xi-star}) is computed by 
substituting \eqref{eq:def-cont-xi} in \eqref{eq:rstarinterp} and
differentiating in $w^*$.

If $C_f^{t,s}$ were twice continuously-differentiable in $(t,s)$, then
setting $\bar C_f^{t,s}=\gamma \partial_t \partial_s C_f^{t,s}$, we may write
(\ref{eq:def-cont-theta}) as
\begin{equation}\label{eq:def-cont-theta-smooth}
\theta^t=\theta^0+\int_0^t \Big[{-}\bar\eta^r\Big({\gamma}
\Gamma^r\theta^r+g(\theta^r)+{\gamma}R_f^r(\theta^{[r]})
+{\gamma}R_f^{r,*}\theta^*\Big)+\bar u^r\Big]\d r
\end{equation}
where $\sqrt{\gamma}\,u^t=\int_0^t \bar u^r \d r$ and $\{\bar u^t\}_{t \in
[0,T]}
\sim \GP(0,\bar C_f)$. Then $r_\theta^{t,s}$ may likewise be understood as the
linear response
\begin{equation}\label{eq:Rthetainterp}
r_\theta^{t,s}=\frac{\partial \theta^t}{\partial \bar u^s}
\end{equation}
computed by formally differentiating (\ref{eq:def-cont-theta-smooth})
in $\bar u^s$. However, in our
settings where $\{z^t\}_{t \in [0,T]}$ is a semi-martingale,
$C_f^{t,s}$ as defined by (\ref{eq:def-cont-C-h})
is not twice differentiable at the diagonal $s=t$, and $\{u^t\}_{t \in [0,T]}
\sim \GP(0,C_f)$ does not admit a representation $\sqrt{\gamma}\,u^t=\int_0^t
\bar u^r \d r$ in the sense of a usual Lebesgue-Stieltjes integral, as it
does not have differentiable sample paths. Thus these
representations \eqref{eq:def-cont-theta-smooth} and \eqref{eq:Rthetainterp} 
should be understood only in a formal sense.

One may also understand the linear operator $r_f^t(x^{[t]})$
in \eqref{eq:def-cont-resp-xi} as an integrated linear response
\begin{equation}\label{eq:responseinterp}
r_f^t(x^{[t]})
=\int_0^t \frac{\partial f(\xi^t,w^*,\eps)}{\partial w^q}\,x^q \d q
\end{equation}
against the test process $x^{[t]}=\{x^s\}_{s \in [0,t]}$.
The evolution equation (\ref{eq:def-cont-resp-xi}) is computed by formally
differentiating $f(\xi^t,w^*,\eps)$ in $w^q$ using (\ref{eq:def-cont-xi}),
\[\frac{\partial f(\xi^t,w^*,\eps)}{\partial w^q}
={-}\D_\xi f(\xi^t,w^*,\eps) \int_q^t \frac{\bar\eta^r}{\bar\kappa}
R_\theta^{t,r}\bigg(\frac{\partial f(\xi^{r-},w^*,\eps)}{\partial
w^q}+\D_\xi f(\xi^{r-},w^*,\eps)\1_{r=q}\bigg)\d z^r.\]
If $\{z^t\}_{t \in [0,T]}$ has continuously-differentiable sample paths, then a
quantity taking the form $\int_q^t F(r,\xi^r)\1_{r=q}\d z^r$ is interpreted as
$F(q,\xi^q)\partial_t z^t|_{t=q}$, leading to a response process 
analogous to the definition in \cite{celentano2021high}. As $\{z^t\}_{t \in
[0,T]}$ is not
continuously-differentiable in our setting,
to give a simple and mathematically rigorous meaning to this quantity,
we further integrate against $x^q\,\d q$ and
exchange the orders of integration in $r$ and $q$, yielding
\begin{align*}
&\int_0^t \frac{\partial f(\xi^t,w^*,\eps)}{\partial w^q}\,x^q\d q\\
&={-}\D_\xi f(\xi^t,w^*,\eps)\int_0^t\frac{\bar\eta^r}{\bar\kappa}
R_\theta^{t,r}\Bigg[\int_0^r \bigg(\frac{\partial f(\xi^{r-},w^*,\eps)}{\partial
w^q}+\D_\xi f(\xi^{q-},w^*,\eps)\1_{r=q}\bigg)x^q\d q\Bigg]\d z^r\\
&={-}\D_\xi f(\xi^t,w^*,\eps)\int_0^t\frac{\bar\eta^r}{\bar\kappa}
R_\theta^{t,r}\Bigg[\int_0^r \frac{\partial f(\xi^r,w^*,\eps)}{\partial
w^q}x^q\d q+\D_\xi f(\xi^{r-},w^*,\eps)x^r\Bigg]\d z^r.
\end{align*}
This gives the equation (\ref{eq:def-cont-resp-xi}) under the identification
(\ref{eq:responseinterp}). We note that for SME where $\{z^t\}_{t \in [0,T]}$
is a Brownian diffusion, an alternative method of defining the expected response
$R_f^t=\E[r_f^t]$ via Gaussian integration-by-parts was used
by \cite{nishiyama2026high}.

We provide these remarks to clarify a more intuitive interpretation of
these response processes, although we emphasize that our proofs will operate
directly on the definitions (\ref{eq:def-cont-theta}--\ref{eq:def-cont-Gamma}),
rather than these heuristic interpretations.
\end{remark}

\subsubsection{Existence and uniqueness of the DMFT fixed point}

The following establishes existence and
uniqueness of the fixed points $(C_\theta,R_\theta,C_f,R_f,R_f^*,\Gamma)$ 
satisfying (\ref{eq:def-cont-C-t}--\ref{eq:def-cont-Gamma})
in a suitable domain $\cS^\text{cont} \equiv \cS^\text{cont}(T,C_0)$.
We defer a precise definition of this
domain to Appendix \ref{sec:existenceuniqueness}.

\begin{thm}\label{thm:uniqueness-existence}
Fix any large enough constant $C_0 \equiv C_0(T)>0$, and let
$\cS \equiv \cS(T,C_0)$ and
$\cS^\text{cont} \equiv \cS^\text{cont}(T,C_0)$ be the spaces for
$(C_\theta,R_\theta,C_f,R_f,R_f^*,\Gamma)$ given in
Definition \ref{def:S}. In both settings of
$\{z^t\}_{t \in [0,T]}$ in (\ref{eq:poisson-z}) for SGD and 
in (\ref{eq:gaussian-z}) for SME, the following hold:
\begin{enumerate}[(a)]
\item Given any $(C_\theta,R_\theta,C_f,R_f,R_f^*,\Gamma) \in \cS$ and
the filtered space $(\Omega,\cF,\{\cF_t\}_{t \in [0,T]},\P)$
containing $(\theta^0,\theta^*),\eps,\{u^t\}_{t \in [0,T]},(\{w^t\}_{t \in
[0,T]},w^*),\{z^t\}_{t \in [0,T]}$,
there exist unique $\cF_t$-adapted c\`adl\`ag solutions to
\eqref{eq:def-cont-theta}, \eqref{eq:def-cont-deriv-t} for each fixed $s \in [0,T]$,
and unique $\cF_t$-adapted c\`adl\`ag solutions to
\eqref{eq:def-cont-xi}, \eqref{eq:def-cont-deriv-xi-star},
\eqref{eq:def-cont-resp-xi} for each fixed deterministic input process
$x \in L^4([0,T],\R^k)$.
\item There exists $(C_\theta,R_\theta,C_f,R_f,R_f^*,\Gamma) \in
\cS^\text{cont}$ satisfying the fixed-point conditions
(\ref{eq:def-cont-C-t}--\ref{eq:def-cont-Gamma}) defined by the solutions of
part (a), and this fixed point is unique in $\cS$.
\end{enumerate}
\end{thm}

\subsection{Convergence to the limit processes}

The following theorem is the main result of this paper,
showing that on fixed time horizons $[0,T]$, finite-time marginals
of the empirical
distributions of coordinates of $(\{\btheta^t\}_{t \in [0,T]},\btheta^*)$
and $(\{\X\btheta^t\}_{t \in [0,T]},\X\btheta^*,\beps)$ for both the SGD
process \eqref{eq:SGDrescaled} and SME process \eqref{eq:SME}
converge in law to the preceding DMFT processes $(\{\theta^t\}_{t \in
[0,T]},\theta^*)$ and $(\{\xi^t\}_{t \in [0,T]},w^*,\eps)$.

\begin{thm}\label{thm:DMFT-limit}
Fix $T>0$ and $0 \leq t_1 \leq \ldots \leq t_m \leq T$
not depending on $n,d$. For any sufficiently large constant $C_0>0$, let
$\cS^\text{cont} \equiv \cS^\text{cont}(T,C_0)$ be the domain of Theorem
\ref{thm:uniqueness-existence}.
\begin{enumerate}[(a)]
\item Let $\{\btheta^t\}_{t \in [0,T]}$
be the time-rescaled SGD process (\ref{eq:SGDrescaled})
where $\btheta^t=[\theta_1^t,\ldots,\theta_d^t]^\top \in \R^{d \times k}$,
and let $\Law(\cdot)$ denote the joint law of
$(\{\theta^t\}_{t \in [0,T]},\theta^*)$ or
$(\{\xi^t\}_{t \in [0,T]},w^*,\eps)$ in the DMFT
system (\ref{eq:thetaeps}--\ref{eq:def-cont-Gamma}) when $\{z^t\}_{t \in [0,T]}$ is the Poisson process
in (\ref{eq:poisson-z}),
defined by the unique fixed point $(C_\theta,R_\theta,C_f,R_f,R_f^*,\Gamma) \in
\cS^\text{cont}$. Then almost surely as $n,d \to \infty$,
weakly and in Wasserstein-2 over $(\R^k)^m \times \R^{k^*}$ and $(\R^k)^m \times \R^{k^*} \times \R$,
\begin{align}
\frac{1}{d}\sum_{j=1}^d
\delta_{(\theta_j^{t_1},\ldots,\theta_j^{t_m},\theta_j^*)}
&\Rightarrow
\Law(\theta^{t_1},\ldots,\theta^{t_m},\theta^*),\label{eq:DMFT-limit-theta}\\
\frac{1}{n}\sum_{i=1}^n \delta_{(\x_i^\top \btheta^{t_1},
\ldots,\x_i^\top \btheta^{t_m},\x_i^\top \btheta^*,\eps_i)}
&\Rightarrow \Law(\xi^{t_1},\ldots,\xi^{t_m},w^*,\eps).\label{eq:DMFT-limit-xi}
\end{align}
\item Let $\{\btheta^t\}_{t \in [0,T]}$ be the SME process
(\ref{eq:SME}). Then the same result holds upon replacing
the DMFT system (\ref{eq:thetaeps}--\ref{eq:def-cont-Gamma}) by that
in the Gaussian setting of $\{z^t\}_{t \in [0,T]}$ given in (\ref{eq:gaussian-z}).
\end{enumerate}
\end{thm}

\begin{remark}
Under Assumption \ref{asp:SGD}, the DMFT equations for SGD in Section \ref{sec:DMFT-SGD} depend only on
$(\bar\kappa,\{\bar\eta^t\}_{t \in [0,T]})$ and not on $\alpha$. Thus the
high-dimensional scaling limit of SGD in Theorem \ref{thm:DMFT-limit}(a)
is the same for all $\alpha \in [0,1)$.
\end{remark}

\begin{remark}\label{remark:observables}
Theorem \ref{thm:DMFT-limit} implies, as a consequence, that for any
pseudo-Lipschitz test functions
$\psi:(\R^k)^m \times \R^{k^*} \to \R$ 
and $\phi:(\R^k)^m \times \R^{k^*} \times \R \to \R$
(i.e.\ satisfying
$|\psi(x)-\psi(y)| \leq C\|x-y\|_2(1+\|x\|_2+\|y\|_2)$ for all inputs $x,y \in (\R^k)^m \times \R^{k^*}$), almost surely
\begin{align*}
\lim_{n,d \to \infty} \frac{1}{d}\sum_{j=1}^d
\psi(\theta_j^{t_1},\ldots,\theta_j^{t_m},\theta_j^*)
&=\E\,\psi(\theta^{t_1},\ldots,\theta^{t_m},\theta^*),\\
\lim_{n,d \to \infty} \frac{1}{n}\sum_{i=1}^n
\phi(\x_i^\top \btheta^{t_1},\ldots,\x_i^\top \btheta^{t_m},\x_i^\top
\btheta^*,\eps_i)
&=\E\,\phi(\xi^{t_1},\ldots,\xi^{t_m},w^*,\eps),
\end{align*}
where the expectations are evaluated over the joint laws of
the DMFT variables.

For example, for any fixed training time $t>0$ (in units of training epochs for
SGD, c.f.\ the time rescaling \eqref{eq:SGDrescaled}), almost surely as $n,d \to \infty$:
\begin{enumerate}
\item The pairwise overlaps between columns of the learned parameter
$\btheta^t \in \R^{d \times k}$ and true parameter
$\btheta^* \in \R^{d \times k^*}$ satisfy
\[\frac{1}{d}\,\btheta^{t\top}\btheta^*
\to \E[\theta^t \otimes \theta^*] \in \R^{k \times k^*}.\]
\item The average training loss satisfies
\[\frac{1}{n}\sum_{i=1}^n L\big(\sigma(\x_i^\top \btheta^t),y_i\big)
\to \E L\big(\sigma(\xi^t),\sigma^*(w^*,\eps))\big).\]
\item Suppose further that the data $\x_i \sim \cN(0,\frac{1}{d}\,\Id)$ is Gaussian. Then for an
independent test
sample $\x_\mathrm{test} \sim \cN(0,\frac{1}{d}\,\Id)$, conditional on $(\btheta^t,\btheta^*)$ we have
\[(\x_\text{test}^\top\btheta^t,\x_\text{test}^\top\btheta^*)
\sim \cN\left(0,\frac{1}{d}(\btheta^t,\btheta^*)^\top(\btheta^t,\btheta^*)\right).\]
Thus for any test loss $L_\text{test}:\R^k \times \R^{k^*} \to \R$, the expected
test error admits a representation
\[\E_{\x_\mathrm{test}} L_\text{test}(\x_{\text{test}}^\top \btheta^t,\x_{\text{test}}^\top \btheta^*)
=h\left(\frac{1}{d}(\btheta^t,\btheta^*)^\top(\btheta^t,\btheta^*)\right)\]
for some smooth function $h:\R^{(k+k^*) \times (k+k^*)} \to \R$. Then also
\[\E_{\x_\mathrm{test}} L_\text{test}(\x_{\text{test}}^\top \btheta^t,\x_{\text{test}}^\top \btheta^*)
\to h\big(\E[(\theta,\theta^*) \otimes (\theta,\theta^*)]\big).\]
\end{enumerate}
\end{remark}

\section{Discussion}

In this section, to provide further interpretation of our results and some
connections to related literature, we discuss some simplifications of the
preceding DMFT equations in the limit of small learning rate, limit of large
sample size, connection to information/generative exponents in single-index
models, and specialization to quadratic optimization for
squared loss and ridge regularizer.

\subsection{Small learning rate limit and gradient flow}\label{sec:comp-gf}

Consider, for simplicity, a time-independent learning rate
$\bar\eta^t=\bar\eta$.
We check that in the small learning rate limit $\bar\eta \to 0$,
the DMFT system (\ref{eq:def-cont-theta}--\ref{eq:def-cont-Gamma}) converges to
that obtained for gradient flow in \cite{celentano2021high}.

Introduce the rescaled time $\tau=\bar\eta t$,
and define the time-rescaled processes
\begin{equation}\label{eq:rescaled1}
\begin{gathered}
\tilde\theta^\tau=\theta^{\tau/\bar\eta},\quad
\tilde \xi^\tau=\xi^{\tau/\bar\eta},\quad
\tilde u^\tau = u^{\tau/\bar\eta},\quad
\tilde w^\tau=w^{\tau/\bar\eta},\quad
\tilde z^\tau=z^{\tau/\bar\eta},\\
\tilde r_\theta^{\tau,\sigma}=r_\theta^{\tau/\bar\eta,\sigma/\bar\eta},\quad
\tilde r_f^{\tau,*}=r_f^{\tau/\bar\eta,*},\quad
\tilde r_f^\tau(\tilde x^{[\tau]})=r_f^{\tau/\bar\eta}(x^{[\tau/\bar\eta]})
\text{ where } \tilde x^\tau=x^{\tau/\bar\eta},\end{gathered}
\end{equation}
together with the corresponding deterministic response and covariance processes
\begin{equation}\label{eq:rescaled2}
\begin{gathered}
\tilde C_\theta^{\tau,\sigma}=C_\theta^{\tau/\bar\eta,\sigma/\bar\eta}, \quad
\tilde R_\theta^{\tau,\sigma}=R_\theta^{\tau/\bar\eta,\sigma/\bar\eta}, \quad
\tilde C_f^{\tau,\sigma}=C_f^{\tau/\bar\eta,\sigma/\bar\eta},\\
\tilde R_f^{\tau}(x^{[\tau]})=R_f^{\tau/\bar\eta}(x^{[\tau/\bar\eta]}), \quad
\tilde R_f^{\tau,*}=R_f^{\tau/\bar\eta,*}, \quad
\tilde \Gamma^\tau=\Gamma^{\tau/\bar\eta}.\end{gathered}
\end{equation}
Note that $\{\tilde u^\tau\} \sim \GP(0,\tilde C_f)$,
$(\{\tilde w^\tau\},w^*) \sim \GP(0,\tilde C_\theta)$, and $\{\tilde z^\tau\}$
is a Poisson process with rate $\bar\kappa/\bar\eta$ in the SGD setting or a
diffusion process $\tilde z^\tau=\int_0^\tau (\bar\kappa/\bar\eta)\d \rho
+\int_0^\tau \sqrt{\bar\kappa/\bar\eta}\,\d b^\rho$ in the SME setting.

Applying this change of variables to
(\ref{eq:def-cont-theta}--\ref{eq:def-cont-resp-xi}), we have
\begin{align}
\tilde\theta^{\tau}&= \theta^0 - \int_0^{\tau} \Big({\gamma}\tilde\Gamma^\rho
\tilde\theta^\rho+g(\tilde\theta^\rho)+{\gamma}\tilde R_f^\rho(\tilde\theta^{[\rho]})+
{\gamma}\tilde R_f^{\rho,*}\theta^*\Big)\d \rho+\sqrt{\gamma}\,\tilde
u^{\tau},\label{eq:thetaGF}\\
\tilde r_\theta^{\tau,\sigma}&=\Id_k-\int_\sigma^\tau
\Big[\Big({\gamma}\tilde\Gamma^\rho+\D g(\tilde
\theta^\rho)\Big)\tilde r_\theta^{\rho,\sigma}
+\gamma\tilde R_f^\rho(\tilde
r_\theta^{[\rho],\sigma})\Big]\d \rho \text{ if } \tau \geq \sigma,
\quad \tilde r_\theta^{\tau,\sigma}=0 \text{ if }
\tau<\sigma,\label{eq:rthetaGF}\\
\tilde\xi^\tau&={-}\frac{\bar\eta}{\bar\kappa}
\int_0^\tau \tilde R_\theta^{\tau,\rho}
f(\tilde\xi^{\rho-},w^*,\eps)\d \tilde z^\rho+\tilde w^\tau,\\
\tilde r_f^{\tau,*}&={-}\D_\xi f(\tilde \xi^\tau,w^*,\eps)
\frac{\bar\eta}{\bar\kappa}
\int_0^\tau \tilde R_\theta^{\tau,\rho}\tilde r_f^{\rho-,*}\d \tilde z^\rho
+\D_{w^*} f(\tilde \xi^\tau,w^*,\eps),\\
\tilde r_f^\tau(\tilde x^{[\tau]})
&={-}\D_\xi f(\tilde \xi^\tau,w^*,\eps)
\frac{\bar\eta}{\bar\kappa}\int_0^\tau \tilde R_\theta^{\tau,\rho}
\Big(\tilde r_f^{\rho-}(\tilde x^{[\rho]})
+\D_\xi f(\tilde \xi^{\rho-},w^*,\eps)\tilde x^\rho\Big)\d \tilde z^\rho.
\end{align}
From this, it may be verified that $\tilde R_\theta,\tilde R_f,\tilde R_f^*$
remain uniformly bounded over finite intervals of rescaled time
$\tau,\sigma \in [0,\tilde T]$, as $\bar\eta \to 0$. Then, decomposing
the evolution of $\tilde \xi^\tau$ as
\[\tilde\xi^\tau= {-}\int_0^\tau \tilde R_\theta^{\tau,\rho}
f(\tilde\xi^\rho,w^*,\eps)\d \rho - \underbrace{\frac{\bar\eta}{\bar\kappa}
\int_0^\tau \tilde R_\theta^{\tau,\rho}
f(\tilde\xi^{\rho-},w^*,\eps)(\d \tilde
z^\rho-(\bar\kappa/\bar\eta)\d\rho)}_{:=M^\tau}+\,\tilde w^\tau,\]
we note that the martingale term $M^\tau$ satisfies
$\E\|M^\tau\|^2 \lesssim (\bar\eta/\bar\kappa)^2
\cdot (\bar\kappa/\bar\eta)=O(\bar\eta)$
over finite time horizons $\tau \in [0,\tilde T]$,
by the It\^o isometry. Then a
high-moment bound may be applied to show
$\lim_{\bar\eta \to 0} \sup_{\tau \in [0,\tilde T]} \|M^\tau\|=0$. The
same argument establishes that the martingale terms for the evolutions of
$\tilde r_f^{\tau,*}$ and $\tilde r_f^\tau(\tilde x^{[\tau]})$ vanish in the
limit $\bar\eta \to 0$, leading to the simplified equations
\begin{align}
\tilde\xi^\tau&={-}
\int_0^\tau \tilde R_\theta^{\tau,\rho}
f(\tilde\xi^{\rho},w^*,\eps)\d \rho+\tilde w^\tau,\label{eq:xiGF}\\
\tilde r_f^{\tau,*}&={-}\D_\xi f(\tilde \xi^\tau,w^*,\eps)
\int_0^\tau \tilde R_\theta^{\tau,\rho}\tilde r_f^{\rho,*}\d \rho
+\D_{w^*} f(\tilde \xi^\tau,w^*,\eps),\label{eq:rfstarGF}\\
\tilde r_f^\tau(\tilde x^{[\tau]})
&={-}\D_\xi f(\tilde \xi^\tau,w^*,\eps)
\int_0^\tau \tilde R_\theta^{\tau,\rho}
\Big(\tilde r_f^\rho(\tilde x^{[\rho]})
+\D_\xi f(\tilde \xi^\rho,w^*,\eps)\tilde x^\rho\Big)\d\rho.
\end{align}
Positing the form
\[\tilde r_f^\tau (\tilde x^{[\tau]})
=\int_0^\tau \frac{\partial f(\tilde\xi^\tau,w^*,\eps)}{\partial \tilde w^\sigma}
\,\tilde x^\sigma \d \sigma\]
as in Remark \ref{rmk:response}, this response process 
$\frac{\partial f(\tilde\xi^\tau,w^*,\eps)}{\partial \tilde w^\sigma}$ must then
satisfy
\begin{equation}
\frac{\partial f(\tilde\xi^\tau,w^*,\eps)}{\partial \tilde w^\sigma}
={-}\D_\xi f(\tilde\xi^\tau,w^*,\eps) \left(\int_\sigma^\tau
\tilde R_\theta^{\tau,\rho}\frac{\partial f(\tilde\xi^\rho,w^*,\eps)}{\partial
\tilde w^\sigma} \d \rho +
\tilde R_\theta^{\tau,\sigma}\D_\xi f(\tilde \xi^{\sigma},w^*,\eps)\right).
\label{eq:rfGF}
\end{equation}
These equations \eqref{eq:thetaGF}, \eqref{eq:rthetaGF}, \eqref{eq:xiGF},
\eqref{eq:rfstarGF}, and \eqref{eq:rfGF} (specialized to a ridge regularizer $G(\theta)=\frac{\lambda}{2}\|\theta\|^2$)
are equivalent to the DMFT system established in \cite{celentano2021high}
for gradient flow with ridge regularization.

This illustrates that our chosen learning rate $\eta \asymp n^{\alpha}$ is
the critical scaling for which stochasticity of mini-batches renders SGD
dynamics different from those of gradient flow, while for
$\eta \ll n^{\alpha}$ the asymptotic dynamics of SGD reduce to those of
gradient flow.

\subsection{Large-sample limit and one-pass SGD}\label{sec:comp-onepass}

Recalling the parameter $\gamma=\lim_{n,d \to \infty} n/d$, we next consider a
data-rich regime where $\gamma \to \infty$. For simplicity, 
consider a constant learning rate $\eta^t=\bar\eta^t=\bar \eta$
and single-sample batch size $\kappa=\bar\kappa=1$.
To keep the stochastic gradient dynamics \eqref{eq:SGD} the same as $n$ grows
(i.e.\ depending on $d$ rather than $n$), we define
\[\tilde g(\theta)=\gamma^{-1} g(\theta)\]
and fix $\tilde g(\cdot)$ as $\gamma \to \infty$.
Rescaling time as $\tau=\gamma t$, let us show that over any fixed time horizon
$\tau \in [0,\tilde T]$ (corresponding to a number of SGD iterations $n(\tilde
T/\gamma) \approx d\tilde T$ that 
does not depend on $n$), the preceding DMFT equations
reduce to a simple diffusion process that characterizes the one-pass or online
SGD procedure
\[\bar\btheta^{k+1}=\bar\btheta^k-\bar \eta
\Big(\x^k \otimes f(\x^{k\top} \bar\btheta^k,\x^{k\top}\btheta^*,\eps^k)
+\frac{1}{d}\,\tilde g(\bar\btheta^k)\Big)\]
with a fresh sample $(\x^k,y^k)=(\x^k,\sigma^*(\x^{k\top}\btheta^*,\eps^k))$
per iteration. We note that this is the standard online SGD procedure for 
minimizing the population risk
\[\E[L\big(\sigma(\x^\top \btheta),y\big)]
+\frac{1}{d}\sum_{j=1}^d \tilde G(\theta_j), \qquad \tilde G(\theta)
=\gamma^{-1} G(\theta).\]

We define the time-rescaled processes $\tilde\theta^\tau=\theta^{\tau/\gamma}$,
$\tilde \xi^\tau=\xi^{\tau/\gamma}$, $\tilde z^\tau=z^{\tau/\gamma}$,
etc.\ analogously to
(\ref{eq:rescaled1}--\ref{eq:rescaled2}). Then applying this change of variables
to (\ref{eq:def-cont-theta}--\ref{eq:def-cont-resp-xi}),
\begin{align*}
\tilde\theta^{\tau}&=\theta^0 - \bar\eta \int_0^{\tau} \Big(\tilde\Gamma^\rho
\tilde\theta^\rho+\tilde g(\tilde\theta^\rho)+\tilde R_f^\rho(\tilde\theta^{[\rho]})+\tilde R_f^{\rho,*}\theta^*\Big)\d \rho+\sqrt{\gamma}\,\tilde
u^{\tau},\\
\tilde r_\theta^{\tau,\sigma}&=\Id_k-\bar\eta\int_\sigma^\tau
\Big[\Big(\tilde\Gamma^\rho+\D g(\tilde \theta^\rho)\Big)
\tilde r_\theta^{\rho,\sigma}+\tilde R_f^\rho(\tilde
r_\theta^{[\rho],\sigma})\Big]\d \rho \text{ if } \tau \geq \sigma,
\quad \tilde r_\theta^{\tau,\sigma}=0 \text{ if }
\tau<\sigma,\\
\tilde\xi^\tau&={-}\bar\eta
\int_0^\tau \tilde R_\theta^{\tau,\rho}
f(\tilde\xi^{\rho-},w^*,\eps)\d \tilde z^\rho+\tilde w^\tau,\\
\tilde r_f^{\tau,*}&={-}\D_\xi f(\tilde \xi^\tau,w^*,\eps)
\cdot \bar\eta
\int_0^\tau \tilde R_\theta^{\tau,\rho}\tilde r_f^{\rho-,*}\d \tilde z^\rho
+\D_{w^*} f(\tilde \xi^\tau,w^*,\eps),\\
\tilde r_f^\tau(\tilde x^{[\tau]})
&={-}\D_\xi f(\tilde \xi^\tau,w^*,\eps)
\cdot \bar\eta \int_0^\tau \tilde R_\theta^{\tau,\rho}
\Big(\tilde r_f^{\rho-}(\tilde x^{[\rho]})
+\D_\xi f(\tilde \xi^{\rho-},w^*,\eps)\tilde x^\rho\Big)\d \tilde z^\rho.
\end{align*}
Here, $\{\tilde z^\tau\}_{\tau \geq 0}$ is a Poisson process with rate
$1/\gamma$ in the setting of SGD, or the diffusion process $\tilde
z^\tau=\int_0^\tau (1/\gamma)\d \rho+\int_0^\tau (1/\sqrt{\gamma})\d b^\rho$
in the setting of SME. Then one may again verify that $\tilde
R_\theta,\tilde R_f,\tilde R_f^*$ remain uniformly bounded as $\gamma \to
\infty$, and that in this limit, all stochastic integrals against $\d \tilde
z^\rho$ vanish to give the simplified equations
\begin{equation}\label{eq:theta-timeres-1}
\tilde\theta^\tau=\theta^0-\bar\eta \int_0^\tau \Big(\tilde \Gamma^\rho
\tilde\theta^\rho+\tilde g(\tilde \theta^\rho)+ \tilde
R_f^{\rho,*}\theta^*\Big)\d \rho+\sqrt{\gamma}\,\tilde u^\tau,
\end{equation}
\[\tilde \xi^\tau=\tilde w^\tau,
\quad \tilde r_f^{\tau,*}=\D_{w^*} f(\tilde \xi^\tau,w^*,\eps)
=\D_{w^*} f(\tilde w^\tau,w^*,\eps),
\quad \tilde r_f^\tau \equiv 0.\]
The covariance kernel of the rescaled Gaussian process $\{\sqrt{\gamma}\,\tilde
u^\tau\}_{\tau \geq 0}$ is given by
\begin{align*}
\gamma\,\tilde C_f^{\tau,\sigma} &= \gamma\,\E\Big[\bar\eta \int_0^\tau
f(\tilde\xi^{\rho-},w^*,\eps)\d \tilde z^\rho
\otimes \bar\eta\int_0^\sigma f(\tilde\xi^{\rho-},w^*,\eps)\d \tilde
z^\rho\Big].
\end{align*}
Applying the It\^o isometry and above simplification
$\tilde \xi^\tau=\tilde w^\tau$ in the $\gamma \to \infty$ limit,
we then obtain
\begin{align*}
\lim_{\gamma \to \infty} \gamma\,\tilde C_f^{\tau,\sigma}
&=\lim_{\gamma \to \infty}\left(O\Big(\frac{1}{\sqrt{\gamma}}\Big) +
\gamma\bar\eta^2\,\E\Big[\Big(\int_0^{\tau\wedge\sigma} f(\tilde
\xi^{\rho-},w^*,\eps)(\d \tilde z^r - (1/\gamma)\d r)\Big)^{\otimes
2}\Big]\right)\\
&=\bar \eta^2\int_0^{\tau\wedge\sigma} \E\Big[
f(\tilde w^\rho,w^*,\eps)^{\otimes 2}\Big]\d \rho.
\end{align*}
This is precisely the covariance kernel of a Brownian diffusion
\[\d \tilde u^\tau=\bar\eta \Sigma^\tau\d B^\tau,
\qquad \Sigma^\tau=\E[f(\tilde w^\tau,w^*,\eps)^{\otimes 2}]^{1/2},\]
where $\{B^\tau\}_{\tau \geq 0}$ is a standard Brownian motion on $\R^k$.

Thus, to summarize, in the limit $\gamma \to \infty$,
the law of $\{\tilde \theta^\tau\}_{\tau \geq 0}$ in \eqref{eq:theta-timeres-1}
is given by a distribution-dependent SDE
\begin{equation}\label{eq:onlineSDE}
\d \tilde\theta^\tau={-}\bar\eta \Big(\tilde \Gamma^\tau
\tilde\theta^\tau+\tilde g(\tilde \theta^\tau)+\tilde R_f^{\tau,*}
\theta^*\Big)\d \tau+\bar \eta\,\Sigma^\tau \d B^\tau,
\end{equation}
where the deterministic drift and diffusion
coefficients $\tilde \Gamma^\tau,\tilde R_f^{\tau,*},\Sigma^\tau$ are
defined self-consistently from the law of $(\tilde\theta^\tau,\theta^*)$ by
\begin{equation}\label{eq:SDEcoefs}
\begin{gathered}
\tilde C_\theta=\E[(\tilde \theta^\tau,\theta^*) \otimes (\tilde
\theta^\tau,\theta^*)] \in \R^{(k+k^*) \times (k+k^*)},
\quad (\tilde w^\tau,w^*) \sim \cN(0,\tilde C_\theta),\\
\tilde \Gamma^\tau=\E[\D_\xi f(\tilde w^\tau,w^*,\eps)],
\quad \tilde R_f^{\tau,*}=\E[\D_{w^*} f(\tilde w^\tau,w^*,\eps)],
\quad \Sigma^\tau=\E[f(\tilde w^\tau,w^*,\eps)^{\otimes 2}]^{1/2}.
\end{gathered}
\end{equation}
This coincides with the high-dimensional limit of the
learning dynamics of online SGD. For example, specializing to the case of a
ridge penalty $\tilde G(\theta)=\frac{\lambda}{2}\|\theta\|^2$ and
$\tilde g(\theta)=\lambda \theta$, we may apply It\^o's formula to
\eqref{eq:onlineSDE} to obtain a closed system of ODEs for the overlap
parameters $\tilde C_\theta=\E[(\tilde \theta^\tau,\theta^*) \otimes (\tilde
\theta^\tau,\theta^*)]$, which take the form
\begin{align}
\frac{\d}{\d \tau} \E [\tilde\theta^\tau \otimes \theta^*] &= -\bar\eta
\left((\tilde \Gamma^\tau+\lambda)\E[\tilde\theta^\tau \otimes \theta^*] +
\tilde R_f^{\tau,*}\E[\theta^* \otimes \theta^*] \right),\label{eq:overlap-online-q}\\
\frac{\d}{\d \tau} \mathbb{E} [\tilde\theta^\tau \otimes \tilde\theta^\tau] &= -\bar\eta \bigg((\tilde \Gamma^\tau+\lambda)
\E[\tilde\theta^\tau \otimes \tilde\theta^\tau]  + 
\tilde R_f^{\tau,*}\E[\theta^* \otimes \tilde\theta^\tau]\notag\\
&\hspace{1in}+\E[\tilde\theta^\tau \otimes \tilde\theta^\tau]
(\tilde \Gamma^\tau+\lambda)^\top + \E[\tilde\theta^\tau \otimes \theta^*]
\tilde R_f^{\tau,*\top} \bigg) + \bar\eta^2(\Sigma^\tau)^2\label{eq:overlap-online-p}
\end{align}
and $\tilde \Gamma^\tau,\tilde R_f^{\tau,*},\Sigma^\tau$ are the functions of
$\tilde C_\theta$ defined by \eqref{eq:SDEcoefs}. This recovers the known ODE
for overlap parameters in online SGD (in a setting of isotropic data), see
e.g.\ \cite[Corollary 1.1]{collins2024hitting}.

Our theory illustrates that \eqref{eq:onlineSDE} is the scaling limit of online
SGD even for non-ridge regularizers $\tilde G(\cdot)$, although in such
settings the equation \eqref{eq:onlineSDE} is non-linear and does not close
on the overlap statistics.

\begin{remark}\label{rmk:onlinevsonepass}
For a single-index model ($k=k^*=1$), several works have
illustrated an important distinction between the sample complexities of online
SGD and multi-pass algorithms for learning a vector
having non-vanishing overlap $\frac{1}{d}\btheta^\top \btheta^*$:
\cite{arous2021online} showed that this sample complexity for online SGD is
governed by the information exponent
\[\IE(\sigma^*)=\min\{k \geq 1:\E_{w^* \sim
\cN(0,1)}[H_k(w^*)\sigma^*(w^*,\eps)] \neq 0\}\]
where $H_k(\cdot)$ is the $k^\text{th}$ Hermite polynomial,
while \cite{lee2024neural} and \cite{dandi2024benefits} illustrated that the
sample complexities for certain variants/initializations of multi-pass SGD and gradient descent are instead governed by the generative exponent \cite{damian2024computational}
\[\GE(\sigma^*)=\min\{k \geq 1:\E_{w^* \sim
\cN(0,1)}[\E[H_k(w^*) \mid \sigma^*(w^*,\eps)]^2] \neq 0\}.\]

A version of this distinction between online and multi-pass SGD
for whether $\btheta^*$ is learnable with linear sample complexity
$n \asymp d$ may also be understood via comparison of the preceding DMFT
equations: Consider, as in these works, the squared loss
$L(\hat y,y)=\frac{1}{2}(\hat y-y)^2$,
$\btheta^* \sim \operatorname{Uniform}(\cS^{d-1}(\sqrt{d}))$,
and an initialization $\btheta^0$ independent of $\btheta^*$.
Then $\theta^* \sim \cN(0,1)$ and $w^* \sim \cN(0,1)$, and \eqref{eq:fgdef} gives
\begin{equation}\label{eq:init-square-f}
f(\xi,w^*,\eps)=\bigl(\sigma(\xi)-\sigma^*(w^*,\eps)\bigr)\sigma'(\xi),
\qquad \D_{w^*}f(\xi,w^*,\eps)=-{\sigma^*}'(w^*,\eps)\sigma'(\xi).
\end{equation}
First consider online SGD:
Since $\tilde\theta^0$ is independent of $\theta^*$, we have
$\E[\tilde\theta^0\theta^*]=0$. Under the condition
\[\IE(\sigma^*) \geq 2,\;\text{i.e.\ } \E[w^*\sigma^*(w^*,\eps)]
=\E[{\sigma^*}'(w^*,\eps)]=0,\]
a straightforward Gr\"onwall argument applied to
the preceding dynamics (\ref{eq:onlineSDE}--\ref{eq:SDEcoefs}) checks that 
\[\tilde R_f^{\tau,*} \equiv 0 \text{ and }
\E[\tilde\theta^\tau \theta^*]
=\E[\tilde w^\tau w^*] \equiv 0 \text{ for all } \tau \geq 0.\]
Thus, as long as $\IE(\sigma^*) \geq 2$, the asymptotic overlap 
$\E[\tilde\theta^\tau\theta^*]$ in online SGD remains 0 for any fixed
$\tau>0$, i.e.\ $\btheta^*$ is not learnable with any linear
number of samples $n \approx \tau d$.

In contrast, we claim that for the DMFT limit of multi-pass SGD, as long as
\[\GE(\sigma^*)=1,\;\text{i.e.\ } \E[w^* \mid \sigma^*(w^*)] \neq 0,\]
there exists a Lipschitz activation $\sigma:\R \to \R$ (in fact, for all
``suitably generic'' activations) for which
\begin{equation}\label{eq:learnthetastar}
\E[\theta^t \theta^*] \neq 0 \text{ for all small } t>0.
\end{equation}
Thus $\btheta^*$ is learnable by 
multi-pass SGD with any linear sample size $n \asymp d$ and $k=tn$
iterations, for any small $t>0$ fixed irrespective of $(n,d)$.

We may check this claim \eqref{eq:learnthetastar} by computing $\frac{\d}{\d
t} \E[\theta^t\theta^*]$ and $\frac{\d^2}{\d t^2} \E[\theta^t\theta^*]$ at
$t=0$: Setting the regularizer
$g(\theta) \equiv 0$ for simplicity, and taking the overlap of
\eqref{eq:def-cont-theta} with $\theta^*$,
\begin{equation}\label{eq:init-q-first-general}
\frac{\d}{\d t} \E[\theta^t\theta^*]
=-\bar\eta\gamma\left(\Gamma^t\E[\theta^t\theta^*]
+\E[R_f^t(\theta^{[t]})\theta^*]+R_f^{t,*}\right).
\end{equation}
At time $t=0$, $R_f^0=0$ and $\E[\theta^0\theta^*]=0$, hence
\begin{equation}\label{eq:init-qgamma-first}
\frac{\d}{\d t} \E[\theta^t\theta^*]\Big|_{t=0}
=-\bar\eta\gamma R_f^{0,*}=\bar\eta\gamma\E[{\sigma^*}'(w^*,\eps)]\E[\sigma'(w^0)]
\end{equation}
If $\IE(\sigma^*)=1$, i.e.\ $\E[{\sigma^*}'(w^*,\eps)] \neq 0$, then this is
non-zero for any activation $\sigma:\R \to \R$
such that $\E[\sigma'(w^0)] \neq 0$, implying
\eqref{eq:learnthetastar} for all small enough $t>0$.

Supposing, then, that $\E[{\sigma^*}'(w^*,\eps)]=0$ so $\IE(\sigma^*) \geq 2$,
differentiating \eqref{eq:init-q-first-general} at $t=0$ and using $
\E[\theta^0\theta^*]=0$ and $\frac{\d}{\d t}\E[\theta^t\theta^*]|_{t=0}=0$ yields
\begin{equation}\label{eq:init-q-second-reduced}
\frac{\d^2}{\d t^2}\E[\theta^t\theta^*]\Big|_{t=0}
=-\bar\eta\gamma\left(\left.\frac{\d}{\d t}
\E[R_f^t(\theta^{[t]})\theta^*]\right|_{t=0}
+\left.\frac{\d}{\d t}R_f^{t,*}\right|_{t=0}\right).
\end{equation}
For the Poisson process $\{z^t\}_{t \geq 0}$, as $t\downarrow 0$,
\begin{equation}\label{eq:init-jump-probabilities}
\P(z^t=0)=1-\bar\kappa t+o(t),\qquad
\P(z^t=1)=\bar\kappa t+o(t),\qquad
\P(z^t\geq2)=o(t).
\end{equation}
On the event $z^t=0$, we have $\xi^t=w^t$ by \eqref{eq:def-cont-xi}, and hence
$r_f^{t,*}=\D_{w^*} f(\xi^t,w^*,\eps)=-{\sigma^*}'(w^*,\eps)\sigma'(w^t)$.
Splitting \eqref{eq:def-cont-deriv-xi-star} according to
\eqref{eq:init-jump-probabilities} gives
\begin{align}
R_f^{t,*}=\E[r_f^{t,*}]
={}&\underbrace{\E[-{\sigma^*}'(w^*,\eps)\sigma'(w^t)
\1_{\{z^t=0\}}]}_{(\mathrm{I})}
+\underbrace{\E[r_f^{t,*}\1_{\{z^t=1\}}]}_{(\mathrm{II})}
+o(t).
\label{eq:init-teacher-response-split}
\end{align}
First, independence of $\{z^t\}_{t \geq 0}$ with
$(\{w^t\}_{t \geq 0},w^*,\eps)$ gives
\begin{align*}
(\mathrm{I})
&=\E[-{\sigma^*}'(w^*,\eps)\sigma'(w^t)]
-\bar\kappa t\,\E[-{\sigma^*}'(w^*,\eps)\sigma'(w^t)]+o(t)\\
&=\E[-{\sigma^*}'(w^*,\eps)\sigma'(w^t)]
-\bar\kappa t\,\E[-{\sigma^*}'(w^*,\eps)\sigma'(w^0)]+o(t)
={-}\E[{\sigma^*}'(w^*)\sigma'(w^t)]+o(t),
\end{align*}
the last equality holding by independence of $w^0,w^*,\eps$
and the assumption $\E[{\sigma^*}'(w^*,\eps)]=0$.
Representing the Gaussian triple
$(w^0,w^*,w^t)\overset{L}{=}(w^0,w^*,p(t)w^0+q(t)w^*+r(t)w')$ in law, where
$w^0,w^*,w' \overset{iid}{\sim} \cN(0,1)$ and
$p(t)=\E[w^0w^t]=\E[\theta^0\theta^t]$, $q(t)=\E[w^tw^*]=\E[\theta^t\theta^*]$,
$r(t)^2=\E[(w^t)^2]-p(t)^2-q(t)^2
=\E[(\theta^t)^2]-p(t)^2-q(t)^2$, we get
\begin{align}
\left.\frac{\d}{\d t}(\mathrm{I})\right|_{t=0}
&={-}\left.\frac{\d}{\d t}
\E[{\sigma^*}'(w^*,\eps)\sigma'(w^t)]\right|_{t=0}
\notag\\
&={-}\E[{\sigma^*}'(w^*,\eps)\sigma''(w^0)
(p'(0)w^0+q'(0)w^*+r'(0)w')]=0,\label{eq:init-no-jump-response}
\end{align}
the last equality using $q'(0)=0$, $\E[{\sigma^*}'(w^*,\eps)]=0$, and
independence of $w^0,w^*,w'$. Next, on the event $z^t=1$, letting $r \in (0,t)$ be
the jump time, as $t \downarrow 0$ we have by \eqref{eq:def-cont-xi} and
\eqref{eq:def-cont-deriv-xi-star} that $R_\theta^{t,r} \to
R_\theta^{0,0}=1$, $\xi^{r-}=w^r \to w^0$,
$r_f^{r-,*}=\D_{w^*} f(\xi^{r-},w^*,\eps) \to \D_{w^*} f(w^0,w^*,\eps)$, and hence
$\xi^t \to w^0-\frac{\bar\eta}{\bar\kappa}f(w^0,w^*,\eps)$. Substituting into
\eqref{eq:def-cont-deriv-xi-star} and applying again
independence of $\{z^t\}_{t \geq 0}$ and $(\{w^t\}_{t \geq 0},w^*,\eps)$ gives
\begin{align}
(\mathrm{II})
&=\bar\kappa t\,\E\Bigg[
{-}\D_\xi f\left(
w^0-\frac{\bar\eta}{\bar\kappa}f(w^0,w^*,\eps),w^*,\eps\right)
\frac{\bar\eta}{\bar\kappa}\D_{w^*}f(w^0,w^*,\eps)\notag\\
&\hspace{1in}
+\D_{w^*}f\left(w^0-\frac{\bar\eta}{\bar\kappa}f(w^0,w^*,\eps),w^*,\eps\right)\Bigg]+o(t)\notag\\
&=\bar\kappa
t\,\E\Bigg[w^*f\left(w^0-\frac{\bar\eta}{\bar\kappa}f(w^0,w^*,\eps),w^*,\eps\right)\Bigg]+o(t),
\label{eq:init-one-jump-response}
\end{align}
the second equality integrating by parts in $w^* \sim \cN(0,1)$. Finally,
for the response process $R_f^t(x^{[t]})$ with any continuous input $x^{[t]}$, 
we have from \eqref{eq:def-cont-resp-xi} that
$r_f^t(x^{[t]})=0$ when $z^t=0$. When $z^t=1$ with jump time $r \in (0,t)$,
as $t \downarrow 0$ we also have
$r_f^{r-} \equiv 0$ and $x^r\to x^0$. Thus
a similar expansion as \eqref{eq:init-teacher-response-split} gives
\begin{align*}
R_f^t(x^{[t]})
&=\bar\kappa t\,\E\Bigg[
{-}\D_\xi f\left(
w^0-\frac{\bar\eta}{\bar\kappa}f(w^0,w^*,\eps),w^*,\eps\right)
\frac{\bar\eta}{\bar\kappa}\D_\xi f(w^0,w^*,\eps)\Bigg]x^0+o(t).
\end{align*}
Applying this with $x^{[t]}=\theta^{[t]}$ shows
$\E[R_f^t(\theta^{[t]})\theta^*]=o(t)$ since $\E[\theta^0\theta^*]=0$, so
\begin{equation}\label{eq:init-path-response-zero}
\left.\frac{\d}{\d t}
\E[R_f^t(\theta^{[t]})\theta^*]\right|_{t=0}=0.
\end{equation}
It follows from \eqref{eq:init-q-second-reduced},
\eqref{eq:init-no-jump-response}, \eqref{eq:init-one-jump-response},
and \eqref{eq:init-path-response-zero} that
\begin{equation*}
\begin{aligned}
\frac{\d^2}{\d t^2} \E[\theta^t\theta^*]\Big|_{t=0}
={-}\bar\eta\gamma\bar\kappa\,
\E\left[w^*f\left(
w^0-\frac{\bar\eta}{\bar\kappa}f(w^0,w^*,\eps),w^*,\eps\right)\right]
\end{aligned}
\end{equation*}
where we recall
$f(\xi,w^*,\eps)=(\sigma(\xi)-\sigma^*(w^*,\eps))\sigma'(\xi)$.

Suppose
$\GE(\sigma^*)=1$, i.e.\ $\E[w^*\mid\sigma^*(w^*,\eps)] \neq 0$ with positive
probability. Since $\sigma^*(w^*,\eps)$
has finite moment-generating-function
in a neighborhood of 0, polynomials are dense in its $L^2$-space, so this
implies there exists a smallest integer $m_* \geq 1$ for which
\[\E[w^* \sigma^*(w^*,\eps)^{m_*}] \neq 0.\]
Consider first a polynomial activation
$\tilde\sigma(\xi)=\sum_{m=0}^{m_*} a_m \xi^m$, and introduce the homogenization
$\sigma^*(w^*)=b\sigma^*(w^*)$ with $b=1$ to write
\[f(\xi,w^*,\eps)=\left(\sum_{m=0}^{m_*} a_m \xi^m-b\sigma^*(w^*)\right)
\sum_{m=1}^{m_*} a_m m\xi^{m-1}.\]
Since $f(\xi,w^*,\eps)$ has degree $2m_*-1$ in $\xi$, Taylor expansion around
$\xi=w^0$ gives
\begin{equation}\label{eq:polynomialsecondderiv}
\E\left[w^*
f\bigl(w^0-\frac{\bar\eta}{\bar\kappa}f(w^0,w^*,\eps),w^*,\eps\bigr)\right]
=\sum_{r=0}^{2m_*-1}\frac{(-\bar\eta/\bar\kappa)^r}{r!}
\E[w^*f(w^0,w^*,\eps)^r\D_\xi^r f(w^0,w^*,\eps)].
\end{equation}
Now consider \eqref{eq:polynomialsecondderiv}
as a polynomial of $(a_0,\ldots,a_{m_*},b)$:
Each $r^\text{th}$ summand is homogeneous of degree $2(r+1)$ in
$(a_0,\ldots,a_{m_*},b)$, with degree at most $r+1$ in $b$ and at least
$r+1$ in $(a_0,\ldots,a_{m_*})$. The definition of $m_*$ implies that all
terms of degrees $0,1,\ldots,m_*-1$ in $b$ are 0, and thus
all summands vanish for
$r=0,\ldots,m_*-2$. Summands for $r=m_*,\ldots,2m_*-1$ have degree at least
$r+1 \geq m_*+1$ in $(a_0,\ldots,a_{m_*})$, while the summand for $r=m_*-1$
consists of the single non-vanishing term
\[\frac{(-\bar\eta/\bar \kappa)^{m_*-1}}{(m_*-1)!}
\E[w^*({-}b\sigma^*(w^*,\eps))^{m_*}]
\cdot \E[\tilde\sigma'(w^0)^{m_*-1}\tilde\sigma^{(m_*)}(w^0)]\]
which is homogeneous of degree $m_*$ in $(a_0,\ldots,a_{m_*})$, and in
particular has a non-zero coefficient for $a_1^{m_*-1}a_{m_*}$. Thus,
evaluated at $b=1$, \eqref{eq:polynomialsecondderiv} is a non-zero 
polynomial function of $(a_0,\ldots,a_{m_*})$, implying that it is non-zero
for Lebesgue-a.e.\ choices of $(a_0,\ldots,a_{m_*})$ defining
$\tilde\sigma(\cdot)$. There then exists a Lipschitz approximation
$\sigma(\cdot)$ of any such $\tilde\sigma(\cdot)$ (with $\sigma=\tilde\sigma$
on a sufficiently large interval $[-K,K]$) for which 
\eqref{eq:polynomialsecondderiv} is also non-zero when
$f(\xi,w^*,\eps)=(\sigma(\xi)-\sigma^*(w^*,\eps))\sigma'(\xi)$, so
\[\frac{\d^2}{\d t^2} \E[\theta^t\theta^*]|_{t=0} \neq 0.\]
These conditions $\frac{\d}{\d t}\E[\theta^t\theta^*]|_{t=0}=0$ and
$\frac{\d^2}{\d t^2} \E[\theta^t\theta^*]|_{t=0} \neq 0$ imply
\eqref{eq:learnthetastar} for all small $t>0$.
\end{remark}

\subsection{Deterministic integro-differential equations for squared loss and
ridge regularizer}\label{sec:linear}

We next consider a special setting of quadratic optimization, given by a
squared loss, linear activation, and ridge regularizer. In this setting,
the DMFT equations become linear, and the relevant mean and covariance
statistics of
$\{\theta^t,\xi^t,r_\theta^{t,s},r_f^t,r_f^{t,*}\}$ may be calculated
analytically to yield closed integro-differential equations for the two mappings
\[(C_\theta,R_\theta) \mapsto (C_f, R_f, R_f^*,\Gamma),
\qquad (C_f, R_f, R_f^*,\Gamma) \mapsto (C_\theta,R_\theta).\]
This calculation/reduction is analogous to the reduction to the CKCHS equations in spherical $p$-spin models \cite{crisanti1993spherical,cugliandolo1993analytical}.
A corollary is that, in this linear setting, 
$(C_\theta,R_\theta,C_f, R_f, R_f^*,\Gamma)$ depend only on the mean and
covariance of $\{z^t\}_{t \geq 0}$, and hence coincide for SGD and SME. Thus,
the asymptotic overlaps $\frac{1}{d}\btheta^{t\top} \btheta^*$ and training
and generalization errors also coincide for SGD and SME. (More generally, they
coincide in linear regression with any regularizer, as the mapping $(C_\theta,R_\theta) \mapsto (C_f, R_f, R_f^*,\Gamma)$ is always the same for both SGD and SME.)

Specifically, let $k=k^*$ and consider the setting
\[L(\hat y,y) = \frac{1}{2} \norm{\hat y - y}^2, \qquad \sigma(\xi) = \xi,
\qquad f(\xi, w^*, \eps) = \xi - \sigma^*(w^*, \eps).\]
Let us denote
\[\bar \xi^t=\xi^t-\sigma^*(w^*,\eps),
\qquad \bar w^t=w^t-\sigma^*(w^*,\eps),
\qquad \bar z^t=z^t-\bar\kappa\,t.\]
Then \eqref{eq:def-cont-xi}, \eqref{eq:def-cont-deriv-xi-star}, and
\eqref{eq:def-cont-resp-xi} simplify to
\begin{align*}
\bar \xi^t&={-}\int_0^t \frac{\bar\eta^r}{\bar\kappa} R_\theta^{t,r} \bar
\xi^{r-}\d z^r+\bar w^t
={-}\int_0^t {\bar\eta^r} R_\theta^{t,r} \bar \xi^r \d r
\underbrace{-\int_0^t \frac{\bar\eta^r}{\bar\kappa} R_\theta^{t,r}
\bar \xi^{r-}\d \bar z^r}_{:=M^t} +\,\bar w^t,\\
r_f^{t,*}&=-\int_0^t \frac{\bar\eta^r}{\bar\kappa}
R_\theta^{t,r}r_f^{r-,*}\d z^r-\D_{w^*}\sigma^*(w^*,\eps),\\
r_f^t(x^{[t]})&={-} \int_0^t \frac{\bar\eta^r}{\bar\kappa}
R_\theta^{t,r}\big(r_f^{r-}(x^{[r]})+x^r\big)\d z^r.
\end{align*}
Noting that $r_f^{r-}(x^{[r]})$ and $r_f^{r-,*}$ are $\cF_t$-predictable,
taking expectations of the last two equations yields
\begin{align*}
R_f^{t,*}&={-}\int_0^t \bar \eta^r R_\theta^{t,r}
R_f^{r,*} \d r-\E[\D_{w^*} \sigma^*(w^*,\eps)],\\
R_f^t(x^{[t]})&={-} \int_0^t {\bar\eta^r}
R_\theta^{t,r}\big(R_f^{r}(x^{[r]})+x^{r}\big)\d r.
\end{align*}
These are Volterra integral equations of the second kind, and explicitly
solvable: Let $K_\theta^{t,s} \in \R^{k \times k}$
denote the unique resolvent kernel (see e.g.~\cite[Chap 2]{brunner2004collocation}) satisfying
\begin{equation}\label{eq:Volterra-Rt-resolvent}
K_\theta^{t,s} = -{\bar\eta^s} R_\theta^{t,s} - \int_s^t {\bar\eta^r} R_\theta^{t,r} K_\theta^{r,s} \d r
={-}\bar\eta^s R_\theta^{t,s}-\int_s^t \bar\eta^s K_\theta^{t,r}
R_\theta^{r,s}\d r.
\end{equation}
Then the explicit solutions for $R_f,R_f^*,\Gamma$ are
\begin{equation}\label{eq:Rf-linear}
R_f^t(x^{[t]}) = \int_0^t K_\theta^{t,s} x^s \d s,
\quad
R_f^{t,*}=-\bigg(1 + \int_0^t K_\theta^{t,s} \d
s\bigg)\E[\D_{w^*}\sigma^*(w^*,\eps)],
\quad \Gamma^t=\Id_k.
\end{equation}
To derive the form of $C_f$, note that the above Volterra integral equation
for $\bar\xi^t$ may also be solved to yield
\begin{equation}\label{eq:xiexplicit}
\bar\xi^t = M^t + \bar w^t + \int_0^t K_\theta^{t,s} (M^s + \bar w^s) \d s.
\end{equation}
We observe that since $\{\bar z^t\}_{t \geq 0}$ is a mean-zero process
independent of $(\{w^t\}_{t \geq 0},w^*,\eps)$, for any $t,s \geq 0$ this
implies
$\E[M^t \otimes \bar w^s] = \E[\E[M^t\mid \{w^t\}_{t\in [0,T]}, w^*, \eps]
\otimes \bar w^s ] = 0$. Thus
\begin{align*}
\bar C_\xi^{t,s}:=\E[\bar\xi^t \otimes \bar\xi^s]
&=\E\left[\left(M^t+\int_0^t K_\theta^{t,r}M^r \d r\right)
\otimes \left(M^s+\int_0^s K_\theta^{s,r}M^r \d r\right)\right]\\
&\hspace{1in}
+\E\left[\left(\bar w^t+\int_0^t K_\theta^{t,r}\bar w^r \d r\right)
\otimes \left(\bar w^s+\int_0^s K_\theta^{s,r}\bar w^r \d r\right)\right].
\end{align*}
Denote
\begin{equation}\label{eq:barCtheta}
\bar C_\theta^{t,s}=\E[\bar w^t \otimes \bar w^s]=\E[
(w^t-\sigma^*(w^*,\eps)) \otimes (w^s-\sigma^*(w^*,\eps))],
\quad (\{w^t\}_{t \geq 0},w^*) \sim \GP(0,C_\theta).
\end{equation}
Then the contribution from $\{\bar w^t\}$ is
\begin{align*}
&\E\brk{\prn{\bar w^t + \int_0^t K_\theta^{t,r}\bar w^r \d r} \otimes \prn{\bar w^s + \int_0^s K_\theta^{s,r'}\bar w^{r'} \d r'}} \\
&= \bar C_\theta^{t,s} + \int_0^t K_\theta^{t,r} \bar C_\theta^{r,s}\d r +
\int_0^s \bar C_\theta^{t,r'} K_\theta^{s,r'\top} \d r' + \int_0^t \int_0^s
K_\theta^{t,r}\bar C_\theta^{r,r'}K_\theta^{s,r'\top} \d r\,\d r'.
\end{align*}
For the contribution from $\{M^t\}$, observe that
\begin{align}
M^t + \int_0^t K_\theta^{t,r}M^r \d r &=  - \int_0^t
\frac{\bar\eta^s}{\bar\kappa} R_\theta^{t,s} \bar \xi^{s-}\d \bar z^s
- \int_0^t K_\theta^{t,s}\int_0^s \frac{\bar\eta^r}{\bar\kappa}
R_\theta^{s,r} \bar \xi^{r-}\d \bar z^r \d s\notag\\
&=  - \int_0^t  \prn{\frac{\bar\eta^r}{\bar\kappa} R_\theta^{t,r} + \int_r^t \frac{\bar\eta^r}{\bar\kappa} K_\theta^{t,s}   R_\theta^{s,r}\d s} \bar \xi^{r-}
\d \bar z^r
=\frac{1}{\bar\kappa}\int_0^t K_\theta^{t,r} \bar \xi^{r-}
\d \bar z^r.\label{eq:Midentity}
\end{align}
Then by the It\^o isometry,
\begin{align*}
\E\brk{\prn{M^t + \int_0^t K_\theta^{t,r}M^r \d r} \otimes \prn{M^s + \int_0^s K_\theta^{s,r}M^r \d r}}
&=\frac{1}{\bar\kappa}\int_0^{t\wedge s} K_\theta^{t,r} \E[\bar
\xi^r \otimes \bar \xi^r] K_\theta^{s,r\top} \d r.
\end{align*}
Combining these arguments, the covariance of $\{\bar \xi^t\}$ satisfies
\begin{equation}\label{eq:cov-barxi}
\begin{aligned}
\bar C_\xi^{t,s}&= \frac{1}{\bar\kappa}\int_0^{t\wedge s}
K_\theta^{t,r}\bar C_\xi^{r,r}
K_\theta^{s,r\top} \d r\\
&\hspace{0.2in}+\bar C_\theta^{t,s} + \int_0^t K_\theta^{t,r} \bar C_\theta^{r,s}\d r +
\int_0^s \bar C_\theta^{t,r'} K_\theta^{s,r'\top} \d r' + \int_0^t \int_0^s
K_\theta^{t,r}\bar C_\theta^{r,r'}K_\theta^{s,r'\top} \d r\,\d r'.
\end{aligned}
\end{equation}
This may be understood as a closed linear integral equation for the diagonal
$\{\bar C_\xi^{t,t}\}_{t \in [0,T]}$, whose solution may then be
substituted back to determine $\{\bar C_\xi^{t,s}\}_{s \neq t}$.
Finally, by definition
\begin{align*}
C_f^{t,s} &= \E\brk{\prn{\int_0^t {\bar\eta^r} \bar \xi^{r} \d r + \int_0^t
\frac{\bar\eta^r}{\bar\kappa} \bar\xi^{r-} \d \bar z^r}
\otimes \prn{\int_0^s {\bar\eta^r} \bar \xi^r \d r + \int_0^s
\frac{\bar\eta^r}{\bar\kappa} \bar\xi^{r-} \d \bar z^r}}\nonumber\\
&=(\mathrm{I}) + (\mathrm{II}) + (\mathrm{III}) + (\mathrm{IV}),
\end{align*}
where the terms are given by
\begin{align*}
(\mathrm{I}) & = 
\int_0^t \int_0^s \bar\eta^r\bar\eta^{r'}\E[\bar\xi^r \otimes \bar\xi^{r'}]\d r
\d r'=\int_0^t \int_0^s \bar\eta^r\bar\eta^{r'}\bar C_\xi^{r,r'}\d r
\d r',\\
(\mathrm{II}) & = \E\brk{\int_0^t \frac{\bar\eta^r}{\bar\kappa} \bar\xi^{r-} \d \bar
z^r \otimes \int_0^s \frac{\bar\eta^{r'}}{\bar\kappa} \bar\xi^{r'-} \d \bar
z^{r'}} = \int_0^{t\wedge s} \frac{(\bar\eta^{r})^2}{\bar\kappa}\bar
C_\xi^{r,r}\d r,\\
(\mathrm{III}) &= \int_0^t {\bar\eta^r}\E\brk{\bar \xi^r  \otimes \int_0^s
\frac{\bar\eta^{r'}}{\bar\kappa} \bar \xi^{r'-} \d \bar z^{r'}}\d r\\
&\overset{(*)}{=}\int_0^t {\bar\eta^r}\E\brk{  \prn{M^r+\int_0^r K_\theta^{r,r'}
M^{r'} \d r'}  \otimes \int_0^s \frac{\bar\eta^{r'}}{\bar\kappa} \bar \xi^{r'-} \d \bar z^{r'}}\d r\\
&\overset{(**)}{=}\int_0^t \frac{\bar\eta^r}{\bar\kappa}\E\brk{\int_0^r K_\theta^{r,r'}\bar
\xi^{r'-} \d \bar z^{r'}  \otimes \int_0^s \frac{\bar\eta^{r'}}{\bar\kappa}
\bar\xi^{r'-} \d \bar z^{r'}}\d r = \int_0^t \int_0^{r\wedge s}
\frac{\bar\eta^r\bar\eta^{r'}}{\bar\kappa}
K_\theta^{r,r'}\bar C_\xi^{r',r'} \d {r'}\d r,\\
(\mathrm{IV})&=\int_0^s \bar\eta^{r'} \E\brk{\int_0^t
\frac{\bar\eta^r}{\bar\kappa} \bar\xi^{r-}\d \bar z^r
\otimes \bar \xi^{r'}} \d r'
=\int_0^s \int_0^{t \wedge r} 
\frac{\bar\eta^r\bar \eta^{r'}}{\bar\kappa} 
\bar C_\xi^{r',r'}K_\theta^{r,r'\top}\d r'\d r.
\end{align*}
In $(*)$, we have applied the form of $\{x^t\}_{t \geq 0}$ from
\eqref{eq:xiexplicit}, and the independence of $(\{\bar w^t\}_{t \geq
0},w^*,\eps)$ and $\{\bar z^t\}_{t \geq 0}$ as above. In $(**)$, we have applied
again the identity \eqref{eq:Midentity}, and the form for $(\mathrm{IV})$
follows from symmetric arguments as $(\mathrm{III})$. Thus, to summarize
\begin{equation}\label{eq:Cf-linear}
\begin{aligned}
C_f^{t,s}&=\int_0^t \int_0^s \bar\eta^r\bar\eta^{r'}\bar C_\xi^{r,r'}\d r \d r'
+\int_0^{t\wedge s} \frac{(\bar\eta^{r})^2}{\bar\kappa}\bar C_\xi^{r,r}\d r\\
&\hspace{0.2in}
+\int_0^t \int_0^{r\wedge s} \frac{\bar\eta^r\bar\eta^{r'}}{\bar\kappa}
K_\theta^{r,r'}\bar C_\xi^{r',r'} \d {r'}\d r
+\int_0^s \int_0^{t \wedge r} \frac{\bar\eta^r\bar \eta^{r'}}{\bar\kappa} 
\bar C_\xi^{r',r'}K_\theta^{r,r'\top}\d r'\d r.
\end{aligned}
\end{equation}
The mapping $(C_\theta,R_\theta) \mapsto (C_f, R_f, R_f^*, \Gamma)$ is thus
given by \eqref{eq:Volterra-Rt-resolvent},
\eqref{eq:Rf-linear}, \eqref{eq:barCtheta}, \eqref{eq:cov-barxi},
and \eqref{eq:Cf-linear}, and this mapping is the same for both SGD and SME.

Consider next the setting of a ridge regularizer,
\[G(\theta)=\frac{\lambda}{2}\|\theta\|^2,
\qquad g(\theta)=\lambda \theta,\]
and suppose that $R_f$ takes the form
$R_f^t(\theta^{[t]})=\int_0^t R_f^{t,s} \theta^s ds$.
Then \eqref{eq:def-cont-theta} and \eqref{eq:def-cont-deriv-t} also simplify
to linear Volterra integral equations,
\[\theta^t=\int_0^t A_f^{t,r} \theta^r \d r + b_f^t,
\qquad r_\theta^{t,s}=\int_s^t A_f^{t,r} r_\theta^{r,s} \d r + \Id_k,\]
where
\begin{align*}
        A_f^{t,r} &= -\bar{\eta}^r \left({\gamma} \Gamma^r + \lambda\,\Id_k\right) - \gamma \int_r^t {\bar{\eta}^s} R_f^{r,s} \d s, \qquad
    b_f^t = \theta^0 - \gamma \left(\int_0^t
{\bar{\eta}^r}R_f^{r,*} \d r\right)  \theta^* + \sqrt{\gamma}\,u^t.
\end{align*}
Let $K_f^{t,s} \in \R^{k \times k}$ be the unique resolvent kernel satisfying
\begin{equation}\label{eq:Volterra-Rf-resolvent}
K_f^{t,s}=A_f^{t,s} + \int_s^t A_f^{t,r} K_f^{r,s} \d r
=A_f^{t,s}+\int_s^t K_f^{t,r}A_f^{r,s}\d r.
\end{equation}
Then
\begin{equation}\label{eq:Rt-ridge}
R_\theta^{t,s}=\E r_\theta^{t,s}=\left(1+\int_s^t K_f^{t,r}\d r \right) \Id_k,
\end{equation} 
and
\[\theta^t = b_f^t + \int_0^t K_f^{t,s} b_f^s \,\d s\]
which implies that $C_\theta^{t,s}=\E[\theta^t \otimes \theta^s]$ is given by
\begin{equation}\label{eq:Ct-ridge}
\begin{aligned}
C_\theta^{t,s}&=C_b^{t,s} 
+ \int_0^t K_f^{t,r} C_b^{r,s} \,\mathrm{d}r 
+ \int_0^s C_b^{t,r} (K_f^{s,r})^\top \,\mathrm{d}r 
+ \int_0^t \int_0^s K_f^{t,r} C_b^{r,r'} (K_f^{s,r'})^\top \,\mathrm{d}r'
\,\mathrm{d}r,\\
C_\theta^{t,*}&=C_b^{t,*}+\int_0^t K_f^{t,s}
C_b^{s,*} \d s,\\
C_b^{t,s}&=\E[b_f^t\otimes b_f^s]=\E\left[\left(\theta^0 
- \gamma \left(\int_0^t
{\bar{\eta}^r}R_f^{r,*} \d r\right)\theta^*\right)\otimes
\left(\theta^0 - \gamma \left(\int_0^s
{\bar{\eta}^r}R_f^{r,*} \d r\right)\theta^*\right)\right] + \gamma\,C_f^{t,s},\\
C_b^{t,*}&=\E[b_f^t\otimes \theta^*]=\E\left[\left(\theta^0 
- \gamma \left(\int_0^t
{\bar{\eta}^r}R_f^{r,*} \d r\right)\theta^*\right) \otimes \theta^* \right].
\end{aligned}
\end{equation}
The mapping $(C_f, R_f, R^*_f, \Gamma) \mapsto (C_\theta, R_\theta)$ is given
by \eqref{eq:Volterra-Rf-resolvent}, \eqref{eq:Rt-ridge}, and
\eqref{eq:Ct-ridge}.

\subsection{Numerical simulation of the DMFT system}

\begin{figure}[t]
    \centering
    \includegraphics[width=0.75\linewidth]{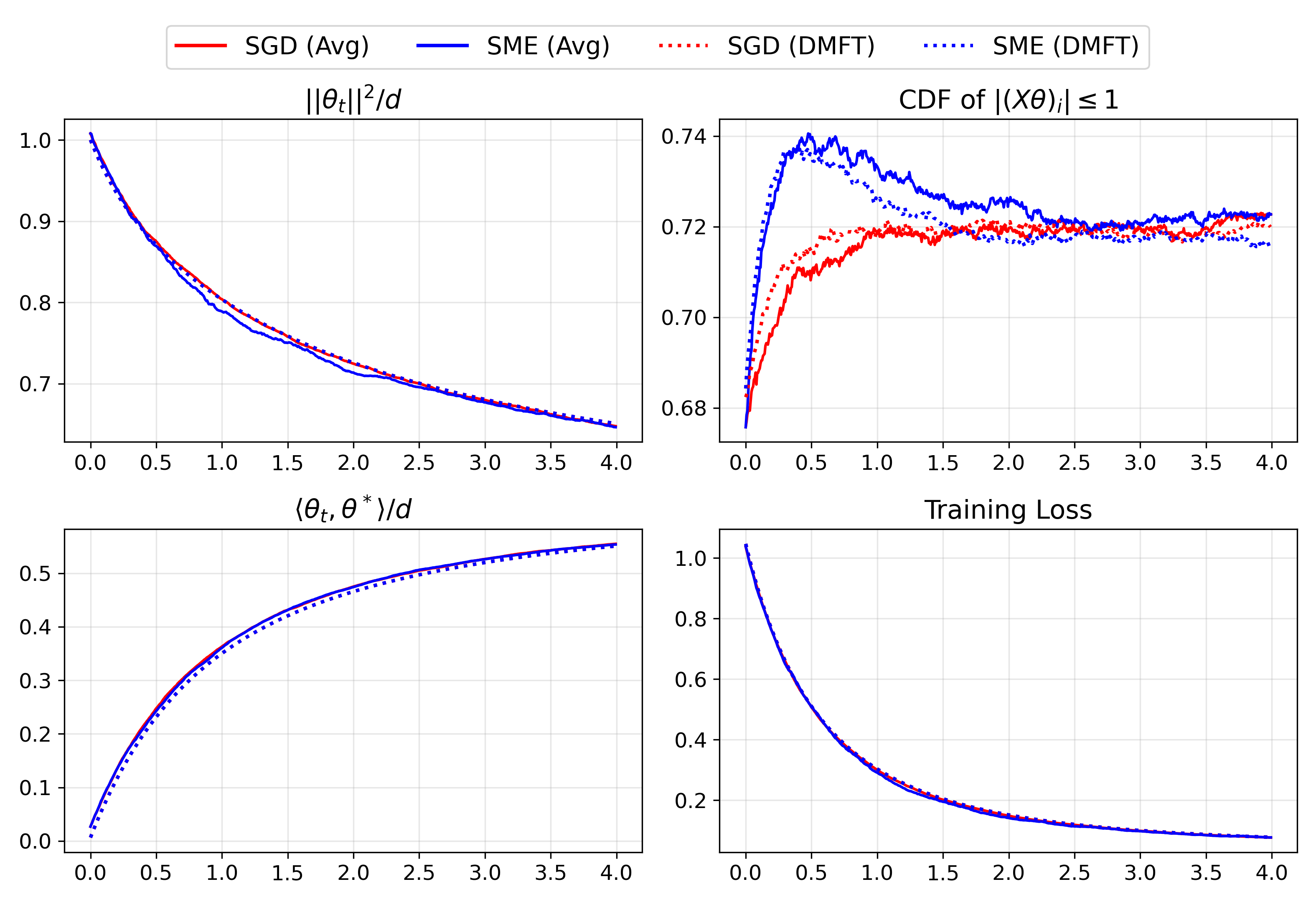}
    \caption{Comparison of SGD and SME dynamics in a linear model.
Parameters: $n=8000$, $d=10000$, $\eta=0.8$, $\lambda=0.1$. The dynamics of the
squared norm $\|\btheta_t\|^2/d$, overlap $\langle \btheta_t, \btheta^*
\rangle/d$, and training loss agree closely for SGD and SME, whereas those of
the empirical CDF $n^{-1}\sum_{i=1}^n \1\{|\x_i^\top \btheta| \le 1\}$ show a
discrepancy.}\label{fig:linear}
    \end{figure}

    \begin{figure}[t]
    \centering
    \includegraphics[width=0.75\linewidth]{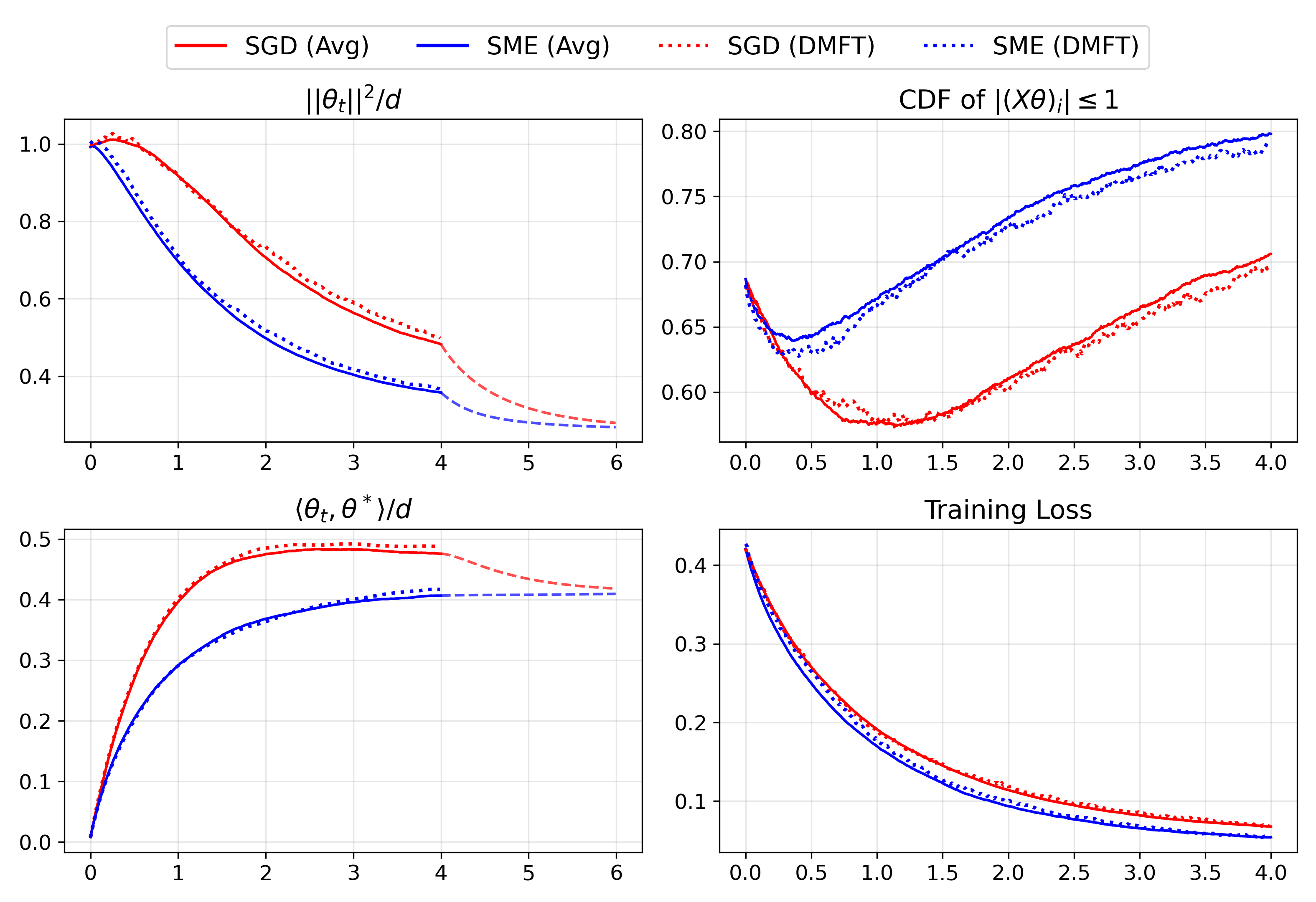}
    \caption{Divergence of SGD and SME dynamics in a model with Huber loss and
$\tanh$ activation.
Parameters: $n=8000$, $d=10000$, $\eta=3.0$, $\lambda=0.1$. Solid/dotted curves
depict the dynamics of SGD/SME and their DMFT predictions over 4
training epochs, exhibiting a divergence that is predicted by DMFT.
Dashed curves for training times $t>4$ depict the subsequent dynamics of
gradient flow initialized from the last SGD/SME iterate, in which the norm
and overlap statistics re-converge.}
    \label{fig:tanh}
\end{figure}


\begin{figure}[t]
    \centering
    \includegraphics[width=0.75\linewidth]{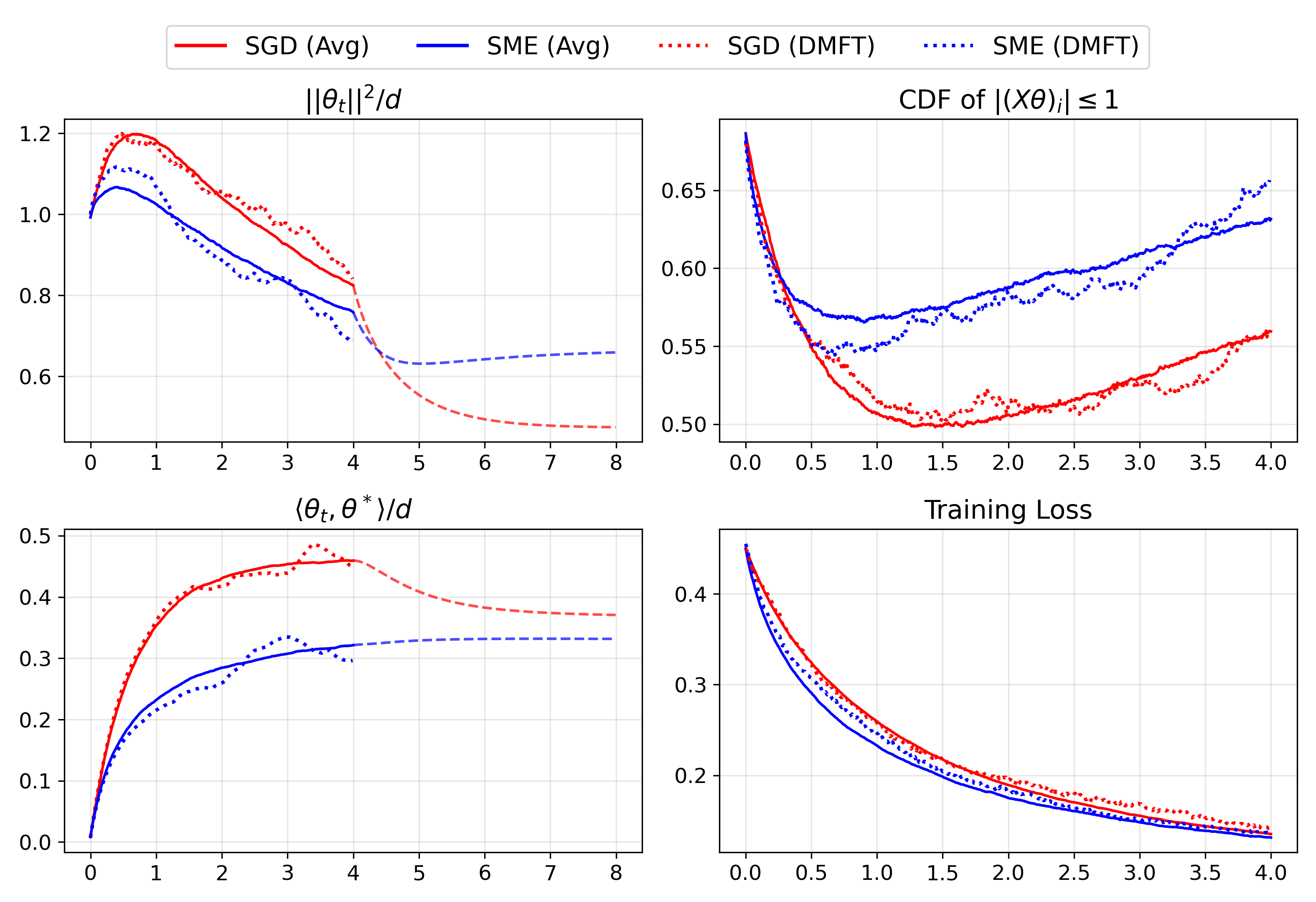}
    \caption{Divergence of SGD and SME dynamics in a model with Huber loss and
$\sin$ activation. Parameters are the same as Figure \ref{fig:tanh}. In this
example, the difference between SGD and SME persists through the later gradient
flow training, suggesting that the dynamics reach local optima
with distinct statistical properties.}
    \label{fig:sine}
    \end{figure}

We simulate the solution to the DMFT
system (\ref{eq:thetaeps}--\ref{eq:def-cont-Gamma}) by iterating the fixed-point
mappings $(C_\theta,R_\theta) \mapsto (C_f,R_f,R^*_f,\Gamma)$
and $(C_f,R_f,R^*_f,\Gamma) \mapsto (C_\theta,R_\theta)$ until convergence, and
then simulating the processes $\{\theta^t,\xi^t\}_{t \geq 0}$ from the
parameters of the fixed point. More concretely:
\begin{itemize}
\item In the setting of a ridge regularizer, the mapping
$(C_f,R_f,R^*_f,\Gamma) \mapsto (C_\theta,R_\theta)$ may be evaluated
analytically via \eqref{eq:Volterra-Rf-resolvent}, \eqref{eq:Rt-ridge}, and
\eqref{eq:Ct-ridge}.
\item In the setting of squared loss and linear activation, the mapping
$(C_\theta,R_\theta) \mapsto (C_f,R_f,R^*_f,\Gamma)$ may likewise be evaluated
analytically via \eqref{eq:Volterra-Rt-resolvent},
\eqref{eq:Rf-linear}, \eqref{eq:barCtheta}, \eqref{eq:cov-barxi},
and \eqref{eq:Cf-linear}.
\item Outside of these settings, the expectations
(\ref{eq:def-cont-C-t}--\ref{eq:def-cont-Gamma}) may be approximated
via Monte Carlo simulation of the processes
(\ref{eq:thetaeps}--\ref{eq:def-cont-resp-xi}).
\end{itemize}
All time integrals are discretized in increments of $\delta$, as described in
Appendix \ref{sec:discreteDMFT}. We remark that simulation of the processes
$\{\xi^t,r_f^{t,*},r_f^{t,s}\}$ for SGD is substantially faster
than for gradient flow or SME, as integrals against the Poisson
process $\{z^t\}_{t \geq 0}$ are rapidly computable due to the sparsity of its
jumps, and we provide an efficient software implementation with the
publication of this manuscript.

We close this section by presenting some numerical results that corroborate our
preceding discussion/findings. Unless otherwise stated, all simulations use a single-index model $k=k^*=1$, single-sample batch size $\kappa=1$,
time discretization $\delta=0.05$,
and ridge regularizer $G(\theta)=\frac{\lambda}{2}\theta^2$ with
analytical computation of the DMFT mapping
$(C_f,R_f,R^*_f,\Gamma) \mapsto (C_\theta,R_\theta)$. The reverse DMFT mapping $(C_f,R_f,R^*_f,\Gamma) \mapsto (C_\theta,R_\theta)$ is
computed analytically for the linear model, and via 10,000 Monte Carlo samples
otherwise. DMFT predictions for $d^{-1}\|\btheta^t\|^2$ and $d^{-1}\btheta^{t\top}\btheta^*$ are contained in $C_\theta$, and the predictions for other observables are also computed via 10,000 Monte Carlo samples. Statistics of the high-dimensional SGD and SME processes
$\{\btheta^t\}_{t \geq 0}$ are averaged over 10 independent trials.\\

{\it Coincidence of quadratic observables of SGD and SME in a linear model.} For a linear model $y=\x^\top \btheta^*$ with loss $L(\hat
y,y)=\frac{1}{2}(\hat y-y)^2$ and activation $\sigma(\xi)=\xi$,
Figure~\ref{fig:linear} verifies the claim of Section \ref{sec:linear}
that the dynamics of the squared norm
$d^{-1}\|\btheta^t\|^2$, overlap $d^{-1}\btheta^{t\top}\btheta^*$,
and training loss $(2n)^{-1}\sum_{i=1}^n (\x_i^\top \btheta^t-\x_i^\top
\btheta^*)^2$ coincide for SGD and SME. The empirical distribution of $\{\x_i^\top
\btheta^t\}_{i \in [n]}$ is predicted (via the law of $\xi^t$) to be different
for SGD and SME, and we indeed observe a discrepancy at the tested
learning rate that is also predicted by the DMFT theory.\\

{\it Difference between SGD and SME dynamics in non-convex landscapes.} Figures \ref{fig:tanh} 
and \ref{fig:sine} display two different non-linear
and non-convex settings, in which the dynamics of SGD and SME
with relatively large learning rate $\eta=3.0$ are markedly different:

\begin{enumerate}
    \item Figure \ref{fig:tanh} uses a Huber loss
$L(\hat y,y)=\begin{cases} \frac{1}{2}(\hat y - y)^2 & \text{if } |\hat y - y| < 1 \\ (|\hat y - y| - \frac{1}{2}) & \text{otherwise} \end{cases}$ and activation $\sigma(x)=\tanh(x)$, with noisy true labels $y=\tanh(\x^\top \btheta^*+\eps)$ where $\eps\sim\mathcal{N}(0,0.1)$.
\item Figure \ref{fig:sine} uses the same Huber loss, activation
$\sigma(x)=\sin(x)$, and noisy true labels $y=\sin(\x^\top \btheta^*)+\eps$ where $\eps \sim\mathcal{N}(0,0.1)$.
\end{enumerate}

We highlight that the mechanism driving the difference between SGD and SME
may be different in
these two examples: Both examples train SGD/SME for $T=4$ epochs, and
then continue to apply gradient flow from the last SGD/SME iterate until
convergence. The dynamics over the first $T=4$ epochs in Figure \ref{fig:tanh}
diverge between SGD and SME, but converge again upon
the subsequent application of gradient flow. This suggests that SGD and SME
reach basins of attraction of (local) minimizers of the empirical risk with
similar statistical properties, and that their differences may be driven by
different oscillatory behaviors around these minimizers.
In contrast, the difference between SGD and SME in
Figure \ref{fig:sine} persists even through the later gradient flow training,
with SGD achieving a higher overlap $d^{-1}\btheta^{t\top}\btheta^*$
and smaller squared norm $d^{-1}\|\btheta^t\|^2$. This suggests that SGD and
SME reach local minimizers of the empirical risk that have different
generalization properties, and such a difference is predicted
by the DMFT theory.\\

{\it The small learning rate limit.} Figure \ref{fig:comp-lr} compares the overlap dynamics of SGD and SME
in units of rescaled time $\tau=\eta t$ across various learning rates $\eta$
for the Tanh activation setting of Figure \ref{fig:tanh}, verifying the claim
in Section \ref{sec:comp-gf} that these dynamics converge to a common
(gradient flow) limit as $\eta \to 0$. The figure additionally includes a
fixed batch-size comparison with $\kappa=n^{1/3}$.\\


\begin{figure}[t]
    \centering
    \includegraphics[width=1\linewidth]{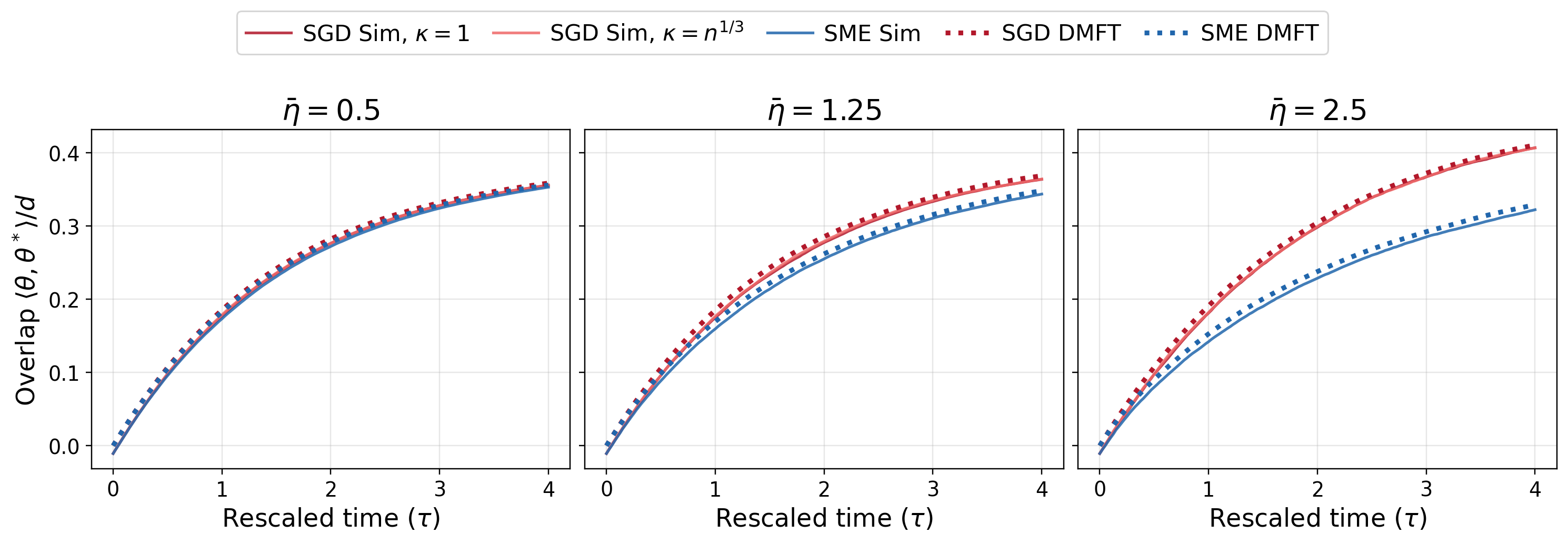}
    \caption{Effect of learning rate on SGD vs.\ SME divergence.
Overlap dynamics are plotted in units of rescaled time $\tau=\eta t$, for
various learning rates $\eta \in \{0.5, 1.25, 2.5\}$. In addition to the
single-sample case, a fixed batch size $\kappa=n^{1/3}$ is included for comparison; the strong alignment between these two simulations provides compelling support for our theory.
At small learning rates ($\eta=0.5$), the SME diffusion provides an accurate
approximation of SGD, and their gap widens as $\eta$ increases.}
    \label{fig:comp-lr}
\end{figure}



\begin{figure}[t]
    \centering
    \includegraphics[width=0.65\linewidth]{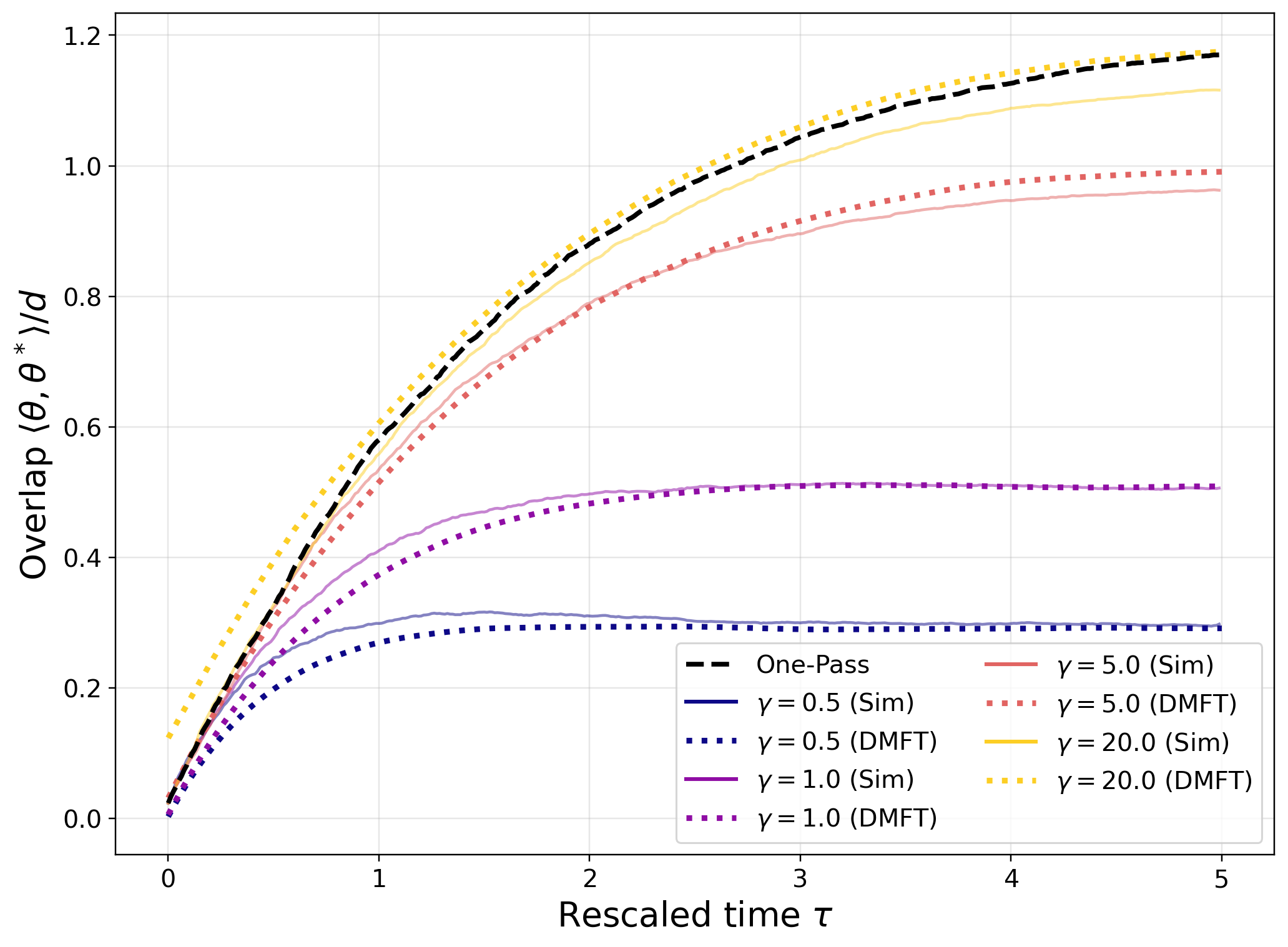}
    \caption{Convergence to one-pass SGD in the large sample limit.
Overlap dynamics are plotted in units of rescaled time $\tau=\gamma t$, fixing
$d=4000$ and varying $\gamma=n/d \in \{0.5, 1.0, 5.0, 20.0\}$.
As the training sample size $\gamma$ increases, multi-pass SGD dynamics (solid
lines) and their DMFT predictions (dotted lines) both converge to the simulated
overlap statistics of one-pass SGD (black dashed line).}
    \label{fig:comp-onepass}
\end{figure}

{\it Transition to one-pass SGD dynamics.} Figure \ref{fig:comp-onepass} compares the overlap dynamics of SGD in units of
rescaled time $\tau=\gamma t$ for various dimension ratios $\gamma=\lim_{n,d \to
\infty} n/d$, verifying the claim in Section \ref{sec:comp-onepass}
that as $\gamma \to \infty$, these overlap curves converge to a limit that
represents one-pass SGD.

\section{Conclusion}

In this work, we have established a new high-dimensional, continuous-time scaling limit for SGD in the multi-index model, in the form of a pair of dimension-free dynamical mean-field equations that characterize the evolving empirical distributions of coordinates of the SGD iterates. This limit implies deterministic trajectories for both training and generalization error in the asymptotic limit $n,d \to \infty$, and serves to clarify the distinctions between multi-pass and online SGD and also between SGD, gradient flow, and stochastic diffusion. The scaling limits of the latter three optimization methods are distinct, with the implication that these methods may converge to different (deterministically predicted) local minimizers of the empirical risk in non-convex settings.

We believe that our characterization of this scaling limit may open a new path to understanding a number of questions that at present remain poorly understood, even in this relatively simple multi-index model:
\begin{enumerate}
\item In regimes of high overparametrization (i.e.\ large but fixed $m \gg n/d$), will SGD typically converge to 0 training loss, and what is the mechanism for this convergence?
\item Over a large but fixed number of training epochs (i.e.\ as $T \to \infty$ in the dynamical mean-field limit we establish), what are the training and generalization errors achievable by SGD, and how do these compare with optimal optimization/estimation procedures?
\item In non-convex settings, how can differences in generalization error between SGD, SME, and gradient flow be more systematically understood?
\end{enumerate}
These questions pertain to the long-time behavior of the scaling limit we establish, and we leave a detailed investigation of this behavior as an open question for future work.

\appendix

\section{Existence and uniqueness of the DMFT fixed
point}\label{sec:existenceuniqueness}

This section defines the spaces $\cS \equiv \cS(T,C_0)$ and $\cS^\text{cont}
\equiv \cS^\text{cont}(T,C_0)$, and proves Theorem \ref{thm:uniqueness-existence}. The
proof is an adaptation of the method introduced in \cite{celentano2021high}
for gradient flow (following more closely its implementation in
\cite{fan2025dynamical1} for Langevin dynamics) and will consist of:
\begin{enumerate}
    \item Defining spaces $\mathcal{S}_\xi$
    for $(\Cf,\Rf,\Rf^*,\Gamma)$ and ${\mathcal{S}}_\theta$ for $(\Ct,\Rt)$,
and showing that on these spaces, the primary processes
$\theta^t,\xi^t$ and auxiliary response processes $r_\theta^{t,s},r_f^{t,*},r_f^t(x^{[t]})$ are well-defined.
    \item Showing that the transformations which represent the fixed point
equations for $(\Ct,\Rt,\Cf,\Rf,\Rf^*,\Gamma)$ define mappings $\cS_\xi
\to \cS_\theta$ and $\cS_\theta \to \cS_\xi$, and are contractive in a suitable metric.
\end{enumerate}

The main differences between our argument and those of past works involve the
treatment of the stochastic Volterra-type integrals in \eqref{eq:def-cont-xi},
\eqref{eq:def-cont-deriv-xi-star}, and \eqref{eq:def-cont-resp-xi} that arise
due to the stochastic gradient noise in both SGD and SME; this necessitates 
modified definitions and treatments of the response processes, as discussed in
Remark \ref{rmk:response}.

\subsection{Definition of the space \texorpdfstring{$\cS$}{S}}

We remind the reader that we abbreviate $\norm{\cdot}$ for the $\ell_2$-norm of
a vector and Frobenius norm of a matrix, and write $\normop{\cdot}$ for the
$\ell_2$-to-$\ell_2$ matrix operator norm.

For any Euclidean space $E$, we denote by
$L^4([0,t],E)$ the space of (Borel-measurable) functions
$s \in [0,t] \mapsto x^s \in E$ for which
$\|x\|_{L^4}:=(\int_0^t \|x^s\|_2^4 \d s)^{1/4}<\infty$.
For a linear operator $R:L^4([0,t],\R^k) \to \R^k$, we define also
its operator norm
\[\normop{R}=\sup_{x \in L^4([0,t],\R^k):\|x\|_{L^4} \leq 1} \norm{R(x)}_2.\]
We say that $R$ is a bounded linear operator if $\normop{R}<\infty$.

Fix any $T>0$. Under Assumptions \ref{asp:SGD} and \ref{asp:lipschitz}, there
exist constants $M_f,M_g,M_\eta>0$ such that
\begin{equation}\label{eq:defMbounds}
\begin{gathered}
\sup_{(\xi,w^*,\eps) \in \R^k \times \R^{k^*} \times \R}
\{\|\D_\xi f(\xi,w^*,\eps)\|_\op,
\|\D_{w^*} f(\xi,w^*,\eps)\|_\op\} \leq M_f,\\
\sup_{\theta \in \R^k} \|\D g(\theta)\|_\op \leq M_g,
\qquad \sup_{t \in [0,T]} \bar \eta^t \cdot \max\{1,1/\sqrt{\bar\kappa}\}
\leq M_\eta.
\end{gathered}
\end{equation}
For a constant $C_0>0$ (which we will take large enough, depending on $T$ and
the above constants $M_f,M_g,M_\eta$),
we define envelope functions $\pRt,\pRf,\pRfs:[0,T] \to [0,\infty)$ by
\begin{align}
\frac{\d}{\d t}\pRt(t)&=C_0\prn{\pRt(t)+\int_0^t \pRf(t-r) \pRt(r) \d r},
\quad \text{ with } \pRt(0)=C_0,\label{eq:def-pRt}\\
\pRf(t) &= C_0\prn{\int_0^t \pRt(t-r) \pRf(r) \d r +
\pRt(t)},\label{eq:def-pRf}\\
\pRfs(t)&=C_0\left(1+\int_0^t \pRt(t-r)\pRfs(r)\d r\right).\label{eq:def-pRfs}
\end{align}
We then define two envelope functions $\pCt,\pCf:[0,T] \to [0,\infty)$ by
\begin{align}
\frac{\d}{\d t}\pCt(t)&=C_0\left(6\gamma \pCf(t)+\pCt(t)
+\int_0^t \Phi_{R_f}(t-r)\pCt(r)\d r+\Phi_{R_f^*}(t)\right),\notag\\
&\hspace{2in} \text{ with } \pCt(0)=C_0,
\label{eq:def-pCt}\\
\pCf(t) &= C_0 \left(1 + \Phi_{C_\theta}(t) + \int_0^t \Phi_{R_\theta}(t-r)\pCf(r) \d r\right).
\label{eq:def-pCf}
\end{align}
One may check using a Laplace transform argument that
these systems admit a unique solution $\pRt,\pRf,\pRfs,\pCt,\pCf:[0,T] \to
[0,\infty)$, where all four envelope functions are nonnegative and increasing;
we refer to \cite[Lemma 3.1]{fan2025dynamical1} for details of such an argument.

\begin{defi}\label{def:S}
Fix $T>0$ and a constant $C_0>0$ (which also defines $\pRt,\pRf,\pRfs,\pCt,\pCf$).
Given a finite subset $D=\{d_1,\ldots,d_m\} \subset [0,T]$, we call
\[[0,d_1),[d_1,d_2),\ldots,[d_{m-1},d_m),[d_m,T]\]
the maximal intervals of $[0,T] \setminus D$.

Let $\mathcal{S}_\xi \equiv \mathcal{S}_\xi(T,C_0)$ be the space of all
tuples $(C_f,R_f,R_f^*,\Gamma)$ such that there exists a finite set
$D \subset [0,T]$ for which:
    \begin{enumerate}
        \item $\Cf \equiv \{\Cf^{t,s}\}_{t,s \in [0,T]}$ is a symmetric covariance kernel
on $[0,T] \otimes \R^k$, i.e.\ $\Cf^{t,s}=(\Cf^{s,t})^\top
\in \R^{k \times k}$ for all $t,s \in [0,T]$,
and for any finite subset $T_0 \subset [0,T]$,
the matrix $(\Cf^{t,s})_{t,s \in T_0} \in \R^{k|T_0| \times k|T_0|}$ is
symmetric positive-semidefinite. Furthermore, we have
\[\Cf^{0,0}=0, \qquad
\Tr \Cf^{t,t} \leq C_0 \int_0^t \pCf(s)\d s \text{ for all } 0\leq t\leq T,\]
and for each maximal interval $I$ of $[0,T]\setminus D$,
        \begin{equation}\label{eq:Cfcontinuity}
\Tr[\Cf^{t,t}-\Cf^{t,s}-\Cf^{s,t}+\Cf^{s,s}] \leq C_0 \pCf(T)|t-s|
\text{ for all } t,s \in I.
        \end{equation}
        \item $\Rf \equiv \{\Rf^t\}_{t \in [0,T]}$ is a family of bounded linear operators
$\Rf^t:L^4([0,t],\R^k) \to \R^k$.
Fixing any process $x \in L^4([0,T],\R^k)$, the map
$t \mapsto R_f^t(x^{[t]})$ is Borel-measurable, and
        \begin{equation}\label{eq:pRf-volterra}
\norm{\Rf^{t}(x^{[t]})}^2 \leq \int_0^t \pRf(t-r) \norm{x^r}^2\d r \text{ for all }
t\in [0,T].
\end{equation}
\item $\Rf^* \equiv \{\Rf^{t,*}\}_{t \in [0,T]}$ is a Borel-measurable
process of matrices in $\R^{k \times k^*}$, where
\[\|\Rf^{t,*}\|^2 \leq \pRfs(t) \text{ for all } t \in [0,T].\]
\item $\Gamma \equiv \{\Gamma^t\}_{t \in [0,T]}$ is a Borel-measurable process
of matrices in $\R^{k \times k}$, where
\[\normop{\Gamma_t} \leq M_f \text{ for all } t \in [0,T].\]
    \end{enumerate}

Let ${\mathcal{S}}_\theta \equiv {\mathcal{S}}_\theta(T,C_0)$ be the space of
all tuples $(C_\theta,R_\theta)$ such that there exists a finite set $D \subset [0,T]$
for which:
    \begin{enumerate}
        \item $\Ct \equiv \{\Ct^{t,s}\}_{t,s \in [0,T] \cup \{*\}}$ is a symmetric
covariance kernel on $([0,T] \otimes \R^k) \times \R^{k^*}$, i.e.\
$\Ct^{t,s}=(\Ct^{s,t})^\top \in \R^{k \times k}$,
$\Ct^{t,*}=(\Ct^{*,t})^\top \in \R^{k \times k^*}$, and
$\Ct^{*,*} \in \R^{k^* \times k^*}$ for all $t,s \in [0,T]$, and
for any finite subset $T_0 \subset [0,T]$,
the matrix $(\Ct^{t,s})_{t,s \in T_0 \cup \{*\}} \in \R^{(k|T_0|+k^*) \times
(k|T_0|+k^*)}$ is symmetric positive-semidefinite. Furthermore,
we have
\[\Ct^{*,*}=\E[\theta^*\theta^{*\top}],
\quad \Ct^{0,0}=\E[\theta^0\theta^{0\top}],
\quad \Ct^{0,*}=\E[\theta^0\theta^{*\top}],\]
\[\Tr \Ct^{t,t} \leq \pCt(t) \text{ for all } 0\leq t\leq T,\]
and for each maximal interval $I$ of $[0,T]\setminus D$,
        \begin{equation}\label{eq:Cthetacontinuity}
            \Tr[\Ct^{t,t}-\Ct^{t,s}-\Ct^{s,t}+\Ct^{s,s}] \leq 
(\Phi_{C_\theta}(T)+C_0^2\Phi_{C_f}(T))|t-s|
\text{ for all } t,s \in I.
        \end{equation}
        \item $\Rt \equiv \{\Rt^{t,s}\}_{t,s \in [0,T]}$ is a Borel-measurable
process of matrices in $\R^{k \times k}$, satisfying
\[\Rt^{t,s}=0 \text{ for all $s>t$},
\qquad
\|\Rt^{t,s}\|^2 \leq \pRt(t-s) \text{ for all } s \leq t,\]
and for each maximal interval $I$ of $[0,T] \setminus D$,
\[\|\Rt^{t,s}-\Rt^{t',s}\|^2 \leq \Phi_{R_\theta}(T)|t-t'|
\text{ for all } s \in [0,T] \text{ and }
t,t' \in I \text{ with } t,t' \geq s.\]
\end{enumerate}
We set
\[\cS \equiv \cS(T,C_0)=\{(C_\theta,R_\theta,C_f,R_f,R_f^*,\Gamma):
(C_\theta,R_\theta) \in \cS_\theta,\,(C_f,R_f,R_f^*,\Gamma) \in \cS_\xi\}.\]
We denote by
$\cS^{\mathrm{cont}},\cS_\xi^{\mathrm{cont}},\cS_\theta^{\mathrm{cont}}$
these spaces where the above conditions hold with $D=\emptyset$,
i.e.\ where the continuity conditions for
$C_f,C_\theta,R_\theta$ hold on the whole interval $[0,T]$. 
\end{defi}

Note that the continuity conditions \eqref{eq:Cfcontinuity} and
\eqref{eq:Cthetacontinuity} for $C_f$ and $C_\theta$ 
ensure (via Kolmogorov's continuity theorem) 
the existence of Gaussian processes $\{u^t\}_{t \in [0,T]} \sim \GP(0,C_f)$ 
and $(\{w^t\}_{t \in [0,T]},w^*) \sim \GP(0,C_\theta)$ with c\`adl\`ag sample
paths (having discontinuities only in $D$),
as claimed in the constructions of Section \ref{sec:DMFTlimits}.

\subsection{Existence and uniqueness of stochastic processes}
We first establish Theorem \ref{thm:uniqueness-existence}(a), showing
the existence and uniqueness of solutions to the
stochastic integro-differential equations defined in
(\ref{eq:def-cont-theta}--\ref{eq:def-cont-resp-xi}), given any
$(C_\theta,R_\theta,C_f,R_f,R_f^*,\Gamma) \in \cS$. (Here
and throughout, a solution $\{Y^t\}_{t \in [0,T]}$ refers to a process that
satisfies the desired equation a.s.\ at each $t \in [0,T]$.
The solution is unique if any other solution $\widetilde Y$ is a modification of $Y$, i.e.\ $Y^t=\widetilde Y^t$ a.s.\ for each $t \in [0,T]$, implying in the case where $Y$ and $\widetilde Y$ are both c\`adl\`ag that they have equal sample paths a.s.)

The following general lemma shows the well-posedness of a class of
Volterra-type SDEs, which we will use to show existence and uniqueness
of the specific processes $\xi^t$, $r_f^{t,*}$, and $r_f^t(x^{[t]})$. 

\begin{lem}\label{lem:exist-unique-poisson-volterra}
Let $(\Omega,\cF,\{\cF\}_{t \geq 0},\P)$ be the filtered probability space
of Theorem \ref{thm:uniqueness-existence}(a), where
$\{z^t\}_{t \in [0,T]}$ is either the Poisson process
\eqref{eq:poisson-z} or the Brownian diffusion process \eqref{eq:gaussian-z}.
Fix any $m \geq 1$, let $\{F^t\}_{t \in [0,T]}$ and $\{G^t\}_{t \in
[0,T]}$ be c\`adl\`ag $\cF_t$-adapted
processes in $\R^{m \times k}$ and $\R^m$, and let 
$H:\R^k \times \R^{k^*} \times \R \to \R^k$ be a function that is
uniformly Lipschitz-continuous in its first argument. Suppose
for some deterministic constants $M>0$ and $\iota>0$ that
\[\sup_{t \in [0,T]} \norm{F^t} \leq M \text{ a.s.}, \qquad \sup_{t \in [0,T]} \E \norm{G^t}^{2+\iota}<\infty,
\qquad \E\|H(0,w^*,\eps)\|^{2+\iota}<\infty.\]
Then for any $\Rt \equiv \{\Rt^{t,s}\}_{t,s \in [0,T]}$
satisfying the conditions of Definition \ref{def:S},
there exists a unique $\cF_t$-adapted and c\`adl\`ag solution
$\{Y^t\}_{t \in [0,T]}$ to
\begin{equation}\label{eq:Picardconstruction}
Y^t=F^t \int_0^t \frac{\bar\eta^r}{\bar\kappa} R_\theta^{t,r}
H(Y^{r-},w^*,\eps)\d z^r+G^t.
\end{equation}
\end{lem}
\begin{proof}
Let $L^{2+\iota,\infty}(\Omega \times [0,T],\R^k)$ be the
space of jointly measurable, $\cF_t$-adapted processes $\{Y^t\}_{t \in [0,T]}$ (identified up to modifications)
for which
$\|Y\|_{L^{2+\iota,\infty}}:=\sup_{t \in [0,T]} (\E \|Y^t\|^{2+\iota})^{\frac{1}{2+\iota}}<\infty$.
Initialize $Y^{(0)}=G \in L^{2+\iota}(\Omega \times [0,T],\R^k)$, and
suppose inductively we have constructed a c\`adl\`ag process
$Y^{(n)} \in L^{2+\iota}(\Omega \times [0,T],\R^k)$.
Consider the Volterra-type It\^o integral
\[I^{(n),t}=\int_0^t \frac{\bar \eta^r}{\bar\kappa}
R_\theta^{t,r}H(Y^{(n),r-},w^*,\eps) \d z^r,\]
which defines a jointly measurable and adapted process $\{I^{(n),t}\}_{t \in [0,T]}$
by \cite[Lemma 2.A]{berger1980volterra} in the setting of Gaussian $\{z^t\}$ or
by a pathwise construction for each $\omega \in \Omega$ in the setting of
Poisson $\{z^t\}$. Recall the maximal intervals of $[0,T] \setminus D$ 
in Definition \ref{def:S} on which $t \in [r,T] \mapsto R_\theta^{t,r}$
is uniformly continuous. For Poisson
$\{z^t\}_{t \in [0,T]}$, it is clear that this continuity implies $I^{(n),t}$
is c\`adl\`ag over each such maximal interval, and hence also over
$[0,T]$, with discontinuities at the jumps of $\{z^t\}_{t \in [0,T]}$ and the
points of $D$. For Gaussian $\{z^t\}$, note that
for any $s<t$ belonging to the same maximal interval of $[0,T] \setminus D$,
we have by the Burkholder-Davis-Gundy inequality of
Lemma \ref{lem:burkholder} and H\"older's inequality,
\begin{align*}
&\E\|I^{(n),t}-I^{(n),s}\|^{2+\iota}\\
&\lesssim \E\left\|\int_0^s \frac{\bar \eta^r}{\bar\kappa}
(R_\theta^{t,r}-R_\theta^{s,r})H(Y^{(n),r-},w^*,\eps) \d z^r\right\|^{2+\iota}
+\E\left\|\int_s^t \frac{\bar \eta^r}{\bar\kappa}
R_\theta^{t,r}H(Y^{(n),r-},w^*,\eps) \d z^r\right\|^{2+\iota}\\
&\lesssim \E\left(\int_0^s \|(R_\theta^{t,r}
-R_\theta^{s,r})H(Y^{(n),r},w^*,\eps)\|^2 \d r\right)^{1+\iota/2}
+\E\left(\int_s^t\|R_\theta^{t,r}H(Y^{(n),r},w^*,\eps)\|^2 \d
r\right)^{1+\iota/2}\\
&\lesssim \int_0^s \E \|(R_\theta^{t,r}
-R_\theta^{s,r})H(Y^{(n),r},w^*,\eps)\|^{2+\iota} \d r
+|t-s|^{\iota/2}\int_s^t \E\|R_\theta^{t,r}H(Y^{(n),r},w^*,\eps)\|^{2+\iota}
\d r.
\end{align*}
Then, by the boundedness and continuity conditions for $R_\theta$ in Definition
\ref{def:S} and the Lipschitz continuity of $H(\cdot)$,
\begin{align}
\E\|I^{(n),t}-I^{(n),s}\|^{2+\iota} &\lesssim
|t-s|^{2+\iota}\left(\E\|H(0,w^*,\eps)\|^{2+\iota}+\int_0^T
\E \|Y^{(n),r}\|^{2+\iota} \d r\right)\notag\\
&\qquad+|t-s|^{\iota/2}\left(|t-s|\,\E\|H(0,w^*,\eps)\|^{2+\iota}
+\int_s^t \|Y^{(n),r}\|^{2+\iota} \d r\right)\notag\\
&\lesssim |t-s|^{2+\iota}+|t-s|^{1+\iota/2}.\label{eq:Icontinuous}
\end{align}
Kolmogorov's continuity theorem then implies that $\{I^{(n),t}\}_{t \in [0,T]}$ 
has a modification which is uniformly continuous on each maximal interval of
$[0,T] \setminus D$ and hence c\`adl\`ag on $[0,T]$.

Then defining
\begin{equation}\label{eq:PicardStepXi}
Y^{(n+1),t}=F^tI^{(n),t}+G^t
=F^t\int_0^t \frac{\bar\eta^r}{\bar\kappa} R_\theta^{t,r}
H(Y^{(n),r-},w^*,\eps)\d z^r+G^t,
\end{equation}
$Y^{(n+1)}$ is adapted and c\`adl\`ag.
Applying the boundedness of $\{F^t\}$, the Burkholder-Davis-Gundy
inequality of Lemma \ref{lem:burkholder} in both the Poisson and Gaussian
settings, and H\"older's inequality as above,
\begin{align*}
\E \norm{Y^{(n+1),t}}^{2+\iota}
&\lesssim \E \left\|\int_0^t \frac{\bar\eta^r}{\bar\kappa} R_\theta^{t,r}
H(Y^{(n),r-},w^*,\eps)\d z^r\right\|^{2+\iota}
+\E \|G^t\|^{2+\iota}\\
&\lesssim \E\|H(0,w^*,\eps)\|^{2+\iota}
+\int_0^T \E\|Y^{(n),r}\|^{2+\iota}\d r
+\E\|G^t\|^{2+\iota}.
\end{align*}
Thus $Y^{(n+1)} \in L^{2+\iota,\infty}(\Omega \times [0,T],\R^k)$, so
$Y^{(n+1)}$ is inductively well-defined for each $n \geq 0$.

Now consider the difference $Y^{(n),t}-Y^{(n-1),t}$,
and set $\phi_n(t)=\E\|Y^{(n),t}-Y^{(n-1),t}\|^{2+\iota}$.
Then again by Lemma \ref{lem:burkholder} and similar arguments as above,
\begin{align}
\phi_1(t)&\lesssim \E \left\|\int_0^t \frac{\bar\eta^r}{\bar\kappa}
R_\theta^{t,r}H(Y^{(0),r-},w^*,\eps) \d z^r\right\|^{2+\iota}
\lesssim \E\|H(0,w^*,\eps)\|^{2+\iota} +\int_0^T \E \|G^r\|^{2+\iota} \d
r,\notag\\
\phi_{n+1}(t) &\lesssim
\E \left\|\int_0^t \frac{\bar\eta^r}{\bar\kappa}R_{\theta}^{t,r}
\Big(H(Y^{(n),r-},w^*,\eps) - H(Y^{(n-1),r-},w^*,\eps)\Big)
\d z^r\right\|^{2+\iota}\notag\\
&\lesssim\int_0^t \E \|Y^{(n),r}-Y^{(n-1),r}\|^{2+\iota}
\d r=\int_0^t \phi_n(r)\d r.\label{eq:Picardbound}
\end{align}
I.e., there exists a constant $C>0$ (depending on $M,T,\pRt(T),M_\eta,H$,
and the law of $(w^*,\eps)$) for which
$\phi_1(t) \leq C$ and $\phi_{n+1}(t)\leq C\int_0^t \phi_n(r)\d r$.
Then $\phi_{n}(t)\leq C^nt^{n-1}/(n-1)!$, implying that $\{Y^{(n)}\}_{n \geq 0}$
is Cauchy in $L^{2+\iota,\infty}(\Omega \times [0,T],\R^k)$ and thus has a
limit $Y \in L^{2+\iota,\infty}(\Omega \times [0,T],\R^k)$ such that
$\lim_{n \to \infty} \sup_{t \in [0,T]} \E\|Y^{(n),t}-Y^t\|^{2+\iota}=0$. Defining the process
\[I^t=\int_0^t \frac{\bar\eta^r}{\bar\kappa}R_\theta^{t,r}
H(Y^{r-},w^*,\eps)\d z^r\]
as above, we then have $\E\|Y^t-(F^tI^t+G^t)\|^{2+\iota}
=\lim_{n \to \infty} \E\|Y^{(n+1),t}-(F^tI^{(n),t}+G^t)\|^{2+\iota}=0$ for each $t \in [0,T]$,
so $\{Y^t\}_{t \in [0,T]}$ solves \eqref{eq:Picardconstruction}. The
same argument as \eqref{eq:Icontinuous} implies that 
$\{I^t\}_{t \in [0,T]}$ and $\{Y^t\}_{t \in [0,T]}$ have c\`adl\`ag
modifications.

For uniqueness, let $\{\widetilde Y^t\}_{t \in [0,T]}$ be any other adapted
and c\`adl\`ag solution
to \eqref{eq:Picardconstruction}. Fix any $K>0$, define
the stopping time $\tau=\inf\{t \in [0,T]:\widetilde Y^t \geq K\}$,
and consider the stopped process 
$(\widetilde Y^\tau)^t=\widetilde Y^{\tau \wedge t}$. Then
by Lemma \ref{lem:burkholder} and a similar argument as above,
\begin{align*}
\E\|(\widetilde Y^\tau)^t\|^{2+\iota}
&=\E\left\|F^{t \wedge \tau}
\int_0^t \1\{r \leq \tau\}\frac{\bar \eta^r}{\bar\kappa}
R_\theta^{t,r}H(\widetilde Y^{r-},w^*,\eps)\d z^r
+G^{t \wedge \tau}\right\|^{2+\iota}\\
&\lesssim 1+\int_0^t
\E \|\1\{r \leq \tau\}\widetilde Y^{r-}\|^{2+\iota}\d r
\lesssim 1+\int_0^t \E\|(\widetilde Y^\tau)^r\|^{2+\iota}\d r.
\end{align*}
Note that the constant underlying $\lesssim$ is independent of $K$, so Gr\"onwall's
inequality implies that $\sup_{t \in [0,T]} \E\|(\widetilde Y^\tau)^t\|^{2+\iota}$ is uniformly
bounded independently of $K$. Taking $K \to \infty$
shows that $\widetilde Y \in L^{2+\iota,\infty}(\Omega \times [0,T],\R^k)$.
Then defining $\phi(t)=\E \|Y^t-\widetilde Y^t\|^{2+\iota}$,
the same argument as \eqref{eq:Picardbound} shows
$\phi(t) \lesssim \int_0^t \phi(s)\d s$, so $\sup_{t \in [0,T]} \phi(t)=0$ by Gr\"onwall's inequality, and $\widetilde Y$ is a modification of $Y$.
\end{proof}

\begin{proof}[Proof of Theorem \ref{thm:uniqueness-existence}(a)]
Let $(C_\theta,R_\theta,C_f,R_f,R_f^*,\Gamma) \in \cS$ be given, and let
$(\theta,\theta^*)$, $\eps$, $\{u^t\}_{t \in [0,T]}$, $(\{w^t\}_{t \geq
[0,T]},w^*)$, and $\{z^t\}_{t \in [0,T]}$ in the filtered space
$(\Omega,\cF,\{\cF\}_{t \in [0,T]},\P)$ also be given. The bound for
$C_f$ in Definition \ref{def:S} implies that $\sup_{t \in [0,T]} \E\|w^t\|^{2+\iota}<\infty$ for any $\iota>0$.
Then there exists a unique $\cF_t$-adapted and c\`adl\`ag
solution $\{\xi^t\}_{t \in [0,T]}$ to \eqref{eq:def-cont-xi} by
Lemma \ref{lem:exist-unique-poisson-volterra}. Given this solution and any
deterministic process $x \in L^4([0,T],\R^k)$, define next
\[G^t={-}\D_\xi
f(\xi^t,w^*,\eps)\int_0^t \frac{\bar\eta^r}{\bar\kappa}R_\theta^{t,r}
\D_\xi f(\xi^{r-},w^*,\eps)x^r \d z^r.\]
This process $\{G^t\}_{t \in [0,T]}$ is c\`adl\`ag and satisfies $\sup_{t \in [0,T]}\|G^t\|^{2+\iota}<\infty$ for
sufficiently small $\iota>0$, by
the same arguments as in Lemma \ref{lem:exist-unique-poisson-volterra}.
Then there exist unique $\cF_t$-adapted and c\`adl\`ag solutions
$\{r_f^{t,*}\}_{t \in [0,T]}$ and
$\{r_f^t(x^{[t]})\}_{t \in [0,T]}$ to
(\ref{eq:def-cont-deriv-xi-star}--\ref{eq:def-cont-resp-xi})
also by Lemma \ref{lem:exist-unique-poisson-volterra}.

For \eqref{eq:def-cont-theta}, set $\tilde
\theta^t=\theta^t-\sqrt{\gamma}\,u^t$. Note that since $C_f^{0,0}=0$, we have
$u^0=0$ and $\tilde \theta^0=\theta^0$. Then
$\{\theta^t\}_{t \in [0,T]}$ solves
\eqref{eq:def-cont-theta} if and only if 
$\{\tilde \theta^t\}_{t \in [0,T]}$ solves the differential equation
\begin{align}
\frac{\d}{\d t}\tilde\theta^t&=\underbrace{-\bar\eta^t \Big({\gamma}\Gamma^t
(\tilde\theta^t + \sqrt{\gamma}\,u^t)+g(\tilde\theta^t + \sqrt{\gamma}\,
u^t)+{\gamma}R_f^t(\tilde\theta^{[t]})+ {\gamma^{3/2}}
R_f^t(u^{[t]})+{\gamma}R_f^{t,*}\theta^*\Big)}_{:=F(t,\tilde\theta^{[t]})},
\quad \tilde\theta^0=\theta^0.\label{eq:cont-theta-rewrite}
\end{align}
The bound \eqref{eq:pRf-volterra} implies that
$R_f^t$ restricts to an operator on $C([0,t],\R^k) \subset L^4([0,t],\R^k)$ that
is Lipschitz in the sup-norm $\|x\|_\infty=\sup_{s \in [0,t]} \|x^s\|_2$,
uniformly over $t \in [0,T]$. Then the function $(t,x) \mapsto F(t,x^{[t]})$
on the right side above is measurable
in $t$ and uniformly Lipschitz in $x \in C([0,T],\R^k)$, implying
(c.f.\ \cite[Theorems 2.1, 2.3]{hale2013introduction})
the existence of a unique continuous solution to \eqref{eq:cont-theta-rewrite} 
for any realization of $(\theta^0,\theta^*)$ and
$\{u^t\}_{t \in [0,T]}$. This defines, pathwise for each $\omega \in \Omega$,
a unique $\cF_t$-adapted and c\`adl\`ag solution
$\theta^t=\tilde \theta^t+\sqrt{\gamma}\,u^t$ to \eqref{eq:def-cont-theta}, with discontinuities at the jumps of $\{u^t\}$.
Given this solution $\{\theta^t\}_{t \in [0,T]}$, the same argument shows that
for each fixed $s \in [0,T]$, there is a unique continuous solution
$\{r_\theta^{t,s}\}_{t \in [s,T]}$ to
\[\frac{\d}{\d t}r_\theta^{t,s}
={-}\bar\eta^t[(\gamma \Gamma^t+\D g(\theta^t)r_\theta^{t,s}
+\gamma R_f^t(r_\theta^{[t],s})],
\quad r_\theta^{s,s}=\Id_k,\]
which is equivalent to \eqref{eq:def-cont-deriv-t}.
\end{proof}

\subsection{Closure of the mappings}

Given $(\Ct,\Rt) \in \cS_\theta$, consider the processes \eqref{eq:def-cont-xi},
\eqref{eq:def-cont-deriv-xi-star}, and \eqref{eq:def-cont-resp-xi}
defined on $(\Omega,\cF,\{\cF_t\}_{t \geq 0},\P)$ and let
$\trsfrmB(\Ct,\Rt)=(\Cf,\Rf,\Rf^*,\Gamma)$ be the mapping defined by
(\ref{eq:def-cont-C-h}--\ref{eq:def-cont-Gamma}). Similarly, given
$(\Cf,\Rf,\Rf^*,\Gamma) \in \cS_\xi$, consider the processes
\eqref{eq:def-cont-theta} and \eqref{eq:def-cont-deriv-t} and let
$\trsfrmA(\Cf,\Rf,\Rf^*,\Gamma)=(\Ct,\Rt)$ be the mapping defined by
(\ref{eq:def-cont-C-t}--\ref{eq:def-cont-R-t}).

In this section, we check that for a sufficiently large choice of constant
$C_0>0$ defining $\cS_\theta$ and $\cS_\xi$,
$\trsfrmB$ maps ${\mathcal{S}}_\theta$ into
$\mathcal{S}^{\mathrm{cont}}_\xi \subset \cS_\xi$, and 
$\trsfrmA$ maps $\mathcal{S}_\xi$ into ${\mathcal{S}}_\theta$.

\begin{lem} \label{lem:solution-in-space}
Fix any $T>0$. Then for any constant $C_0>0$ large enough (depending on $T$,
$M_f$, $M_g$, $M_\eta$, $\E\|\theta^0\|^2$, $\E\|\theta^*\|^2$, and
$\|g(0)\|^2$),
$\trsfrmA$ maps $\mathcal{S}_\xi(T,C_0)$ into ${\mathcal{S}}_\theta(T,C_0)$, and
$\mathcal{S}^{\mathrm{cont}}_\xi(T,C_0)$ into
${\mathcal{S}}^{\mathrm{cont}}_\theta(T,C_0)$.
\end{lem}
\begin{proof}
Assume $(C_f,R_f,R_f^*,\Gamma) \in S_\xi(T,C_0)$, and denote
$(C_\theta,R_\theta)=\trsfrmA(C_f,R_f,R_f^*,\Gamma)$.

We note that by its definition in \eqref{eq:def-cont-C-t},
the kernel
$\{C_\theta^{t,s}\}_{t,s \in [0,T] \cup \{*\}}$ is positive-semidefinite 
with $C_\theta^{*,*},C_\theta^{0,*},C_\theta^{0,0}$ as specified in Definition
\ref{def:S}. Recall from \eqref{eq:def-cont-theta} that
\begin{align}\label{eq:bound-theta-second-moment-1}
    \theta^t&=\theta^0-\int_0^t \bar\eta^r \Big({\gamma}\Gamma^r
\theta^r+g(\theta^r)+{\gamma}R_f^{r}(\theta^{[r]})+
{\gamma}R_f^{r,*}\theta^*\Big)\d r+\sqrt{\gamma}\,u^t.
\end{align}
Applying Cauchy-Schwarz and the bound $(\int_0^t f(s)\d s)^2 \leq t\int_0^t
f(s)^2 \d s \leq T\int_0^t f(s)^2 \d s$, we have $\E\|\theta^t\|^2 \leq
6(\nI+\nII+\nIII+\nIV+\nV+\nVI)$ where
\begin{align*}
\nI&=\E\|\theta^0\|^2\\
\nII&=\E\left\|\int_0^t
\bar\eta^r\gamma\Gamma^r\theta^r \d r\right\|^2
\leq \gamma^2 M_f^2 M_\eta^2 T\int_0^t \E\norm{\theta^r}^2 \d r,\\
\nIII&=\E\left\|\int_0^t \bar\eta^r g(\theta^r) \d
r\right\|^2\leq M_\eta^2 T\int_0^t \E\|g(\theta^r)\|^2 \d r
    \leq 2M_g^2 M_\eta^2 T\int_0^t \E\norm{\theta^r}^2\d r
+2M_\eta^2 T^2\norm{g(0)}^2,\,\\
\nIV&=\E\left\|\int_0^t \bar\eta^r\gamma
R_f^{r}(\theta^{[r]}) \d r\right\|^2 \leq \gamma^2 M_\eta^2 T\int_0^t \E\norm{R_f^{r}(\theta^{[r]})}^2 \d r\\
    &\overset{(*)}{\leq} \gamma^2 M_\eta^2 T\int_0^t
\int_0^r \pRf(r-r')\E \|\theta^{r'}\|^2\d r'\,\d r,\\
\nV&=\E\left\|\int_0^t \bar\eta^r \gamma R_f^{r,*}\theta^*
\d r\right\|^2
    \leq \gamma^2 M_\eta^2 T\,\E\norm{\theta^*}^2
\int_0^t \Phi_{R_f^*}(r)\d r,\\
\nVI&=\gamma\,\E\norm{u^t}^2=\gamma \Tr \Cf^{t,t} \leq \gamma C_0\int_0^t
\Phi_{C_f}(s)\d s.
\end{align*}
Here, we have used the property \eqref{eq:pRf-volterra} of $R_f^t(\cdot)$
to bound $(*)$.
Collecting these bounds shows, for any constant $C_0>0$ large enough,
\begin{align*}
\E\|\theta^t\|^2&\leq C_0\left(1+\int_0^t \left(6\gamma \Phi_{C_f}(r)+\E\|\theta^r\|^2
+\int_0^r \Phi_{R_f}(r-r')\E\|\theta^{r'}\|^2\d r'+\Phi_{R_f^*}(r) \right)\d
r\right).
\end{align*}
The definition of $\pCt$ in \eqref{eq:def-pCt} implies
\begin{equation}\label{eq:PhiCthetaexplicit}
\pCt(t)=C_0\left(1+\int_0^t \left(6\gamma \Phi_{C_f}(r)+\pCt(r)
+\int_0^r \Phi_{R_f}(r-r')\pCt(r')\d r'+\Phi_{R_f^*}(r) \right)\d r\right),
\end{equation}
so comparing with the above shows
\[\Tr C_\theta^{t,t}=\E\|\theta^t\|^2 \leq \Phi_{\Ct}(t).\]
Next, let $D$ be the discontinuity set of $(\Cf,\Rf,\Rf^*,\Gamma) \in
\cS_\xi(T,C_0)$.
For each maximal interval $I$ of $[0,T] \setminus D$
and any $s,t\in I$ with $t \geq s$, by an analogous application of Cauchy-Schwarz
and this bound $\E\|\theta^t\|^2 \leq \Phi_{C_\theta}(t)$, for any constant
$C_0>0$ large enough,
\begin{align}
&\Tr[\Ct^{t,t} - \Ct^{t,s} - \Ct^{s,t} + \Ct^{s,s}]
=\E \norm{\theta^t-\theta^s}^2\notag\\
&=\E\left\|\int_s^t \bar\eta^r \Big({\gamma}\Gamma^r
\theta^r+g(\theta^r)+{\gamma}R_f^{r}(\theta^{[r]})+
{\gamma}R_f^{r,*}\theta^*\Big)\d r+\sqrt{\gamma}(u^t-u^s)\right\|^2\notag\\
&\leq C_0(t-s)\left(1+\int_s^t \E\|\theta^r\|^2\d r
+\int_s^t\int_0^r \Phi_{R_f}(r-r')\E\|\theta^{r'}\|^2\d r'\d r
+\int_s^t\Phi_{R_f^*}(r)\d r\right)\notag\\
&\hspace{3in}+C_0\E\|u^t-u^s\|^2\notag\\
&\leq C_0(t-s)\left(1+\int_0^t \Phi_{C_\theta}(r)\d r
+\int_0^t \int_0^r \Phi_{R_f}(r-r')\Phi_{C_\theta}(r')\d r' \d r
+\int_0^t \Phi_{R_f^*}(r)\d r\right)\notag\\
&\hspace{3in}+C_0\E\|u^t-u^s\|^2\notag\\
&\leq (\Phi_{C_\theta}(T)+C_0^2\Phi_{C_f}(T))|t-s|\label{eq:bound-cont-Ct}
\end{align}
where the last inequality applies
$\E\norm{u^t-u^s}^2=\Tr[\Cf^{t,t}-\Cf^{t,s}-\Cf^{s,t}+\Cf^{s,s}]
\leq C_0\Phi_{C_f}(T)|t-s|$ by \eqref{eq:Cfcontinuity}, and the form of
$\Phi_{\Ct}(t)$ in \eqref{eq:PhiCthetaexplicit}
and monotonicity $\Phi_{\Ct}(t) \leq \Phi_{\Ct}(T)$.
This checks all conditions of Definition \ref{def:S} for $C_\theta$.

For $R_\theta$, recall from \eqref{eq:def-cont-deriv-t} that
\begin{align}\label{eq:rthetarestate}
r_\theta^{t,s}=\Id-\int_s^t \bar\eta^r\bigg[\Big({\gamma}\Gamma^r+\D
g(\theta^r)\Big)r_\theta^{r,s}+{\gamma}
R_f^r\Big(r_\theta^{[r],s}\Big)\bigg]\d r \text{ for } t \geq s.
\end{align}
Then $\|r_\theta^{t,s}\|^2 \leq 4(\nI+\nII+\nIII+\nIV)$ where
\begin{align*}
\nI&=\|\Id\|^2=k\\
\nII&=\left\|\int_s^t \bar\eta^r\gamma\Gamma^r r_\theta^{r,s}\d
r\right\|^2
\leq \gamma^2M_f^2M_\eta^2T\int_s^t \|r_\theta^{r,s}\|^2 \d r,\\
\nIII&=\left\|\int_s^t \bar\eta^r \D g(\theta^r)r_\theta^{r,s}\d
r\right\|^2
\leq M_g^2 M_\eta^2 T\int_s^t \|r_\theta^{r,s}\|^2 \d r,\\
\nIV&=\left\|\int_s^t \bar\eta^r \gamma R_f^r\Big(r_\theta^{[r],s}\Big)\d
r\right\|^2
\leq \gamma^2 M_\eta^2 T\int_s^t
\Big\|R_f^r\Big(r_\theta^{[r],s}\Big)\Big\|^2 \d r\\
&\leq \gamma^2 M_\eta^2 T\int_s^t \int_s^r \Phi_{R_f}(r-r')
\|r_\theta^{r',s}\|^2 \d r'\d r,\\
\end{align*}
the last inequality again using \eqref{eq:pRf-volterra} and the fact that
$r_\theta^{r',s}=0$ for $r'<s$. Collecting these bounds shows
\[\|r_\theta^{t,s}\|^2
\leq C_0\left(1+\int_s^t \left(\norm{r_\theta^{r,s}}^2
+\int_s^r \Phi_{R_f}(r-r')\|r_\theta^{r',s}\|^2\d r'\right) \d r\right)\]
for any constant $C_0>0$ large enough.
The definition of $\Phi_{R_\theta}(t)$ in \eqref{eq:def-pRt} implies
\begin{equation}\label{eq:PhiRthetaexplicit}
\Phi_{R_\theta}(t-s)=C_0\left(1+\int_0^{t-s}\left(
\Phi_{R_\theta}(r)+\int_0^r \Phi_{R_f}(r-r')\Phi_{R_\theta}(r')\d
r'\right)\d r\right),
\end{equation}
so comparing with the above shows
\[\|R_\theta^{t,s}\|^2=\|\E r_\theta^{t,s}\|^2\leq \E\|r_\theta^{t,s}\|^2
\leq \Phi_{R_\theta}(t-s).\]
Then, applying this bound
$\E\|r_\theta^{t,s}\|^2 \leq \Phi_{R_\theta}(t-s)$,
we have similarly for any $C_0>0$ large enough and any $t \geq t' \geq s$ that
\begin{align}
\|R_\theta^{t,s}-R_\theta^{t',s}\|^2
&\leq \E\|r_\theta^{t,s}-r_\theta^{t',s}\|^2\notag\\
&\leq C_0|t-t'|\left(\int_{t'}^t \E\|r_\theta^{r,s}\|^2 \d r
+\int_{t'}^t \int_s^r \Phi_{R_f}(r-r')\E \|r_\theta^{r',s}\|^2 \d r' \d
r\right)\notag\\
&\leq C_0|t-t'|\left(\int_s^t \pRt(r-s)\d r
+\int_s^t \int_s^r \Phi_{R_f}(r-r')\pRt(r'-s)\d r' \d
r\right)\notag\\
&\leq \pRt(t-s)|t-t'| \leq \pRt(T)|t-t'|,\label{eq:Rcontbound}
\end{align}
where the last two inequalities apply \eqref{eq:PhiRthetaexplicit} and
monotonicity of $\pRt$.
This checks all conditions of Definition \ref{def:S} for $R_\theta$, so
$(\Ct,\Rt) \in {\mathcal{S}}_\theta(T,C_0)$.

We note that \eqref{eq:Rcontbound} holds for all $t,t',s \in [0,T]$ with
$t \geq t' \geq s$, irrespective of the discontinuity set
$D$ for $(\Cf,\Rf,\Rf^*,\Gamma) \in \cS_\xi(T,C_0)$.
If $(\Cf,\Rf,\Rf^*,\Gamma)\in \mathcal{S}^{\mathrm{cont}}_\xi(T,C_0)$, then
\eqref{eq:bound-cont-Ct} also holds for all $t,s \in [0,T]$, implying by the
above arguments that the conditions of Definition \ref{def:S} for
$(C_\theta,R_\theta)$ hold with $D=\emptyset$.
Thus also $(\Ct,\Rt) \in {\mathcal{S}}_\theta^\mathrm{cont}(T,C_0)$.
\end{proof}

\begin{lem} \label{lem:solution-in-space-2}
Fix any $T>0$. Then for any constant $C_0>0$ large enough (depending on $T$,
$M_f$, and $M_\eta$), $\trsfrmB$ maps ${\mathcal{S}}_\theta(T,C_0)$ into
$\mathcal{S}^{\mathrm{cont}}_\xi(T,C_0)$.
\end{lem}
\begin{proof}
In both the Poisson and Gaussian settings,
$\{z^t-\bar \kappa t\}_{t \in [0,T]}$ is a
martingale whose predictable quadratic variation is given by
$\langle z^t-\bar \kappa t
\rangle=\bar \kappa t$. Then for any $\cF_t$-predictable square-integrable
process $\{X^t\}_{t \in [0,T]}$,
\begin{align}
\E\left\|\int_0^t X^r \d z^r\right\|_2^2
&\leq 2\,\E\left\|\int_0^t X^r \bar \kappa\,\d r\right\|_2^2
+2\,\E\left\|\int_0^t X^r (\d z^r-\bar \kappa\,\d r)\right\|_2^2\notag\\
&=2\,\E\left\|\int_0^t X^r \bar \kappa\,\d r\right\|_2^2
+2\int_0^t \E\|X^r\|_2^2\,\bar\kappa\,\d r
\leq 2(t+1)\int_0^t \E\|X^r\|_2^2\,\bar\kappa\,\d r.\label{eq:basicdzrbound}
\end{align}

Assume $(\Ct,\Rt) \in {\mathcal{S}}_\theta(T,C_0)$, and denote
$(C_f,R_f,R_f^*,\Gamma)=\trsfrmB(C_\theta,R_\theta)$.
By its definition in \eqref{eq:def-cont-C-h},
$C_f$ is a positive-semidefinite kernel with $C_f^{0,0}=0$. 
Applying \eqref{eq:basicdzrbound} and \eqref{eq:def-cont-xi},
using the fact that $f$ is Lipschitz, we have
\begin{align*}
    &\E\|f(\xi^t, w^*, \eps)\|^2 \le 2M_f^2 \E\|\xi^t\|^2 + 2 \E\|f(0, w^*, \eps)\|^2 \\
    &\le 4M_f^2 \E\left\|\int_0^t \frac{\bar\eta^r}{\bar\kappa} R_\theta^{t,r}
f(\xi^{r-},w^*,\eps)\d z^r\right\|^2 + 4M_f^2\E \|w^t\|^2 + 2 \E\|f(0, w^*, \eps)\|^2\\
&\le 8(t+1)M_f^2M_\eta^2 \int_0^t \E\left\| R_\theta^{t,r}
f(\xi^{r},w^*,\eps)\right\|^2 \d r + 4M_f^2 \Tr C_\theta^{t,t} + 2 \E\|f(0, w^*, \eps)\|^2\\
&\le 8(t+1)M_f^2M_\eta^2 \int_0^t \Phi_{R_\theta}(t-r)\E\left\|
f(\xi^{r},w^*,\eps)\right\|^2 \d r + 4M_f^2 \Phi_{C_\theta}(t) + 2 \E\|f(0, w^*, \eps)\|^2.
\end{align*}
Therefore, for a large enough constant $C_0>0$, we have
\[
\E\|f(\xi^t, w^*, \eps)\|^2 \le C_0 \left(1 + \Phi_{C_\theta}(t) + \int_0^t \Phi_{R_\theta}(t-r)\E\left\|
f(\xi^{r},w^*,\eps)\right\|^2 \d r\right).
\]
Comparing with the definition of $\pCf(t)$ in \eqref{eq:def-pCf}, we have
$\E\|f(\xi^t, w^*, \eps)\|^2 \le \pCf(t)$.
Then for any $t \in [0,T]$,
\[\Tr \Cf^{t,t}=\E\left\|\int_0^t
\frac{\bar\eta^r}{\bar\kappa}f(\xi^{r-},w^*,\eps)\d z^r\right\|^2
\leq 2(t+1)\int_0^t \E\left\|\frac{\bar \eta^r}{\bar \kappa}
f(\xi^r,w^*,\eps)\right\|^2\bar \kappa\,\d r
\leq C_0 \int_0^t \pCf(s)\d s\]
for a large enough constant $C_0>0$. Similarly, for all $t,s \in [0,T]$
with $t \geq s$,
\begin{align*}
\Tr[\Cf^{t,t} - \Cf^{t,s} - \Cf^{s,t} + \Cf^{s,s}] &=\E\left\|\int_s^t
\frac{\bar\eta^r}{\bar\kappa}f(\xi^{r-},w^*,\eps)\d z^r\right\|^2
\leq C_0\pCf(T)|t-s|
\end{align*}
for large enough $C_0>0$.
This checks all conditions of Definition \ref{def:S} for $C_f$.

For $R_f$, given any $x \in L^4([0,T],\R^k)$, 
it is clear from the definition of $r_f^t(x^{[t]})$
in \eqref{eq:def-cont-resp-xi} that
$t \mapsto R_f^t(x^{[t]})=\E r_f^t(x^{[t]})$ is Borel-measurable.
By \eqref{eq:def-cont-resp-xi} and \eqref{eq:basicdzrbound},
we have
\begin{align}
\E\|r_f^t(x^{[t]})\|^2 &\leq M_f^2
\E\left\|\int_0^t \frac{\bar\eta^r}{\bar\kappa}
R_\theta^{t,r}\Big(r_f^{r-}(x^{[r]})+\D_\xi f(\xi^{r-},w^*,\eps)x^{r-}\Big)\d
z^r\right\|^2\notag\\
&\leq C_0\int_0^t \Phi_{R_\theta}(t-r)
\Big(\E \|r_f^r(x^{[r]})\|^2+\|x^r\|^2\Big)\d r\label{eq:bound-Acal-moment}
\end{align}
for large enough $C_0>0$. From the definition of $\pRf$ in \eqref{eq:def-pRf}, we have
\begin{equation}\label{eq:def-pRf-lem}
\pRf(t-r)=C_0\left(\int_r^t \pRt(t-r')\pRf(r'-r)\d r'
+\Phi_{R_\theta}(t-r)\right).
\end{equation}
Integrating both sides of the equation with respect to ${\norm{x^r}^2}$, we get
\begin{align*}
\int_0^t \pRf(t-r) \norm{x^r}^2\d r
&=C_0\int_0^t \prn{\int_r^t \pRt(t-r') \pRf(r'-r) \d r' + \pRt(t-r)} \norm{x^r}^2\d r\nonumber\\
    &=C_0\int_0^t \pRt(t-r)\prn{\int_0^r \pRf(r-r')\norm{x^{r'}}^2\d {r'}+
\norm{x^r}^2}\d r.
\end{align*}
Comparing with \eqref{eq:bound-Acal-moment}, this shows that
\begin{equation}\label{eq:Erfsquaredbound}
\|R_f^t(x^{[t]})\|^2
=\|\E r_f^{t}(x^{[t]})\|^2
\leq \E\norm{r_f^{t}(x^{[t]})}^2 \leq \int_0^t \pRf(t-r) \norm{x^r}^2\d r.
\end{equation}
Further bounding
$(\int_0^t \pRf(t-r) \norm{x^r}^2\d r)^2
\leq t\pRf(t)^2\int_0^t \norm{x^r}^4\d r$, this verifies that
$R_f^t$ is a bounded linear operator from $L^4([0,t],\R^k)$ to $\R^k$. This
checks all conditions of Definition \ref{def:S} for $R_f$.

For $R_f^*$, we have similarly
\begin{align*}
    \E \norm{r_f^{t,*}}^2
&\leq 2\E\left\|\D_\xi f(\xi^t,w^*,\eps) \int_0^t
\frac{\bar\eta^r}{\bar\kappa}
\Rt^{t,r}r_f^{r-,*}\d z^r\right\|^2
+2\E\|\D_{w^*} f(\xi^t,w^*,\eps)\|^2\\
&\leq C_0\left(1+\int_0^t \pRt(t-r)\E\|r_f^{r,*}\|^2\d r\right)
\end{align*}
for large enough $C_0>0$. Comparing with the definition of $\Phi_{R_f^*}$ in
\eqref{eq:def-pRfs} shows
$\|R_f^{t,*}\|^2 \leq \E\|r_f^{t,*}\|^2 \leq \pRfs(t)$. Finally, 
the measurability of $t \mapsto \Gamma^t$ and bound
$\|\Gamma_t\|_\op \leq M_f$ follow directly from the definition of $\Gamma^t$ 
in \eqref{eq:def-cont-Gamma}, and \eqref{eq:defMbounds}.
Thus $(C_f,R_f,R_f^*,\Gamma) \in \cS_\xi^{\mathrm{cont}}(T,C_0)$.
\end{proof}

\subsection{Contractivity of the mappings}\label{sec:contractivity}

We now endow $\mathcal{S}_\xi \equiv \cS_\xi(T,C_0)$ and ${\mathcal{S}}_\theta
\equiv \cS_\theta(T,C_0)$ with metrics, under which we will show that
$\trsfrm \equiv \trsfrmA \circ \trsfrmB$ is contractive.

Fix a constant $\lambda>0$ (which we will also take large enough, depending
on $T,C_0$). For covariance kernels $C_{f,1},C_{f,2}$ on $[0,T] \otimes
\R^k$ and $C_{\theta,1},C_{\theta,2}$ on $([0,T] \otimes \R^k) \times \R^{k^*}$
satisfying the conditions of Definition \ref{def:S}, define
\begin{align*}
\lbddst{C_{f,1}}{C_{f,2}}=\inf_{(u_1,u_2) \sim \gamma \in
\Gamma(C_{f,1},C_{f,2})} \sup_{t \in [0,T]} e^{-\lambda t} \big(\E {\norm{u^t_1 -
u^t_2}^4}\big)^{1/4},\\
\lbddst{C_{\theta,1}}{C_{\theta,2}}=\inf_{(w_1,w_2) \sim \gamma \in
\Gamma(C_{\theta,1},C_{\theta,2})} \sup_{t \in [0, T]} e^{-\lambda t} \big(\E
\norm{w^t_1 - w^t_2}^4\big)^{1/4},
\end{align*}
where $\Gamma(C_{f,1},C_{f,2})$ denotes the set of all couplings between the
Gaussian process laws $\GP(0,C_{f,1})$ and $\GP(0,C_{f,2})$, and similarly for
$\Gamma(C_{\theta,1},C_{\theta,2})$. Note that under the constraints and
continuity conditions for $C_f,C_\theta$ in Definition \ref{def:S}, if 
$\lbddst{C_{\theta,1}}{C_{\theta,2}}=0$, then
$C_{\theta,1}^{t,s}=C_{\theta,2}^{t,s}$ pointwise for all $t,s \in [0,T]
\cup \{*\}$, and similarly for $C_f$.

For pairs of response kernels/operators satisfying Definition \ref{def:S}, we
define also
\begin{align*}
    \lbddst{R_{f,1}}{R_{f,2}} &= \sup_{0 \leq t \leq T} e^{-\lambda t}
\normop{R_{f,1}^t - R_{f,2}^t},\\
    \lbddst{R_{f,1}^*}{R_{f,2}^*} &= \sup_{0 \leq t \leq T} e^{-\lambda t}
\norm{R_{f,1}^{t,*} - R_{f,2}^{t,*}},\\
    \lbddst{\Gamma_1}{\Gamma_2} &= \sup_{0 \leq t \leq T} e^{-\lambda t}
\norm{\Gamma_1 - \Gamma_2},\\
    \lbddst{R_{\theta,1}} {R_{\theta,2}} &= \sup_{0 \leq s\leq t \leq T}
e^{-\lambda t} \norm{R_{\theta,1}^{t,s} - R_{\theta,2}^{t,s}},
\end{align*}
where we recall that $\normop{R_f^t}$ is the operator norm as a linear map
from $L^4([0,t],\R^k)$ to $\R^k$.
Finally, for any $X_i=(C_{f,i},R_{f,i},R_{f,i}^*,\Gamma_i) \in \mathcal{S}_\xi$
and $Y_i=(C_{\theta,i},R_{\theta,i}) \in {\mathcal{S}}_\theta$, we define 
\begin{subequations}
\begin{align*}
	\lbddst{X_1}{X_2} & = \lbddst{C_{f,1}}{C_{f,2}} +
\lbddst{R_{f,1}}{R_{f,2}} + \lbddst{R_{f,1}^*}{R_{f,2}^*}
+ \lbddst{\Gamma_1}{\Gamma_2},\\
	\lbddst{Y_1}{Y_2} & = \lbddst{C_{\theta,1}}{C_{\theta,2}}
+\lbddst{R_{\theta,1}}{R_{\theta,2}}.
\end{align*}
\end{subequations}

Throughout these arguments, we will fix $T,C_0>0$, and write
$\lesssim$ to mean inequality up to a constant depending on $T,C_0$.

We first study the contractivity properties for $\trsfrmA$.

\begin{lem}\label{lem:exist-unique-theta}
Fix any $T,C_0>0$, suppose $X=(C_f,R_f,R_f^*,\Gamma) \in
\cS_\xi(T,C_0)$, and let $\{\theta^t\}_{t \in [0,T]}$ and
$\{r_\theta^{t,s}\}_{t,s \in [0,T]}$ be defined
by \eqref{eq:def-cont-theta} and \eqref{eq:def-cont-deriv-t}.
Then there is a constant $C \equiv C(T,C_0)>0$ such that for all $t,s \in [0,T]$,
\[\E\|\theta^t\|^4 \leq C,
\qquad \E\|r_\theta^{t,s}\|^4 \leq C.\]
\end{lem}

\begin{proof}
By \eqref{eq:pRf-volterra},
\[\norm{R_f^{r}(\theta^{[r]})}^4 \leq \prn{\int_0^r \pRf(r-s)\norm{\theta^s}^2\d
s }^2\lesssim \int_0^r \norm{\theta^s}^4\d s.\]
Then from the definition \eqref{eq:def-cont-theta}, we have
\begin{align*}
    \norm{\theta^t}^4&\lesssim
\norm{\theta^0}^4+\int_0^t\Big(1+\norm{\theta^r}^4+\norm{R_f^{r}(\theta^{[r]})}^4+\norm{\theta^*}^4\Big)\d r+\norm{u^t}^4\\
    &\lesssim \norm{\theta^0}^4+\int_0^t
\Big(1+\norm{\theta^r}^4+\norm{\theta^*}^4\Big)\d r+\norm{u^t}^4.
\end{align*}
Since $\E\norm{u^t}^4 \lesssim \prn{\E\norm{u^t}^2}^2 \lesssim \pCf(T)^2 \lesssim 1$,
taking expectations and applying Gr\"onwall's inequality shows
the bound for $\E\|\theta^t\|^4$. A similar argument shows the bound
for $\E\|r_\theta^{t,s}\|^4$.
\end{proof}

\begin{lem}\label{lem:contract-Ct}
Fix any $T,C_0>0$. For $i=1,2$, 
suppose $X_i=(C_{f,i},R_{f,i},R_{f,i}^*,\Gamma_i) \in \mathcal{S}_\xi(T,C_0)$,
and let $(C_{\theta,i},R_{\theta,i})=\trsfrmA(X_i)$.
Then there exist constants $C \equiv C(T,C_0)>0$ and $\lambda_0 \equiv
\lambda_0(T,C_0)>0$ such that for any $\lambda>\lambda_0$,
\begin{align}
\lbddst{C_{\theta,1}}{C_{\theta,2}} & \leq C\,\lbddst
{C_{f,1}}{C_{f,2}}\notag\\
&\hspace{0.2in}+\frac{C}{\lambda}\prn{\lbddst{\Gamma_1} {\Gamma_2} +
\lbddst{R_{f,1}}{R_{f,2}} + \lbddst{R_{f,1}^*}{R_{f,2}^*}},
\label{eq:contractCtheta}\\
    \lbddst{R_{\theta,1}}{R_{\theta,2}} &\leq
\frac{C}{\lambda}\lbddst{X_1}{X_2}.\label{eq:contractRtheta}
\end{align}
\end{lem}
\begin{proof}
By definition of $\lbddst{\cdot}{\cdot}$,
there exists a coupling $\{u_1^t,u_2^t\}_{t \in [0,T]}$ of $\GP(0,C_{f,1})$ and
$\GP(0,C_{f,2})$ such that
\begin{align*}
\sup_{t \in [0, T]} e^{-4\lambda t} \big(\mathbb{E}
\norm{u_1^t-u_2^t}^4\big)^{1/4}
\leq 2\,\lbddst{C_{f,1}}{C_{f,2}}^4.
\end{align*}
For $i=1,2$, define $\theta_i^t$ via \eqref{eq:def-cont-theta} using
$X_i,\{u_i^t\}_{t \in [0,T]}$. Then
$\E \|\theta^t_1 - \theta^t_2\|^4 \lesssim
\int_0^t((\mathrm{I})_r+(\mathrm{II})_r)\d r +(\mathrm{III})$ where 
\begin{align*}
    (\mathrm{I})_r &=\E\left\|g(\theta_1^r)-g(\theta_2^r)+\gamma
\Gamma_1^r(\theta_1^r-\theta_2^r)+\gamma
R_{f,1}^r(\theta_1^{[r]}-\theta_2^{[r]})\right\|^4\\
    &\lesssim \E\norm{\theta_1^r-\theta_2^r}^4 + \left(\int_0^r \Phi_{\Rf}(r-r')
\E\norm{\theta^{r'}_1 - \theta^{r'}_2}^2 \d r'\right)^2\\
    &\lesssim e^{4\lambda r}e^{-4\lambda r}\E\norm{\theta_1^r-\theta_2^r}^4 +\int_0^r
e^{4\lambda r'}e^{-4\lambda r'}\E\norm{\theta^{r'}_1 - \theta^{r'}_2}^4 \d r'\\
    &\lesssim e^{4\lambda r} \sup_{t\in[0,T]}e^{-4\lambda t}
\E\norm{\theta_1^t-\theta_2^t}^4,\\
    (\mathrm{II})_r&=\E\left\|(\Gamma_1^r-\Gamma_2^r)\theta_2^r
+(R_{f,1}^r-R_{f,2}^r)\theta_2^{[r]}+
(R_{f,1}^{r,*}-R_{f,2}^{r,*})\theta^*\right\|^4\\
    &\lesssim
\norm{\Gamma_1^r-\Gamma_2^r}^4\E\|\theta_2^r\|^4+\normop{{R_{f,1}^r-R_{f,2}^r}}^4\E\|\theta_2^{[r]}\|_{L^4}^4 +
\norm{R_{f,1}^{r,*}-R_{f,2}^{r,*}}^4\E\|\theta^*\|^4\\
&\overset{(*)}\lesssim e^{4\lambda r}\Big(\lbddst{\Gamma_1}{\Gamma_2}^4
+\lbddst{R_{f,1}}{R_{f,2}}^4+\lbddst{R_{f,1}^*}{R_{f,2}^*}^4\Big),\\
(\mathrm{III})&=\E \norm{u_1^t-u_2^t}^4
\leq e^{4\lambda t} \cdot 2\lbddst{C_{f,1}}{C_{f,2}}^4.
\end{align*}
For $(*)$, we have used Lemma \ref{lem:exist-unique-theta} to bound
$\E\|\theta^t_2\|^4,\E\|\theta_2^{[r]}\|_{L^4}^4 \lesssim 1$, and applied the
definitions of $\lbddst{\cdot}{\cdot}$.
Combining these bounds, multiplying by $e^{-4\lambda t}$, and taking the
supremum of over $t \in [0,T]$, we get
\begin{align*}
    &\sup_{t\in[0,T]}e^{-4\lambda t}\E\norm{\theta_1^t-\theta_2^t}^4\\
    &\lesssim \sup_{t\in[0,T]}e^{-4\lambda t}\E\norm{\theta_1^t-\theta_2^t}^4\cdot
\sup_{t\in[0,T]}e^{-4\lambda t} \int_0^t e^{4\lambda r}\d r\\
    &\hspace{0.5in}+
\Big(\lbddst{\Gamma_1}{\Gamma_2}^4
+\lbddst{R_{f,1}}{R_{f,2}}^4+\lbddst{R_{f,1}^*}{R_{f,2}^*}^4\Big)
\sup_{t\in[0,T]}e^{-4\lambda t}\int_0^t e^{4\lambda r} \d r\\
&\hspace{0.5in}+ \lbddst{C_{f,1}}{C_{f,2}}^4\\
&\lesssim \frac{1}{\lambda}\left(
\sup_{t\in[0,T]}e^{-4\lambda t}\E\norm{\theta_1^t-\theta_2^t}^4
+\lbddst{\Gamma_1}{\Gamma_2}^4
+\lbddst{R_{f,1}}{R_{f,2}}^4+\lbddst{R_{f,1}^*}{R_{f,2}^*}^4\right)\\
&\hspace{0.5in}+\lbddst{C_{f,1}}{C_{f,2}}^4.
\end{align*}
Thus for any $\lambda>\lambda_0(T,C_0)$ large enough, we may rearrange to get
\begin{align}
\sup_{t\in[0,T]}e^{-4\lambda t}\E\norm{\theta_1^t-\theta_2^t}^4 &\lesssim
\lbddst{C_{f,1}}{C_{f,2}}^4\label{eq:theta4diffbound}\\
&\hspace{0.2in}+\frac{1}{\lambda}\prn{\lbddst{\Gamma_1}{\Gamma_2}^4 +
\lbddst{R_{f,1}}{R_{f,2}}^4 + \lbddst{R_{f,1}^*}{R_{f,2}^*}^4}.\notag
\end{align}
Letting $\{w_1^t,w_2^t\}_{t \in [0,T] \cup \{*\}}$ be a jointly Gaussian process
having the same covariance as $\{\theta_1^t,\theta_2^t\}_{t \in [0,T] \cup
\{*\}}$, this gives a coupling of $C_{\theta,1},C_{\theta,2}$, so
\begin{align*}
\lbddst{C_{\theta,1}}{C_{\theta,2}}^4
&\leq \sup_{t \in [0, T]} e^{-4\lambda t} 
\E {\norm{w^t_1 - w^t_2}^4}
\lesssim \sup_{t \in [0, T]} e^{-4\lambda t} 
(\E {\norm{w^t_1 - w^t_2}^2})^2\\
&= \sup_{t \in [0, T]} e^{-4\lambda t} 
(\E{\norm{\theta^t_1 - \theta^t_2}^2})^2
\lesssim \sup_{t \in [0, T]} e^{-4\lambda t} 
\E{\norm{\theta^t_1 - \theta^t_2}^4}.
\end{align*}
Combining the above shows \eqref{eq:contractCtheta}.

For the analysis of $R_\theta$, we have similarly
from \eqref{eq:def-cont-deriv-t} that
$\E\|r_{\theta,1}^{t,s}-r_{\theta,2}^{t,s}\|^2
\lesssim \int_s^t ((\mathrm{I})_r+(\mathrm{II})_r)\d r$ where
\begin{align*}
(\mathrm{I})_r&=\E\Big\|(\gamma \Gamma_1^r+\D g(\theta_1^r))
(r_{\theta,1}^{r,s}-r_{\theta,2}^{r,s})
+\gamma R_f^r(r_{\theta,1}^{[r],s}-r_{\theta,2}^{[r],s})\Big\|^2\\
&\lesssim \E\|r_{\theta,1}^{r,s}-r_{\theta,2}^{r,s}\|^2
+\int_s^r \Phi_{R_f}(r-r')\E\|r_{\theta,1}^{r',s}-r_{\theta,2}^{r',s}\|^2\d
r'\\
&\lesssim e^{2\lambda r} \sup_{t \in [s,T]}
e^{-2\lambda t}\E\|r_{\theta,1}^{t,s}-r_{\theta,2}^{t,s}\|^2,\\
(\mathrm{II})_r&=\E\Big\|\gamma(\Gamma^t_1 - \Gamma^t_2)r_{\theta,2}^{r,s}
+(\D g(\theta^r_1)-\D g(\theta^r_2))r_{\theta,2}^{r,s}
+ \gamma (R_{f,1}^r-R_{f,2}^r)(r_{\theta,2}^{[r],s})\Big\|^2\\
&\lesssim \|\Gamma_1^r-\Gamma_2^r\|_\op^2 \E\|r_{\theta,2}^{r,s}\|^2
+\E[\|\D g(\theta_1)-\D g(\theta_2)\|_\op^2\|r_{\theta,2}^{r,s}\|^2]
+\normop{R_{f,1}^r-R_{f,2}^r}^2\E\|r_{\theta,2}^{[r],s}\|_{L^4}^2\\
&\lesssim \|\Gamma_1^r-\Gamma_2^r\|_\op^2 \E\|r_{\theta,2}^{r,s}\|^2
+(\E\|\theta_1^r-\theta_2^r\|^4)^{1/2}(\E\|r_{\theta,2}^{r,s}\|^4)^{1/2}
+\normop{R_{f,1}^r-R_{f,2}^r}^2\E\|r_{\theta,2}^{[r],s}\|_{L^4}^2\\
&\lesssim e^{2\lambda r}\lbddst{X_1}{X_2}^2.
\end{align*}
Here, $\|r_\theta^{[r],s}\|_{L^4}$ denotes
$(\int_0^r \|r_\theta^{r',s}\|^4 \d r')^{1/4}$,
and the last step using again Lemma \ref{lem:exist-unique-theta} 
to bound $\E\|r_{\theta,2}^{[r],s}\|_{L^4}^4 \lesssim 1$, together with the
bound $\sup_{t \in [0,T]} e^{-2\lambda t}(\E\|\theta_1^t-\theta_2^t\|^4)^{1/2}
\lesssim \lbddst{X_1}{X_2}^2$ implied by \eqref{eq:theta4diffbound}.
Multiplying by $e^{-\lambda t}$ and taking the supremum over $0 \leq s \leq t
\leq T$,
\begin{align*}
&\sup_{0 \leq s \leq t \leq T} e^{-\lambda t}
\E\|r_{\theta,1}^{t,s}-r_{\theta,2}^{t,s}\|^2\\
&\lesssim \left(\sup_{0 \leq s \leq t \leq T} e^{-\lambda t}
\E\|r_{\theta,1}^{t,s}-r_{\theta,2}^{t,s}\|^2
+\lbddst{X_1}{X_2}^2\right)
\sup_{0 \leq s \leq t \leq T} e^{-\lambda t} \int_s^t e^{\lambda r}\d r\\
&\lesssim \frac{1}{\lambda}\left(\sup_{0 \leq s \leq t \leq T} e^{-\lambda t}
\E\|r_{\theta,1}^{t,s}-r_{\theta,2}^{t,s}\|^2
+\lbddst{X_1}{X_2}^2\right),
\end{align*}
and rearranging shows \eqref{eq:contractRtheta}.
\end{proof}

We next derive the contraction properties for the transformation $\trsfrmB$.

\begin{lem} \label{lem:contract-xi}
Fix any $T,C_0>0$.
Let $(\Omega,\cF,\{\cF\}_{t \in [0,T]},\P)$ be the filtered probability
space of Theorem \ref{thm:uniqueness-existence}, and
let $\{z^t\}_{t \in [0,T]}$ be either the Poisson process \eqref{eq:poisson-z}
or Gaussian process \eqref{eq:gaussian-z}.

For $i=1,2$, suppose $Y_i=(C_{\theta,i},R_{\theta,i}) \in
{\mathcal{S}}_\theta(T,C_0)$, and let $\{w_1^t,w_2^t\}_{t \in [0,T] \cup \{*\}}$
be any coupling of $w_1 \sim \GP(0,C_{\theta,1})$
and $w_2 \sim \GP(0,C_{\theta,2})$. Define $\xi_i^t$, 
$r_{f,i}^{t,*}$, and $r_{f,i}^t(x^{[t]})$,
via \eqref{eq:def-cont-xi}, \eqref{eq:def-cont-deriv-xi-star}, and
\eqref{eq:def-cont-resp-xi} using
$Y_i,\{w_i^t\}_{t \in [0,T]}$.
Then there exists a constant $C \equiv C(T,C_0)>0$ such that for all $t \in
[0,T]$ and any process $x \in L^4([0,T],\R^k)$,
\begin{align}
\E\left\|\xi^t_1 - \xi^t_2\right\|^4 &\leq C
\left(\sup_{r \in [0,t]} \|R_{\theta,1}^{t,r}-R_{\theta,2}^{t,r}\|^4
+\E\|w_1^t-w_2^t\|^4\right),\label{eq:contract-xi}\\
\E\|r_{f,1}^{t}(x^{[t]})-r_{f,2}^{t}(x^{[t]})\|^2 &\leq C
\Bigg(\sup_{r \in [0,t]} \|R_{\theta,1}^{t,r}-R_{\theta,2}^{t,r}\|^2
+(\E\|w_1^t-w_2^t\|^4)^{1/2}\notag\\
&\hspace{1in}
+\Big(\int_0^t \E\|w_1^r-w_2^r\|^4 \d r\Big)^{1/2}\Bigg)
\times \left(\int_0^t \|x^r\|^4 \d r\right)^{1/2},\label{eq:contract-rfxi}\\
\E\norm{r_{f,1}^{t,*}-r_{f,2}^{t,*}}^2
&\leq C\left(\sup_{r \in [0,t]} \|R_{\theta,1}^{t,r}-R_{\theta,2}^{t,r}\|^2
+(\E\|w_1^t-w_2^t\|^4)^{1/2}\right).\label{eq:contract-rfxi-star}
\end{align}
\end{lem}
\begin{proof}
For $i=1,2$, Gaussianity and Assumption~\ref{asp:model}(c) give
\[
\sup_{t\in[0,T]}\E\|w_i^t\|^4\leq 3\Phi_{C_\theta}(T)^2,
\qquad \E\|w^*\|^4+\E|\eps|^4<\infty.
\]
Then by the linear growth of $f$ in Assumption~\ref{asp:lipschitz},
definition of $\xi^t$ in
\eqref{eq:def-cont-xi}, Burkholder-Davis-Gundy inequality of
Lemma~\ref{lem:burkholder} with $p=4$,
and H\"older's inequality, we obtain
\begin{align*}
\E\|\xi_i^t\|^4
&\lesssim \E\|w_i^t\|^4
+\int_0^t\|R_{\theta,i}^{t,r}\|^4
\E\|f(\xi_i^r,w^*,\eps)\|^4\,\d r
\lesssim 1+\int_0^t\E\|\xi_i^r\|^4\,\d r,
\end{align*}
where we used $\|R_{\theta,i}^{t,r}\|^4\leq\Phi_{R_\theta}(T)^2 \lesssim 1$.
Thus Gr\"onwall's inequality and the linear growth bound for $f$ yield
\begin{equation}\label{eq:xi-fourth-moment}
\sup_{t\in[0,T]}
\E\|\xi_i^t\|^4 \lesssim 1,
\qquad \sup_{t\in[0,T]}
\E\|f(\xi_i^t,w^*,\eps)\|^4 \lesssim 1.
\end{equation}
Applying again the definition of $\xi_i^t$ and now
using \eqref{eq:xi-fourth-moment} and Lipschitz continuity of $f$,
\begin{align*}
\E\left\|\xi^t_1 - \xi^t_2\right\|^4 &\lesssim \E\left\|\int_0^t
\frac{\bar\eta^r}{\bar\kappa}
\prn{R_{\theta,1}^{t,r}f(\xi_1^{r-},w^*,\eps)-R_{\theta,2}^{t,r}f(\xi_2^{r-},w^*,\eps)}\d z^r\right\|^4 + \E\left\|w^t_1 - w^t_2\right\|^4\nonumber\\
  &\lesssim \int_0^t
\E\left\|R_{\theta,1}^{t,r}f(\xi_1^{r},w^*,\eps)-R_{\theta,2}^{t,r}f(\xi_2^r,w^*,\eps)\right\|^4 \d r+\E\left\|w^t_1 - w^t_2\right\|^4\\
&\lesssim \int_0^t\|R_{\theta,1}^{t,r}-R_{\theta,2}^{t,r}\|^4
\E\|f(\xi_2^r,w^*,\eps)\|^4\,\d r+\int_0^t\E\|\xi_1^r-\xi_2^r\|^4\,\d r
+\E\|w_1^t-w_2^t\|^4\\
&\lesssim \sup_{r \in [0,t]} \|R_{\theta,1}^{t,r}-R_{\theta,2}^{t,r}\|^4
+\int_0^t \E\|\xi_1^r-\xi_2^r\|^4\d r+\E\left\|w^t_1 - w^t_2\right\|^4.
\end{align*}
Then Gr\"onwall's inequality shows \eqref{eq:contract-xi}.

For \eqref{eq:contract-rfxi}, applying the definition of $r_{f,i}^t(x^{[t]})$
in \eqref{eq:def-cont-resp-xi} and Lipschitz continuity of $f$, we have
$\E\|r_{f,1}^{t}(x^{[t]}) - r_{f,2}^{t}(x^{[t]})\|^2 \lesssim \nI+\nII+\nIII$,
where
\begin{align*}
\nI&=\E\left\|\int_{0}^{t}
\frac{\bar\eta^r}{\bar\kappa}
  \prn{R_{\theta,1}^{t,r}(r_{f,1}^{r-}(x^{[r]})-r_{f,2}^{r-}(x^{[r]}))+\prn{R_{\theta,1}^{t,r}-R_{\theta,2}^{t,r}}r_{f,2}^{r-}(x^{[r]})}
\d z^r\right\|^2\nonumber\\
\nII&=\E\left\|\int_{0}^{t}
\frac{\bar\eta^r}{\bar\kappa} \left(R_{\theta,1}^{t,r}(\D_\xi f(\xi^{r-}_1,w^*,\eps)-\D_\xi f(\xi^{r-}_2,w^*,\eps))\right).\right.\nonumber\\
&\hspace{1in}\left.\left.+ \prn{R_{\theta,1}^{t,r}-R_{\theta,2}^{t,r}}\D_\xi f(\xi^{r-}_2,w^*,\eps)\right)
x^{r-} \d z^r\right\|^2\nonumber\\
\nIII&=\E\left\|\prn{\D_\xi f(\xi^t_1,w^*,\eps) - \D_\xi f(\xi^t_2,w^*,\eps)}
\int_{0}^{t}
  \frac{\bar\eta^r}{\bar\kappa} R_{\theta,2}^{t,r}\bigl(r_{f,2}^{r-}(x^{[r]})+\D_\xi f(\xi^{r-}_2,w^*,\eps)x^{r-}\bigr) \d z^r\right\|^2.
\end{align*}
Using \eqref{eq:basicdzrbound} and \eqref{eq:Erfsquaredbound},
the first term can be bounded as
\begin{align*}
\nI &\lesssim
\int_0^t\prn{\pRt(t-r)^2\E\big\|r_{f,1}^{r}(x^{[r]})-r_{f,2}^{r}(x^{[r]})\big\|^2
+ \norm{R_{\theta,1}^{t,r}-R_{\theta,2}^{t,r}}^2\E\norm{r_{f,2}^{r}(x^{[r]})}^2 }\d r\nonumber\\
&\lesssim \int_0^t \E\big\|r_{f,1}^{r}(x^{[r]})-r_{f,2}^{r}(x^{[r]})\big\|^2\d r
+\int_0^t \norm{R_{\theta,1}^{t,r}-R_{\theta,2}^{t,r}}^2 \int_0^r
\pRf(r-r')^2\norm{x^{r'}}^2\d r'\d r\nonumber\\
&\lesssim \int_0^t \E\big\|r_{f,1}^{r}(x^{[r]})-r_{f,2}^{r}(x^{[r]})\big\|^2\d
r+\sup_{r\in[0,t]}\norm{R_{\theta,1}^{t,r}-R_{\theta,2}^{t,r}}^2\Big(\int_0^t
\norm{x^r}^4\d r\Big)^{1/2}.
\end{align*}
Similarly, the second term is bounded as
\begin{align*}
\nII &\lesssim \int_0^t\E\norm{\xi_1^{r}-\xi_2^{r}}^2 \cdot \norm{x^r}^2\d r +
\sup_{r\in[0,t]}\norm{R_{\theta,1}^{t,r}-R_{\theta,2}^{t,r}}^2\int_0^t
\norm{x^r}^2\d r\\
&\lesssim 
\Big(\int_0^t \E\norm{\xi_1^{r}-\xi_2^{r}}^4 \d r\Big)^{1/2}
\Big(\int_0^t \norm{x^r}^4\d r\Big)^{1/2} +
\sup_{r\in[0,t]}\norm{R_{\theta,1}^{t,r}-R_{\theta,2}^{t,r}}^2
\Big(\int_0^t \norm{x^r}^4\d r\Big)^{1/2}
\end{align*}
The third term is bounded as
\begin{align*}
\nIII &\lesssim \E\brk{\norm{\xi_1^{t} - \xi_2^{t}}^2 \cdot \left\|
\int_{0}^{t} \frac{\bar\eta^r}{\bar\kappa} R_{\theta,2}^{t,r}\bigl(r_{f,2}^{r-}(x^{[r]})+\D_\xi
f(\xi^{r-}_2,w^*,\eps)x^{r-}\bigr)
\d z^r\right\|^2}\nonumber\\
&\leq (\E \norm{\xi_1^{t} - \xi_2^{t}}^4)^{1/2}\left(\E\left\|
\int_{0}^{t} \frac{\bar\eta^r}{\bar\kappa} R_{\theta,2}^{t,r}\bigl(r_{f,2}^{r-}(x^{[r]})+\D_\xi
f(\xi^{r-}_2,w^*,\eps)x^{r-}\bigr)
\d z^r\right\|^4\right)^{1/2}.
\end{align*}
Applying Lemma \ref{lem:burkholder} to bound the second factor, we have
\[\nIII \lesssim
(\E \norm{\xi_1^{t} - \xi_2^{t}}^4)^{1/2}\cdot \left(\int_0^t
\prn{\E\norm{r_{f,2}^{r}(x^{[r]})}^4 + \norm{x^{r}}^4} \d
r\right)^{1/2}.\]
Applying Lemma \ref{lem:burkholder} again to the definition of
$r_f^t(\cdot)$ shows
\[\E\norm{r_f^{t}(x^{[t]})}^4 \lesssim \int_0^t
\left(\E\norm{r_f^{r}(x^{[r]})}^4+\norm{x^r}^4\right)\d r,\]
so Gr\"onwall's inequality implies
$\E\norm{r_f^{t}(x^{[t]})}^4 \lesssim \int_{0}^{t}\norm{x^r}^4\d r$.
Therefore
\begin{align*}
\nIII &\leq  (\E\norm{\xi_1^{t} - \xi_2^{t}}^4)^{1/2}\left(\int_0^t
\norm{x^{r}}^4 \d r\right)^{1/2}.
\end{align*}
Combining the above bounds, we have
\begin{align*}
&\E\big\|r_{f,1}^{t}(x^{[t]}) - r_{f,2}^{t}(x^{[t]})\big\|^2
\lesssim
\int_0^t \E\big\|r_{f,1}^{r}(x^{[r]}) - r_{f,2}^{r}(x^{[r]})\|^2\d r \\
&\hspace{0.2in}+\left(\sup_{r\in[0,t]}\norm{R_{\theta,1}^{t,r}-R_{\theta,2}^{t,r}}^2+(\E\norm{\xi_1^t-\xi_2^t}^4)^{1/2}
+\Big(\int_0^t \E\norm{\xi_1^r-\xi_2^r}^4\d r\Big)^{1/2}\right)
\left(\int_{0}^{t}\norm{x^{r}}^{4}\d r \right)^{1/2}.
\end{align*}
Applying again Gr\"onwall's inequality and \eqref{eq:contract-xi} to bound
$\E\norm{\xi_1^r-\xi_2^r}^4$, we obtain \eqref{eq:contract-rfxi}.

Finally, from the form of $r_{f,i}^{t,*}$, we have
$\E\norm{r_{f,1}^{t,*}-r_{f,2}^{t,*}}^2 \lesssim \nI+\nII+\nIII$,
where
\begin{align*}
\nI&=\E\left\|\int_0^t \frac{\bar\eta^r}{\bar\kappa}
\prn{R_{\theta,1}^{t,r}r_{f,1}^{r-,*}-R_{\theta,2}^{t,r}r_{f,2}^{r-,*}}\d z^r\right\|^2, \\
\nII&= \E\left\|\Big(\D_\xi f(\xi_1^t,w^*,\eps)-\D_\xi f(\xi_2^t,w^*,\eps)\Big)
\int_0^t \frac{\bar\eta^r}{\bar\kappa}
R_{\theta,1}^{t,r}r_{f,1}^{r-,*}\d z^r\right\|^2,\\
\nIII&=\E\norm{\D_{w^*} f(\xi_1^t,w^*,\eps)-\D_{w^*} f(\xi_2^t,w^*,\eps)}^2.
\end{align*}
By similar arguments as above, we may bound each term to obtain
\[\E\norm{r_{f,1}^{t,*}-r_{f,2}^{t,*}}^2
\lesssim \int_0^t \E\norm{r_{f,1}^{r,*}-r_{f,2}^{r,*}}^2 \d r
+\sup_{r \in [0,t]}\|R_{\theta,1}^{t,r}-R_{\theta,2}^{t,r}\|^2
+(\E\|\xi_1^t-\xi_2^t\|^4)^{1/2}.\]
Applying Gr\"onwall's inequality and \eqref{eq:contract-xi} to bound
$\E\norm{\xi_1^t-\xi_2^t}^4$ shows \eqref{eq:contract-rfxi-star}.
\end{proof}

\begin{lem}\label{lem:contract-fsys}
Fix any $T,C_0>0$. For $i=1,2$, suppose
$Y_i=(C_{\theta,i},R_{\theta,i}) \in {\mathcal{S}}_\theta(T,C_0)$, and
let $(C_{f,i},R_{f,i},R_{f,i}^*,\Gamma_i)=\trsfrmB(Y_i)$. Then there exist
constants $C \equiv C(T,C_0)>0$ and $\lambda_0 \equiv \lambda_0(T,C_0)>0$
such that for any $\lambda>\lambda_0$,
\[\begin{gathered}
\lbddst{C_{f,1}}{C_{f,2}} \leq \frac{C}{\sqrt{\lambda}}\,\lbddst{Y_1}{Y_2}, \\
\lbddst{R_{f,1}}{R_{f,2}},\lbddst{R_{f,1}^*}{R_{f,2}^*},
\lbddst{\Gamma_1}{\Gamma_2} \leq C\,\lbddst{Y_1}{Y_2}.
\end{gathered}\]
\end{lem}
\begin{proof}
Let $\{w^t_1,w^t_2\}_{t \in [0,T] \cup \{*\}}$ be a coupling
of $\GP(0,C_{\theta,1})$ and $\GP(0,C_{\theta,2})$ that satisfies
\begin{equation}\label{eq:wcoupling}
\sup_{t \in [0, T]} e^{-\lambda t} (\mathbb{E}\norm{w^t_1 -
w^t_2}^4)^{1/4}  \leq 2\,\lbddst{C_{\theta,1}}{C_{\theta,2}}.
\end{equation}
Then by \eqref{eq:contract-xi} and the definition of $\lbddst{\cdot}{\cdot}$,
\[e^{-4\lambda t}
\E\|\xi_1^t-\xi_2^t\|^4 \lesssim 
\lbddst{R_{\theta,1}}{R_{\theta,2}}^4+\lbddst{C_{\theta,1}}{C_{\theta,2}}^4
\lesssim \lbddst{Y_1}{Y_2}^4.\]
For $i=1,2$, define
\[h_i^t=\int_0^t \frac{\bar\eta^r}{\bar\kappa}f(\xi_i^{r-},w^*,\eps) \d z^r.\]
Then by the Lipschitz property of $f$ and \eqref{eq:basicdzrbound},
\[\E\|h_1^t-h_2^t\|^2
\lesssim \int_0^t \E\|\xi_1^r-\xi_2^r\|^2 \d r
\leq \sup_{r \in [0,t]}
e^{-2\lambda r}\,\E\|\xi_1^r-\xi_2^r\|^2
\cdot \int_0^t e^{2\lambda r} \d r
\lesssim \lbddst{Y_1}{Y_2}^2 \cdot \int_0^t e^{2\lambda r} \d r,\]
so
\[\sup_{t \in [0,T]} e^{-2\lambda t}
\E\|h_1^t-h_2^t\|^2 \lesssim \frac{1}{\lambda}\,\lbddst{Y_1}{Y_2}^2.\]
Letting $\{u_1^t,u_2^t\}_{t \in [0,T]}$ be a jointly Gaussian process with the
same covariance as $\{h_1^t,h_2^t\}_{t \in [0,T]}$, this gives a coupling of
$\GP(0,C_{f,1})$ and $\GP(0,C_{f,2})$, so
\begin{align*}
\lbddst{C_{f,1}}{C_{f,2}}^4
&\leq \sup_{t \in [0, T]} e^{-4\lambda t} 
\E {\norm{u^t_1 - u^t_2}^4}
\lesssim \sup_{t \in [0, T]} e^{-4\lambda t} 
(\E {\norm{u^t_1 - u^t_2}^2})^2\\
&=\sup_{t \in [0, T]} e^{-4\lambda t} 
(\E{\norm{h^t_1 - h^t_2}^2})^2
\lesssim \frac{1}{\lambda^2}\,\lbddst{Y_1}{Y_2}^4.
\end{align*}

For $R_f$, note that by \eqref{eq:contract-rfxi},
\begin{align*}
&\normop{R_{f,1}^t-R_{f,2}^t}^2
=\sup_{x:\|x\|_{L^4} \leq 1}
\|\E r_{f,1}^t(x^{[t]})-\E r_{f,2}^t(x^{[t]})\|^2\\
&\lesssim \sup_{r \in [0,t]}
\|R_{\theta,1}^{t,r}-R_{\theta,2}^{t,r}\|^2
+(\E\|w_1^t-w_2^t\|^4)^{1/2}
+\Big(\int_0^t \E\|w_1^r-w_2^r\|^4 \d r\Big)^{1/2}\\
&\lesssim \sup_{r \in [0,t]}
\|R_{\theta,1}^{t,r}-R_{\theta,2}^{t,r}\|^2
+\sup_{r \in [0,t]} 
(e^{-4\lambda r}\E \|w_1^r-w_2^r\|^4)^{1/2}
\left(e^{2\lambda t}+\left(\int_0^t e^{4\lambda r}\d r\right)^{1/2}\right).
\end{align*}
Multiplying by $e^{-2\lambda t}$ and applying \eqref{eq:wcoupling} shows
\[\lbddst{R_{f,1}}{R_{f,2}}^2
\lesssim
\lbddst{R_{\theta,1}}{R_{\theta,2}}^2+\lbddst{C_{\theta,1}}{C_{\theta,2}}^2
\lesssim \lbddst{Y_1}{Y_2}^2.\]
A similar argument using \eqref{eq:contract-rfxi-star} 
and $\|\Gamma_1^t-\Gamma_2^t\| \lesssim \E\|\xi_1^t-\xi_2^t\|$ shows
$\lbddst{R_{f,1}^*}{R_{f,2}^*} \lesssim \lbddst{Y_1}{Y_2}$
and $\lbddst{\Gamma_1}{\Gamma_2} \lesssim \lbddst{Y_1}{Y_2}$.
\end{proof}


\begin{proof}[Proof of Theorem \ref{thm:uniqueness-existence}(b)]
Fix $T>0$, and any constant $C_0>0$ large enough.
By Lemmas \ref{lem:solution-in-space} and \ref{lem:solution-in-space-2},
\[\trsfrm:=\trsfrmA \circ \trsfrmB\]
defines a mapping from $\cS_\theta^\text{cont} \equiv
\cS_\theta^\text{cont}(T,C_0)$ to itself. For $i=1,2$, consider any
$Y_i=(C_{\theta,i},R_{\theta,i}) \in {\mathcal{S}}_\theta^\text{cont}$,
let $X_i=(C_{f,i},R_{f,i},R_{f,i}^*,\Gamma_i)=\trsfrmB(Y_i)$.
Then by Lemmas \ref{lem:contract-Ct} and \ref{lem:contract-fsys},
for any large enough $\lambda>0$,
\begin{align}\label{eq:T-contraction}
    &\lbddst{\trsfrm(Y_1)}{\trsfrm(Y_2)}
\lesssim 
\lbddst{C_{f,1}}{C_{f,2}}+\frac{1}{\lambda}\lbddst{X_1}{X_2}
\lesssim \frac{1}{\sqrt{\lambda}}\lbddst{Y_1}{Y_2}.
\end{align}
Thus for large enough $\lambda>0$, we have
$\lbddst{\trsfrm(Y_1)}{\trsfrm(Y_2)} \leq \frac{1}{2}\,\lbddst{Y_1}{Y_2}$,
implying that $\trsfrm:\cS_\theta^\text{cont} \to \cS_\theta^\text{cont}$
is contractive in the metric $\lbddst{\cdot}{\cdot}$.

It is clear that $\lbddst{R_{\theta,1}}{R_{\theta,2}}$ defines a complete metric
on the space of Borel-measurable $\R^{k \times k}$-valued
processes $\{R_\theta^{t,s}\}_{t,s \in [0,T]}$, and the conditions in
Definition \ref{def:S} for
$\{R_\theta^{t,s}\}_{t,s \in [0,T]}$ to belong to $\cS_\theta^{\text{cont}}$ are
closed under $\lbddst{R_{\theta,1}}{R_{\theta,2}}$. For
$\lbddst{C_{\theta,1}}{C_{\theta,2}}$, define also
\[\|C_{\theta,1}-C_{\theta,2}\|_\infty
=\sup_{t,s \in [0,T] \cup \{*\}} |C_{\theta,1}^{t,s}-C_{\theta,2}^{t,s}|.\]
We have by Cauchy-Schwarz that
for any $C_{\theta,1},C_{\theta,2}$ satisfying the conditions of
$\cS_\theta^{\text{cont}}$,
\begin{align*}
\|C_{\theta,1}-C_{\theta,2}\|_\infty &\leq
\inf_{w_1,w_2 \sim \gamma \in \Gamma(C_{\theta,1},C_{\theta,2})}
\sqrt{\E \brk{\norm{w_1^t}^2} \cdot \E \brk{\norm{w_1^s - w_2^s}^2}} + \sqrt{\E
\brk{\norm{w_2^s}^2} \cdot \E \brk{\norm{w_1^t - w_2^t}^2}} \nonumber \\
	& \leq 2\sqrt{\pCt(T)} e^{\lambda T} \cdot
\lbddst{R_{\theta,1}}{R_{\theta,2}}.
\end{align*}
Conversely, by the condition \eqref{eq:Cthetacontinuity} and
\cite[Lemma D.1]{fan2025dynamical2}, there exists a jointly
Gaussian coupling $(w_1,w_2) \sim \gamma \in \Gamma(C_{\theta,1},C_{\theta,2})$
for which
\begin{align*}
\lbddst{C_{\theta,1}}{C_{\theta,2}}
&\leq \sup_{t \in [0,T]} e^{-\lambda t}(\E\|w_1^t-w_2^t\|^4)^{1/4}
\leq C\sup_{t \in [0,T]} e^{-\lambda t}(\E\|w_1^t-w_2^t\|^2)^{1/2}\\
&\leq C'\Big(\sqrt{\|C_{\theta,1}-C_{\theta,2}\|_\infty}
+\|C_{\theta,1}-C_{\theta,2}\|_\infty\Big)^{1/2},
\end{align*}
where $C,C'>0$ are some constants depending on $T,C_0,k$. Thus, the metrics
$\lbddst{C_{\theta,1}}{C_{\theta,2}}$ and
$\|C_{\theta,1}-C_{\theta,2}\|_\infty$ are uniformly equivalent over the space
of covariance kernels $\{C_\theta^{t,s}\}_{t,s \in [0,T] \cup \{*\}}$
satisfying the conditions of $\cS_\theta^{\text{cont}}$. Since this space is
complete under $\|C_{\theta,1}-C_{\theta,2}\|_\infty$, it is also complete under
$\lbddst{C_{\theta,1}}{C_{\theta,2}}$.

Then $\lbddst{\cdot}{\cdot}$ defines a complete metric on
$\cS_\theta^{\text{cont}}$, so by the Banach fixed-point theorem, there exists a
unique fixed point $Y \in \mathcal{S}^{\mathrm{cont}}_\theta$ such that
$\mathcal{T}(Y)=Y$. Since the image of $\mathcal{T}:\cS_\theta \to \cS_\theta$
is contained in $\cS_\theta^\text{cont}$, this fixed point is also unique in
$\cS_\theta$. Finally, this implies that for any
$C_0>0$ large enough, there is a unique fixed point
$(C_\theta,R_\theta,C_f,R_f,R_f^*,\Gamma)
\in \cS \equiv \cS(T,C_0)$ satisfying
(\ref{eq:def-cont-C-t}--\ref{eq:def-cont-Gamma}).
\end{proof}

\section{DMFT approximation in discrete time}

In this section, we introduce and analyze
a time-discretized version of the SGD and SME
processes: Fixing a discretization size $\delta>0$, let $\{\z_\delta^t\}_{t
\geq 0}$ be i.i.d.\ random vectors in $\R^n$ encapsulating the data sampling
noise, with law
\[\begin{cases}
\z_\delta^t=\bPi \bP_\delta^t & \text{for SGD},\\
\z_\delta^t \sim
\cN(\delta \bar\kappa1_n,\delta\bar\kappa\,\Id_{n\times n}) & \text{for SME}.
\end{cases}\]
In the SGD setting, $\bP_\delta^t \in \R^{\binom{n}{\kappa}}$ is a random vector
with i.i.d.\ $\Pois(\delta n^{1-\alpha} / \binom{n}{\kappa})$ entries,
and $\bPi\in\R^{n\times \binom{n}{\kappa}}$ is an incidence matrix having
entries
\[\Pi_{i,S}=\1_{i \in S} \text{ for all $i \in [n]$ and $S \subset \cS$}.\]
Note that marginally, each entry of $\bPi\bP_\delta^t$ has law
$\Pois(\binom{n-1}{\kappa-1} \cdot \delta n^{1-\alpha} /
\binom{n}{\kappa})=\Pois(\delta \kappa/n^\alpha)$,
where $\delta \kappa/n^\alpha \approx \delta\bar\kappa$
for large $n$ under Assumption \ref{asp:SGD}.

For both SGD and SME, we consider the discrete-time dynamics
\begin{align}
{\btheta}_\delta^{t+1}&={\btheta}_\delta^{t}
-\frac{\bar\eta^{t\delta}}{\bar\kappa}
\sum_{i=1}^n \x_i \otimes \Big(f(\x_i^\top \btheta_\delta^t,\x_i^\top
\btheta^*,\eps_i) \cdot z_{\delta,i}^t\Big)
-\delta\bar\eta^{t\delta} g(\btheta_\delta^t)\notag\\
&={\btheta}_\delta^{t}
-\frac{\bar\eta^{t\delta}}{\bar\kappa} \X^\top \big(f(\X \btheta_\delta^t,
\X\btheta^*,\beps) \odot \z_\delta^t\big)
- \delta\bar\eta^{t\delta} g({\btheta}_\delta^{t}).\label{eq:discretedynamics}
\end{align}
Here and throughout, $\X=[\x_1,\ldots,\x_n]^\top \in \R^{n \times d}$,
$f:\R^{n \times (k+k^*+1)} \to \R^{n \times k}$
and $g:\R^{d \times k} \to \R^{d \times k}$ denote the applications of
$f(\cdot)$ and $g(\cdot)$ row-wise, and $\odot$ denotes the row-wise
scalar product. We will write also
\[{\bxi}_\delta^{t}=\X {\btheta}_\delta^{t},
\qquad \bxi^*=\X \btheta^*.\]
We remark that these discrete-time dynamics \eqref{eq:discretedynamics}
may be understood as a variant of SGD/SME where the stochastic
gradient or drift/diffusion coefficients are only updated in (rescaled) time
increments of $\delta$, i.e.\ they are computed in \eqref{eq:SGDrescaled}
and \eqref{eq:SME} using $\btheta^{k\delta}$ instead of $\btheta^t$ if $t
\in [k\delta,(k+1)\delta)$.

We show in Section \ref{sec:discreteDMFT} that
these discrete-time dynamics are characterized by a system of discrete-time
DMFT equations as $n,d \to \infty$, and in Section
\ref{sec:deltaDMFTconvergence} that this discrete-time DMFT system converges
to the continuous-time system defined in Section \ref{sec:DMFTlimits} as
$\delta \to 0$.

\subsection{Discrete DMFT approximation via AMP}\label{sec:discreteDMFT}

Given the same variables $\theta^0,\theta^*,\eps$ as in
\eqref{eq:thetaeps}, the discrete DMFT system is defined by the primary
processes
\begin{align}
\theta^t_\delta&=
\theta^0-\sum_{r=0}^{t-1} \delta\bar\eta^{r\delta}\prn{{\gamma}\Gamma_\delta^r \theta^r_\delta+g(\theta^r_\delta)+\gamma\sum_{q=0}^{r-1}R_{f,\delta}^{r,q}
\theta^q_\delta
+ {\gamma} R_{f,\delta}^{r,*}\theta^*}+\sqrt{\gamma}u_\delta^t,
\label{eq:def-dis-theta}\\
\xi^t_\delta&={-}\sum_{r=0}^{t-1}\frac{\bar\eta^{r\delta}}{\bar\kappa}
R_{\theta,\delta}^{t,r}f(\xi_\delta^r,w^*,\eps)z^r_\delta+w^t_\delta,\label{eq:def-dis-xi}
\end{align}
and the response processes
\begin{align}
r_{\theta,\delta}^{t,s}&=\Id_k
-\sum_{r=s}^{t-1} \delta\bar\eta^{r\delta}\bigg[\left({\gamma}\Gamma_\delta^r+
\D g(\theta_\delta^r)\right)\cdot r_{\theta,\delta}^{r,s}
+{\gamma} \sum_{q=s+1}^{r-1}  R_{f,\delta}^{r,q}
r_{\theta, \delta}^{q,s}\bigg] \text{ for } s<t,\notag\\
&\hspace{1in} r_{\theta,\delta}^{t,s}=0 \text{ for } s \geq t,\label{eq:def-dis-deriv-t}\\
r_{f,\delta}^{t,*}&={-}\D_\xi f(\xi^t_\delta,w^*,\eps) \sum_{r=0}^{t-1}
\frac{\bar\eta^{r\delta}}{\bar\kappa}
R_{\theta,\delta}^{t,r}r_{f,\delta}^{r,*}z^r_\delta
+\D_{w^*} f(\xi^t_\delta,w^*,\eps),\label{eq:def-dis-deriv-xi-star}\\
    r_{f,\delta}^{t,s}&=-\D_\xi
f(\xi_\delta^{t},w^*,\eps)\prn{\sum_{r=s+1}^{t-1}\frac{\bar\eta^{r\delta}}{\bar\kappa}
R_{\theta,\delta}^{t,r} r_{f,\delta}^{r,s}z^r_\delta +
\frac{\bar\eta^{s\delta}}{\bar\kappa} R_{\theta,\delta}^{t,s}\cdot\D_\xi
f(\xi_\delta^{s},w^*,\eps)z^s_\delta} \text{ for } s<t,\notag\\
&\hspace{1in}r_{f,\delta}^{t,s}=0 \text{ for } s \geq t,\label{eq:def-dis-deriv-xi}
\end{align}
where
\[\{u_\delta^t\}_{t \geq 0} \sim \GP(0,C_{\theta,\delta}),
\qquad (\{w_\delta^t\}_{t \geq 0},w^*) \sim \GP(0,C_{f,\delta})\]
are discrete-time Gaussian processes,
\[\{z_\delta^t\}_{t \geq 0} \overset{iid}{\sim} \begin{cases}
\Pois(\delta\bar\kappa) & \text{ for SGD}, \\
\cN(\delta\bar\kappa,\delta\bar\kappa) & \text{ for SME},
\end{cases}\]
and these are independent of each other and of $\theta^0,\theta^*,\eps$.
The deterministic correlation and response kernels are defined as
expectations with respect to the above laws,
\begin{align}
C_{\theta,\delta}^{t,s}&=\E[\theta^t_\delta \otimes \theta^s_\delta] \text{ for } t,s \in [0,\infty)
\cup \{*\}\label{eq:def-dis-C-t}\\
C_{f,\delta}^{t,s}&=\E\brk{\sum_{r=0}^t \frac{\bar\eta^{r\delta}}{\bar\kappa}
f(\xi_\delta^r,w^*,\eps)z^r_\delta
\otimes \sum_{r=0}^s \frac{\bar\eta^{s\delta}}{\bar\kappa}
f(\xi_\delta^s,w^*,\eps)z^s_\delta}\label{eq:def-dis-C-h}\\
 R_{\theta,\delta}^{t,s}&=\E[ r_{\theta,\delta}^{t,s}]
\label{eq:def-dis-R-t}\\
R_{f,\delta}^{t,s}&=\E[r_{f,\delta}^{t,s}]\label{eq:def-dis-R-h}\\
R_{f,\delta}^{t,*}&=\E[ r_{f,\delta}^{t,*}]
\label{eq:def-dis-R-h-star}\\
\Gamma^t_\delta&=\E[\D_\xi f(\xi^t_\delta,w^*,\eps)]\label{eq:def-dis-Gamma}
\end{align}
We clarify that in contrast to the continuous-time setting, it is immediate to
see that these processes/kernels are uniquely defined recursively in time via
\[\begin{gathered}
\{\theta_\delta^t,r_{\theta,\delta}^{t,s}\}_{s,t=0}
\mapsto \{C_{\theta,\delta}^{t,s},R_{\theta,\delta}^{t,s}\}_{s,t=0}
\mapsto \{\xi_\delta^t,r_{f,\delta}^{t,s},r_{f,\delta}^{t,*}\}_{s,t=0}
\mapsto
\{C_{f,\delta}^{t,s},R_{f,\delta}^{t,s},R_{f,\delta}^{t,*},\Gamma_\delta^t\}_{s,t=0}
\mapsto\\
\{\theta_\delta^t,r_{\theta,\delta}^{t,s}\}_{s,t \leq 1}
\mapsto \{C_{\theta,\delta}^{t,s},R_{\theta,\delta}^{t,s}\}_{s,t \leq 1}
\mapsto \{\xi_\delta^t,r_{f,\delta}^{t,s},r_{f,\delta}^{t,*}\}_{s,t \leq 1}
\mapsto
\{C_{f,\delta}^{t,s},R_{f,\delta}^{t,s},R_{f,\delta}^{t,*},\Gamma_\delta^t\}_{s,t \leq 1} \mapsto \ldots
\end{gathered}\]

The main result of this section is the following lemma, which shows that
this discrete-time DMFT system approximates the dynamics
\eqref{eq:discretedynamics} in the large $n,d$ limit.

\begin{lem}\label{lem:finite_dim_converge}
Let $\{\theta_{\delta,j}^t,\theta_j^*\}$ be the rows of
$\btheta_\delta^t,\btheta^*$, and let $\{\xi_{\delta,i}^t,\xi_i^*,\eps_i\}$ be 
those of $\bxi_\delta^t,\bxi^*,\beps$.
Then for any fixed integer $T \geq 0$, almost surely as $n,d\rightarrow\infty$,
\begin{align*}
\frac{1}{d}\sum_{j=1}^d \delta_{(\theta^0_{\delta,j},\ldots,
\theta^T_{\delta,j},\theta^\ast_j)}
\Rightarrow \Law\big(\theta^0_\delta,\ldots, \theta^T_\delta,\theta^\ast\big),
\quad \frac{1}{n}\sum_{i=1}^n \delta_{(\xi^0_{\delta,i}, \ldots,
\xi^T_{\delta,i},\xi_i^*,\eps_i)}
\Rightarrow \Law\big(\xi^0_\delta,\ldots,\xi^T_\delta,w^*,\eps\big)
\end{align*}
weakly and in Wasserstein-2, where the right side denotes the joint law of these
variables defined via the discrete-time DMFT recursions
(\ref{eq:def-dis-theta}--\ref{eq:def-dis-Gamma}).
\end{lem}

We proceed to prove Lemma \ref{lem:finite_dim_converge},
by re-expressing the dynamics \eqref{eq:discretedynamics} as an Approximate
Message Passing (AMP) procedure
for which rigorous state evolution results are available. 

\begin{lem}\label{lem:initialconditions}
Let $\{\theta_j^0,\theta_j^*\}$ and
$\{\eps_i,z_{\delta,i}^t\}$ denote the rows of
$\btheta^0,\btheta^*$ and $\beps,\z_\delta^t$.
Then for any fixed integer $T \geq 0$, almost surely as $n,d \to \infty$,
\[\frac{1}{d}\sum_{j=1}^d
\delta_{(\theta_j^0,\theta_j^*)} \Rightarrow \Law(\theta^0,\theta^*)
\qquad \frac{1}{n}\sum_{i=1}^n
\delta_{(\eps_i,z_{\delta,i}^0,\ldots,z_{\delta,i}^T)}
\Rightarrow \Law(\eps,z_\delta^0,\ldots,z_\delta^T)\]
weakly and in Wasserstein-$p$ for each fixed order $p \geq 1$.
\end{lem}
\begin{proof}
The convergence for $(\btheta^0,\btheta^*)$ is assumed in
Assumption \ref{asp:model}(c).
In the SME setting, the convergence for
$(\beps,\z_\delta^0,\ldots,\z_\delta^T)$ is also immediate from Assumption
\ref{asp:model}(c) and the independence of entries of $\z_\delta^t$,
see e.g.\ \cite[Proposition E.1]{fan2022approximate}.

In the setting of SGD, fix any $k \geq 0$ and
monic monomial $q:\R^{T+1} \to \R$, and consider 
\[Q=\frac{1}{n}\sum_{i=1}^n \eps_i^k q(z_{\delta,i}^0,\ldots,z_{\delta,i}^T)\]
as a polynomial function of the i.i.d.\ Poisson variables
$\bP_\delta^0,\ldots,\bP_0^T$ defining
$\z_\delta^t=\bPi\bP_\delta^t$. Note that since $\bPi$ has nonnegative entries,
each term $q(z_{\delta,i}^0,\ldots,z_{\delta,i}^T)$ is then some polynomial of
the coordinates of $\bP_\delta^0,\ldots,\bP_\delta^T$ with
nonnegative coefficients. By the concentration inequality of
\cite[Theorem 1.4, Lemma 7.6]{schudy2012concentration}, for any $\lambda \geq 0$, we have
\begin{equation}\label{eq:poissonconcentration}
\P[|Q-\E Q| \geq \lambda]
\leq C\max\left\{\max_{r=1,\ldots,\deg(q)}
e^{-\frac{c\lambda^2}{\mu_0\mu_r}},
\max_{r=1,\ldots,\deg(q)}
e^{-\left(\frac{c\lambda}{\mu_r}\right)^{1/r}}\right\}
\end{equation}
for some $(k,q)$-dependent constants $C,c>0$. Here $\deg(q)$ is the degree of
$q$, and (c.f.\ \cite[Section 1.4]{schudy2012concentration})
\[\mu_0 \leq \frac{1}{n}\sum_{i=1}^n
|\eps_i|^k \E q(z_{\delta,i}^0,\ldots,z_{\delta,i}^T),
\qquad \mu_r \leq \max_{|S|=r} \E\left[\partial_S 
\frac{1}{n}\sum_{i=1}^n 
|\eps_i|^k q(z_{\delta,i}^0,\ldots,z_{\delta,i}^T)\right]\]
where this bound for $\mu_r$ denotes the maximum expected mixed partial
derivative of order $r$ in any combination of the coordinates of
$\bP_\delta^0,\ldots,\bP_\delta^T$. Since
each coordinate of $\bP_\delta^0,\ldots,\bP_\delta^T$ can influence at most
$\kappa$ samples $i=1,\ldots,n$, it is clear by Assumption \ref{asp:model}(c) that 
for all large $n,d$, each $r \geq 1$, and some $(k,q,r)$-dependent constants
$C,C'>0$, we have
\begin{align*}
\mu_0 &\leq C,\\
\mu_r &\leq \max_{S \subset [n]:|S|=\kappa} \frac{C}{n}\sum_{i \in S}
|\eps_i|^k
\leq \max_{S \subset [n]:|S|=\kappa}
C\left(\frac{1}{n}\sum_{i=1}^n |\eps_i|^{2k}\right)^{1/2}
\left(\frac{1}{n}\sum_{i=1}^n \1_{i \in S}\right)^{1/2}
\leq C'\sqrt{\frac{\kappa}{n}}.
\end{align*}
Since $\kappa \asymp n^\alpha$
and $\alpha<1$, applying these bounds to \eqref{eq:poissonconcentration} shows,
by the Borel-Cantelli lemma,
that there exists a constant $\iota>0$ for which $|Q-\E Q|<n^{-\iota}$
a.s.\ for all large $n,d$. On the other hand, letting
$\{\tilde z_\delta^t\}_{t \geq 0} \overset{iid}{\sim}
\Pois(\delta\kappa/n^\alpha)$ and 
$\{z_\delta^t\}_{t \geq 0} \overset{iid}{\sim}
\Pois(\delta\bar \kappa)$, by Assumptions \ref{asp:SGD} and \ref{asp:model}
we have
\[\E Q=\E q(\tilde z_\delta^0,\ldots,\tilde z_\delta^T)
\cdot \frac{1}{n}\sum_{i=1}^n \eps_i^k
=\E q(z_\delta^0,\ldots,z_\delta^T) \cdot \E \eps^k+o(1).\]
This shows that $Q$ converges a.s.\ to $\E q(z_\delta^0,\ldots,z_\delta^T) \cdot
\E \eps^k$, i.e.\ each mixed moment of the empirical measure
$\frac{1}{n}\sum_{i=1}^n
\delta_{(\eps_i,z_{\delta,i}^0,\ldots,z_{\delta,i}^T)}$
converges a.s.\ to that of $\Law(\eps,z_\delta^0,\ldots,z_\delta^T)$.
Since $\eps$ has finite moment generating function around 0, so
does $(\eps,z_\delta^0,\ldots,z_\delta^T)$. Then this convergence of mixed
moments implies convergence weakly and in Wasserstein-$p$
\cite[Thm 6.9]{villani2008optimal}, establishing the lemma.
\end{proof}

\begin{proof}[Proof of Lemma \ref{lem:finite_dim_converge}]
The proof relies on mapping the dynamics \eqref{eq:discretedynamics} to an
AMP algorithm: Let us set
\[\V^0=[\V_1^0,\V_2^0]:=[\btheta^0,\btheta^*] \in \R^{d \times (k+k^*)},
\quad \V^t=[\V_1^t,\V_2^t]:=[\btheta_\delta^t-\btheta_\delta^{t-1},0]
\in \R^{d \times (k+k^*)} \text{ for } t \geq 1,\]
\[\Y^t=[\Y_1^t,\Y_2^t]:=\left[{-}\frac{\bar\eta^{t\delta}}{\bar\kappa}f(\X\btheta_\delta^t,\X\btheta^*,\beps)
\odot \z_\delta^t,\,0\right] \in \R^{n
\times (k+k^*)} \text{ for } t \geq 0.\]
Then
\begin{equation}\label{eq:thetaV}
[\btheta_\delta^t,\btheta^*]=\sum_{r=0}^t \V^r.
\end{equation}
For any matrices $A^{t,q},B^{t,q} \in \R^{(k+k^*) \times
(k+k^*)}$, the dynamics \eqref{eq:discretedynamics} are then equivalent to the
iterations, for $t=0,1,2,\ldots$
\begin{align}
\W^t&=\X\V^t-\sum_{r=0}^{t-1} \Y^r (A^{t,r})^\top,\label{eq:AMPW}\\
\Y^t&=\left[{-}\frac{\bar\eta^{t\delta}}{\bar\kappa}
f\left(\sum_{q=0}^t \left(\W^q+\sum_{r=0}^{q-1} \Y^r
(A^{q,r})^\top\right),\beps\right) \odot \z_\delta^t,\;0\right],\label{eq:AMPY}\\
\U^t&=\X^\top \Y^t-\sum_{q=0}^t \V^q (B^{t,q})^\top,\label{eq:AMPU}\\
\V^{t+1}&=\U^t+\sum_{q=0}^t \V^q (B^{t,q})^\top
-\left[\delta\bar\eta^{\delta t}g\left(\sum_{q=0}^t
\V_1^q\right),\,0\right].\label{eq:AMPV}
\end{align}

We may identify \eqref{eq:AMPY} and \eqref{eq:AMPV} as two mappings
\begin{equation}\label{eq:ytvt}
\Y^t=y_t(\W^0,\ldots,\W^t,\beps,\z_\delta^0,\ldots,\z_\delta^t),
\qquad \V^{t+1}=v_{t+1}(\U^0,\ldots,\U^t,\V^0)
\end{equation}
defined recursively for $t=0,1,2,\ldots$ via
\begin{align}
&y_t(W^0,\ldots,W^t,\eps,z_\delta^0,\ldots,z_\delta^t)\nonumber\\
&\hspace{0.5in}=\begin{pmatrix} {-}\frac{\bar\eta^{t\delta}}{\bar\kappa}z_\delta^t 
f\left(\sum_{q=0}^t \left(W^q+\sum_{r=0}^{q-1}
A^{q,r}y_r(W^0,\ldots,W^r,\eps,z_\delta^0,\ldots,z_\delta^r)\right),\eps\right)
\\ 0 \end{pmatrix},\label{eq:yt}\\
&v_{t+1}(U^0,\ldots,U^t,V^0)
=U^t+B^{t,0}V^0+\sum_{q=1}^t B^{t,q}v_q(U^0,\ldots,U^{q-1},V^0)\notag\\
&\hspace{2in}-\begin{pmatrix} \delta\bar\eta^{\delta t}g\left(V_1^0+\sum_{q=1}^t
[v_q(U^0,\ldots,U^{q-1},V^0)]_1\right) \\ 0 \end{pmatrix},\label{eq:vt}
\end{align}
and applied row-wise to their inputs in \eqref{eq:ytvt}.
Let $\theta^0,\theta^*,\eps,\{z_\delta^t\}_{t \geq 0}$ be as defined in the
discrete-time DMFT system. We set
\[V^0=(\theta^0,\theta^*) \in \R^{k+k^*}\]
and define recursively for $t=0,1,2,\ldots$
\begin{equation}\label{eq:AMPstateevolution}
\begin{gathered}
A^{t,q}=\E\D_q V^t \text{ for all } q \in \{0,\ldots,t-1\},
\quad \Sigma^t=\E[(V^0,\ldots,V^t) \otimes (V^0,\ldots,V^t)],\\
(W^0,\ldots,W^t) \sim \cN(0,\Sigma^t) \text{ (independent of
$\eps,\{z_\delta^t\}_{t \geq 0}$)}, \quad
Y^t=y_t(W^0,\ldots,W^t,\eps,z_\delta^0,\ldots,z_\delta^t),\\
B^{t,q}=\gamma\,\E\D_q Y^t \text{ for all } s \in \{0,\ldots,t\}, \quad
\Omega^t=\gamma\,\E[(Y^0,\ldots,Y^t) \otimes (Y^0 \ldots,Y^t)],\\
(U^0,\ldots,U^t) \sim \cN(0,\Omega^t) \text{ (independent of $\theta^0,\theta^*$)},
\quad V^{t+1}=v_{t+1}(U^0,\ldots,U^t,V^0),
\end{gathered}
\end{equation}
where $\D_q V^t \equiv \D_{U^q} v_t(U^0,\ldots,U^{t-1},V^0) \in \R^{(k+k^*)
\times (k+k^*)}$ denotes the partial derivative (Jacobian) in $U^q$, and
similarly $\D_q Y^t \equiv \D_{W^q}
y_t(W^0,\ldots,W^t,\eps,z_\delta^0,\ldots,z_\delta^t) \in
\R^{(k+k^*) \times (k+k^*)}$ denotes that in $W^q$. Then, using these choices of
$\{A^{t,q},B^{t,s}\}$, (\ref{eq:AMPW}--\ref{eq:AMPV}) is a standard form of an
AMP algorithm with state evolution covariances $\{\Sigma^t,\Omega^t\}$, see
e.g.\ \cite{javanmard2013state}.

We apply the result of \cite[Theorem 2.21]{wang2024universality} for this AMP
algorithm: It is clear that
$T$ iterations of (\ref{eq:AMPW}--\ref{eq:AMPV}) may be directly mapped to
$(k+k^*)T$ vector iterations of
\cite[Eqs.\ (2.14) and (D.1--D.2)]{wang2024universality} with
side information vectors $(\beps,\z_\delta^0,\ldots,\z_\delta^T)$
and the notational identification $(1/\sqrt{\gamma})\X \leftrightarrow \W^\top$.
Each function
$y_t(\cdot)$ and $v_{t+1}(\cdot)$ recursively defined via \eqref{eq:AMPY} and
\eqref{eq:AMPV} has polynomial growth and is Lipschitz in its first $t+1$
arguments $W^0,\ldots,W^t$ and $U^0,\ldots,U^t$, by the conditions for
$f(\cdot)$ and $g(\cdot)$ in Assumption \ref{asp:lipschitz}.\footnote{Here,
$(W^0,\ldots,W^t) \mapsto y(W^0,\ldots,W^t,\eps,z_\delta^0,\ldots,z_\delta^t)$
is $(\prod_{s=0}^t C\max(1,|z_\delta^s|))$-Lipschitz where $C>0$ depends on the
Lipschitz constant of $f$. The argument of \cite[Theorem
2.21]{wang2024universality} implicitly assumes that
$(W^0,\ldots,W^t) \mapsto y(W^0,\ldots,W^t,\eps,z_\delta^0,\ldots,z_\delta^t)$
is $L$-Lipschitz for a deterministic constant $L>0$, but it is clear from the
proof that the argument holds as long as this is
$L(\eps,z_\delta^0,\ldots,z_\delta^t)$-Lipschitz where
$\E|L(\eps,z_\delta^0,\ldots,z_\delta^t)|^{2+\iota}<\infty$ for some $\iota>0$.}
The conditions of \cite[Definition 2.18]{wang2024universality} hold for
the data matrix
$(1/\sqrt{\gamma})\X$ by Assumption \ref{asp:model}, with constant 
variance profile across all entries, and \cite[Assumption
2.17]{wang2024universality} is verified by Lemma \ref{lem:initialconditions}.
Then by \cite[Theorem 2.21]{wang2024universality}, for any fixed $T \geq 1$,
almost surely as $n,d\to\infty$,
\begin{equation}\label{eq:state_evolution}
\begin{gathered}
\frac{1}{d}\sum_{j=1}^d \delta_{(U_j^0,\ldots,U_j^T,V_j^0,\ldots,V_j^T)} \Rightarrow
\Law(U^0,\ldots,U^T,V^0,\ldots,V^T),\\
\frac{1}{n}\sum_{i=1}^n
\delta_{(W_i^0,\ldots,W_i^T,Y_i^0,\ldots,Y_i^T,\eps_i,z_{\delta,i}^0,\ldots,z_{\delta,i}^T)} \Rightarrow
\Law(W^0,\ldots,W^T,Y^0,\ldots,Y^T,\eps,z_\delta^0,\ldots,z_\delta^T)
\end{gathered}
\end{equation}
weakly and in Wasserstein-2.

Finally, we show that this implies the statements of the lemma.
Writing the two components of $V^t,Y^t \in \R^{k \times k^*}$ as
$V^t=(V_1^t,V_2^t)$ and $Y^t=(Y_1^t,Y_2^t)$, note that by the definitions
\eqref{eq:yt} and \eqref{eq:vt} we have
\[V^t=(V_1^t,0) \text{ for all } t \geq 1,
\quad Y^t=(Y_1^t,0) \text{ for all } t \geq 0,\]
and hence also
\[A^{t,q}=\begin{pmatrix} (A^{t,q})_{11} & (A^{t,q})_{12} \\
0 & 0 \end{pmatrix}, \; W_2^t=0 \text{ for all } t \geq 1,\]
\[B^{t,q}=\begin{pmatrix} (B^{t,q})_{11} & (B^{t,q})_{12} \\
0 & 0 \end{pmatrix}, \; U_2^t=0 \text{ for all } t \geq 0.\]
Let us define
\begin{align}
(\theta_\delta^t,\theta^*)&:=\sum_{r=0}^t V^r
=\left(\sum_{r=0}^t V_1^r,\,V_2^0\right),\label{eq:AMPthetadef}\\
(\xi_\delta^t,w^*)&:=\sum_{q=0}^t \left(W^q+\sum_{r=0}^{q-1}
A^{q,r}Y^r\right)
=\left(\sum_{q=0}^t W_1^q+\sum_{q=1}^t \sum_{r=0}^{q-1} (A^{q,r})_{11}Y_1^r,\,
W_2^0\right).\label{eq:AMPxidef}
\end{align}
Comparing these definitions with
\eqref{eq:thetaV} and \eqref{eq:AMPW}, the convergence
\eqref{eq:state_evolution} implies that also
\begin{equation}\label{eq:AMPDMFTlimits}
\frac{1}{d}\sum_{j=1}^d \delta_{(\theta^0_{\delta,j}, \ldots,
\theta^T_{\delta,j},\theta^\ast_j)} \Rightarrow
\Law\big(\theta^0_\delta,\ldots, \theta^T_\delta,\theta^\ast\big),
\quad \frac{1}{n}\sum_{i=1}^n \delta_{(\xi^0_{\delta,i}, \ldots,
\xi^T_{\delta,i},\xi_i^*,\eps_i)}
\Rightarrow \Law\big(\xi^0_\delta,\ldots,\xi^T_\delta,w^*,\eps\big).
\end{equation}
So it remains to check that these definitions of
$\theta_\delta^t,\xi_\delta^t,w^*$ in (\ref{eq:AMPthetadef}--\ref{eq:AMPxidef})
coincide with those defined by the discrete DMFT recursions
(\ref{eq:def-dis-theta}--\ref{eq:def-dis-Gamma}).

Differentiating \eqref{eq:yt}
and \eqref{eq:vt} by the chain rule and applying $A^{q,r}=\E \D_rV^q$
and $B^{t,s}=\gamma\,\E \D_s Y^t$, observe that
\begin{align}
\D_s Y^t&:=\begin{pmatrix} (\D_s Y^t)_{11} & (\D_s Y^t)_{12} \\ 0 & 0
\end{pmatrix}\notag\\
&={-}\frac{\bar\eta^{t\delta}}{\bar\kappa}z_\delta^t \begin{pmatrix} 
\D_\xi f(\xi_\delta^t,w^*,\eps)
& \D_{w^*} f(\xi_\delta^t,w^*,\eps)\\ 0 & 0 \end{pmatrix}
\left(\Id_{k+k^*}+\sum_{q=s+1}^t \sum_{r=s}^{q-1}(\E \D_r V^q)(\D_s
Y^r)\right),\label{eq:AMPchainrule1}\\
\D_s V^{t+1}&:=\begin{pmatrix} (\D_s V^{t+1})_{11} & (\D_s V^{t+1})_{12}
\\ 0 & 0 \end{pmatrix}\notag\\
&=\Id_{k+k^*}\1_{s=t}
+\sum_{q=s+1}^t (\gamma \E \D_q Y^t)(\D_s V^q)\notag\\
&\hspace{0.5in}- \delta \bar\eta^{\delta t}\begin{pmatrix}\D g(\theta_\delta^t)
\sum_{q=s+1}^t (\D_s V^q)_{11} & \D g(\theta_\delta^t)
\sum_{q=s+1}^t (\D_s V^q)_{12} \\ 0 & 0 \end{pmatrix}.
\label{eq:AMPchainrule2}
\end{align}
In particular, for $s=t$ we have $(\D_t
Y^t)_{11}={-}\frac{\bar\eta^{t\delta}}{\bar\kappa}z_\delta^t \cdot \D_\xi
f(\xi_\delta^t,w^*,\eps)$. Let us define
\begin{equation}\label{eq:AMPresponsedefs}
\begin{gathered}
(\D_sY^t-\D_{s+1}Y^t)_{11}
:={-}\frac{\bar\eta^{t\delta}}{\bar\kappa}z_\delta^t r_{f,\delta}^{t,s},
\quad (\D_0Y^t)_{12}:=
{-}\frac{\bar\eta^{t\delta}}{\bar\kappa}z_\delta^t r_{f,\delta}^{t,*},
\quad R_{f,\delta}^{t,s}=\E r_{f,\delta}^{t,s},
\quad R_{f,\delta}^{t,*}=\E r_{f,\delta}^{t,*},\\
\sum_{r=s+1}^t (\D_s V^r)_{11}
=\sum_{r=s}^{t-1} (\D_s V^{r+1})_{11}:=r_{\theta,\delta}^{t,s},
\quad R_{\theta,\delta}^{t,s}=\E r_{\theta,\delta}^{t,s},
\quad \Gamma_\delta^t=\E \D_\xi f(\xi_\delta^t,w^*,\eps),
\end{gathered}
\end{equation}
and apply the identities
\begin{align}
&\sum_{q=s+1}^t \sum_{r=s}^{q-1} [(\E \D_r V^q)(\D_s Y^r)]_{1i}\notag\\
&=\sum_{r=s}^{t-1}\left(\sum_{q=r+1}^t \E(\D_r V^q)_{11} \right)(\D_s
Y^r)_{1i}=\sum_{r=s}^{t-1} R_{\theta,\delta}^{t,r}(\D_s Y^r)_{1i}
\text{ for } i=1,2,\label{eq:sumidentity1}\\
&\sum_{q=s+1}^t [(\E \D_q Y^t)(\D_s V^q)]_{11}\notag\\
&=\sum_{q=s+1}^t \left(\E(\D_t Y^t)_{11}+\sum_{r=q}^{t-1}
\E(\D_r Y^t-\D_{r+1} Y^t)_{11}\right)(\D_s V^q)_{11}\notag\\
&=\E(\D_t Y^t)_{11} \sum_{q=s+1}^t (\D_s V^q)_{11}
+\sum_{r=s+1}^{t-1}
\E(\D_r Y^t-\D_{r+1} Y^t)_{11}\left(\sum_{q=s+1}^r (\D_s V^q)_{11}\right)\notag\\
&={-}\delta\bar\eta^{t\delta} \Gamma_\delta^t r_{\theta,\delta}^{t,s}
-\delta\bar \eta^{t\delta}\sum_{r=s+1}^{t-1} R_{f,\delta}^{t,r}
r_{\theta,\delta}^{r,s}.
\label{eq:sumidentity2}
\end{align}
In the last step, we have used $\E z_\delta^t=\delta\bar\kappa$ and the fact
that $z_\delta^t$ is independent of
$\D_\xi f(\xi_\delta^t,w^*,\eps)$ and $r_{f,\delta}^{t,s}$.
Then, applying \eqref{eq:AMPresponsedefs},
\eqref{eq:sumidentity1}, \eqref{eq:sumidentity2} to
(\ref{eq:AMPchainrule1}--\ref{eq:AMPchainrule2}), we obtain the recursions
\begin{align*}
r_{f,\delta}^{t,s}&=\D_\xi f(\xi_\delta^t,w^*,\eps)
\left(\sum_{r=s+1}^{t-1}
R_{\theta,\delta}^{t,r}(\D_s Y^r-\D_{s+1}Y^r)_{11}
+R_{\theta,\delta}^{t,s}(\D_s Y^s)_{11}\right)\\
&={-}\D_\xi f(\xi_\delta^t,w^*,\eps)
\left(\sum_{r=s+1}^{t-1}
\frac{\bar\eta^{r\delta}}{\bar\kappa}
R_{\theta,\delta}^{t,r}r_{f,\delta}^{r,s}z_\delta^r
+\frac{\bar\eta^{s\delta}}{\bar\kappa} R_{\theta,\delta}^{t,s} \cdot
\D_\xi f(\xi_\delta^s,w^*,\eps)z_\delta^s\right),\\
r_{f,\delta}^{t,*}&=
\D_{w^*} f(\xi_\delta^t,w^*,\eps)
+\D_\xi f(\xi_\delta^t,w^*,\eps)\sum_{r=0}^{t-1}R_{\theta,\delta}^{t,r}
(\D_0 Y^r)_{12}\\
&={-}\D_\xi f(\xi_\delta^t,w^*,\eps)
\sum_{r=0}^{t-1} \frac{\bar\eta^{r\delta}}{\bar\kappa}
R_{\theta,\delta}^{t,r}r_{f,\delta}^{r,*}z_\delta^r
+\D_{w^*} f(\xi_\delta^t,w^*,\eps),\\
r_{\theta,\delta}^{t,s}&=\sum_{r=s}^{t-1} \left(
\Id_k\1_{s=r}-\gamma \delta\bar\eta^{r\delta} \Gamma_\delta^r
r_{\theta,\delta}^{r,s}
-\gamma\delta\bar\eta^{r\delta}\sum_{q=s+1}^{r-1}
R_{f,\delta}^{r,q}r_{\theta,\delta}^{q,s}\right)
-\delta \bar\eta^{r\delta}
\D g(\theta_\delta^r)r_{\theta,\delta}^{r,s}\\
&=\Id_k-\sum_{r=s}^{t-1} \delta\bar\eta^{r\delta}
\bigg[(\gamma \Gamma_\delta^r+\D g(\theta_\delta^r)) \cdot
r_{\theta,\delta}^{r,s}
+\gamma \sum_{q=s+1}^{r-1} R_{f,\delta}^{r,q}r_{\theta,\delta}^{q,s}\bigg]
\end{align*}
which are precisely (\ref{eq:def-dis-deriv-t}--\ref{eq:def-dis-deriv-xi}).
Furthermore, let us define
\begin{equation}\label{eq:AMPuwdef}
\sum_{r=0}^t U_1^r:=\sqrt{\gamma}u_\delta^t,
\qquad \sum_{q=0}^t W_1^q:=w_\delta^t.
\end{equation}
Then by \eqref{eq:AMPthetadef} and \eqref{eq:vt},
\begin{align*}
\theta_\delta^t&=V_1^0+\sum_{r=1}^t V_1^r\\
&=\theta^0+\sum_{r=0}^{t-1}\left(U_1^r+\sum_{q=0}^r (B^{r,q})_{11}V_1^r
+(B^{r,0})_{12}V_2^0-\delta\bar\eta^{r\delta} g(\theta_\delta^r)\right)\\
&=\theta^0+\sum_{r=0}^{t-1} \left(U_1^r+\gamma\sum_{q=0}^r
(\E\D_q Y^r)_{11}V_1^r+\gamma (\E \D_0 Y^r)_{12}V_2^0
-\delta\bar\eta^{r\delta} g(\theta_\delta^r)\right)\\
&\overset{(*)}{=}\theta^0+\sum_{r=0}^{t-1}\left(U_1^r+\gamma
\left({-}\delta\bar\eta^{r\delta}\Gamma_\delta^r \theta_\delta^r-
\delta\bar\eta^{r\delta}
\sum_{q=0}^{r-1} R_{f,\delta}^{r,q}
\theta_\delta^q -\delta\bar\eta^{r\delta} R_{f,\delta}^{t,*}\theta^*\right)
-\delta\bar\eta^{r\delta} g(\theta_\delta^r)\right)\\
&=\theta^0-\sum_{r=0}^{t-1}\delta\bar\eta^{r\delta}
\left(\gamma \Gamma_\delta^r \theta_\delta^r
+g(\theta_\delta^r)+\gamma \sum_{q=0}^{r-1}
R_{f,\delta}^{r,q}\theta_\delta^q\right)+\sqrt{\gamma} u_\delta^t,
\end{align*}
where $(*)$ applies an identity analogous to \eqref{eq:sumidentity2}. Similarly,
applying \eqref{eq:AMPxidef} and \eqref{eq:yt},
\begin{align*}
\xi_\delta^t=\sum_{q=0}^t W_1^q+\sum_{q=1}^t \sum_{r=0}^{q-1}
(A^{q,r})_{11}Y_1^r
&=\sum_{q=0}^t W_1^q+\sum_{q=1}^t \sum_{r=0}^{q-1}
(\E \D_r V^q)_{11}Y_1^r\\
&\overset{(*)}{=}
\sum_{q=0}^t W_1^q+\sum_{r=0}^{t-1} R_{\theta,\delta}^{t,r} Y_1^r
={-}\sum_{r=0}^{t-1} \frac{\bar\eta^{r\delta}}{\bar\kappa}
R_{\theta,\delta}^{t,r}f(\xi_\delta^r,w^*,\eps)z_\delta^r+w_\delta^t,
\end{align*}
where $(*)$ applies an identity analogous to \eqref{eq:sumidentity1}. These
recursions are precisely (\ref{eq:def-dis-theta}--\ref{eq:def-dis-xi}).
Finally, by \eqref{eq:AMPstateevolution},
we note that $\{u_\delta^t\}_{t \geq 0}$ and $(\{w_\delta^t\}_{t \geq 0},
w_\delta^*)$ defined via \eqref{eq:AMPxidef} and \eqref{eq:AMPuwdef} are
Gaussian with covariance
\begin{align*}
\E[u_\delta^t \otimes u_\delta^s]
&=\gamma^{-1}\E\left[\sum_{r=0}^t U_1^r \otimes \sum_{r=0}^s U_1^r\right]
=\E\left[\sum_{r=0}^t Y_1^t \otimes \sum_{r=0}^s Y_1^r\right]\\
&=\E\left[\sum_{r=0}^t \frac{\bar\eta^{r\delta}}{\bar\kappa}
f(\xi_\delta^r,w^*,\eps)z_\delta^r
\otimes \sum_{r=0}^s \frac{\bar\eta^{r\delta}}{\bar\kappa}
f(\xi_\delta^r,w^*,\eps)z_\delta^r\right],\\
\E[w_\delta^t \otimes w_\delta^s]
&=\E\left[\sum_{q=0}^t W_1^q \otimes \sum_{q=0}^s W_1^q\right]
=\E\left[\sum_{q=0}^t V_1^q \otimes \sum_{q=0}^s V_1^q\right]
=\E[\theta_\delta^t \otimes \theta_\delta^s],\\
\E[w_\delta^t \otimes w_\delta^*]
&=\E\left[\sum_{q=0}^t W_1^q \otimes W_2^0\right]
=\E\left[\sum_{q=0}^t V_1^q \otimes V_2^0\right]
=\E[\theta^t \otimes \theta^*],\\
\E[w_\delta^* \otimes w_\delta^*]
&=\E\left[W_2^0 \otimes W_2^0\right]
=\E\left[V_2^0 \otimes V_2^0\right]
=\E[\theta^* \otimes \theta^*],
\end{align*}
which is precisely (\ref{eq:def-dis-C-t}--\ref{eq:def-dis-C-h}). This verifies
that the limits in \eqref{eq:AMPDMFTlimits} coincide with the definitions from
(\ref{eq:def-dis-theta}--\ref{eq:def-dis-Gamma}), showing the lemma.
\end{proof}

\subsection{Convergence to the continuous DMFT system}\label{sec:deltaDMFTconvergence}

We now show that a continuous-time embedding of the preceding
discrete-time DMFT system converges to (the unique fixed
point of) the continuous-time DMFT system defined in Section
\ref{sec:DMFTlimits}, as the discretization step size $\delta \to 0$.

To ease notation in this section, for all times $t \in [0,T]$,
we will use the time-discretization conventions
\begin{equation}\label{eq:flrcil}
\flr{t} = \max\{i\delta: i\delta\leq t, i\in \Z_+\} \in \delta\Z_+,
\quad \cil{t} = \flr{t} + \delta \in \delta\Z_+
\end{equation}
Let $\{\theta_\delta^t\}_{t \in \Z_+}$, $\{\xi_\delta^t\}_{t \in \Z_+}$,
and the correlation and response kernels
\[C_{\theta,\delta},R_{\theta,\delta},C_{f,\delta},R_{f,\delta},R_{f,\delta}^*,\Gamma_\delta\]
be the components of the discrete-time DMFT system defined via the recurions
(\ref{eq:def-dis-theta}--\ref{eq:def-dis-Gamma}).
We define their continuous-time embeddings
$\{\bar\theta_\delta^t\}_{t \in [0,T]}$, 
$\{\bar \xi_\delta^t\}_{t \in [0,T]}$,
$\{\bar C_{\theta,\delta}^{t,s}\}_{t,s \in [0,T]}$,
$\{\bar C_{f,\delta}^{t,s}\}_{t,s \in [0,T]}$,
$\{\bar R_{f,\delta}^{t,*}\}_{t \in [0,T]}$, and
$\{\bar \Gamma_\delta^t\}_{t \in [0,T]}$ by
\begin{equation}\label{eq:embeddings}
\bar\theta_\delta^t=\theta_\delta^k,
\; \bar\xi_\delta^t=\xi_\delta^k,
\; \bar C_{\theta,\delta}^{t,s}=C_{\theta,\delta}^{k,j},
\; \bar C_{f,\delta}^{t,s}=C_{f,\delta}^{k,j},
\; \bar R_{f,\delta}^{t,*}=R_{f,\delta}^{k,*},
\; \bar \Gamma_{f,\delta}^t=\Gamma_{f,\delta}^k
\text{ if } \flr{t}=k \text{ and } \flr{s}=j.
\end{equation}
These embeddings are piecewise-constant and right-continuous with jumps at
$\delta\Z_+$ (henceforth abbreviated as $\delta$-p.c.r.c.)
We define also the $\delta$-p.c.r.c.\ embedding
$\{\bar R_{\theta,\delta}^{t,s}\}_{t,s \in [0,T]}$ by
\[\bar R_{\theta,\delta}^{t,s}=\begin{cases} 0 & \text{ if } s>t \\
\Id & \text{ if } s \leq t \text{ and } \flr{s}=\flr{t} \\
R_{\theta,\delta}^{k,j} & \text{ if } \flr{s}=j<\flr{t}=k,\end{cases}\]
and a family of linear operators $\{\bar R_{f,\delta}^t\}_{t \in [0,T]}$
with $\bar R_{f,\delta}^t:L^4([0,t],\R^k) \to \R^k$ by
\begin{equation}\label{eq:Rfembed}
\bar R_{f,\delta}^t(x)
=\int_0^{\flr{t}} \bar R_{f,\delta}^{t,s} x^s\d s,
\quad \bar R_{f,\delta}^{t,s}=\delta^{-1}R_{f,\delta}^{k,j} \text{ if }
\flr{s}=j<\flr{t}=k.
\end{equation}
We denote the tuples of these embedded correlation and response functions by
\[X^\delta=(\bar C_{f,\delta},\bar R_{f,\delta},
\bar R_{f,\delta}^*,\bar \Gamma_\delta),
\quad Y^\delta=(\bar C_{\theta,\delta},\bar R_{\theta,\delta}).\]

The main result of this section is the following lemma, where we recall
the domain $\cS(T,C_0)$ and metric $\lbddst{\cdot}{\cdot}$ defined in Section
\ref{sec:existenceuniqueness}.

\begin{lem}\label{lem:dmft_discretization_error}
\begin{enumerate}[(a)]
\item For any sufficiently large constant $C_0 \equiv C_0(T)>0$ and any
$\delta>0$,
\[(Y^\delta,X^\delta) \in \cS(T,C_0).\]
\item For any sufficiently large constant $\lambda \equiv \lambda(T,C_0)>0$,
\[\lim_{\delta \to 0} \lbddst{X^\delta}{X}=0,
\quad \lim_{\delta \to 0} \lbddst{Y^\delta}{Y}=0\]
where $(Y,X) \in \cS^\text{cont}(T,C_0)$ is the fixed point given by Theorem
\ref{thm:uniqueness-existence}.
\item For any fixed $m \geq 0$ and $0 \leq t_1 \leq \ldots \leq t_m \leq T$,
as $\delta \to 0$,
\begin{align}
\Law(\bar \theta^{t_1}_\delta,\ldots,\bar \theta^{t_m}_\delta,\theta^*)
&\Rightarrow
\Law(\theta^{t_1},\ldots,\theta^{t_m},\theta^*),\,\label{eq:discretedmftthetaconv}\\
\Law(\bar \xi_\delta^{t_1},\ldots,\bar \xi_\delta^{t_m},w^*,\eps) &\Rightarrow
\Law(\xi^{t_1},\ldots,\xi^{t_m},w^*,\eps)\label{eq:discretedmftetaconv}
\end{align}
weakly and in Wasserstein-2.
\end{enumerate}
\end{lem}

\begin{proof}[Proof of Lemma \ref{lem:dmft_discretization_error}(a)]
We define also the embeddings
\begin{equation}\label{eq:embeddings2}
\bar u_\delta^t=u_\delta^k,
\;\bar w_\delta^t=w_\delta^k,
\;\bar r_{f,\delta}^{t,*}=r_{f,\delta}^{k,*} \text{ if } \flr{t}=k,
\quad
\bar r_{\theta,\delta}^{t,s}=\begin{cases} 0 & \text{ if } s>t \\
\Id & \text{ if } s \leq t \text{ and } \flr{s}=\flr{t} \\
r_{\theta,\delta}^{k,j} & \text{ if } \flr{s}=k<\flr{t}=j.\end{cases}
\end{equation}
Let $D([0,T],\R^k)$ denote the space of $\delta$-p.c.r.c.\ $\R^k$-valued
processes $\{x^t\}_{t \in [0,T]}$, i.e.\ satisfying $x^t=x^{\flr{t}}$
for all $t \in [0,T]$. We embed
$r_{f,\delta}^{k,j}$ as a family of linear operators
$\bar r_{f,\delta}^t:D([0,T],\R^k) \to \R^k$ given by
\begin{equation}\label{eq:rfembed}
\bar r_{f,\delta}^t(x^{[t]})
=\sum_{j=0}^{k-1} r_{f,\delta}^{k,j}x^{j\delta} \text{ if } \flr{t}=k.
\end{equation}
Finally, we identify the variables $\{z_\delta^t\}_{t \geq 0}$ of the
discrete-time DMFT system with the increments of the continuous-time Gaussian
or Poisson process $\{z^t\}_{t \geq 0}$ defined by
\eqref{eq:poisson-z}/\eqref{eq:gaussian-z},
\[z_\delta^t=z^{(t+1)\delta}-z^{t\delta}.\]
Then the discrete-time DMFT recursions
(\ref{eq:def-dis-theta}--\ref{eq:def-dis-Gamma})
imply the following equations for these embeddings:
For any $x \in D([0,T],\R^k)$,
\begin{align}
\bar\theta^t_\delta&=\bar\theta^0_\delta-\int_0^{\flr{t}} \bar\eta^{\flr{r}} \Big({\gamma}\bar\Gamma^r_\delta
\bar\theta^r_\delta+g(\bar\theta^r_\delta)+{\gamma}\bar R_{f,\delta}^r(\bar\theta^{[r]}_\delta)+
{\gamma}\bar R_{f,\delta}^{r,*}\theta^*\Big)\d r+\sqrt{\gamma}\,\bar u^t_\delta,
\label{eq:def-dist-theta}\\
\bar\xi^t_\delta&={-}\int_0^{\flr{t}} \frac{\bar\eta^{\flr{r}}}{\bar\kappa} \bar R_{\theta,\delta}^{t,r}
f(\bar\xi^{r-}_\delta,w_\delta^*,\eps)\d z^r+\bar
w^t_\delta,\label{eq:def-dist-xi}\\
\bar r_{\theta,\delta}^{t,s}&=\Id- 
\mathbf{1}_{\cil{s} \leq \flr{t}} \int_{\cil{s}}^{\flr{t}} \bar\eta^{\flr{r}}\bigg[\Big({\gamma}\bar\Gamma^r_\delta+\D
g(\bar\theta^r_\delta)\Big)\bar r_{\theta,\delta}^{r,s}+{\gamma}
\bar R_{f,\delta}^r\Big(\bar r_{\theta,\delta}^{[r],s}\Big)\bigg]\d r
\text{ for } s \leq t,\notag\\
&\hspace{1in}r_{\theta,\delta}^{t,s}=0 \text{ for } s>t,
\label{eq:def-dist-deriv-t}\\
\bar r_{f,\delta}^{t,*}&={-}\D_\xi f(\bar\xi^t_\delta,w_\delta^*,\eps) \int_0^{\flr{t}}
\frac{\bar\eta^{\flr{r}}}{\bar\kappa}
\bar R_{\theta,\delta}^{t,r}\bar r_{f,\delta}^{r-,*}\d z^r
+\D_{w^*} f(\bar\xi^t_\delta,w_\delta^*,\eps),\label{eq:def-dist-deriv-xi-star}\\
\bar r_{f,\delta}^t(x^{[t]})&={-}\D_\xi
f(\bar\xi^t_\delta,w_\delta^*,\eps) \int_0^{\flr{t}} \frac{\bar\eta^{\flr{r}}}{\bar\kappa}
\bar R_{\theta,\delta}^{t,r}\bigg(\bar r_{f,\delta}^{r-}(x^{[r]})+\D_\xi
f(\bar\xi^{r-}_\delta,w_\delta^*,\eps)x^{r-}\bigg)\d z^r,\label{eq:def-dist-resp-xi}
\end{align}
and
\begin{align}
\bar C_{\theta,\delta}^{t,s}&=\E[\bar\theta^t_\delta \otimes \bar\theta^s_\delta] \text{ for } t,s \in [0,\infty)
\cup \{*\}\label{eq:def-dist-C-t}\\
\bar R_{\theta,\delta}^{t,s}&=\E[\bar r_{\theta,\delta}^{t,s}]
\label{eq:def-dist-R-t}\\
\bar C_{f,\delta}^{t,s}&=\E\brk{\int_0^{\flr{t}}
\frac{\bar\eta^{\flr{r}}}{\bar\kappa}f(\bar\xi^{r-}_\delta,w_\delta^*,\eps)\d z^r
\otimes \int_0^{\flr{s}}
\frac{\bar\eta^{\flr{r}}}{\bar\kappa}f(\bar\xi^{r-}_\delta,w_\delta^*,\eps)\d z^r}\label{eq:def-dist-C-h}\\
\bar R_{f,\delta}^t(x^{[t]})&=\E\big[\bar r_{f,\delta}^t(x^{[t]})\big]\label{eq:def-dist-R-h}\\
\bar R_{f,\delta}^{t,*}&=\E[\bar r_{f,\delta}^{t,*}]
\label{eq:def-dist-R-h-star}\\
\bar \Gamma^t_\delta&=\E[\D_\xi f(\bar\xi^t_\delta,w_\delta^*,\eps)].\label{eq:def-dist-Gamma}
\end{align}
For example, since $\{\bar\theta_\delta^t\}$ is $\delta$-p.c.r.c., we have
by \eqref{eq:Rfembed} for $\flr{t}=k$ that
\[\bar R_{f,\delta}^t(\bar\theta_\delta^{[t]})
=\int_0^{\flr{t}} \bar R_{f,\delta}^{t,s}\bar\theta_\delta^s\d s
=\delta\sum_{q=0}^{k-1} (\delta^{-1}R_{f,\delta}^{k,q})\theta_\delta^q
=\sum_{q=0}^{k-1} R_{f,\delta}^{k,q} \theta_\delta^q.\]
Then, since $\{\bar\theta_\delta^t,\bar u_\delta^t\}$ are both
$\delta$-p.c.r.c.\ and defined by \eqref{eq:embeddings} and
\eqref{eq:embeddings2}, we have that \eqref{eq:def-dist-theta} is equivalent to
the discrete-time equation \eqref{eq:def-dis-theta},
\[\theta_\delta^k=\theta_\delta^0
-\delta \sum_{r=0}^k
\bar\eta^{r\delta}\Big(\gamma \Gamma_\delta^r\theta_\delta^r
+g(\theta_\delta^r)+\gamma \sum_{q=0}^{r-1} R_{f,\delta}^{r,q}
\theta_\delta^q+\gamma R_{f,\delta}^{r,*}\theta^*\Big)
+\sqrt{\gamma}u_\delta^k.\]
Similar arguments show that
(\ref{eq:def-dis-xi}--\ref{eq:def-dis-deriv-xi-star}) are equivalent to
(\ref{eq:def-dist-xi}--\ref{eq:def-dist-deriv-xi-star}), and multiplying
\eqref{eq:def-dis-deriv-xi} on both sides by $x^s$ and summing over
$s=0,\ldots,t-1$ shows \eqref{eq:def-dist-resp-xi}. It is clear from the
definitions that
(\ref{eq:def-dis-C-t}--\ref{eq:def-dis-Gamma}) are equivalent to
(\ref{eq:def-dist-C-t}--\ref{eq:def-dist-Gamma}), where \eqref{eq:def-dist-R-h}
follows from \eqref{eq:Rfembed} and \eqref{eq:rfembed}.

Now let $\cD_\theta^\delta$ denote the space of all tuples
$Y^\delta=(\bar C_{\theta,\delta},\bar R_{\theta,\delta})$ where
\begin{itemize}
\item $\bar C_{\theta,\delta} \equiv \{\bar C_{\theta,\delta}^{t,s}\}_{t,s \in
[0,T]}$ is a covariance kernel on $(\R^k \otimes [0,T]) \times \R^{k^*}$
that is $\delta$-p.c.r.c.\ in $(t,s)$.
\item $\bar R_{\theta,\delta} \equiv \{\bar R_{\theta,\delta}^{t,s}\}_{t,s \in
[0,T]}$ is a $\R^{k \times k}$-valued process satisfying
$\bar R_{\theta,\delta}^{t,s}=0$ if $s>t$,
$\bar R_{\theta,\delta}^{t,s}=\Id$ if $s \leq t$ with $\flr{s}=\flr{t}$,
and $\bar R_{\theta,\delta}^{t,s}$ is $\delta$-p.c.r.c.\ in $(t,s)$.
\end{itemize}
Likewise, let $\cD_\xi^\delta$ denote the space of all tuples
$X^\delta=(\bar C_{f,\delta},\bar R_{f,\delta}, \bar R_{f,\delta}^*,
\bar \Gamma_\delta)$ where
\begin{itemize}
\item $\bar C_{f,\delta} \equiv \{\bar C_{f,\delta}^{t,s}\}_{s,t \in [0,T]}$
is a covariance kernel on $\R^k \otimes [0,T]$ that is $\delta$-p.c.r.c.\ in
$(t,s)$,
\item $\bar R_{f,\delta}^* \equiv \{\bar R_{f,\delta}^{t,*}\}_{t \in [0,T]}$
and $\bar \Gamma_\delta \equiv \{\bar \Gamma_\delta^t\}_{t \in [0,T]}$ are
$\R^{k \times k^*}$-valued and $\R^{k \times k}$-valued processes,
respectively, that are $\delta$-p.c.r.c.\ in $t$,
\item $\bar R_{f,\delta} \equiv \{\bar R_{f,\delta}^t\}_{t \in [0,T]}$ is a
family of linear operators $\bar R_{f,\delta}^t:D([0,t],\R^k) \to \R^k$ having
the forms, for some matrix-valued coefficients $R_{f,\delta}^{k,j}$,
\begin{equation}\label{eq:barRform}
\bar R_{f,\delta}^t(x^{[t]})=
\sum_{j=0}^{k-1} R_{f,\delta}^{k,j}x^{j\delta} \text{ if } \flr{t}=k.
\end{equation}
\end{itemize}
Given any $X^\delta=(\bar C_{f,\delta},\bar R_{f,\delta},
\bar R_{f,\delta}^*,\bar \Gamma_\delta) \in \cD_\xi^\delta$, note that
$\{\bar u_\delta^t\} \sim \GP(0,\bar C_{f,\delta})$ is $\delta$-p.c.r.c.\ with
probability 1. Then it may be checked inductively in $k=0,1,2,\ldots$ that 
\eqref{eq:def-dist-theta} and \eqref{eq:def-dist-deriv-t} have
unique solutions over $t,s \in [0,k\delta)$, which are $\delta$-p.c.r.c. Then
$Y^\delta=(\bar C_{\theta,\delta},\bar R_{\theta,\delta})$ defined by
(\ref{eq:def-dist-C-t}--\ref{eq:def-dist-R-t}) belongs to $\cD_\theta^\delta$,
i.e.\ \eqref{eq:def-dist-theta}, \eqref{eq:def-dist-deriv-t}, and 
(\ref{eq:def-dist-C-t}--\ref{eq:def-dist-R-t}) define a mapping
\[\trsfrmDA:\cD_\xi^\delta \to \cD_\theta^\delta.\]
Similarly, given any
$Y^\delta=(\bar C_{\theta,\delta},\bar R_{\theta,\delta})
\in \cD_\theta^\delta$, it may be checked inductively that 
\eqref{eq:def-dist-xi} and \eqref{eq:def-dist-deriv-xi-star} have unique
solutions which are $\delta$-p.c.r.c.\ in $t$, and that there exist (random)
matrix coefficients $r_{f,\delta}^{k,j}$ such that for any
$x \in D([0,T],\R^k)$, the solution to \eqref{eq:def-dist-resp-xi} takes a form
\[\bar r_{f,\delta}^t(x^{[t]})
=\sum_{j=0}^{k-1} r_{f,\delta}^{k,j}x^{j\delta}
\text{ if } \flr{t}=k.\]
Then $X^\delta=(\bar C_{f,\delta},\bar R_{f,\delta},
\bar R_{f,\delta}^*,\bar \Gamma_\delta)$ defined by
(\ref{eq:def-dist-C-h}--\ref{eq:def-dist-Gamma}) belongs to $\cD_\xi^\delta$,
i.e.\ \eqref{eq:def-dist-xi},
(\ref{eq:def-dist-deriv-xi-star}--\ref{eq:def-dist-resp-xi}), and
(\ref{eq:def-dist-C-h}--\ref{eq:def-dist-Gamma}) define a mapping
\[\trsfrmDB:\cD_\theta^\delta \to \cD_\xi^\delta.\]
These mappings may be understood as
discretized versions of $\trsfrmA$ and $\trsfrmB$ from Section
\ref{sec:existenceuniqueness}. Now denoting by
\[X^\delta=(\bar C_{f,\delta},\bar R_{f,\delta},
\bar R_{f,\delta}^*,\bar \Gamma_\delta),
\quad Y^\delta=(\bar C_{\theta,\delta},\bar R_{\theta,\delta})\]
the specific elements of $\cD_\xi^\delta$ and $\cD_\theta^\delta$ that
correspond to the previously defined embeddings of the discrete-time DMFT
recursions,
since these satisfy (\ref{eq:def-dist-theta}--\ref{eq:def-dist-Gamma}), they
are a fixed point of these mappings, i.e.
\[\trsfrmDA(X^\delta)=Y^\delta,
\quad \trsfrmDB(Y^\delta)=X^\delta.\]

Lemma \ref{lem:dmft_discretization_error}(a) then holds by the following
reasoning: On one hand, $(X^\delta,Y^\delta) \in \cD_\xi^\delta \times
\cD_\theta^\delta$ is the unique such
fixed point of these mappings, because the equations
(\ref{eq:def-dist-theta}--\ref{eq:def-dist-Gamma}) uniquely determine,
recursively in time,
\[\begin{gathered}
\{\bar C_{\theta,\delta}^{t,s},\bar R_{\theta,\delta}^{t,s}\}_{t,s:\flr{t},\flr{s}=0}
\Rightarrow
\{\bar C_{f,\delta}^{t,s},\bar
R_{f,\delta}^t,\bar R_{f,\delta}^{t,*},\bar \Gamma_\delta^t\}_{t,s:\flr{t},\flr{s}=0}
\Rightarrow\\
\{\bar C_{\theta,\delta}^{t,s},\bar
R_{\theta,\delta}^{t,s}\}_{t,s:\flr{t},\flr{s} \in \{0,\delta\}}
\Rightarrow
\{\bar C_{f,\delta}^{t,s},\bar
R_{f,\delta}^t,\bar R_{f,\delta}^{t,*},\bar \Gamma_\delta^t\}_{t,s:\flr{t},\flr{s} \in \{0,\delta\}} \Rightarrow\\
\{\bar C_{\theta,\delta}^{t,s},\bar
R_{\theta,\delta}^{t,s}\}_{t,s:\flr{t},\flr{s} \in \{0,\delta,2\delta\}}
\Rightarrow \{\bar C_{f,\delta}^{t,s},\bar
R_{f,\delta}^t,\bar R_{f,\delta}^{t,*},\bar
\Gamma_\delta^t\}_{t,s:\flr{t},\flr{s} \in \{0,\delta,2\delta\}} \Rightarrow
\ldots
\end{gathered}\]
On the other hand, let us extend each operator
$\bar R_{f,\delta}^t:D([0,t],\R^k) \to \R^k$ having the form
\eqref{eq:barRform} to an operator
$\bar R_{f,\delta}^t:L^4([0,t],\R^k) \to \R^k$ given by
\[\bar R_{f,\delta}^t(x^{[t]})=\int_0^{\flr{t}} \bar R_{f,\delta}^{t,s}x^s\d s,
\quad \bar R_{f,\delta}^{t,s}=\delta^{-1}R_{f,\delta}^{k,j} \text{ if }
\flr{s}=j<\flr{t}=k.\]
In particular, this extension of $\bar R_{f,\delta}^t:D([0,t],\R^k) \to \R^k$
defined via \eqref{eq:def-dist-R-h} and
\eqref{eq:rfembed} is precisely the embedded operator we have defined in
\eqref{eq:Rfembed}. In this way,
we may identify $\cD_\xi^\delta \cap \cS_\xi(T,C_0)$ as a subset of the
space $\cS_\xi(T,C_0)$ of Section \ref{sec:existenceuniqueness}, where the
discontinuity set $D \subset [0,T]$ of Definition \ref{def:S}
is $\delta\Z_+ \cap [0,T]$.

Then for any sufficiently large constants $C_0 \equiv C_0(T)>0$
and $\lambda \equiv \lambda(T,C_0)>0$,
the same argument as in Lemma \ref{lem:solution-in-space} shows that 
if $X^\delta \in \cD_\xi^\delta \cap \cS_\xi(T,C_0)$, then 
$Y^\delta=\trsfrmDA(X^\delta)$ defined via
\eqref{eq:def-dist-theta}, \eqref{eq:def-dist-deriv-t}, and 
(\ref{eq:def-dist-C-t}--\ref{eq:def-dist-R-t}) satisfies
$Y^\delta \in \cD_\theta^\delta \cap \cS_\theta(T,C_0)$;
the same argument as in Lemma \ref{lem:solution-in-space-2} shows that if
$Y^\delta \in \cD_\theta^\delta \cap \cS_\theta(T,C_0)$, then 
$X^\delta=\trsfrmDB(Y^\delta)$ defined via
\eqref{eq:def-dist-xi},
(\ref{eq:def-dist-deriv-xi-star}--\ref{eq:def-dist-resp-xi}), and
(\ref{eq:def-dist-C-h}--\ref{eq:def-dist-Gamma}) satisfy
$X^\delta \in \cD_\xi^\delta \cap \cS_\xi(T,C_0)$;
and the same arguments as in Lemmas \ref{lem:contract-Ct} and
\ref{lem:contract-fsys} show that the combined map
\[\trsfrmD \equiv \trsfrmDA \circ \trsfrmDB\]
is contractive in the metric $\lbddst{\cdot}{\cdot}$ on
$\cD_\theta^\delta \cap \cS_\theta(T,C_0)$. The conditions that
$(C_\theta,R_\theta) \in \cS_\theta(T,C_0)$ are $\delta$-p.c.r.c.\ and
hence belong to $\cD_\theta^\delta$ are closed under this metric, so
$\cD_\theta^\delta \cap \cS_\theta(T,C_0)$ is also complete 
in $\lbddst{\cdot}{\cdot}$. Then by the Banach fixed-point theorem,
there exists a fixed point of $\trsfrmD:\cD_\theta^\delta \cap \cS_\theta(T,C_0)
\to \cD_\theta^\delta \cap \cS_\theta(T,C_0)$,
which must coincide with $Y^\delta$
by the above claim of uniqueness of this fixed point in $\cD_\theta^\delta$.
This shows that $Y^\delta \in 
\cD_\theta^\delta \cap \cS_\theta(T,C_0)$, and hence also
$X^\delta \in \cD_\xi^\delta \cap \cS_\xi(T,C_0)$,
i.e.\ $(Y^\delta,X^\delta) \in \cS(T,C_0)$ as claimed.
\end{proof}

\begin{proof}[Proof of Lemma \ref{lem:dmft_discretization_error}(b)]
We compare the outputs of the continuous map $\trsfrmA$
and the discretized map $\trsfrmDA$ when both are
given the same discrete input $X^\delta \in \mathcal D_\xi^\delta \cap
\cS_\xi(T,C_0)$. Given any such input
$X^\delta=(\bar C_{f,\delta},\bar R_{f,\delta},\bar
R_{f,\delta}^*,\bar\Gamma_\delta)$, consider the continuous
process $\{\theta^t\}_{t \in [0,T]}$ and the embedded discrete process
$\{\bar\theta^t_\delta\}_{t \in [0,T]}$ coupled using the same underlying
realization of $\{\bar u_\delta^t\}_{t \in [0,T]} \sim \GP(0,\bar C_{f,\delta})$:
\begin{align*}
    \theta^t&=\theta^0-\int_0^t \bar\eta^r \Big({\gamma}\bar\Gamma^r_\delta
\theta^r+g(\theta^r)+{\gamma}\bar R_{f,\delta}^r(\theta^{[r]})+
{\gamma}\bar R_{f,\delta}^{r,*}\theta^*\Big)\d r+\sqrt{\gamma}\,\bar u^{t}_\delta\, , \\
    \bar\theta^t_\delta&=\theta^0-\int_0^{\flr{t}} \bar\eta^{\flr{r}} \Big({\gamma}\bar\Gamma^r_\delta
\bar\theta^r_\delta+g(\bar\theta^r_\delta)+{\gamma}\bar R_{f,\delta}^r(\bar\theta^{[r]}_\delta)+
{\gamma}\bar R_{f,\delta}^{r,*}\theta^*\Big)\d r+\sqrt{\gamma}\,\bar u^t_\delta \, .
\end{align*}
Then
\begin{align*}
    \bar\theta^t_\delta - \theta^t &= - \int_0^{\flr{t}} \bar\eta^r \prn{{\gamma}\bar\Gamma_\delta^{r}(\bar \theta^{r}_\delta - \theta^r)+ (g(\bar \theta^{r}_\delta)-g(\theta^r)) + {\gamma} \bar R_{f,\delta}^{r}\prn{\prn{\bar\theta_\delta - \theta}^{[r]}} } \d r \nonumber\\
    &\qquad+ \int_0^{\flr{t}} \prn{\bar\eta^r - \bar\eta^{\flr{r}}}\prn{g(\bar \theta^{r}_\delta) +  {\gamma}\prn{\bar\Gamma_\delta^{r}\bar \theta^{r}_\delta+ \bar R_{f,\delta}^{r}\prn{{\bar\theta_\delta}^{[r]}} + \bar R_{f,\delta}^{r,*}\theta^* }} \d r \nonumber\\
    &\qquad+ \int_{\flr{t}}^t \bar\eta^r \Big({\gamma}\bar\Gamma^r_\delta
\theta^r+g(\theta^r)+{\gamma}\bar R_{f,\delta}^r(\theta^{[r]})+
{\gamma}\bar R_{f,\delta}^{r,*}\theta^*\Big)\d r.
\end{align*}
As a consequence, $\E\brk{\norm{\bar\theta^t_\delta - \theta^t}^2} \leq
4((\mathrm{I})+(\mathrm{II})+(\mathrm{III})+(\mathrm{IV}))$ where
\begin{align*}
(\mathrm{I})&=\E\left\|\int_0^{\flr{t}}
\bar\eta^r \prn{{\gamma}\bar\Gamma_\delta^{r}(\bar \theta^{r}_\delta -
\theta^r)+ (g(\bar \theta^{r}_\delta)-g(\theta^r))}\d r\right\|^2,\\
(\mathrm{II})&=\E\left\|\int_0^{\flr{t}}
\bar\eta^r {{\gamma} \bar R_{f,\delta}^{r}\prn{\prn{\bar\theta_\delta -
\theta}^{[r]}} } \d r\right\|^2,\\
(\mathrm{III})&=\E\left\|\int_0^{\flr{t}}
\prn{\bar\eta^r - \bar\eta^{\flr{r}}}\prn{g(\bar \theta^{r}_\delta) +
{\gamma}\prn{\bar\Gamma_\delta^{r}\bar \theta^{r}_\delta+ \bar
R_{f,\delta}^{r}\prn{{\bar\theta_\delta}^{[r]}} + \bar
R_{f,\delta}^{r,*}\theta^* }} \d r \right\|^2,\\
(\mathrm{IV})&= \E\left\|\int_{\flr{t}}^t \bar\eta^r \Big({\gamma}\bar\Gamma^r_\delta
\theta^r+g(\theta^r)+{\gamma}\bar R_{f,\delta}^r(\theta^{[r]})+
{\gamma}\bar R_{f,\delta}^{r,*}\theta^*\Big)\d r\right\|^2.
\end{align*}
By the Lipschitz continuity of $g(\cdot)$ and bound \eqref{eq:pRf-volterra}
for $\bar R_{f,\delta}$, the first two terms can be bounded by
\[\nI+\nII \lesssim
\frac{1}{\lambda}e^{2\lambda t}\cdot\sup_{s\in[0,T]} e^{-2\lambda s}
\E\brk{\norm{\theta^s-\bar \theta^s_\delta}^2}\]
where $\lesssim$ denotes inequality up to a constant not depending on $\delta$.
For the remaining terms, note we may show analogously to Lemma
\ref{lem:exist-unique-theta} that $\E\|\bar \theta_\delta^t\|^4 \lesssim 1$
and $\E\|R_{f,\delta}^t(\bar \theta_\delta^{[t]})\|^4 \lesssim 1$. Then
for the third term, by the Lipschitz continuity of
$\{\bar \eta^t\}_{t \in [0,T]}$
in Assumption \ref{asp:SGD}, we have $\int_0^T (\bar \eta^r-\bar
\eta^{\flr{r}})^2 \d r \lesssim \delta^2$ for a constant $C>0$, and hence also
$\nIII \lesssim \delta^2$ by Cauchy-Schwarz. Finally $\nIV \lesssim \delta^2$
by the conditions for $X^\delta \in \cS_\xi(T,C_0)$. Then
choosing large enough $\lambda$ yields $\sup_{s\in[0,T]} e^{-2\lambda s}
\E\norm{\theta^s-\bar \theta^s_\delta}^2 \lesssim \delta^2$, which implies
$\lbddst{C_\theta}{\bar C_{\theta,\delta}} \lesssim \delta$.
A similar argument shows $\lbddst{R_\theta}{\bar R_{\theta,\delta}} \lesssim
\delta$, establishing
\begin{equation}\label{eq:trsfrmAclose}
\lbddst{\trsfrmA(X^\delta)}{\trsfrmDA(X^\delta)} \lesssim \delta
\text{ for all } X^\delta \in \cD_\xi^\delta \cap \cS_\xi(T,C_0).
\end{equation}

Analogously, we may compare $\trsfrmB$
and $\trsfrmDB$ given the same discrete input $Y^\delta=(\bar C_{\theta,\delta},\bar R_{\theta,\delta}) \in \cD_\theta^\delta \cap \cS_\theta(T,C_0)$. Coupling by the same realizations of $(\{\bar w_\delta^t\}_{t \in [0,T]},w^*_\delta) \sim \GP(0,\bar C_{\theta,\delta})$ an $\{z^t\}_{t \in [0,T]}$, we have
\begin{align*}
    \xi^t&={-}\int_0^{{t}} \frac{\bar\eta^{{r}}}{\bar\kappa} \bar R_{\theta,\delta}^{t,r}
f(\xi^{r-},w_\delta^*,\eps)\d z^r+\bar
w^t_\delta\, , \\
    \bar\xi^t_\delta&={-}\int_0^{\flr{t}} \frac{\bar\eta^{\flr{r}}}{\bar\kappa} \bar R_{\theta,\delta}^{t,r}
f(\bar\xi^{r-}_\delta,w_\delta^*,\eps)\d z^r+\bar
w^t_\delta \, .
\end{align*}
Then 
\begin{align*}
    \xi^t - \bar\xi^t_\delta &= - \underbrace{\int_0^{\flr{t}} \frac{\bar\eta^{{r}}}{\bar\kappa} \bar R_{\theta,\delta}^{t,r} \prn{f(\xi^{r-},w_\delta^*,\eps) - f(\bar\xi^{r-}_\delta,w_\delta^*,\eps)} \d z^r}_{(\mathrm{I})} \nonumber\\
    &\qquad- \underbrace{\int_0^{\flr{t}} \frac{\bar\eta^r - \bar\eta^{\flr{r}}}{\bar\kappa}R_{\theta,\delta}^{t,r} {f(\bar\xi^{r-}_\delta,w_\delta^*,\eps)} \d z^r}_{(\mathrm{II})} -  \underbrace{\int_{\flr{t}}^t \frac{\bar\eta^{{r}}}{\bar\kappa} \bar R_{\theta,\delta}^{t,r}
f(\xi^{r-},w_\delta^*,\eps)\d z^r}_{(\mathrm{III})}. 
\end{align*}
We can bound $\E[\norm{\xi^t - \bar\xi^t_\delta}^2] \leq 3(\E[(\mathrm{I})^2]+ \E[(\mathrm{II})^2] + \E[(\mathrm{III})^2])$.
For $\E[(\mathrm{I})^2]$, by using the Burkholder-Davis-Gundy inequality of
Lemma \ref{lem:burkholder}, boundedness of $\bar R_{\theta,\delta}$, and Lipschitz continuity of $f$, as in the above argument, we have
\[\E[(\mathrm{I})^2] \lesssim
\frac{1}{\lambda}e^{2\lambda t}\cdot\sup_{s\in[0,T]} e^{-2\lambda s}
\E\brk{\norm{\xi^s-\bar \xi^s_\delta}^2}.\]
Similarly to \eqref{eq:xi-fourth-moment} we have
$\E\|f(\bar \xi_\delta^t,w_\delta^*,\eps)\|^4 \lesssim 1$, and hence
analogously $\E[(\mathrm{II})^2] \lesssim \delta^2$ and $\E[(\mathrm{III})^2]
\lesssim \delta$. Then choosing $\lambda$ large enough and rearranging, $\sup_{s\in[0,T]} e^{-2\lambda s}
\E\norm{\xi^s-\bar \xi^s_\delta}^2 \lesssim \delta$. Then by an analogous argument,
\[
\lbddst{C_f}{\bar C_{f,\delta}}^2 \lesssim \sup_{t\in [0,T]} e^{-2\lambda t}\E \left\|\int_0^t \frac{\bar\eta^r}{\bar \kappa}f(\xi^{r-},w^*_\delta,\eps) \d z^r- \int_0^{\flr{t}} \frac{\bar\eta^{\flr{r}}}{\bar \kappa}f(\bar\xi^{r-}_\delta,w^*_\delta,\eps)\d z^r\right\|^2 \lesssim \delta\,,
\]
that is,
$\lbddst{C_f}{\bar C_{f,\delta}} \lesssim \delta^{1/2}$.
Similarly, one may analyze $R_f$ and $R_f^*$ to show that
\begin{equation}\label{eq:trsfrmBclose}
\lbddst{\trsfrmB(Y^\delta)}{\trsfrmDB(Y^\delta)} \lesssim \delta^{1/2}
\text{ for all } Y^\delta \in \cD_\theta^\delta \cap \cS_\theta(T,C_0),
\end{equation}
and we omit this verification for brevity.

Now let $Y^\delta \in \cD_\theta^\delta \cap \cS_\theta(T,C_0)$ be the fixed
point of $\trsfrmD$, and let $Y \in \cS_\theta^\text{cont}(T,C_0)$ be the fixed
point of $\trsfrm$. Then
\begin{align*}
	\lbddst{Y}{Y^\delta} & =
\lbddst{\mathcal{T}(Y)}{\mathcal{T}^\delta(Y^\delta)}  \leq
\underbrace{\lbddst{\mathcal{T}(Y)}{\mathcal{T}(Y^\delta)}}_{\mathrm{(I)}} +
\underbrace{\lbddst{\mathcal{T}(Y^\delta)}{
\mathcal{T}^\delta(Y^\delta)}}_{\mathrm{(II)}}.
\end{align*}
By the contractivity \eqref{eq:T-contraction} of $\trsfrm$, we can choose
$\lambda$ large enough such that
\begin{align}
	\mathrm{(I)} \leq \frac 1 2 \lbddst{Y}{Y^\delta} \, . \label{eq:mid-integral-differential-approximation-2}
\end{align}
To control $\mathrm{(II)}$, we have
\begin{align*}
\nII&=\lbddst{\trsfrmA \circ \trsfrmB(Y^\delta)}{\trsfrmDA \circ
\trsfrmDB(Y^\delta)}\\
&\leq \lbddst{\trsfrmA \circ \trsfrmB(Y^\delta)}{\trsfrmA \circ
\trsfrmDB(Y^\delta)}\\
&\hspace{1in}+\lbddst{\trsfrmA \circ \trsfrmDB(Y^\delta)}{\trsfrmDA \circ
\trsfrmDB(Y^\delta)} \lesssim \delta^{1/2}
\end{align*}
where the last inequality holds
by \eqref{eq:trsfrmAclose}, \eqref{eq:trsfrmBclose}, and Lemma
\ref{lem:contract-Ct}. Thus $\lbddst{Y}{Y^\delta} \lesssim \delta^{1/2}$.
Setting $X=\trsfrmB(Y)$ and $X^\delta=\trsfrmDB(Y^\delta)$, this
implies by Lemma \ref{lem:contract-fsys} that also
$\lbddst{X}{X^\delta} \lesssim \delta^{1/2}$, so
\[\lim_{\delta \to 0} \lbddst{X}{X^\delta}=0,
\qquad \lim_{\delta \to 0} \lbddst{Y}{Y^\delta}=0.\]
\end{proof}

\begin{proof}[Proof of Lemma \ref{lem:dmft_discretization_error}(c)]
The convergence $\lim_{\delta\to 0}\lbddst{C_f}{\bar C_{f,\delta}}=0$
implies that there exists a coupling of $\{u^t\}_{t \in [0,T]} \sim
\GP(0,C_f)$ and $\{\bar u^t_\delta\}_{t \in [0,T]} \sim \GP(0,\bar C_{f,\delta})$ for which
\[\lim_{\delta \to 0} \sup_{t \in [0,T]} \E\norm{u^t-\bar u^t_\delta}^4=0.\]
Then defining $\{\theta^t\}_{t \in [0,T]}$ via $(C_f,R_f,R_f^*,\Gamma)$
and $\{\bar \theta_\delta^t\}_{t \in [0,T]}$ via
$(\bar C_{f,\delta},\bar R_{f,\delta},\bar R_{f,\delta}^*,\bar
\Gamma_\delta)$ using this coupling of $\{u^t\}$ and $\{\bar u_\delta^t\}$, the same arguments as leading to (\ref{eq:theta4diffbound}) show
\[\lim_{\delta \to 0} \sup_{t \in [0,T]} e^{-4\lambda t}
\E\norm{\theta^t-\bar \theta^t_\delta}^4=0.\]
In particular,
$\lim_{\delta \to 0} \E[\norm{\theta^{t_1}-\bar
\theta^{t_1}_\delta}^2+\ldots+\norm{\theta^{t_m}-\bar \theta^{t_m}_\delta}^2]
=0$, which implies the Wasserstein-2 convergence
(\ref{eq:discretedmftthetaconv}). A similar argument shows
(\ref{eq:discretedmftetaconv}).
\end{proof}

\section{Discretization Error of SGD and SME}\label{sec:disc-sgd-sme}

In this section, we conclude the proof of Theorem \ref{thm:DMFT-limit} by
analyzing the SGD/SME dynamics \eqref{eq:SGDrescaled} and
\eqref{eq:SME} in $\R^d$, and
showing that they are also well-approximated by the discrete-time dynamics
\eqref{eq:discretedynamics} as $\delta \to 0$.

\subsection{Discretization of SGD}

Recall the SGD update of \eqref{eq:SGD},
\[\bar\btheta^{k+1}-\bar\btheta^k=-\eta^k {\left(\frac{1}{\kappa}
\sum_{i \in S^k} \x_i \otimes
f(\x_i^\top \bar\btheta^k,\x_i^\top \btheta^*,\eps_i)
+\frac{1}{n}\,g(\bar\btheta^k)\right)}, \qquad |S^k|=\kappa,\]
where $g(\cdot)$ is applied row-wise. We will analyze a Poissonized
version of this process: Recall $\set=\{S \subset [n]:|S|=\kappa\}$, and let
\[\bN^t=\{N_S^t\}_{S \in \set}\]
be a $\R^{\binom{n}{\kappa}}$-valued process indexed by $\set$, whose
coordinates are independent homogeneous Poisson jump processes with rate
$n^{1-\alpha}/\binom{n}{\kappa}$. Set
\[N^t=\1^\top \bN^t=\sum_{S \in \set} N_S^t.\]
Marginally, $N^t$ is a Poisson jump process with rate $n^{1-\alpha}$. We define
a Poissonized version of the SGD dynamics $\{\bar\btheta^k\}_{k \geq 0}$
(without time rescaling) by
\begin{equation}\label{eq:SGDPoissonized}
\bnu^t=\bar\btheta^{N^t}, \qquad \bnu^0=\btheta^0.
\end{equation}
Then $\bnu^t$ is a continuous-time jump process given by
\begin{align}
\bnu^t&=\bnu^0-\int_0^t \eta^{N^{s-}} \sum_{S \in \set}\prn{\frac{1}{\kappa}
\sum_{i \in S} \x_i \otimes f(\x_i^\top \bnu^{s-},\x_i^\top \btheta^*,\eps_i)
+ \frac{1}{n} g(\bnu^{s-})} \d N^s_S\notag\\
&=\bnu^0-\int_0^t \eta^{N^{s-}} \prn{\frac{1}{\kappa}\sum_{i=1}^n
\x_i \otimes f(\x_i^\top \bnu^{s-},\x_i^\top \btheta^*,\eps_i)\sum_{S \in
\set}
\1_{i\in S} \d N^s_S+\frac{1}{n} g(\bnu^{s-}) \sum_{S \in \set}
\d N^s_S}\notag\\
&=\bnu^0-\int_0^t \eta^{N^{s-}} \prn{\frac{1}{\kappa}\,\X^\top
\diag(\bPi\,\d\bN^s)f(\X\bnu^{s-},\X\btheta^*,\beps))
+\frac{1}{n} g(\bnu^{s-}) \d N^s}. \label{eq:poisson-nu}
\end{align} 
Here $f(\cdot)$ is also applied row-wise, and $\bPi\in\R^{n\times{\binom{n}{\kappa}}}$ is the incidence matrix with entries $\Pi_{i,S}=\1_{i \in S}$.

Fixing a discretization step size $\delta>0$, in this section, we continue
to write $\flr{t}$ and $\cil{t}$ for the time discretization notation
previously defined in \eqref{eq:flrcil}. We then consider an approximation
$\{\bnu_\delta^t\}_{t \geq 0}$ to the above Poissonized process, defined by
\begin{equation}\label{eq:deltaSGD}
\bnu_\delta^t = \bnu^0 - \int_0^{\flr{t}}
\frac{\bar\eta^{\flr{s}}}{\bar\kappa} \X^\top \diag(\bPi\,\d\bN^s)
f(\X\bnu_\delta^{s-},\X\btheta^*,\beps)
-\int_0^{\flr{t}} \bar \eta^{\flr{s}} g(\bnu_\delta^{s}) \d s.
\end{equation}
It may be checked that $\{\bnu_\delta^t\}_{t \geq 0}$ is a
piecewise-constant embedding of the discrete-time dynamics
\eqref{eq:discretedynamics} in the SGD setting of
$\z_\delta^k=\bPi\bP_\delta^k$,
\[\bnu_\delta^t=\btheta_\delta^k \text{ if } \flr{t}=k,\]
upon identifying $\bP_\delta^k=\int_{k\delta}^{(k+1)\delta} \d\bN^s
=\bN^{(k+1)\delta}-\bN^{k\delta}$.

The main result of this section is the following comparison lemma.

\begin{lem}\label{lem:disc-sgd}
For any fixed $T>0$, $\delta \in (0,1)$,
and some constant $C \equiv C(T,\alpha)>0$ not depending on $\delta$,
almost surely as $n,d \to \infty$,
\[\limsup_{n,d\to
\infty}\sup_{t\in[0,T]}\frac{\norm{\bnu^t-\bnu_\delta^t}^2}{d} \leq
C\delta.\]
\end{lem}

We proceed to prove Lemma \ref{lem:disc-sgd}. Let
\[\bZ^t=\bN^t-t\frac{n^{1-\alpha}}{\binom{n}{\kappa}}\mathbf{1}_{\binom{n}{\kappa}}\]
be the compensated Poisson process obtained by centering $\{\bN^t\}_{t \geq 0}$,
i.e.\ $\1_{\binom{n}{\kappa}}$ denotes the all-1's vector in $\R^{\binom{n}{\kappa}}$ and $\E \bZ^t=0$. The proof will rely on martingale
bounds for integrals against $\{\bPi\bZ^t\}_{t \geq 0}$, which are encapsulated
in the following lemma.

\begin{lem}\label{lemma:Poissonmartingales}
Consider any $m \equiv m(n) \geq 1$,
let $\{\A^t\}_{t \geq 0}$ be any $\R^{m \times n}$-valued process that is
predictable with respect to $\{\bZ^t\}_{t \geq 0}$, and define
\[\M^t=\int_0^t \A^s \bPi \d \bZ^s,
\qquad E_{p,t}=\int_0^t \E\|\A^s\|_{\sF}^p\d s
+\int_0^t \E\|\A^s\1_n\|^p\d s.\]
Then for any $t \geq 0$ and integer $p \geq 1$,
there exists a constant $C_{2p,t}>0$ such that
\begin{equation}\label{eq:Mtmomentbound}
\E\sup_{u \in [0,t]} \|\M^u\|^{2p} \leq C_{2p,t}E_{2p,t}.
\end{equation}
Furthermore, there exists a constant $C>0$ and
martingale $\{Q^t\}_{t \geq 0}$ such that (with probability 1)
\begin{equation}\label{eq:Mtasbound}
\|\M^t\|^2 \leq C\int_0^t (\|\A^s\|_\sF^2
+\|\A^s\1_n\|^2)\d s+|Q^t|,
\end{equation}
and
\begin{equation}\label{eq:Qtmomentbound}
\E \sup_{u \in [0,t]} (Q^u)^{2p} \leq C_{2p,t}\left(n^{-(p-1)(1-\alpha)}E_{4p,t}
+n^{p+2}\int_0^t \E\|\A^s\|_\op^{4p} \d s\right).
\end{equation}
\end{lem}
\begin{proof}
We write $\lesssim$ to absorb constants depending on $p$ and $t$.

For \eqref{eq:Mtmomentbound}, we apply the Burkholder-Davis-Gundy inequality of 
Lemma \ref{lem:burkholderfull}. Let $\e_S$ be the standard basis
vector corresponding to the coordinate $S \in \set$. Then 
Lemma \ref{lem:burkholderfull} gives, for any $p \geq 1$,
\begin{align}
\E\sup_{u\in[0,t]}\|\M^u\|^{2p}
&\lesssim 
\E\left(\int_0^t\frac{n^{1-\alpha}}{\binom n\kappa}
\sum_{S\in\mathscr S_\kappa} \|\A^s\bPi \e_S\|^2 \d s\right)^p
+\E\int_0^t\frac{n^{1-\alpha}}{\binom n\kappa}
\sum_{S\in\mathscr S_\kappa}
\|\A^s\bPi\e_S\|^{2p} \d s.\label{eq:poisson-M-BDG-direct}
\end{align}
Let $\A^s=[\a_1^s,\ldots,\a_n^s]$. Observe that
\[\frac{1}{\binom n\kappa} \sum_{S\in\mathscr S_\kappa}
\|\A^s\bPi\e_S\|^{2p}
=\E_S \left\|\sum_{i \in S} \a_i^s\right\|^{2p}
\leq \E_{S'}\left\|\sum_{i \in S'} \a_i^s\right\|^{2p},\]
where $\E_S$ is the expectation over the set $S \subset [n]$ of $\kappa$
indices sampled uniformly without replacement, $\E_{S'}$ is the
expectation over a multi-set $S'$ of $\kappa$ indices sampled uniformly and
independently with replacement,
and the inequality holds since $\a \mapsto \|\a\|^{2p}$ is convex
\cite[Theorem 4]{hoeffding1963probability}. Writing $\bar \a^s=(1/n)\A^s\1_n$
for the mean of $\a_i^s$ under uniform sampling of $i \in [n]$,
the vector Rosenthal inequality \cite[Theorem 1]{hansen2015integrated} gives
\begin{align*}
\E_{S'}\left\|\sum_{i \in S'} (\a_i^s-\bar \a^s)\right\|^{2p}
&\lesssim \left(\frac{\kappa}{n}\sum_{i=1}^n \|\a_i^s-\bar \a^s\|^2\right)^p
+\frac{\kappa}{n}\sum_{i=1}^n \|\a_i^s-\bar \a^s\|^{2p}
\end{align*}
Hence, using $\kappa/n \asymp n^{-(1-\alpha)}$,
\begin{align}
\frac{1}{\binom n\kappa} \sum_{S\in\mathscr S_\kappa}
\|\A^s\bPi\e_S\|^{2p}
&\lesssim \|\kappa \bar \a^s\|^{2p}
+\left(\frac{\kappa}{n}\sum_{i=1}^n \|\a_i^s\|^2\right)^p
+\frac{\kappa}{n}\sum_{i=1}^n \|\a_i^s\|^{2p}\notag\\
&\lesssim n^{-p(1-\alpha)}\|\A^s\|_\sF^{2p}
+n^{-2p(1-\alpha)}\|\A^s\1_n\|^{2p}
+n^{-(1-\alpha)}\sum_{i=1}^n \|\a_i^s\|^{2p}.\label{eq:APipthmoment}
\end{align}
For $p=1$, the first and third terms of \eqref{eq:APipthmoment} coincide,
while for $p \geq 1$ we have
$\sum_{i=1}^n \|\a_i^s\|^{2p} \leq \|\A^s\|_\sF^{2p}$. Then applying this
and H\"older's inequality
to the two terms of \eqref{eq:poisson-M-BDG-direct} yields
\begin{align*}
\E\sup_{u \in [0,t]} \|\M^u\|^{2p}
&\lesssim \E\left(\int_0^t (\|\A^s\|_\sF^2+n^{-(1-\alpha)}\|\A^s\1_n\|^2)\d s
\right)^p\\
&\qquad+\E\int_0^t \left(n^{-(p-1)(1-\alpha)}\|\A^s\|_\sF^{2p}
+n^{-(2p-1)(1-\alpha)}\|\A^s\1_n\|^{2p}+\sum_{i=1}^n \|\a_i\|^{2p}\right)\d s\\
&\lesssim \int_0^t (\E \|\A^s\|_\sF^{2p}+n^{-p(1-\alpha)}\|\A^s\1_n\|^{2p})\d s
\lesssim E_{2p,t}
\end{align*}
which is \eqref{eq:Mtmomentbound}.

For \eqref{eq:Mtasbound},
by It\^o's formula, for any continuously-differentiable $F:\R^m \to \R$,
\begin{equation}\label{eq:Ito-jump}
F(\M^t)-F(\mathbf{0})=\int_0^t \nabla F(\M^{s-})^\top \d \M^s + \sum_{0<s\leq t}
(F(\M^s)-F(\M^{s-})-\nabla F(\M^{s-})^\top \Delta \M^s)
\end{equation}
where $\Delta\M^s=\M^s-\M^{s-}$. In particular,
\[
\norm{\M^t}^2=\underbrace{2\int_0^t {\M^{s-}}^{\top}\A^s\bPi\d\bZ^s}_{:=Q_1^t}
+\sum_{0<s\leq t}\norm{\Delta \M^s}^2\, .
\]
Here, each jump of $\M^s$ occurs at a jump of $N_S$ for some $S \in \set$, in
which case $\Delta \M^s = \A^s\bPi\e_S$. Hence
\begin{align*}
\sum_{0<s\leq t}\norm{\Delta \M^s}^2 &= \sum_{S \in \set} \int_0^t
\norm{\A^s\bPi\e_S}^2 \d {N}^s_S\\
&=\underbrace{\sum_{S \in \set} \int_0^t \norm{\A^s\bPi\e_S}^2 \d
Z^s_S}_{:=Q_2^t} + \underbrace{\frac{n^{1-\alpha}}{\binom{n}{\kappa}}\sum_{S
\in \set} \int_0^t \norm{\A^s\bPi\e_S}^2 \d s}_{:=P^t}.
\end{align*}
Thus
\begin{equation}\label{eq:Mtdecomp}
\|\M^t\|^2=P^t+Q_1^t+Q_2^t.
\end{equation}

Applying again \eqref{eq:APipthmoment} with $p=1$,
\begin{equation}\label{eq:Ptbound}
P^t \leq C\int_0^t (\|\A^s\|_\sF^2+\|\A^s\1_n\|^2)\d s
\end{equation}
for a constant $C>0$. For $Q_2^t$, applying again
the Burkholder-Davis-Gundy inequality of Lemma \ref{lem:burkholderfull},
for any $p \geq 1$,
\begin{equation}\label{eq:Q1tmomentbound}
\E \sup_{u \in [0,t]} (Q_2^u)^{2p}
\lesssim \E\left(\int_0^t 
\frac{n^{1-\alpha}}{\binom{n}{\kappa}}
\sum_{S \in \set} 
\|\A^s\bPi\e_S\|^4 \d s\right)^p
+\E\int_0^t \frac{n^{1-\alpha}}{\binom{n}{\kappa}}
\sum_{S \in \set} 
\|\A^s\bPi\e_S\|^{4p}\d s.
\end{equation}
Then applying \eqref{eq:APipthmoment}, $\|\a_i^s\| \leq \|\A^s\|_\op$,
and H\"older's inequality, for any $p \geq 1$,
\begin{align}
\E \sup_{u \in [0,t]} (Q_2^u)^{2p}
&\lesssim 
\E\left(\int_0^t \left(n^{-(1-\alpha)}\|\A^s\|_\sF^4
+n^{-3(1-\alpha)}\|\A^s\1_n\|^4
+\sum_{i=1}^n\|\a_i^s\|^4\right)\d s\right)^p\notag\\
&\qquad
+\E \int_0^t \left(n^{-(2p-1)(1-\alpha)}\|\A^s\|_\sF^{4p}
+n^{-(4p-1)(1-\alpha)}\|\A^s\1_n\|^{4p}
+\sum_{i=1}^n\|\a_i^s\|^{4p}\right)\d s\notag\\
&\lesssim n^{-p(1-\alpha)}\int_0^t \E\|\A^s\|_\sF^{4p} \d s
+n^{-3p(1-\alpha)}\int_0^t \E\|\A^s\1_n\|^{4p} \d s
+\int_0^t \E\left(n\|\A^s\|_\op^4\right)^p\d s\notag\\
&\lesssim n^{-p(1-\alpha)}E_{4p,t}
+n^p\int_0^t \E\|\A^s\|_\op^{4p}\d s.\label{eq:Qt1bound}
\end{align}
For $Q_1^t$, Lemma \ref{lem:burkholderfull} gives
\[\E \sup_{u \in [0,t]} (Q_1^u)^{2p}
\lesssim \underbrace{\E\left(\int_0^t 
\frac{n^{1-\alpha}}{\binom{n}{\kappa}}
\sum_{S \in \set} 
|{\M^s}^\top \A^s \bPi\e_S|^2 \d s\right)^p}_{:=A^t}
+\underbrace{\E\int_0^t \frac{n^{1-\alpha}}{\binom{n}{\kappa}}
\sum_{S \in \set} 
|{\M^s}^\top \A^s \bPi\e_S|^{2p}\d s}_{:=B^t}.\]
We bound $B^t$ as above, applying
$|{\M^s}^\top \A^s \bPi\e_S|^{2p} \leq
\|\M^s\|^{2p}\|\A^s\bPi\e_S\|^{2p}$. Then similarly by
\eqref{eq:APipthmoment}, Cauchy-Schwarz, and \eqref{eq:Mtmomentbound},
\begin{align*}
B^t &\lesssim \E\int_0^t \|\M^s\|^{2p}
\left(n^{-(p-1)(1-\alpha)}\|\A^s\|_\sF^{2p}
+n^{-(2p-1)(1-\alpha)}\|\A^s\1_n\|^{2p}
+\sum_{i=1}^n \|\a_i^s\|^{2p}\right)\d s\\
&\lesssim \left(\int_0^t \E \|\M^s\|^{4p} \d s\right)^{1/2}
\left(\int \E[n^{-(2p-2)(1-\alpha)}\|\A^s\|_\sF^{4p}
+n^{-(4p-2)(1-\alpha)}\|\A^s\1_n\|^{4p}
+(n\|\A^s\|_\op^{2p})^2]\d s\right)^{1/2}\\
&\lesssim E_{4p,t}^{1/2}
\left(n^{-(2p-2)(1-\alpha)}E_{4p,t}+n^2\int_0^t \E\|\A^s\|_\op^{4p}
\,\d s\right)^{1/2}.
\end{align*}
To bound $A^t$, note that
$\bPi \bPi^\top = (\binom{n-1}{\kappa-1} - \binom{n-2}{\kappa-2}) \Id_n + \binom{n-2}{\kappa-2}
\mathbf{1}_n\mathbf{1}_n^\top$,
where $\binom{n-1}{\kappa-1} - \binom{n-2}{\kappa-2}=(1-\frac{\kappa-1}{n-1})\binom{n-1}{\kappa-1} \lesssim
\frac{\kappa}{n}\binom{n}{\kappa} \lesssim n^{-(1-\alpha)}\binom{n}{\kappa}$ and
$\binom{n-2}{\kappa-2} \lesssim \frac{\kappa^2}{n^2}\binom{n}{\kappa}
\lesssim n^{-2(1-\alpha)}\binom{n}{\kappa}$. Then
\begin{align*}
\frac{1}{\binom{n}{\kappa}} \sum_{S \in \set} 
|{\M^s}^\top \A^s \bPi\e_S|^2
&\leq
\frac{1}{\binom{n}{\kappa}}\|\M^s\|^2\|\A^s\bPi\bPi^\top{\A^s}^\top\|_\op\\
&\lesssim n^{-(1-\alpha)}\|\M^s\|^2\|\A^s\|_\op^2
+n^{-2(1-\alpha)}\|\M^s\|^2\|\A^s\1_n\|^2,
\end{align*}
so
\begin{align*}
A^t &\lesssim \E\left(\int_0^t \|\M^s\|^2(\|\A^s\|_\op^2
+n^{-(1-\alpha)}\|\A\1_n\|^2)\d s\right)^p\\
&\lesssim \left(\int_0^t \E\|\M^s\|^{4p}\d s\right)^{1/2}
\left(\int_0^t (\E\|\A^s\|_\op^{4p}
+n^{-2p(1-\alpha)}\E\|\A\1_n\|^{4p})\d s\right)^{1/2}\\
&\lesssim E_{4p,t}^{1/2}\left(n^{-2p(1-\alpha)}E_{4p,t}
+\int_0^t \E \|\A^s\|_\op^{4p}\d s\right)^{1/2}.
\end{align*}
Combining these bounds gives
\begin{align}
\E\sup_{u \in [0,t]} (Q_1^u)^{2p}
&\lesssim E_{4p,t}^{1/2}\left(n^{-(2p-2)(1-\alpha)}E_{4p,t}+n^2\int_0^t
\E\|\A^s\|_\op^{4p}\,\d s\right)^{1/2}\notag\\
&\lesssim n^{-(p-1)(1-\alpha)}E_{4p,t}
+nE_{4p,t}^{1/2}\left(\int_0^t \E\|\A^s\|_\op^{4p}\right)^{1/2}\notag\\
&\lesssim n^{-(p-1)(1-\alpha)}E_{4p,t}
+n^{2+(p-1)(1-\alpha)}\int_0^t \E\|\A^s\|_\op^{4p}\d s\label{eq:Qt2bound}
\end{align}
Finally, combining
\eqref{eq:Ptbound}, \eqref{eq:Qt1bound}, and \eqref{eq:Qt2bound}
and identifying $Q^t=Q_1^t+Q_2^t$ yields \eqref{eq:Mtasbound}.
\end{proof}

We use this to establish an almost-sure bound on
the norm of $\{\bnu^t\}_{t \in [0,T]}$, which
will be important for controlling error terms in the subsequent analysis.

\begin{lem}\label{lem:l2-bound-sgd}
For any fixed $T>0$, there exists a constant $C \equiv C(T,\alpha)>0$ such
that almost surely for all large $n,d$,
\begin{equation}\label{eq:nutasbound}
\sup_{t\in [0,T]}\norm{\bnu^t} \leq C\sqrt{n}.
\end{equation}
\end{lem}
\begin{proof}
We continue to write $\|\cdot\|$ for the vector $\ell_2$-norm and matrix
Frobenius norm, and denote $\|\M\|_\infty=\max_{i,j} |M_{ij}|$. We abbreviate
$f(\X\bnu) \equiv f(\X\bnu,\X\btheta^*,\beps)$
and $f(\x_i^\top \bnu) \equiv f(\x_i^\top \bnu,\x_i^\top \btheta^*,\beps)$.
We denote by
\begin{equation}\label{eq:poisson-Z}
\bZ^t=\bN^t-t\frac{n^{1-\alpha}}{\binom{n}{\kappa}}\mathbf{1}_{\binom{n}{\kappa}}, \qquad Z^t=\1^\top\bZ^t
\end{equation}
the centered versions of $\bN^t$ and $N^t=\1^\top \bN^t$.\\

{\bf Step 1: Moment bound for $\|\bnu^t\|$.} Fixing a constant $C>0$ large enough, consider the event
\begin{equation}\label{eq:goodevent}
\cE=\left\{\norm{\X}_{\mathrm{op}}\leq C,\;\max_{i=1}^n \norm{\x_i} \leq
C,\;\|\bnu^0\| \leq C\sqrt{n},\;\|\btheta^*\| \leq C\sqrt{n},\;\|\beps\| \leq
C\sqrt{n}\right\}.
\end{equation}
The data assumptions imply that $\cE$ holds almost surely for all large $n,d$.
We momentarily restrict to this event, write $\lesssim$
for inequality on $\cE$ up to constants depending on $C,T,\alpha$,
and write $\E_{\bZ}$ for the expectation over only $\{\bZ^t\}_{t \geq 0}$.

Note that
\begin{equation}\label{eq:PiN}
\bPi\,\d\bN^s
=\bPi\,\d\bZ^s+n^{1-\alpha}\frac{\binom{n-1}{\kappa-1}}{\binom{n}{\kappa}}
\1_{\binom{n}{\kappa}} \d s
=\bPi\,\d\bZ^s+\kappa n^{-\alpha}\1_{\binom{n}{\kappa}} \d s
\end{equation}
where we have used
$\bPi\cdot \mathbf{1}_{\binom{n}{\kappa}} = \binom{n-1}{\kappa-1}
\mathbf{1}_{\binom{n}{\kappa}}$.
Therefore, defining
\begin{align}
 \M_1^t
 &=\int_0^t\frac{\eta^{N^{s-}}}{\kappa}\,
 \X^\top\diag(\bPi\,\d\bZ^s)f(\X\bnu^{s-}),
\quad \M_2^t=\int_0^t\frac{\eta^{N^{s-}}}{n}
 g(\bnu^{s-})\,\d Z^s,
 \label{eq:poisson-M}
\end{align}
we have the martingale decomposition of \eqref{eq:poisson-nu},
\begin{equation}\label{eq:poisson-decomposition}
 \bnu^t=\bnu^0-
 \int_0^t\frac{\eta^{N^s}}{n^\alpha}
 \left\{\X^\top f(\X\bnu^s)+g(\bnu^s)\right\}\,\d s-\M_1^t-\M_2^t.
\end{equation}
The growth conditions for $f$ and $g$ in Assumption~\ref{asp:lipschitz} and
definition of the event $\cE$ give
\begin{equation}\label{eq:poisson-crude-f}
 \norm{f(\X\bnu^s)}+\norm{g(\bnu^s)}
\lesssim \sqrt{n}+\|\X\bnu^s\|+\|\X\btheta^*\|+\|\beps\|
 \lesssim\sqrt n+\norm{\bnu^s}.
\end{equation}
Consequently, \eqref{eq:poisson-decomposition} implies, for any $p \geq 1$,
\begin{equation}\label{eq:poisson-crude-nu}
 \E_\bZ\norm{\bnu^t}^{2p}
 \lesssim n^p+\int_0^t \E_\bZ\norm{\bnu^s}^{2p}\,\d s
 +\E_\bZ\norm{\M_1^t}^{2p}+\E_\bZ\norm{\M_2^t}^{2p}.
\end{equation}
To bound $\|\M_1^t\|$, we apply Lemma \ref{lemma:Poissonmartingales} with
$\a_i^s=\frac{\eta^{N^{s-}}}{\kappa}\x_i \otimes f(\x_i^\top \bnu^{s-})$
and $\A^s=[\a_1^s,\ldots,\a_n^s] \in \R^{dk \times n}$. Then
\begin{align}
\|\A^s\|_\sF^2
&\lesssim \sum_{i=1}^n \|\x_i\|^2\|f(\x_i^\top \bnu^{s-})\|^2
\lesssim \|f(\X\bnu^{s-})\|^2
\lesssim n+\|\bnu^{s-}\|^2,\label{eq:AFboundpois}\\
\|\A^s\1_n\|^2
&\lesssim \left\|\sum_{i=1}^n \x_i \otimes f(\x_i^\top \bnu^{s-})\right\|^2
=\|\X^\top f(\X\bnu^{s-})\|^2
\lesssim n+\|\bnu^{s-}\|^2.\label{eq:A1boundpois}
\end{align}
Then by \eqref{eq:Mtmomentbound} of Lemma \ref{lemma:Poissonmartingales},
\[\E_\bZ \sup_{u \in [0,t]} \|\M_1^u\|^{2p} \lesssim E_{2p,t} \lesssim
n^p+\int_0^t \E_\bZ\|\bnu^s\|^{2p}\d s.\]
By the Burkholder-Davis-Gundy inequality of Lemma \ref{lem:burkholderfull}
applied to $\M_2^t$,
\begin{align}
\E_\bZ \sup_{u \in [0,t]} \|\M_2^u\|^{2p} &\lesssim \E_\bZ\left(\int_0^t n^{1-\alpha}
\left\|\frac{\eta^{N^s}}{n}g(\bnu^s)\right\|^2 \d s\right)^p
+\E_\bZ\int_0^t n^{1-\alpha}
\left\|\frac{\eta^{N^s}}{n}g(\bnu^s)\right\|^{2p} \d s\notag\\
&\lesssim n^{-p(1-\alpha)}\int_0^t \E_\bZ\|g(\bnu^s)\|^{2p}\d s
\lesssim n^{-p(1-\alpha)}\left(n^p+\int_0^t \E_\bZ\|\bnu^s\|^{2p}\d s\right).
\label{eq:M2bound}
\end{align}
Applying these to \eqref{eq:poisson-crude-nu} gives, by Gronwall's inequality,
\begin{equation}\label{eq:nutmomentbound}
\sup_{t \in [0,T]} \E_\bZ\|\bnu^t\|^{2p} \lesssim n^p.
\end{equation}

{\bf Step 2: Leave-one-out comparison.} For $i\in[n]$, let $\X^{[i]}$ be $\X$ with its $i^\text{th}$ row replaced by
zero, and let $\bnu^{t,[i]}$ solve \eqref{eq:poisson-nu} with $\X$ replaced by
$\X^{[i]}$, having the same initialization $\bnu^0$
and driven by the same Poisson process $\{\bN^t\}_{t \geq 0}$.
We continue to restrict to
the event $\cE$, on which $\|\X^{[i]}\|_\op \lesssim 1$.
Then the preceding argument shows also, for each $p \geq 1$,
\begin{equation}\label{eq:nutimomentbound}
\max_{i=1}^n\sup_{t\in[0,T]} \E_\bZ \norm{\bnu^{t,[i]}}^{2p}\lesssim n^p.
\end{equation}
Analogous to \eqref{eq:poisson-decomposition}, we have
\begin{align*}
 \bnu^t-\bnu^{t,[i]}
 &=-\int_0^t\frac{\eta^{N^s}}{n^\alpha}
 \left\{
 \X^\top f(\X\bnu^s)
 -(\X^{[i]})^\top f(\X^{[i]}\bnu^{s,[i]})
 +g(\bnu^s)-g(\bnu^{s,[i]})
 \right\}\,\d s-\M_1^t-\M_2^t,\\
 \M_1^t
 &=\int_0^t\frac{\eta^{N^{s-}}}{\kappa}
 \left\{
 \X^\top\diag(\bPi\,\d\bZ^s)f(\X\bnu^{s-})
 -(\X^{[i]})^\top\diag(\bPi\,\d\bZ^s)
 f(\X^{[i]}\bnu^{s-,[i]})
 \right\},\\
 \M_2^t&=\int_0^t\frac{\eta^{N^{s-}}}{n}
 \left\{g(\bnu^{s-})-g(\bnu^{s-,[i]})\right\}\,\d Z^s.
\end{align*}
By the growth and Lipschitz conditions for $f,g$, note that
\begin{align*}
 \norm{\X^\top f(\X\bnu^s)
 -(\X^{[i]})^\top f(\X^{[i]}\bnu^{s,[i]})}
 &=\left\|
 (\X^{[i]})^\top
 \left\{f(\X^{[i]}\bnu^s)-f(\X^{[i]}\bnu^{s,[i]})\right\}
 +\x_i f(\x_i^\top\bnu^s)\right\|\\
 &\lesssim
 \norm{\bnu^s-\bnu^{s,[i]}}
 +1+\|\x_i^\top\bnu^{s,[i]}\|+\|\x_i^\top \btheta^*\|+|\eps_i|,\\
\|g(\bnu^s)-g(\bnu^{s,[i]})\| &\lesssim \|\bnu^s-\bnu^{s,[i]}\|.
\end{align*}
Then
\begin{align}
 \E_\bZ \norm{\bnu^t-\bnu^{t,[i]}}^{2p}
 &\lesssim
 \int_0^t \left(\E_\bZ \norm{\bnu^s-\bnu^{s,[i]}}^{2p}
 +1+\E_\bZ\|\x_i^\top\bnu^{s,[i]}\|^{2p}
+\|\x_i^\top \btheta^*\|^{2p}+|\eps_i|^{2p}\right)\,\d s\notag\\
&\hspace{1in}+\E_\bZ\norm{\M_1^t}^{2p}+\E_\bZ\norm{\M_2^t}^{2p}.
 \label{eq:poisson-loo-difference}
\end{align}
To bound $\M_1^t$, we apply Lemma \ref{lemma:Poissonmartingales} with
$\a_j^s=\frac{\eta^{N^{s-}}}{\kappa}\x_j \otimes f(\x_j^\top \bnu^{s-})
-\x_j \otimes f(\x_j^\top \bnu^{s-,[i]})$ for $j \neq i$,
$\a_i^s=\frac{\eta^{N^{s-}}}{\kappa}\x_i \otimes f(\x_i^\top \bnu^{s-})$, and
$\A^s=[\a_1^s,\ldots,\a_n^s] \in \R^{dk \times n}$. Then
\begin{align*}
\|\A^s\|_\sF^2
&\lesssim \|\x_i\|^2\|f(\x_i^\top \bnu^{s-})\|^2
+\sum_{j \neq i}
\|\x_j\|^2\|f(\x_j^\top \bnu^{s-})
-f(\x_j^\top \bnu^{s-,[i]})\|^2\\
&\lesssim \|f(\x_i^\top \bnu^{s-})\|^2
+\|f(\X^{[i]}\bnu^{s-})-f(\X^{[i]}\bnu^{s-,[i]})\|^2\\
&\lesssim \norm{\bnu^s-\bnu^{s,[i]}}^2+1+\|\x_i^\top\bnu^{s,[i]}\|^2
+\|\x_i^\top\btheta^*\|^2+|\eps_i|^2,\\
\|\A^s\1_n\|^2
&\lesssim \|\X^\top f(\X\bnu^{s-})
-\X^{[i]\top} f(\X^{[i]}\bnu^{s-,[i]})\|^2\\
&\lesssim \norm{\bnu^s-\bnu^{s,[i]}}^2+1+\|\x_i^\top\bnu^{s,[i]}\|^2
+\|\x_i^\top\btheta^*\|^2+|\eps_i|^2.
\end{align*}
Applying \eqref{eq:Mtmomentbound} of Lemma \ref{lemma:Poissonmartingales} gives
\[\E_\bZ\|\M_1^t\|^{2p} \lesssim E_{2p,t} \lesssim
\int_0^t \left(\E_\bZ\norm{\bnu^s-\bnu^{s,[i]}}^{2p}
+1+\E_\bZ\|\x_i^\top\bnu^{s,[i]}\|^{2p}
+\|\x_i^\top \btheta^*\|^{2p}+|\eps_i|^{2p}\right)\d s.\]
An argument analogous to \eqref{eq:M2bound} shows
\[\E_\bZ\|\M_2^t\|^{2p} \lesssim n^{-p(1-\alpha)}
\int_0^t \E_\bZ\norm{\bnu^s-\bnu^{s,[i]}}^{2p} \d s.\]
Then substituting into \eqref{eq:poisson-loo-difference} and applying Gronwall's
inequality shows
\[\sup_{t \in [0,T]} \E_\bZ \|\bnu^t-\bnu^{t,[i]}\|^{2p}
\lesssim 1+\|\x_i^\top\btheta^*\|^{2p}+|\eps_i|^{2p}
+\int_0^t \E_\bZ \|\x_i^\top \bnu^{s,[i]}\|^{2p} \d s.\]
Together with
\begin{align*}
\|f(\x_i^\top\bnu^s)\| 
&\leq \|f(\x_i^\top\bnu^{s,[i]})\| 
+\|f(\x_i^\top\bnu^s)-f(\x_i^\top\bnu^{s,[i]})\| \\
&\lesssim 1+\|\x_i^\top\bnu^{s,[i]}\|
+\|\x_i^\top\btheta^*\|+|\eps_i|
+\|\bnu^s-\bnu^{s,[i]}\|
\end{align*}
this shows that
\begin{align}
 \sup_{s\in[0,T]}
 \E_\bZ\norm{f(\X\bnu^s)}_\infty^{2p}
&\lesssim
\sup_{s\in[0,T]} \sum_{i=1}^n
\E_\bZ\norm{f(\x_i^\top \bnu^s)}^{2p}\notag\\
&\lesssim \sum_{i=1}^n \left(1+\|\x_i^\top \btheta^*\|^{2p}
+|\eps_i|^{2p}
+\int_0^t \E_\bZ\|\x_i^\top \bnu^{s,[i]}\|^{2p}\d s\right).
 \label{eq:poisson-f-max}
\end{align}

{\bf Step 3: Almost sure bound for $\|\bnu^t\|$.} We now return to the martingale decomposition
\eqref{eq:poisson-decomposition}, and apply \eqref{eq:poisson-f-max}
together with the almost sure statement
\eqref{eq:Mtasbound} of Lemma \ref{lemma:Poissonmartingales} to obtain an
almost sure bound for $\|\bnu^t\|$.

For the martingale term $\M_1^t$, recall $\A^s$ defined prior to
\eqref{eq:AFboundpois} and \eqref{eq:A1boundpois}.
Note that $\A^s$ is a stacking of the $k$ blocks
$\frac{\eta^{N^{s-}}}{\kappa}\X^\top \diag(\{f(\x_i^\top
\bnu^{s-})_j\}_{i=1}^n) \in \R^{d \times n}$ for $j=1,\ldots,k$, so
$\|\A^s\|_\op \lesssim \|\X\|_\op \|f(\X\bnu^{s-})\|_\infty$. Then
\begin{align*}
\E_\bZ\|\A^s\|_\op^{4p}
&\lesssim \sum_{i=1}^n \left(1+\|\x_i^\top \btheta^*\|^{4p}
+|\eps_i|^{4p}
+\int_0^t \E_\bZ\|\x_i^\top \bnu^{s,[i]}\|^{4p}\d s\right)
\end{align*}
by \eqref{eq:poisson-f-max}. Together with \eqref{eq:AFboundpois},
\eqref{eq:A1boundpois}, and the moment bound $\E_\bZ\|\bnu\|^{4p} \lesssim
n^{2p}$ established in \eqref{eq:nutmomentbound}, we have $E_{4p,t} \lesssim
n^{2p}$ in Lemma \ref{lemma:Poissonmartingales}, so
\eqref{eq:Qtmomentbound} of Lemma \ref{lemma:Poissonmartingales} shows
\[\E_\bZ \sup_{t \in [0,T]} (Q^t)^{2p}
\lesssim n^{-(p-1)(1-\alpha)}n^{2p}
+n^{p+2}\sum_{i=1}^n \left(1+\|\x_i^\top \btheta^*\|^{4p}
+|\eps_i|^{4p}+\int_0^t \E_\bZ\|\x_i^\top \bnu^{s,[i]}\|^{4p}\d s\right).\]
We remind the reader that this means an inequality up to a constant factor
on the event $\cE$ of \eqref{eq:goodevent}. Thus, taking the full expectation
on this event $\cE$,
\begin{align*}
&\E\left[\1\{\cE\}\sup_{t \in [0,T]} (Q^t)^{2p}\right]\\
&\lesssim n^{-(p-1)(1-\alpha)}n^{2p}
+n^{p+2} \sum_{i=1}^n \E\left[\1\{\cE\}\left(1+\|\x_i^\top \btheta^*\|^{4p}
+|\eps_i|^{4p} +\int_0^t \E_\bZ\|\x_i^\top
\bnu^{s,[i]}\|^{4p}\d s\right)\right].
\end{align*}
By Rosenthal's inequality applied to $\x_i$
and the data conditions of Assumption \ref{asp:model},
$\E\|\x_i^\top \btheta^*\|^{4p}
\lesssim n^{-2p}\E\|\btheta^*\|^{4p} \lesssim 1$, so also
\[\sum_{i=1}^n \E\left(1+\|\x_i^\top \btheta^*\|^{4p}
+|\eps_i|^{4p} \right) \lesssim n.\]
Consider the event $\cE^{[i]} \supset \cE$ that replaces the conditions
$\|\X\|_\op \leq C$ and $\max_{i=1}^n \|\x_i\| \leq C$ with
$\|\X^{[i]}\|_\op \leq 2C$ and $\max_{j:j \neq i} \|\x_j\| \leq C$. This event
$\cE^{[i]}$ and $\bnu^{s,[i]}$ are both independent of $\x_i$. Then
\begin{align*}
\E\left[\1\{\cE\}\E_\bZ\|\x_i^\top \bnu^{s,[i]}\|^{4p}\right]
&\leq \E\left[\1\{\cE^{[i]}\}\E_\bZ\|\x_i^\top \bnu^{s,[i]}\|^{4p}\right]\\
&=\E\left[\1\{\cE^{[i]}\}\|\x_i^\top \bnu^{s,[i]}\|^{4p}\right]
\lesssim n^{-2p}\,\E\left[\1\{\cE^{[i]}\}\|\bnu^{s,[i]}\|^{4p}\right],
\end{align*}
the last inequality using again
Rosenthal's inequality applied to $\x_i$. Then applying
\eqref{eq:nutimomentbound} on the event $\cE^{[i]}$, we obtain
$\E[\1\{\cE\}\E_\bZ \|\x_i^\top \bnu^{s,[i]}\|^{4p}] \lesssim 1$, hence
\begin{equation}\label{eq:QtmomentboundonE}
\E\left[\1\{\cE\}\sup_{t \in [0,T]} (Q^t)^{2p}\right]
\lesssim n^{-(p-1)(1-\alpha)}n^{2p}+n^{p+3}.
\end{equation}
Since $1-\alpha>0$, choosing $p$ large enough and applying Markov's inequality
and the Borel-Cantelli lemma shows that
\[\1\{\cE\}\sup_{t \in [0,T]} |Q^t| \leq n \text{ a.s.\ for
all large $n,d$}.\]
Then, together with the bounds 
\eqref{eq:AFboundpois}, \eqref{eq:A1boundpois} 
and the fact that $\cE$ holds a.s.\ for all large
$n,d$, the conclusion \eqref{eq:Mtasbound} of
Lemma \ref{lemma:Poissonmartingales} implies there exists a constant $C>0$
such that
\[\|\M_1^t\|^2 \leq C\left(n+\int_0^t \|\bnu^s\|^2 \d s\right) \text{ for all
} t \in [0,T] \text{ a.s.\ for all large $n,d$}.\]
The bounds \eqref{eq:M2bound}, \eqref{eq:nutmomentbound}, Markov's inequality,
and the Borel-Cantelli lemma imply also that for a sufficiently small constant
$\iota>0$,
\begin{equation}\label{eq:M2asbound}
\sup_{t \in [0,T]} \|\M_2^t\|^2 \leq Cn^{1-\iota} \text{ a.s.\ for all large
$n,d$}.
\end{equation}
Applying these bounds, \eqref{eq:poisson-crude-f}, and
Gronwall's inequality to
\eqref{eq:poisson-decomposition} shows $\sup_{t \in [0,T]} \|\bnu^t\|^2 \leq
C'n$ a.s.\ for all large $n,d$, as desired.
\end{proof}

\begin{proof}[Proof of Lemma \ref{lem:disc-sgd}]
We use the same notation $\bZ^t,Z^t$ as in \eqref{eq:poisson-Z}, and continue
to abbreviate 
$f(\X\bnu) \equiv f(\X\bnu,\X\btheta^*,\beps)$
and $f(\x_i^\top \bnu) \equiv f(\x_i^\top \bnu,\x_i^\top \btheta^*,\beps)$.
Introduce the intermediate process
\begin{equation}\label{eq:tildenu}
\tilde\bnu^t=\bnu^0-\int_0^t\frac{\bar\eta^s}{\bar\kappa}
\X^\top\diag(\bPi\,\d\bN^s)f(\X\tilde\bnu^{s-})
-\int_0^t\bar\eta^s g(\tilde\bnu^s)\d s.
\end{equation}
We first compare $\bnu$ to $\tilde\bnu$, and then $\tilde\bnu$ to
$\bnu_\delta$.\\

{\bf Step 1: Bounding $\norm{\bnu^t-\tilde\bnu^t}$.}
By Doob's maximal inequality and the central moment bounds for a Poisson
random variable, for each fixed $p\geq2$,
\begin{equation}\label{eq:Poissontimeapprox}
\E\sup_{t\in [0,T]}\left|\frac{N^t}{n^{1-\alpha}}-t\right|^p
=\E \sup_{t\in [0,T]}\left|\frac{Z^t}{n^{1-\alpha}}\right|^p
\lesssim n^{-(1-\alpha)p/2}.
\end{equation}
Note that by conditions of Assumption \ref{asp:SGD} including Lipschitz
continuity of $t \mapsto \bar\eta^t$,
\begin{align*}
\P\left[\sup_{t\in [0,T]}
\left|\frac{\eta^{N^t}}{\kappa}-\frac{\bar\eta^t}{\bar\kappa}\right|>\eps\right]
&\leq \P[N^T>(T+1)n^{1-\alpha}]\\
&\qquad+\P\left[\sup_{s \in [0,T+1]}
\left|\frac{\eta^{\lfloor sn^{1-\alpha}\rfloor}/n^\alpha}{\kappa/n^\alpha}
-\frac{\bar \eta^s}{\bar \kappa}\right|
+\sup_{t\in [0,T]}
\left|\frac{\bar \eta^{N^t/n^{1-\alpha}}}{\bar\kappa}-\frac{\bar
\eta^t}{\bar\kappa}\right|>\eps\right]\\
&\leq \P[N^T>(T+1)n^{1-\alpha}]
+\P\left[\sup_{t\in [0,T]}
\left|\frac{N^t}{n^{1-\alpha}}-t\right|>c\eps\right]
\end{align*}
for some constant $c>0$ and all large $n,d$.
Then \eqref{eq:Poissontimeapprox}, Markov's inequality,
and the Borel--Cantelli Lemma imply
\begin{equation}\label{eq:Poissonprocapprox}
\sup_{s\in [0,T]}\left(
\left|\frac{\eta^{N^s}}{\kappa}-\frac{\bar\eta^s}{\bar\kappa}\right|
+\left|\frac{\eta^{N^{s-}}}{n^\alpha}-\bar\eta^s\right|
\right)=o(1) \text{ a.s.\ for all large $n,d$}.
\end{equation}

Using \eqref{eq:PiN} and $\d N^s=\d Z^s+n^{1-\alpha}\d s$, decompose
\begin{equation}\label{eq:nutildenudiff}
\bnu^t-\tilde\bnu^t=\mathrm{(I)}+\mathrm{(II)}+\mathrm{(III)}+\mathrm{(IV)}
\end{equation}
where
\begin{align*}
\mathrm{(I)}
&=\int_0^t\frac{\bar\eta^s}{\bar\kappa}
\X^\top\bigl(f(\X\tilde\bnu^s)-f(\X\bnu^s)\bigr)\kappa n^{-\alpha}\d s
-\int_0^t\left(\frac{\eta^{N^s}}{\kappa}
-\frac{\bar\eta^s}{\bar\kappa}\right)
\X^\top f(\X\bnu^s)\kappa n^{-\alpha}\d s,\\
\mathrm{(II)}
&=\underbrace{\int_0^t\frac{\bar\eta^s}{\bar\kappa}
\X^\top\diag(\bPi\,\d\bZ^s)
\bigl(f(\X\tilde\bnu^{s-})-f(\X\bnu^{s-})\bigr)}_{:=\M_1^t}\\
&\qquad-\underbrace{\int_0^t\left(\frac{\eta^{N^{s-}}}{\kappa}
-\frac{\bar\eta^s}{\bar\kappa}\right)
\X^\top\diag(\bPi\,\d\bZ^s)f(\X\bnu^{s-})}_{:=\M_2^t},\\
\mathrm{(III)}
&=\int_0^t\bar\eta^s g(\tilde\bnu^s)\d s
-\int_0^t\frac{\eta^{N^{s-}}}{n^\alpha}g(\bnu^{s-})\d s,\qquad
\mathrm{(IV)}
=-\int_0^t\frac{\eta^{N^{s-}}}{n}g(\bnu^{s-})\d Z^s.
\end{align*}
By \eqref{eq:poisson-crude-f} and Lemma \ref{lem:l2-bound-sgd},
$\sup_{t \in [0,T]} \norm{f(\X\bnu^t)}+\norm{g(\bnu^s)}
\lesssim \sqrt{n}$ a.s.\ for all large $n,d$.
Thus by Lipschitz continuity of $f,g$ and \eqref{eq:Poissonprocapprox},
\begin{align}
\norm{\mathrm{(I)}},\norm{\mathrm{(III)}}
&\lesssim\int_0^t \norm{\tilde\bnu^s-\bnu^s}\d s+o(\sqrt{n}) \text{ for all 
$t\in[0,T]$ a.s.\ for
all large $n,d$}.\label{eq:SGDIIIbound}
\end{align}
The bound \eqref{eq:M2asbound} shows also
\begin{equation}\label{eq:SGDIVbound}
\norm{\mathrm{(IV)}}=o(\sqrt{n}) \text{ for all $t \in [0,T]$ a.s.\ for all large $n,d$}.
\end{equation}
For the term $\M_1^t$ of $\mathrm{(II)}$, we apply Lemma \ref{lemma:Poissonmartingales} with
$\a_i^s=\frac{\bar \eta^s}{\bar\kappa} \x_i \otimes (f(\x_i^\top
\tilde\bnu^{s-})-f(\x_i^\top \bnu^{s-}))$ and $\A^s=[\a_1^s,\ldots,\a_n^s]$.
Then by Lipschitz continuity of $f$,
\begin{align*}
\|\A^s\|_\sF^2 &\lesssim \sum_{i=1}^n
\|\x_i\|^2\|f(\x_i^\top \tilde\bnu^{s-})-f(\x_i^\top \bnu^{s-})\|^2
\lesssim \|f(\X\tilde\bnu^{s-})-f(\X\bnu^{s-})\|^2 \lesssim
\|\tilde \bnu^{s-}-\bnu^{s-}\|^2,\\
\|\A^s\1_n\|^2 &\lesssim 
\|\X^\top f(\X\tilde\bnu^{s-})-f(\X\bnu^{s-})\|^2
\lesssim \|\tilde \bnu^{s-}-\bnu^{s-}\|^2,\\
\|\A^s\|_\op
&\lesssim \|\X\|_\op\|f(\X\tilde \bnu^{s-})-f(\X\bnu^{s-})\|_\infty
\lesssim 1,
\end{align*}
so \eqref{eq:Mtasbound} of Lemma \ref{lemma:Poissonmartingales} implies
\[\|\M_1^t\|^2 \lesssim \int_0^t \|\tilde \bnu^s-\bnu^s\|^2\d s
+|Q^t| \text{ for all $t\in[0,T]$ a.s.\ for all large $n,d$}.\]
Recall the event $\cE$ of \eqref{eq:goodevent}. On this event, by
the bound $\E_\bZ\|\bnu^t\|^{2p} \lesssim n^p$ shown in
\eqref{eq:nutmomentbound} and the analogous bound for $\widetilde\bnu^t$,
we have $\E_\bZ\|\A^s\|_\sF^{4p},\E_\bZ\|\A^s\1_n\|^{4p}
 \lesssim \E_\bZ\|\bnu^s\|^{4p}
+\E_\bZ\|\tilde \bnu^s\|^{4p} \lesssim n^{2p}$. Then the bound
\eqref{eq:Qtmomentbound} in Lemma \ref{lemma:Poissonmartingales} gives
\[\E\sup_{t \in [0,T]} (Q^t)^{2p}
\lesssim n^{-(p-1)(1-\alpha)}n^{2p}+n^{p+2}.\]
Hence Markov's inequality and the Borel-Cantelli lemma show, for some small
enough $\iota>0$, $\sup_{t \in [0,T]}
|Q^t| \lesssim n^{1-\iota}$ a.s.\ for all large $n,d$, so
\[\|\M_1^t\|^2 \lesssim \int_0^t \|\tilde \bnu^s-\bnu^s\|^2\d s+o(n) \text{
for all $t \in [0,T]$ a.s.\ for all large $n,d$}.\]
Similarly for $\M_2^t$, we take
$\a_i^s=(\frac{\eta^{N^{s-}}}{\kappa}-\frac{\bar \eta^s}{\bar\kappa}) \x_i
\otimes f(\x_i^\top \bnu^{s-})$ and $\A^s=[\a_1^s,\ldots,\a_n^s]$. Then
\begin{align*}
\|\A^s\|_\sF^2 &\lesssim 
\left(\frac{\eta^{N^{s-}}}{\kappa}-\frac{\bar \eta^s}{\bar\kappa}\right)^2
\sum_{i=1}^n
\|\x_i\|^2\|f(\x_i^\top \bnu^{s-})\|^2
\lesssim
\left(\frac{\eta^{N^{s-}}}{\kappa}-\frac{\bar \eta^s}{\bar\kappa}\right)^2
(n+\|\bnu^{s-}\|^2),\\
\|\A^s\1_n\|^2 &\lesssim 
\left(\frac{\eta^{N^{s-}}}{\kappa}-\frac{\bar \eta^s}{\bar\kappa}\right)^2
\|\X^\top f(\X\bnu^{s-})\|^2
\lesssim \left(\frac{\eta^{N^{s-}}}{\kappa}-\frac{\bar
\eta^s}{\bar\kappa}\right)^2 (n+\|\bnu^{s-}\|^2),\\
\|\A^s\|_\op &\lesssim \left|\frac{\eta^{N^{s-}}}{\kappa}-\frac{\bar
\eta^s}{\bar\kappa}\right|
\|\X\|_\op\|f(\X\bnu^{s-})\|_\infty,
\end{align*}
so \eqref{eq:Poissonprocapprox}, Lemma \ref{lem:l2-bound-sgd}, and the
conclusion \eqref{eq:Mtasbound} of Lemma \ref{lemma:Poissonmartingales} give
\[\|\M_2^t\|^2 \lesssim o(n)+|Q^t| \text{ for all $t \in [0,T]$ a.s.\ for all
large $n,d$}.\]
On the event $\cE$, by \eqref{eq:nutmomentbound},
$\E_\bZ\|\A^s\|_\sF^{4p},\E_\bZ\|\A^s\1_n\|^{4p} \lesssim n^{2p}$.
Also by \eqref{eq:poisson-f-max},
\[\E_\bZ\|\A^s\|_\op^{4p} \lesssim
\E\|f(\X\bnu^{s-})\|_\infty^{4p}
\lesssim \sum_{i=1}^n \left(1+\|\x_i^\top \btheta^*\|^{4p}
+|\eps_i|^{4p}
+\int_0^t \E_\bZ\|\x_i^\top \bnu^{s,[i]}\|^{4p}\d s\right).\]
Then the same arguments as leading to \eqref{eq:QtmomentboundonE} in Lemma
\ref{lem:l2-bound-sgd} show
\[\E\left[\1\{\cE\}\sup_{u \in [0,t]} (Q^t)^{2p}\right]
\lesssim n^{-(p-1)(1-\alpha)}n^{2p}+n^{p+3},\]
and hence also $\sup_{t \in [0,T]} |Q^t| \lesssim n^{1-\iota}$ and
\[\sup_{t \in [0,T]} \|\M_2^t\|^2=o(n) \text{ a.s.\ for all large $n,d$}.\]
So
\begin{equation}\label{eq:SGDIIbound}
\|\mathrm{(II)}\|^2\lesssim \int_0^t \|\tilde\bnu^s-\bnu^s\|^2\d s+o(n)
\text{ for all $t \in [0,T]$ a.s.\ for all large $n,d$}.
\end{equation}
Applying \eqref{eq:SGDIIIbound}, \eqref{eq:SGDIVbound},
\eqref{eq:SGDIIbound}, and Gronwall's
inequality in \eqref{eq:nutildenudiff} shows
\begin{equation}\label{eq:disc-sgd1} 
\sup_{t \in [0,T]} \|\bnu^t-\tilde\bnu^t\|^2=o(n) \text{ a.s.\ for all
large $n,d$}.
\end{equation}

{\bf Step 2: Bounding $\norm{\tilde\bnu^t - \bnu_\delta^t}$.} Next, we compare $\tilde\bnu^t$ to the discretized process $\bnu_\delta^t$: 
\begin{align*}
       \bnu_\delta^t &= \bnu^0 - \int_0^{\flr{t}}
\frac{\bar\eta^{\flr{s}}}{\bar\kappa} \X^\top
\diag(\bPi\,\d\bN^s)f(\X\bnu_\delta^{s-}) - \int_0^{\flr{t}}
\bar\eta^{\flr{s}} g(\bnu_\delta^{s}) \d s.
\end{align*}
We likewise decompose
    $\tilde \bnu^t -  \bnu_\delta^t=\mathrm{(I)} + \mathrm{(II)} +
\mathrm{(III)} + \mathrm{(IV)} + \mathrm{(V)}$
where
\begin{align*}
    \mathrm{(I)} &= \int_0^{\lfloor t\rfloor} \frac{\bar\eta^{\lfloor s
\rfloor}}{\bar\kappa} \X^\top
\big(f(\X\bnu_\delta^s)-f(\X\tilde\bnu^s)\big)\kappa n^{-\alpha}\d s
-\int_0^{\lfloor t\rfloor} \frac{\bar\eta^{s}-\bar\eta^{\lfloor s\rfloor}
}{\bar\kappa} \X^\top
f(\X\tilde\bnu^s)\kappa n^{-\alpha} \d s, \\
    \mathrm{(II)} &= \int_0^{\lfloor t\rfloor} \frac{\bar\eta^{\lfloor s
\rfloor}}{\bar\kappa} \X^\top \diag(\bPi\,\d\bZ^s)
\big(f(\X\bnu_\delta^{s-})-f(\X\tilde\bnu^{s-})\big)
-\int_0^{\lfloor t\rfloor} \frac{\bar\eta^{s}-\bar\eta^{\lfloor s\rfloor}
}{\bar\kappa} \X^\top
\diag(\bPi\,\d\bZ^s)f(\X\tilde\bnu^{s-}), \\
\mathrm{(III)} &= \int_0^{\lfloor t\rfloor} \bar\eta^{\lfloor s\rfloor}
g(\bnu_\delta^{s})-\int_0^{\lfloor t\rfloor} \bar\eta^{s}
g(\tilde\bnu^{s})\d s,\\
\mathrm{(IV)} &= - \int_{\lfloor t\rfloor}^{t} \frac{\bar\eta^{s}}{\bar\kappa}
\X^\top \diag(\bPi\,\d\bN^s)f(\X\tilde\bnu^{s-}),\\
\mathrm{(V)}&=- \int_{\lfloor t\rfloor}^{t} \bar\eta^{s} g(\tilde\bnu^{s})\d s.
\end{align*}
Note that $t-\flr{t} \lesssim \delta$, and
by the Lipschitz continuity of $\bar\eta^t$ in Assumption \ref{asp:SGD},
$|\bar \eta^t-\bar\eta^{\flr{t}}| \lesssim \delta$.
Then similar arguments as in Step 1 (omitted here for brevity) show, for all
$t \in [0,T]$ a.s.\ for all large $n,d$,
\[
\begin{gathered}
\norm{\mathrm{(I)}} \lesssim \int_0^t \|\bnu_\delta^s-\tilde\bnu^s\|\d s
+\delta\sqrt{n},
\qquad
\norm{\mathrm{(II)}}^2 \lesssim \int_0^t \|\bnu_\delta^s-\tilde\bnu^s\|^2 \d s
+\delta^2 n+o(n),\\
    \norm{\mathrm{(III)}} \lesssim \int_0^t
\norm{\bnu_\delta^{s}-\tilde\bnu^{s}}\d s+\delta\sqrt{n},
\qquad
    \norm{\mathrm{(IV)}}^2 \lesssim \delta n+o(n),
\qquad
\norm{\mathrm{(V)}}\lesssim \delta\sqrt{n}.
\end{gathered}\]
Applying these and Gr\"onwall's inequality gives
\begin{equation}\label{eq:disc-sgd2}
\sup_{t \in [0,T]}
\norm{\bnu_\delta^t - \tilde\bnu^t}^2 \lesssim \delta n+o(n).
\end{equation}
The lemma follows from combining these bounds \eqref{eq:disc-sgd1} and
\eqref{eq:disc-sgd2}.
\end{proof}

\subsection{Discretization of SME}
We now apply a similar analysis to the Stochastic Modified Equation. The
SME process \eqref{eq:SME} may be written concisely as
\begin{align}
\btheta^t &= \btheta^0-\int_0^t \bar\eta^s \prn{\X^\top
f(\X\btheta^s,\X\btheta^*,\beps) + g(\btheta^s)} \d s + \int_0^t
\frac{\bar\eta^s}{\sqrt{\bar \kappa}} \X^\top
\diag(\d\B^s)f(\X\btheta^{s},\X\btheta^*,\beps)
\label{eq:def-btheta}
\end{align}
where $\B^s$ is a standard $n$-dimensional Brownian motion.
We compare this to the discretized process $\{\bnu_\delta^t\}_{t \geq 0}$
defined by
\begin{equation}\label{eq:deltaSME}
\bnu_\delta^t = \btheta^0-\int_0^{\flr{t}} \bar\eta^{\flr{s}}
\prn{\X^\top f(\X\bnu_\delta^s,\X\btheta^*,\beps) +
g(\bnu_\delta^s)} \d s + \int_0^{\flr{t}}
\frac{\bar\eta^{\flr{s}}}{\sqrt{\bar \kappa}} \X^\top \diag(\d\B^s)
f(\X\bnu_\delta^s,\X\btheta^*,\beps).
\end{equation}
One may again check that $\{\bnu_\delta^t\}_{t \geq 0}$ is precisely a
piecewise-constant embedding of the discrete-time dynamics
\eqref{eq:discretedynamics} in the SME setting of $\z_\delta^k \sim
\cN(\delta\bar\kappa\,\1_n,\delta\bar\kappa\,\Id_n)$,
\[\bnu_\delta^t=\btheta_\delta^k \text{ if } \flr{t}=k,\]
upon identifying $\z_\delta^k=\int_{k\delta}^{(k+1)\delta}
(\bar\kappa \1_n \d t+\sqrt{\bar\kappa}\,\d\B^t)$.
The following lemma is then analogous to Lemma \ref{lem:disc-sgd}.

\begin{lem}\label{lem:disc-sme}
For any fixed $T>0$, $\delta \in (0,1)$, and some constant $C \equiv
C(T)>0$ not depending on $\delta$, almost surely as $n,d \to \infty$,
\begin{equation}
\limsup_{n,d\to \infty} \sup_{t\in[0,T]}
\frac{\norm{\btheta^t-\bnu_\delta^t}^2}{d} \leq C\delta.
\end{equation}
\end{lem}
\begin{proof}
Consider any $m \equiv m(n) \geq 1$,
let $\{\A^t\}_{t \geq 0}$ be any $\R^{m \times n}$-valued process that is
predictable with respect to $\{\B^t\}_{t \geq 0}$, and define
\[\M^t=\int_0^t \A^s \d \B^s,
\qquad E_{p,t}=\int_0^t \E\|\A^s\|_{\sF}^p\d s.\]
We note the following martingale bounds analogous to Lemma
\ref{lemma:Poissonmartingales}: First,
the Burkholder-Davis-Gundy inequality of Lemma \ref{lem:burkholderfull}
and H\"older's inequality show
\[\E \sup_{u \in [0,t]} \|\M^u\|^{2p}
\lesssim \E\left(\int_0^t \|\A^s\|_\sF^2 \d s\right)^p \lesssim
\int_0^t \E\|\A^s\|_\sF^{2p}\d s=E_{2p,t}\]
which is the analogue of \eqref{eq:Mtmomentbound}. Next, It\^o's formula gives
\[\|\M^t\|^2=\underbrace{2\int_0^t {\M^s}^\top \A^s \d\B^s}_{:=Q^t}
+\int_0^t \|\A^s\|_\sF^2 \d s,\]
where Lemma \ref{lem:burkholderfull} applied to $Q^t$ shows
\begin{align*}
\E\sup_{u \in [0,t]} (Q^u)^{2p}
\lesssim \E\left(\int_0^t \|{\M^s}^\top\A^s\|^2 \d s\right)^p
&\lesssim \left(\int_0^t \E \|\M^s\|^{4p}\d s\right)^{1/2}
\left(\int_0^t \E \|\A^s\|_\op^{4p}\d s\right)^{1/2}\\
&\lesssim n^{-p}E_{4p,t}+n^p \int_0^t \E \|\A^s\|_\op^{4p}\d s.
\end{align*}
These are the analogues of \eqref{eq:Mtasbound} and \eqref{eq:Qtmomentbound}.
Equipped with these martingale bounds, the proof is then the same as
Lemmas \ref{lem:disc-sgd} and \ref{lem:l2-bound-sgd},
first showing $\E \sup_{t \in [0,T]}
\|\btheta^t\|^{2p} \lesssim n^p$, then showing
$\sup_{t \in [0,T]} \|\btheta^t\|^2 \lesssim n$ a.s.\ using a leave-one-out
argument, and finally showing the desired comparison
$\sup_{t \in [0,T]} \norm{\btheta^t-\bnu_\delta^t}^2 \lesssim \delta n$ a.s.
We omit the details for brevity.
\end{proof}

\begin{proof}[Proof of Theorem \ref{thm:DMFT-limit}(a)]
Let $\btheta^t=\bar\btheta^{\lfloor tn^{1-\alpha} \rfloor}$ be the time-rescaled
SGD process in \eqref{eq:SGDrescaled} (where here $\flr{\cdot}$ denotes the
usual integer part), and let $\bnu^t=\bar\btheta^{N^t}$ be
the Poissonized SGD process in \eqref{eq:SGDPoissonized}. For any $\delta>0$,
let $\bnu_\delta^t$ be the discretized approximation in \eqref{eq:deltaSGD}.
Applying an induction over discrete time increments of $\delta$ to
\eqref{eq:deltaSGD}, we see that $\{\bnu_\delta^t\}_{t \geq 0}$
is piecewise-constant with jumps at $\delta\Z_+$, taking values
\begin{equation}\label{eq:nudeltaembedding}
\bnu_\delta^t=\btheta_\delta^k \text{ for all } t \in [k\delta,(k+1)\delta)
\end{equation}
where $\btheta_\delta^k$ are the iterations of the discrete-time dynamics
\eqref{eq:discretedynamics} defined using the Poisson variables
$\z_\delta^k=\bPi\bP_\delta^k$
and $\bP_\delta^k=\int_{k\delta}^{(k+1)\delta} \d\bN^s$.

Fix any $T>0$. For each $t \in [0,T]$, set
$\tau(t)=\inf\{\tau \geq 0:N^\tau \geq \lfloor tn^{1-\alpha} \rfloor\}$,
so that
\[\btheta^t=\bnu^{\tau(t)}.\]
Note that the bound \eqref{eq:Poissontimeapprox} implies that for some
$\iota>0$, on an event holding a.s.\ for all large $n,d$,
\begin{equation}\label{eq:tauapprox}
\sup_{t \in [0,T]} |\tau(t)-t| \lesssim n^{-\iota}.
\end{equation}
Fix any $0 \leq t_1 \leq \ldots \leq t_m \leq T$, and fix any
$\delta \in (0,1)$ such that none of $t_1,\ldots,t_m$ is an integer multiple of
$\delta$. Then a.s.\ for all large $n,d$, the piecewise-constant nature of
$\{\bnu_\delta^t\}_{t \geq 0}$ and \eqref{eq:tauapprox} imply that
$\bnu_\delta^{\tau(t_i)}=\bnu_\delta^{t_i}$ for each $i=1,\ldots,m$. Then
by Lemma \ref{lem:disc-sgd}, a.s.\ for all large $n,d$,
\[\sum_{i=1}^m \frac{\|\btheta^{t_i}-\bnu_\delta^{t_i}\|^2}{d}
=\sum_{i=1}^m \frac{\|\bnu^{\tau(t_i)}-\bnu_\delta^{\tau(t_i)}\|^2}{d}
\lesssim \delta.\]
This implies that
\begin{equation}\label{eq:W2compare1}
\limsup_{n,d \to \infty} W_2\left(\frac{1}{d}\sum_{j=1}^d
\delta_{(\theta_j^{t_1},\ldots,\theta_j^{t_m},\theta_j^*)},\,
\frac{1}{d}\sum_{j=1}^d
\delta_{(\nu_{\delta,j}^{t_1},\ldots,\nu_{\delta,j}^{t_m},\theta_j^*)}\right)^2
\lesssim \delta.
\end{equation}
Let $k_1,\ldots,k_m$ be the integers for which $t_i \in
[k_i\delta,(k_i+1)\delta)$ for each $i=1,\ldots,m$. 
Then by
\eqref{eq:nudeltaembedding} and Lemma \ref{lem:finite_dim_converge},
\[\limsup_{n,d \to \infty} W_2\left(\frac{1}{d}\sum_{j=1}^d
\delta_{(\nu_{\delta,j}^{t_1},\ldots,\nu_{\delta,j}^{t_m},\theta_j^*)},
\,\Law(\theta_\delta^{k_1},\ldots,\theta_\delta^{k_m},\theta^*)\right)=0\]
where the right side denotes the joint law under the discrete DMFT system of
Section \ref{sec:discreteDMFT}. Recalling the embedding 
\eqref{eq:embeddings} of this discrete system, where
$\bar\theta_\delta^{t_i}=\theta_\delta^{k_i}$ for each $i=1,\ldots,m$, this is
equivalently written as
\begin{equation}\label{eq:W2compare2}
\limsup_{n,d \to \infty} W_2\left(\frac{1}{d}\sum_{j=1}^d
\delta_{(\nu_{\delta,j}^{t_1},\ldots,\nu_{\delta,j}^{t_m},\theta_j^*)},
\,\Law(\bar\theta_\delta^{t_1},\ldots,\bar\theta_\delta^{t_m},\theta^*)\right)
=0.
\end{equation}
Finally, by Lemma \ref{lem:dmft_discretization_error}(c),
\begin{equation}\label{eq:W2compare3}
\lim_{\delta \to 0} W_2\Big(
\Law(\bar\theta_\delta^{t_1},\ldots,\bar\theta_\delta^{t_m},\theta^*),\,
\Law(\theta^{t_1},\ldots,\theta^{t_m},\theta^*)\Big)=0
\end{equation}
where the right side denotes the joint law under the continuous DMFT system
defined via the fixed point of
Theorem \ref{thm:DMFT-limit}. Combining \eqref{eq:W2compare1},
\eqref{eq:W2compare2}, and \eqref{eq:W2compare3} and taking the limit $n,d \to
\infty$ followed by $\delta \to 0$, we obtain \eqref{eq:DMFT-limit-theta}.
The argument for \eqref{eq:DMFT-limit-xi} is the same.
\end{proof}

\begin{proof}[Proof of Theorem \ref{thm:DMFT-limit}(b)]
Let $\{\btheta^t\}_{t \geq 0}$ be the SME process \eqref{eq:SME}, and let
$\bnu_\delta^t$ be its discretized approximation \eqref{eq:deltaSME}.
Again by induction over discrete time increments of $\delta$, we see that
$\{\bnu_\delta^t\}_{t \geq 0}$
is piecewise-constant with jumps at $\delta\Z_+$, taking values
\begin{equation}\label{eq:nudeltaembedding-SME}
\bnu_\delta^t=\btheta_\delta^k \text{ for all } t \in [k\delta,(k+1)\delta)
\end{equation}
where $\btheta_\delta^k$ are the iterations of the discrete-time SME dynamics
\eqref{eq:discretedynamics} defined using the Gaussian variables
$\z_\delta^k=\int_{k\delta}^{(k+1)\delta} (\bar\kappa\1_n\d t
+\sqrt{\bar\kappa}\,\d\B^t)$. 
The proof is then completed as above, using 
Lemmas \ref{lem:disc-sme}, \ref{lem:finite_dim_converge}, and
\ref{lem:dmft_discretization_error}(c) for the SME.
\end{proof}

\section{Burkholder-Davis-Gundy inequality}

Let $\{\cF_t\}_{t \geq 0}$ be a complete right-continuous filtration on
$(\Omega,\cF,\P)$.

\begin{lem}\label{lem:burkholderfull}
Let $\{\B^t\}_{t \geq 0}$ be a
standard $m$-dimensional $\cF_t$-adapted Brownian motion,
let $\{\bN^t\}_{t \geq 0}$ be a vector of $q$ independent $\cF_t$-adapted
homogeneous Poisson processes with rate $\kappa$, let $\1_q$ and
$\e_1,\ldots,\e_q$ be the all-1's and standard basis vectors in $\R^q$, and
let
\[\bZ^t=\bN^t-t\kappa \1_q\]
be the centering of $\bN^t$. Let
$\{\A^t\}_{t \geq 0}$ and $\{\M^t\}_{t \geq 0}$ be $\cF_t$-predictable $\R^{d
\times m}$- and $\R^{d \times q}$-valued processes, let
$\{\v^t\}_{t \geq 0}$ be a $\cF_t$-adapted $\R^d$-valued process, and consider
\[\X^t=\int_0^t \v^s \d s+\int_0^t \M^s \d\B^s
+\int_0^t \A^s \d\bZ^s.\]
Then for any $p \geq 2$, there exists a constant $C_p>0$ such that
for all $t \geq 0$,
\begin{align*}
\E \sup_{s \in [0,t]} \|\X^s\|_2^p
&\leq C_p\left(\E \left(\int_0^t \|\v^s\|_2 \d s\right)^p
+\E\left(\int_0^t \|\M^s\|_\sF^2 \d s\right)^{p/2}
+\E\left(\int_0^t \kappa \|\A^s\|_\sF^2 \d s\right)^{p/2}\right.\\
&\hspace{1in}
+\left.\E\int_0^t \kappa \sum_{j=1}^q \|\A^s \e_j\|_2^p \d s \right)
\end{align*}
\end{lem}
\begin{proof}
This is a specialization of \cite[Thm 2.11]{kunita2004stochastic},
identifying $\{\bZ^t\}_{t \in (0,T]}$ as a compensated Poisson random measure
on $(0,T] \times \{1,\ldots,q\}$ with intensity measure $\nu(S)=\kappa|S|$ for
$S \subset \{1,\ldots,q\}$.
\end{proof}

We record the following direct specialization to the stochastic integrals
against $\{z^t\}_{t \geq 0}$ defining the DMFT system.

\begin{lem}\label{lem:burkholder}
Let $\{\mathcal{F}_t\}_{t \geq 0}$ be a complete right-continuous filtration on
$(\Omega,\cF,\P)$, and let
$\{X^t\}_{t \geq 0}$ be a $\cF_t$-predictable $\R^m$-valued process.
If $\{z^t\}_{t \geq 0}$
is the $\cF_t$-adapted diffusion process \eqref{eq:gaussian-z},
then for any $p \geq 2$,
there exists a constant $C_p>0$ such that for all $t \geq 0$,
\[\E \sup_{s \in [0,t]} \left\|\int_0^s X^r\,\d z^r\right\|_2^p \leq
C_p\left(\E\left(\bar\kappa \int_0^t \|X^r\|_2\,\d r\right)^p
+\E\left(\bar\kappa \int_0^t \|X^r\|_2^2\,\d r\right)^{p/2}\right).\]
If $\{z^t\}_{t \in [0,T]}$ is the $\cF_t$-adapted Poisson process 
\eqref{eq:poisson-z}, then for any $p \geq 2$,
there exists a constant $C_p>0$ such that for all $t \geq 0$,
\[\E \sup_{s \in [0,t]} \left\|\int_0^s X^r\,\d z^r\right\|_2^p \leq
C_p\left(\E\left(\bar\kappa \int_0^t \|X^r\|_2\,\d r\right)^p
+\E\left(\bar\kappa \int_0^t \|X^r\|_2^2\,\d r\right)^{p/2}
+\E\int_0^t \bar\kappa\|X^r\|_2^p\d r\right).\]
\end{lem}

\section*{Acknowledgments}
We would like to thank Sinho Chewi, Alex Damian, Bruno Loureiro,
and Theodor Misiakiewicz for helpful discussions. This research was
supported in part by NSF DMS-2142476 and a Sloan Research Fellowship.

GPT-5.6 (Sol) was used for computation of the first and second derivatives of $\E[\theta^t\theta^*]$ in Remark \ref{rmk:onlinevsonepass}. All other results were obtained without use of AI.

\bibliographystyle{alpha}
\bibliography{references}

@article{nishiyama2026high,
  title={High-Dimensional Limit of Stochastic Gradient Flow via Dynamical Mean-Field Theory},
  author={Nishiyama, Sota and Imaizumi, Masaaki},
  journal={arXiv preprint arXiv:2602.06320},
  year={2026}
}

@book{brunner2004collocation,
  title={Collocation methods for Volterra integral and related functional differential equations},
  author={Brunner, Hermann},
  volume={15},
  year={2004},
  publisher={Cambridge University Press}
}

@book{hale2013introduction,
  title={Introduction to functional differential equations},
  author={Hale, Jack K and Verduyn Lunel, Sjoerd M},
  series={Applied Mathematical Sciences},
  volume={99},
  year={2013},
  publisher={Springer Science \& Business Media},
  doi={10.1007/978-1-4612-4342-7}
}

@book{villani2008optimal,
  title={Optimal transport: old and new},
  author={Villani, C{\'e}dric },
  volume={338},
  year={2008},
  publisher={Springer}
}

@incollection{kunita2004stochastic,
  title={Stochastic differential equations based on {L}{\'e}vy processes and stochastic flows of diffeomorphisms},
  author={Kunita, Hiroshi},
  booktitle={Real and Stochastic Analysis: New Perspectives},
  pages={305--373},
  year={2004},
  publisher={Springer}
}

@article{wang2024universality,
  title={Universality of approximate message passing algorithms and tensor networks},
  author={Wang, Tianhao and Zhong, Xinyi and Fan, Zhou},
  journal={The Annals of Applied Probability},
  volume={34},
  number={4},
  pages={3943--3994},
  year={2024},
  publisher={Institute of Mathematical Statistics}
}

@article{fan2025dynamical1,
  title={Dynamical mean-field analysis of adaptive Langevin diffusions: Propagation-of-chaos and convergence of the linear response},
  author={Fan, Zhou and Ko, Justin and Loureiro, Bruno and Lu, Yue M and Shen, Yandi},
  journal={arXiv preprint arXiv:2504.15556},
  year={2025}
}

@article{fan2025dynamical2,
  title={Dynamical mean-field analysis of adaptive Langevin diffusions: Replica-symmetric fixed point and empirical Bayes},
  author={Fan, Zhou and Ko, Justin and Loureiro, Bruno and Lu, Yue M and Shen, Yandi},
  journal={arXiv preprint arXiv:2504.15558},
  year={2025}
}

@article{celentano2021high,
  title={The high-dimensional asymptotics of first order methods with random data},
  author={Celentano, Michael and Cheng, Chen and Montanari, Andrea},
  journal={arXiv preprint arXiv:2112.07572},
  year={2021}
}

@article{gerbelot2024rigorous,
  title={Rigorous dynamical mean-field theory for stochastic gradient descent methods},
  author={Gerbelot, C{\'e}dric and Troiani, Emanuele and Mignacco, Francesca and Krzakala, Florent and Zdeborov{\'a}, Lenka},
  journal={SIAM Journal on Mathematics of Data Science},
  volume={6},
  number={2},
  pages={400--427},
  year={2024},
  publisher={SIAM}
}

@article{nemirovski2009robust,
title={Robust stochastic approximation approach to stochastic programming},
author={Nemirovski, Arkadi and Juditsky, Anatoli and Lan, Guanghui and Shapiro, Alexander},
  journal={SIAM Journal on Optimization},
volume={19},
number={4},
pages={1574--1609},
year={2009},
publisher={SIAM}
}

@inproceedings{ge2015escaping,
title={Escaping from saddle points—online stochastic gradient for tensor decomposition},
author={Ge, Rong and Huang, Furong and Jin, Chi and Yuan, Yang},
booktitle={Conference on Learning Theory},
pages={797--842},
year={2015}
}

@inproceedings{fang2019sharp,
  title={Sharp analysis for nonconvex {SGD} escaping from saddle points},
  author={Fang, Cong and Lin, Zhouchen and Zhang, Tong},
  booktitle={Conference on Learning Theory},
  pages={1192--1234},
  year={2019},
  organization={PMLR}
}

@book{kushner2003stochastic,
title={Stochastic approximation and recursive algorithms and applications},
author={Kushner, Harold and Yin, G George},
volume={35},
year={2003},
publisher={Springer Science \& Business Media}
}

@article{li2019stochastic,
title={Stochastic Modified Equations and Dynamics of Stochastic Gradient Algorithms {I}: {M}athematical Foundations},
author={Li, Qianxiao and Tai, Cheng and E, Weinan},
journal={Journal of Machine Learning Research},
volume={20},
number={40},
pages={1--47},
year={2019}
}

@article{feng2018semigroups,
  title={Semigroups of stochastic gradient descent and online principal component analysis: properties and diffusion approximations},
  author={Feng, Yuanyuan and Li, Lei and Liu, Jian-Guo},
  journal={Communications in Mathematical Sciences},
  volume={16},
  number={3},
  pages={777--789},
  year={2018}
}

@inproceedings{li2017stochastic,
title={Stochastic modified equations and adaptive stochastic gradient algorithms},
author={Li, Qianxiao and Tai, Cheng and E, Weinan},
booktitle={Proceedings of the 34th International Conference on Machine Learning},
series={Proceedings of Machine Learning Research},
volume={70},
pages={2101--2110},
year={2017},
publisher={PMLR}
}

@article{robbins1951stochastic,
title={A stochastic approximation method},
author={Robbins, Herbert and Monro, Sutton},
journal={The Annals of Mathematical Statistics},
volume={22},
number={3},
pages={400--407},
year={1951},
publisher={Institute of Mathematical Statistics},
doi={10.1214/aoms/1177729586}
}

@article{hu2019diffusion,
  title={On the diffusion approximation of nonconvex stochastic gradient descent},
  author={Hu, Wenqing and Li, Chris Junchi and Li, Lei and Liu, Jian-Guo},
  journal={Annals of Mathematical Sciences and Applications},
  volume={4},
  number={1},
  pages={3--32},
  year={2019}
}

@article{bottou2018optimization,
  title={Optimization methods for large-scale machine learning},
  author={Bottou, L{\'e}on and Curtis, Frank E and Nocedal, Jorge},
journal={SIAM Review},
  volume={60},
  number={2},
  pages={223--311},
  year={2018},
  publisher={SIAM}
}

@incollection{bottou2010large,
title={Large-scale machine learning with stochastic gradient descent},
author={Bottou, L{\'e}on},
booktitle={Proceedings of COMPSTAT'2010},
pages={177--186},
year={2010},
publisher={Springer}
}

@article{goldt2019dynamics,
title={Dynamics of stochastic gradient descent for two-layer neural networks in the teacher-student setup},
author={Goldt, Sebastian and Advani, Madhu S and Saxe, Andrew M and Krzakala, Florent and Zdeborov{\'a}, Lenka},
journal={Journal of Statistical Mechanics: Theory and Experiment},
volume={2020},
number={12},
pages={124010},
year={2020},
doi={10.1088/1742-5468/abc61e}
}

@article{jastrzkebski2017three,
title={Three factors influencing minima in {SGD}},
author={Jastrz{k{e}}bski, Stanis{\l}aw and Kenton, Zachary and Arpit, Devansh and Ballas, Nicolas and Fischer, Asja and Bengio, Yoshua and Storkey, Amos},
journal={arXiv preprint arXiv:1711.04623},
year={2017}
}

@inproceedings{keskar2017large,
  title={On Large-Batch Training for Deep Learning: Generalization Gap and Sharp Minima},
  author={Keskar, Nitish Shirish and Mudigere, Dheevatsa and Nocedal, Jorge and Smelyanskiy, Mikhail and Tang, Ping Tak Peter},
  booktitle={International Conference on Learning Representations},
  year={2017}
}

@inproceedings{bach2013non,
title={Non-strongly-convex smooth stochastic approximation with convergence rate $O(1/n)$},
author={Bach, Francis and Moulines, Eric},
booktitle={Advances in Neural Information Processing Systems},
pages={773--781},
year={2013}
}

@inproceedings{moulines2011non,
title={Non-asymptotic analysis of stochastic approximation algorithms for machine learning},
author={Moulines, Eric and Bach, Francis R},
booktitle={Advances in Neural Information Processing Systems},
pages={451--459},
year={2011}
}

@article{mandt2017stochastic,
title={Stochastic gradient descent as approximate {B}ayesian inference},
author={Mandt, Stephan and Hoffman, Matthew D and Blei, David M},
journal={The Journal of Machine Learning Research},
volume={18},
number={1},
pages={4873--4907},
year={2017},
publisher={JMLR. org}
}

@article{javanmard2013state,
  title={State evolution for general approximate message passing algorithms, with applications to spatial coupling},
  author={Javanmard, Adel and Montanari, Andrea},
  journal={Information and Inference: A Journal of the IMA},
  volume={2},
  number={2},
  pages={115--144},
  year={2013},
  publisher={OUP}
}

@article{fan2022approximate,
  title={Approximate message passing algorithms for rotationally invariant matrices},
  author={Fan, Zhou},
  journal={The Annals of Statistics},
  volume={50},
  number={1},
  pages={197--224},
  year={2022},
  publisher={Institute of Mathematical Statistics}
}

@inproceedings{schudy2012concentration,
  title={Concentration and moment inequalities for polynomials of independent random variables},
  author={Schudy, Warren and Sviridenko, Maxim},
  booktitle={Proceedings of the twenty-third annual ACM-SIAM symposium on Discrete Algorithms},
  pages={437--446},
  year={2012},
  organization={SIAM}
}

@inproceedings{rakhlin2012making,
  title={Making Gradient Descent Optimal for Strongly Convex Stochastic Optimization},
  author={Rakhlin, Alexander and Shamir, Ohad and Sridharan, Karthik},
  booktitle={Proceedings of the 29th International Conference on Machine Learning},
  year={2012},
  url={https://icml.cc/2012/papers/261.pdf}
}

@book{ljung1992stochastic,
  title={Stochastic approximation and optimization of random systems},
  author={Ljung, Lennart and Pflug, Georg Ch and Walk, Harro},
  volume={17},
  year={1992},
  publisher={Springer Science \& Business Media}
}

@article{ljung1977analysis,
  title={Analysis of recursive stochastic algorithms},
  author={Ljung, Lennart},
  journal={IEEE Transactions on Automatic Control},
  volume={22},
  number={4},
  pages={551--575},
  year={1977},
  publisher={IEEE}
}

@inproceedings{blanc2020implicit,
  title={Implicit regularization for deep neural networks driven by an {Ornstein-Uhlenbeck} like process},
  author={Blanc, Guy and Gupta, Neha and Valiant, Gregory and Valiant, Paul},
  booktitle={Conference on Learning Theory},
  pages={483--513},
  year={2020},
  organization={PMLR}
}

@inproceedings{li2022happens,
  title={What Happens after {SGD} Reaches Zero Loss?---A Mathematical Framework},
  author={Li, Zhiyuan and Wang, Tianhao and Arora, Sanjeev},
  booktitle={10th International Conference on Learning Representations, ICLR 2022},
  year={2022}
}

@inproceedings{saad1995dynamics,
  title={Dynamics of on-line gradient descent learning for multilayer neural networks},
  author={Saad, David and Solla, Sara},
  booktitle={Advances in Neural Information Processing Systems},
  volume={8},
  year={1995}
}

@article{saad1995exact,
  title={Exact solution for on-line learning in multilayer neural networks},
  author={Saad, David and Solla, Sara A},
  journal={Physical Review Letters},
  volume={74},
  number={21},
  pages={4337--4340},
  year={1995},
  publisher={[Woodbury, NY, etc.] American Physical Society.}
}

@article{riegler1995line,
  title={On-line backpropagation in two-layered neural networks},
  author={Riegler, Peter and Biehl, Michael},
  journal={Journal of Physics A: Mathematical and General},
  volume={28},
  number={20},
  note={pp. L507--L513},
  year={1995},
  publisher={IoP Publishing},
  doi={10.1088/0305-4470/28/20/002}
}

@article{biehl1995learning,
  title={Learning by on-line gradient descent},
  author={Biehl, Michael and Schwarze, Holm},
  journal={Journal of Physics A: Mathematical and General},
  volume={28},
  number={3},
  pages={643--656},
  year={1995},
  publisher={IOP Publishing}
}

@inproceedings{wang2019solvable,
  title={A solvable high-dimensional model of {GAN}},
  author={Wang, Chuang and Hu, Hong and Lu, Yue},
  booktitle={Advances in Neural Information Processing Systems},
  volume={32},
  year={2019}
}

@inproceedings{ben2022high,
  title={High-dimensional limit theorems for {SGD}: Effective dynamics and critical scaling},
  author={Ben Arous, G{\'e}rard and Gheissari, Reza and Jagannath, Aukosh},
  booktitle={Advances in Neural Information Processing Systems},
  volume={35},
  pages={25349--25362},
  year={2022}
}

@article{arous2021online,
  title={Online stochastic gradient descent on non-convex losses from high-dimensional inference},
  author={Ben Arous, G{\'e}rard and Gheissari, Reza and Jagannath, Aukosh},
  journal={Journal of Machine Learning Research},
  volume={22},
  number={106},
  pages={1--51},
  year={2021}
}

@inproceedings{abbe2023sgd,
  title={{SGD} learning on neural networks: leap complexity and saddle-to-saddle dynamics},
  author={Abbe, Emmanuel and Boix-Adser{\`a}, Enric and Misiakiewicz, Theodor},
  booktitle={The Thirty Sixth Annual Conference on Learning Theory},
  pages={2552--2623},
  year={2023},
  organization={PMLR}
}

@inproceedings{damian2023smoothing,
  title={Smoothing the landscape boosts the signal for {SGD}: Optimal sample complexity for learning single index models},
  author={Damian, Alex and Nichani, Eshaan and Ge, Rong and Lee, Jason D},
  booktitle={Advances in Neural Information Processing Systems},
  volume={36},
  pages={752--784},
  year={2023}
}

@inproceedings{damian2024computational,
  title = {Computational-Statistical Gaps in {G}aussian Single-Index Models (Extended Abstract)},
  author = {Damian, Alex and Pillaud-Vivien, Loucas and Lee, Jason and Bruna, Joan},
  booktitle = {Proceedings of Thirty Seventh Conference on Learning Theory},
  pages = {1262--1262},
  year = {2024},
  editor = {Agrawal, Shipra and Roth, Aaron},
  volume = {247},
  series = {Proceedings of Machine Learning Research},
}

@article{tan2023online,
  title={Online stochastic gradient descent with arbitrary initialization solves non-smooth, non-convex phase retrieval},
  author={Tan, Yan Shuo and Vershynin, Roman},
  journal={Journal of Machine Learning Research},
  volume={24},
  number={58},
  pages={1--47},
  year={2023}
}

@inproceedings{damian2025generative,
  title={The Generative Leap: Tight Sample Complexity for Efficiently Learning {G}aussian Multi-Index Models},
  author={Damian, Alex and Lee, Jason D and Bruna, Joan},
  booktitle={Advances in Neural Information Processing Systems},
  volume={38},
  pages={32081--32116},
  year={2025},
  doi={10.52202/085713-0953}
}

@inproceedings{joshi2025learning,
  title={Learning single index models via harmonic decomposition},
  author={Joshi, Nirmit and Koubbi, Hugo and Misiakiewicz, Theodor and Srebro, Nati},
  booktitle={Advances in Neural Information Processing Systems},
  volume={38},
  pages={50320--50395},
  year={2025},
  doi={10.52202/085713-1501}
}

@inproceedings{barak2022hidden,
  title={Hidden progress in deep learning: {SGD} learns parities near the computational limit},
  author={Barak, Boaz and Edelman, Benjamin and Goel, Surbhi and Kakade, Sham and Malach, Eran and Zhang, Cyril},
  booktitle={Advances in Neural Information Processing Systems},
  volume={35},
  pages={21750--21764},
  year={2022}
}

@inproceedings{lee2024neural,
  title={Neural network learns low-dimensional polynomials with {SGD} near the information-theoretic limit},
  author={Lee, Jason D and Oko, Kazusato and Suzuki, Taiji and Wu, Denny},
  booktitle={Advances in Neural Information Processing Systems},
  volume={37},
  pages={58716--58756},
  year={2024}
}

@inproceedings{dandi2024benefits,
  title={The benefits of reusing batches for gradient descent in two-layer networks: breaking the curse of information and leap exponents},
  author={Dandi, Yatin and Troiani, Emanuele and Arnaboldi, Luca and Pesce, Luca and Zdeborov{\'a}, Lenka and Krzakala, Florent},
  booktitle={Proceedings of the 41st International Conference on Machine Learning},
  pages={9991--10016},
  year={2024}
}

@inproceedings{arnaboldi2023high,
  title={From high-dimensional \& mean-field dynamics to dimensionless {ODE}s: {A} unifying approach to {SGD} in two-layers networks},
  author={Arnaboldi, Luca and Stephan, Ludovic and Krzakala, Florent and Loureiro, Bruno},
  booktitle={The Thirty Sixth Annual Conference on Learning Theory},
  pages={1199--1227},
  year={2023},
  organization={PMLR}
}

@inproceedings{veiga2022phase,
  title={Phase diagram of stochastic gradient descent in high-dimensional two-layer neural networks},
  author={Veiga, Rodrigo and Stephan, Ludovic and Loureiro, Bruno and Krzakala, Florent and Zdeborov{\'a}, Lenka},
  booktitle={Advances in Neural Information Processing Systems},
  volume={35},
  pages={23244--23255},
  year={2022}
}

@article{collins2024hitting,
  title={Hitting the High-Dimensional Notes: An {ODE} for {SGD} Learning Dynamics on {GLM}s and Multi-Index Models},
  author={Collins-Woodfin, Elizabeth and Paquette, Courtney and Paquette, Elliot and Seroussi, Inbar},
  journal={Information and Inference: A Journal of the IMA},
  volume={13},
  number={4},
  note={Article iaae028},
  year={2024},
  publisher={Oxford University Press},
  doi={10.1093/imaiai/iaae028}
}

@inproceedings{collins2024high,
  title={The high line: Exact risk and learning rate curves of stochastic adaptive learning rate algorithms},
  author={Collins-Woodfin, Elizabeth and Seroussi, Inbar and Garc{\'\i}a Malaxechebarr{\'\i}a, Bego{\~n}a and Mackenzie, Andrew W and Paquette, Elliot and Paquette, Courtney},
  booktitle={Advances in Neural Information Processing Systems},
  volume={37},
  pages={6500--6548},
  year={2024}
}

@article{sompolinsky1981dynamic,
  title={Dynamic theory of the spin-glass phase},
  author={Sompolinsky, Haim and Zippelius, Annette},
  journal={Physical Review Letters},
  volume={47},
  number={5},
  pages={359--362},
  year={1981},
  publisher={APS}
}

@article{sompolinsky1982relaxational,
  title={Relaxational dynamics of the {E}dwards-{A}nderson model and the mean-field theory of spin-glasses},
  author={Sompolinsky, Haim and Zippelius, Annette},
  journal={Physical Review B},
  volume={25},
  number={11},
  pages={6860--6875},
  year={1982},
  publisher={APS}
}

@article{crisanti1993spherical,
  title={The spherical p-spin interaction spin-glass model: the dynamics},
  author={Crisanti, Andrea and Horner, Heinz and Sommers, H-J},
  journal={Zeitschrift f{\"u}r Physik B Condensed Matter},
  volume={92},
  number={2},
  pages={257--271},
  year={1993},
  publisher={Springer}
}

@article{cugliandolo1993analytical,
  title={Analytical solution of the off-equilibrium dynamics of a long-range spin-glass model},
  author={Cugliandolo, Leticia F and Kurchan, Jorge},
  journal={Physical Review Letters},
  volume={71},
  number={1},
  pages={173--176},
  year={1993},
  publisher={APS}
}

@article{arous1995large,
  title={Large deviations for Langevin spin glass dynamics},
  author={Ben Arous, G{\'e}rard and Guionnet, Alice},
  journal={Probability Theory and Related Fields},
  volume={102},
  number={4},
  pages={455--509},
  year={1995},
  publisher={Springer}
}

@article{arous1997symmetric,
  title={Symmetric Langevin spin glass dynamics},
  author={Ben Arous, G{\'e}rard and Guionnet, Alice},
  journal={The Annals of Probability},
  volume={25},
  number={3},
  pages={1367--1422},
  year={1997},
  publisher={Institute of Mathematical Statistics}
}

@article{guionnet1997averaged,
  title={Averaged and quenched propagation of chaos for spin glass dynamics},
  author={Guionnet, Alice},
  journal={Probability Theory and Related Fields},
  volume={109},
  number={2},
  pages={183--215},
  year={1997},
  publisher={Springer}
}

@article{grunwald1996sanov,
  title={Sanov results for {G}lauber spin-glass dynamics},
  author={Grunwald, Malte},
  journal={Probability Theory and Related Fields},
  volume={106},
  number={2},
  pages={187--232},
  year={1996},
  publisher={Springer}
}

@article{ben2006cugliandolo,
  title={Cugliandolo-{K}urchan equations for dynamics of spin-glasses},
  author={Ben Arous, G{\'e}rard and Dembo, Amir and Guionnet, Alice},
  journal={Probability Theory and Related Fields},
  volume={136},
  number={4},
  pages={619--660},
  year={2006},
  publisher={Springer}
}

@inproceedings{bordelon2022self,
  title={Self-consistent dynamical field theory of kernel evolution in wide neural networks},
  author={Bordelon, Blake and Pehlevan, Cengiz},
  booktitle={Advances in Neural Information Processing Systems},
  volume={35},
  pages={32240--32256},
  year={2022}
}

@article{mignacco2021stochasticity,
  title={Stochasticity helps to navigate rough landscapes: comparing gradient-descent-based algorithms in the phase retrieval problem},
  author={Mignacco, Francesca and Urbani, Pierfrancesco and Zdeborov{\'a}, Lenka},
  journal={Machine Learning: Science and Technology},
  volume={2},
  number={3},
  pages={035029},
  year={2021},
  publisher={IOP Publishing}
}

@article{mignacco2021dynamical,
  author={Mignacco, Francesca and Krzakala, Florent and Urbani, Pierfrancesco and Zdeborov{\'a}, Lenka},
  title={Dynamical mean-field theory for stochastic gradient descent in {G}aussian mixture classification},
  journal={Journal of Statistical Mechanics: Theory and Experiment},
  volume={2021},
  number={12},
  pages={124008},
  year={2021},
  doi={10.1088/1742-5468/ac3a80}
}

@article{liang2023high,
  author={Liang, Tengyuan and Sen, Subhabrata and Sur, Pragya},
  title={High-dimensional asymptotics of {L}angevin dynamics in spiked matrix models},
  journal={Information and Inference: A Journal of the IMA},
  volume={12},
  number={4},
  pages={2720--2752},
  year={2023},
  doi={10.1093/imaiai/iaad042}
}

@article{sarao2020marvels,
  title={Marvels and pitfalls of the {L}angevin algorithm in noisy high-dimensional inference},
  author={Sarao Mannelli, Stefano and Biroli, Giulio and Cammarota, Chiara and Krzakala, Florent and Urbani, Pierfrancesco and Zdeborov{\'a}, Lenka},
  journal={Physical Review X},
  volume={10},
  number={1},
  pages={011057},
  year={2020},
  publisher={APS}
}

@inproceedings{sarao2020complex,
  title={Complex dynamics in simple neural networks: Understanding gradient flow in phase retrieval},
  author={Sarao Mannelli, Stefano and Biroli, Giulio and Cammarota, Chiara and Krzakala, Florent and Urbani, Pierfrancesco and Zdeborov{\'a}, Lenka},
  booktitle={Advances in Neural Information Processing Systems},
  volume={33},
  pages={3265--3274},
  year={2020}
}

@article{mignacco2022effective,
  title={The effective noise of stochastic gradient descent},
  author={Mignacco, Francesca and Urbani, Pierfrancesco},
  journal={Journal of Statistical Mechanics: Theory and Experiment},
  volume={2022},
  number={8},
  pages={083405},
  year={2022},
  publisher={IOP Publishing}
}

@inproceedings{sarao2021analytical,
  title={Analytical study of momentum-based acceleration methods in paradigmatic high-dimensional non-convex problems},
  author={Sarao Mannelli, Stefano and Urbani, Pierfrancesco},
  booktitle={Advances in Neural Information Processing Systems},
  volume={34},
  pages={187--199},
  year={2021}
}

@inproceedings{sarao2019afraid,
  title={Who is afraid of big bad minima? {A}nalysis of gradient-flow in spiked matrix-tensor models},
  author={Sarao Mannelli, Stefano and Biroli, Giulio and Cammarota, Chiara and Krzakala, Florent and Zdeborov{\'a}, Lenka},
  booktitle={Advances in Neural Information Processing Systems},
  volume={32},
  year={2019}
}

@inproceedings{mannelli2019passed,
  title={Passed \& spurious: Descent algorithms and local minima in spiked matrix-tensor models},
  author={Sarao Mannelli, Stefano and Krzakala, Florent and Urbani, Pierfrancesco and Zdeborov{\'a}, Lenka},
  booktitle={International Conference on Machine Learning},
  pages={4333--4342},
  year={2019}
}

@article{han2025precise,
  title={Precise gradient descent training dynamics for finite-width multi-layer neural networks},
  author={Han, Qiyang and Imaizumi, Masaaki},
  journal={arXiv preprint arXiv:2505.04898},
  year={2025}
}

@article{chen2025learning,
  title={Learning single index model with gradient descent: spectral initialization and precise asymptotics},
  author={Chen, Yuchen and Shen, Yandi},
  journal={arXiv preprint arXiv:2509.23527},
  year={2025}
}

@article{martin2026high,
  title={High-Dimensional Analysis of Gradient Flow for Extensive-Width Quadratic Neural Networks},
  author={Martin, Simon and Biroli, Giulio and Bach, Francis},
  journal={Journal of Machine Learning Research},
  volume={27},
  number={136},
  pages={1--182},
  year={2026},
  url={https://www.jmlr.org/papers/v27/26-0672.html}
}

@article{han2025long,
  title={Long-time dynamics and universality of nonconvex gradient descent},
  author={Han, Qiyang},
  journal={arXiv preprint arXiv:2509.11426},
  year={2025}
}

@article{han2025entrywise,
  title={Entrywise dynamics and universality of general first order methods},
  author={Han, Qiyang},
  journal={The Annals of Statistics},
  volume={53},
  number={4},
  pages={1783--1807},
  year={2025},
  publisher={Institute of Mathematical Statistics}
}

@article{celentano2025state,
  title={State evolution beyond first-order methods {I}: {R}igorous predictions and finite-sample guarantees},
  author={Celentano, Michael and Cheng, Chen and Pananjady, Ashwin and Verchand, Kabir Aladin},
  journal={arXiv preprint arXiv:2507.19611},
  year={2025}
}

@inproceedings{bordelon2024infinite,
  title={Infinite limits of multi-head transformer dynamics},
  author={Bordelon, Blake and Chaudhry, Hamza and Pehlevan, Cengiz},
  booktitle={Advances in Neural Information Processing Systems},
  volume={37},
  pages={35824--35878},
  year={2024}
}

@inproceedings{bordelon2024dynamical,
  title={A Dynamical Model of Neural Scaling Laws},
  author={Bordelon, Blake and Atanasov, Alexander and Pehlevan, Cengiz},
  booktitle={Proceedings of the 41st International Conference on Machine Learning},
  volume={235},
  series={Proceedings of Machine Learning Research},
  pages={4345--4382},
  year={2024},
  publisher={PMLR},
  url={https://proceedings.mlr.press/v235/bordelon24a.html}
}

@inproceedings{montanari2025dynamical,
  title={Dynamical Decoupling of Generalization and Overfitting in Large Two-Layer Networks},
  author={Montanari, Andrea and Urbani, Pierfrancesco},
  booktitle={Advances in Neural Information Processing Systems},
  volume={38},
  pages={30086--30164},
  year={2025},
  doi={10.52202/085713-0895}
}

@inproceedings{paquette2021dynamics,
  title={Dynamics of stochastic momentum methods on large-scale, quadratic models},
  author={Paquette, Courtney and Paquette, Elliot},
  booktitle={Advances in Neural Information Processing Systems},
  volume={34},
  pages={9229--9240},
  year={2021}
}

@inproceedings{lee2022trajectory,
  title={Trajectory of mini-batch momentum: batch size saturation and convergence in high dimensions},
  author={Lee, Kiwon and Cheng, Andrew and Paquette, Elliot and Paquette, Courtney},
  booktitle={Advances in Neural Information Processing Systems},
  volume={35},
  pages={36944--36957},
  year={2022}
}

@inproceedings{paquette2021sgd,
  title={{SGD} in the large: {A}verage-case analysis, asymptotics, and stepsize criticality},
  author={Paquette, Courtney and Lee, Kiwon and Pedregosa, Fabian and Paquette, Elliot},
  booktitle={Conference on Learning Theory},
  pages={3548--3626},
  year={2021},
  organization={PMLR}
}

@inproceedings{marshall2024clip,
  title={To clip or not to clip: the dynamics of {SGD} with gradient clipping in high-dimensions},
  author={Marshall, Noah and Xiao, Ke Liang and Agarwala, Atish and Paquette, Elliot},
  booktitle={The Thirteenth International Conference on Learning Representations},
  year={2025},
  url={https://openreview.net/forum?id=jmN1zXMq0O}
}

@inproceedings{paquette2022implicit,
  title={Implicit regularization or implicit conditioning? {E}xact risk trajectories of {SGD} in high dimensions},
  author={Paquette, Courtney and Paquette, Elliot and Adlam, Ben and Pennington, Jeffrey},
  booktitle={Advances in Neural Information Processing Systems},
  volume={35},
  pages={35984--35999},
  year={2022}
}

@article{paquette2025homogenization,
  title={Homogenization of {SGD} in high-dimensions: Exact dynamics and generalization properties},
  author={Paquette, Courtney and Paquette, Elliot and Adlam, Ben and Pennington, Jeffrey},
  journal={Mathematical Programming},
  volume={214},
  number={1-2},
  pages={1--90},
  year={2025},
  publisher={Springer},
  doi={10.1007/s10107-024-02171-3}
}

@inproceedings{he2019control,
  title={Control batch size and learning rate to generalize well: {T}heoretical and empirical evidence},
  author={He, Fengxiang and Liu, Tongliang and Tao, Dacheng},
  booktitle={Advances in Neural Information Processing Systems},
  volume={32},
  year={2019}
}

@article{dandi2025sequential,
  title={Sequential Dynamics in {I}sing Spin Glasses},
  author={Dandi, Yatin and Gamarnik, David and Pernice, Francisco and Zdeborov{\'a}, Lenka},
  journal={arXiv preprint arXiv:2506.09877},
  year={2025}
}

@article{berger1980volterra,
  title={Volterra equations with It{\^o} integrals—{I}},
  author={Berger, Marc A and Mizel, Victor J},
  journal={Journal of Integral Equations},
  volume={2},
  number={3},
  pages={187--245},
  year={1980},
  publisher={Rocky Mountain Mathematics Consortium}
}

@article{hoeffding1963probability,
  title={Probability inequalities for sums of bounded random variables},
  author={Hoeffding, Wassily},
  journal={Journal of the American Statistical Association},
  volume={58},
  number={301},
  pages={13--30},
  year={1963},
  publisher={Taylor \& Francis}
}

@article{hansen2015integrated,
  title={The integrated mean squared error of series regression and a {R}osenthal {H}ilbert-space inequality},
  author={Hansen, Bruce E},
  journal={Econometric Theory},
  volume={31},
  number={2},
  pages={337--361},
  year={2015},
  publisher={Cambridge University Press}
}

\end{document}